\documentclass[ twoside,openright,titlepage,numbers=noenddot,headinclude,
                footinclude=true,cleardoublepage=empty,abstractoff, 
                BCOR=5mm,paper=a4,fontsize=11pt, ngerman,american,openany]{scrreprt}
\pdfoutput=1

\PassOptionsToPackage{eulerchapternumbers,listings,drafting,%
    				  pdfspacing,
    				 subfig,beramono,eulermath,parts}{classicthesis}

\usepackage{ifthen}
\newboolean{enable-backrefs} 
\setboolean{enable-backrefs}{true} 

\newcommand{\myTitle}{Exploiting random projections and sparsity with random forests and gradient boosting methods\xspace}
\newcommand{\mySubtitle}{Application to multi-label and multi-output learning, random forest model compression and leveraging input sparsity\xspace}

\newcommand{\myName}{Arnaud Joly\xspace}

\newcommand{\myFaculty}{Faculty of Applied Sciences\xspace}
\newcommand{\myDepartment}{Department of Electrical Engineering \& Computer Science\xspace}
\newcommand{\myUni}{University of Li\`ege\xspace}

\newcommand{\myTime}{December 2016\xspace}

\newcounter{dummy} 
\providecommand{\mLyX}{L\kern-.1667em\lower.25em\hbox{Y}\kern-.125emX\@}

\PassOptionsToPackage{utf8x}{inputenc}	
 \usepackage{inputenc}

 \usepackage{babel}

\PassOptionsToPackage{colon,sort}{natbib}
 \usepackage{natbib}

\PassOptionsToPackage{fleqn}{amsmath}		
 \usepackage{amsmath}

\PassOptionsToPackage{T1}{fontenc} 
	\usepackage{fontenc}
\usepackage{textcomp} 
\usepackage{scrhack} 
\usepackage{xspace} 
\usepackage{mparhack} 
\usepackage{fixltx2e} 
\PassOptionsToPackage{printonlyused,smaller}{acronym}
	\usepackage{acronym} 

\usepackage{tabularx} 
\usepackage{caption}
\usepackage{subfig}

\usepackage{listings}
\PassOptionsToPackage{pdftex,hyperfootnotes=false,pdfpagelabels}{hyperref}
	\usepackage{hyperref}  
\PassOptionsToPackage{pdftex}{graphicx}
	\usepackage{graphicx}

\newcommand{\backrefnotcitedstring}{\relax}
\newcommand{\backrefcitedsinglestring}[1]{(Cited on page~#1.)}
\newcommand{\backrefcitedmultistring}[1]{(Cited on pages~#1.)}
\ifthenelse{\boolean{enable-backrefs}}%
{%
		\PassOptionsToPackage{hyperpageref}{backref}
		\usepackage{backref} 
		   \renewcommand*{\backref}[1]{}  
		   \renewcommand*{\backrefalt}[4]{
		      \ifcase #1 %
		         \backrefnotcitedstring%
		      \or%
		         \backrefcitedsinglestring{#2}%
		      \else%
		         \backrefcitedmultistring{#2}%
		      \fi}%
}{\relax}

\hypersetup{%
    colorlinks=true, linktocpage=true, pdfstartpage=3, pdfstartview=FitV,%
    breaklinks=true, pdfpagemode=UseNone, pageanchor=true, pdfpagemode=UseOutlines,%
    plainpages=false, bookmarksnumbered, bookmarksopen=true, bookmarksopenlevel=1,%
    hypertexnames=true, pdfhighlight=/O,
    urlcolor=webbrown, linkcolor=RoyalBlue, citecolor=webgreen, 
    pdftitle={\myTitle},%
    pdfauthor={\textcopyright\ \myName, \myUni, \myFaculty},%
    pdfsubject={},%
    pdfkeywords={},%
    pdfcreator={pdfLaTeX},%
    pdfproducer={LaTeX with hyperref and classicthesis}%
}

\makeatletter
\@ifpackageloaded{babel}%
    {%
       \addto\extrasamerican{%
				}%
       \addto\extrasngerman{%
				}%
    }{\relax}
\makeatother

\usepackage{classicthesis}
\usepackage{amsmath,amsthm,amssymb,amsfonts}
\numberwithin{equation}{chapter}

\usepackage{dashundergaps}
\usepackage[chapter]{algorithm}
\usepackage{algpseudocode}

\numberwithin{figure}{chapter}
\numberwithin{table}{chapter}

\makeatletter
\newcommand{\algmargin}{\the\ALG@thistlm}
\makeatother
\newlength{\whilewidth}
\algdef{SE}[parWHILE]{parWhile}{EndparWhile}[1]
  {\parbox[t]{\dimexpr\linewidth-\algmargin}{%
     \hangindent\whilewidth\strut\algorithmicwhile\ #1\ \algorithmicdo\strut}}{\algorithmicend\ \algorithmicwhile}%
\algnewcommand{\parState}[1]{\State%
  \parbox[t]{\dimexpr\linewidth-\algmargin}{\strut #1\strut}}

\usepackage{varwidth}

\allowdisplaybreaks

\graphicspath{{py_figures/}{figures/}{figures/ch02}{figures/ch04}{figures/ch06}{figures/ch06/friedman-subsample}{figures/ch07}{figures/ch08}}

\usepackage{framed}
\newenvironment{remark}[1]{%
    \definecolor{shadecolor}{gray}{0.9}%
    \begin{shaded}{\color{Maroon}\noindent\textsc{#1}}\\%
    }{%
\end{shaded}%
}

\newtheorem{theorem}{Theorem}

\usepackage{tikz}
\usetikzlibrary{shapes,arrows,positioning,fit,trees}

\newcommand{\ra}[1]{\renewcommand{\arraystretch}{#1}}

\DeclareMathOperator{\E}{E}
\DeclareMathOperator{\Var}{Var}
\DeclareMathOperator{\Bias}{Bias}

\DeclareMathOperator{\logit}{logit}
\DeclareMathOperator{\sign}{sign}
\definecolor{burgundy}{rgb}{0.5, 0.0, 0.13}

\renewcommand\cite{\citep}

\begin{document}

\frenchspacing{}
\raggedbottom{}
\selectlanguage{american}
\pagenumbering{roman}
\pagestyle{plain}
\sloppy


\begin{titlepage}


	\begin{addmargin}[-1cm]{-3cm}
    \begin{center}
        \large

		\myUni \\
		{\normalsize \myFaculty \\
		\myDepartment \\
		Montefiore Institute} \\

        \hfill

        \vfill

\textbf{PhD Thesis in Engineering Sciences}
		\vfill

        \begingroup
            \color{Maroon}\spacedallcaps{\myTitle} \\ \bigskip
        \endgroup

		\mySubtitle{} \\ \medskip

        \spacedlowsmallcaps{\myName}

        \vfill

		\includegraphics[width=0.4\textwidth]{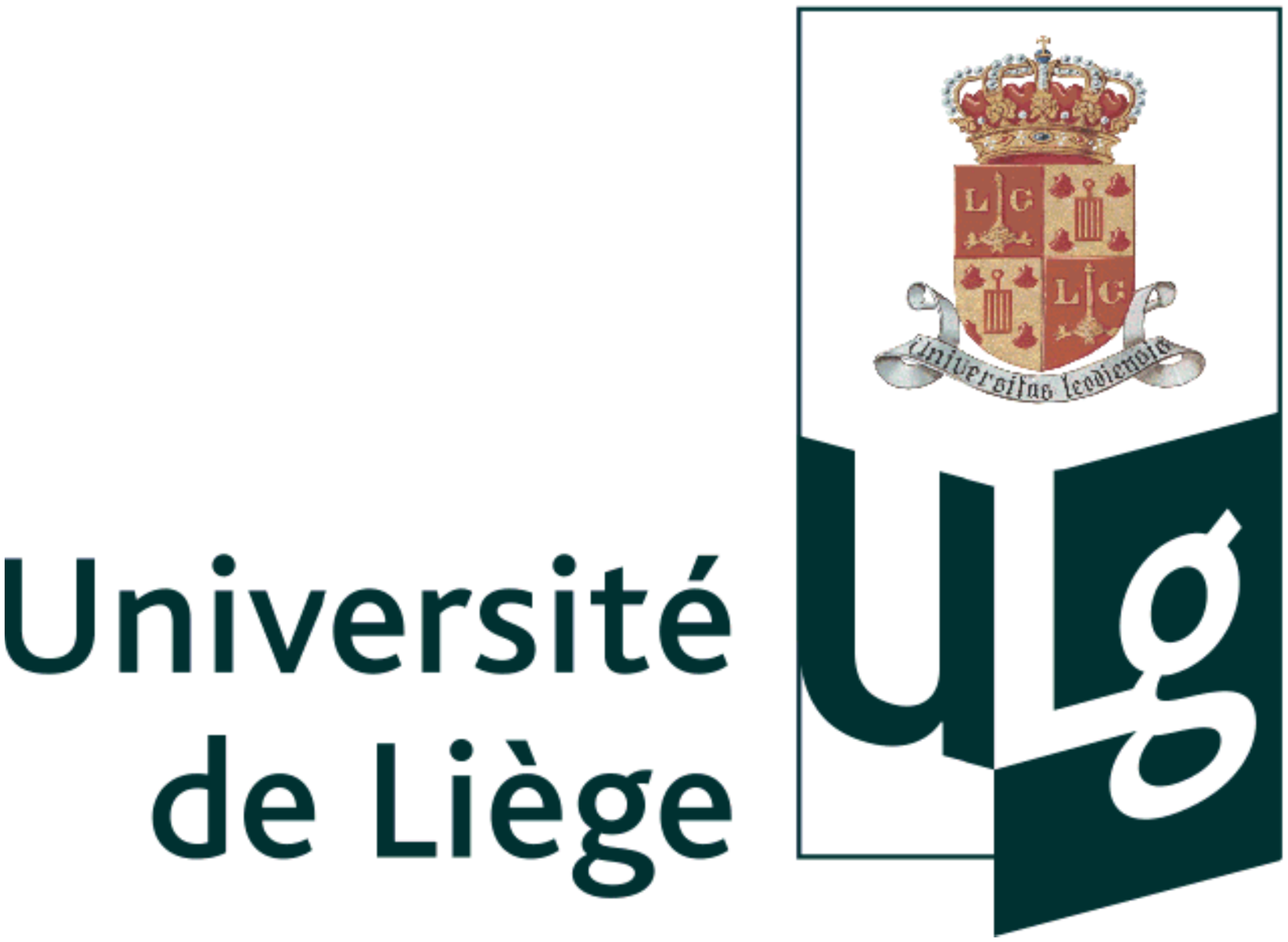}

		\bigskip

		\spacedlowsmallcaps{Advisor: Louis Wehenkel}

		\spacedlowsmallcaps{Co-Advisor: Pierre Geurts}

		\bigskip \myTime\

    \end{center}
  \end{addmargin}
\end{titlepage}

\pdfbookmark[1]{Copyright}{Copyright}

\begingroup
\let\clearpage\relax
\let\cleardoublepage\relax
\let\cleardoublepage\relax

\vspace*{\fill}

\begin{center}

\copyright~2017

\spacedallcaps{\myName}

\spacedallcaps{ALL RIGHTS RESERVED}

\end{center}

\vspace*{\fill}

\endgroup

\cleardoublepage 

\pdfbookmark[1]{Jury members}{Jury members}

\begingroup
\let\clearpage\relax
\let\cleardoublepage\relax
\let\cleardoublepage\relax

\chapter*{Jury members}

\noindent \textsc{Damien Ernst}, Professor at the Université de Li{\`e}ge (President); \\

\noindent \textsc{Louis Wehenkel}, Professor at the Université de  Li{\`e}ge (Advisor); \\

\noindent \textsc{Pierre Geurts}, Professor at the Université de  Li{\`e}ge (Co-Advisor); \\

\noindent \textsc{Quentin Louveaux}, Professor at Université de  Li{\`e}ge; \\

\noindent \textsc{Ashwin Ittoo}, Professor at the Université de  Li{\`e}ge; \\

\noindent \textsc{Grigorios Tsoumakas}, Professor at the Aristotle University of Thessaloniki; \\

\noindent \textsc{Celine Vens}, Professor at the Katholieke Universiteit Leuven; \\

\endgroup

\vfill


\pdfbookmark[1]{Abstract}{Abstract}
\begingroup
\let\clearpage\relax
\let\cleardoublepage\relax
\let\cleardoublepage\relax

\chapter*{Abstract}

Within machine learning, the supervised learning field aims at modeling the
input-output relationship of a system,
from past observations of its behavior. Decision trees characterize the input-output
relationship through a series of nested $if-then-else$ questions, the testing
nodes, leading to a set of predictions, the leaf nodes. Several of such trees
are often combined together for state-of-the-art performance: random forest
ensembles average the predictions of randomized decision trees trained
independently in parallel, while tree boosting ensembles train decision trees
sequentially to refine the predictions made by the previous ones.

The emergence of new applications requires scalable supervised learning
algorithms in terms of computational power and memory space with respect to the
number of inputs, outputs, and observations without sacrificing accuracy. In this
thesis, we identify three main areas where decision tree methods could be
improved for which we provide and evaluate original algorithmic solutions: (i)
learning over high dimensional output spaces, (ii) learning with large sample
datasets and stringent memory constraints at prediction time and (iii) learning
over high dimensional sparse input spaces.

A first approach to solve \emph{learning tasks with a high dimensional output
space}, called binary relevance or single target, is to train one decision tree
ensemble per output. However, it completely neglects the potential correlations
existing between the outputs. An alternative approach called multi-output
decision trees fits a single decision tree ensemble targeting simultaneously all
the outputs, assuming that all outputs are correlated. Nevertheless, both
approaches have (i) exactly the same computational complexity and (ii) target
extreme output correlation structures. In our first contribution, we show how to
combine random projection of the output space, a dimensionality reduction
method, with the random forest algorithm decreasing the learning time
complexity. The accuracy is preserved, and may even be improved by reaching a
different bias-variance tradeoff. In our second contribution, we first formally
adapt the gradient boosting ensemble method to multi-output supervised learning
tasks such as multi-output regression and multi-label classification. We then
propose to combine single random projections of the output space with gradient
boosting on such tasks to adapt automatically to the output correlation
structure.

The random forest algorithm often generates large ensembles of complex models
thanks to the availability of a large number of observations. However, the space
complexity of such models, proportional to their total number of nodes, is often
prohibitive, and therefore these modes are not well suited under \emph{stringent
memory constraints at prediction time}. In our third contribution, we propose to
compress these ensembles by solving a $\ell_1$-based regularization problem
over the set of indicator functions defined by all their nodes.

Some supervised learning tasks have a \emph{high dimensional but sparse input space},
where each observation has only a few of the input variables that have non zero
values. Standard decision tree implementations are not well adapted to treat
sparse input spaces, unlike other supervised learning techniques such as support
vector machines or linear models. In our fourth contribution, we show how to
exploit algorithmically the input space sparsity within decision tree methods.
Our implementation yields a significant speed up both on synthetic and real datasets, while
leading to exactly the same model. It also reduces the required memory to grow
such models by exploiting sparse instead of dense memory storage
for the input matrix.


\endgroup

\vfill

\clearpage

\pdfbookmark[1]{Résumé}{Résumé}
\begingroup
\let\clearpage\relax
\let\cleardoublepage\relax
\let\cleardoublepage\relax

\chapter*{Résumé}

Parmi les techniques d'apprentissage automatique, l'apprentissage supervisé vise à
modéliser les relations entrée-sortie d'un système, à partir d'observations de
son fonctionnement. Les arbres de décision caractérisent cette relation
entrée-sortie à partir d'un ensemble hiérarchique de questions appelées les
noeuds tests amenant à une prédiction, les noeuds feuilles. Plusieurs de ces
arbres sont souvent combinés ensemble afin d'atteindre les performances de
l'état de l'art: les ensembles de forêts aléatoires calculent la moyenne des
prédictions d'arbres de décision randomisés, entraînés indépendamment et en
parallèle alors que les ensembles d'arbres de boosting entraînent des arbres de
décision séquentiellement, améliorant ainsi les prédictions faites par les
précédents modèles de l'ensemble.

L'apparition de nouvelles applications requiert des algorithmes d'apprentissage
supervisé efficaces en terme de puissance de calcul et d'espace mémoire par
rapport au nombre d'entrées, de sorties, et d'observations sans sacrifier la
précision du modèle. Dans cette thèse, nous avons identifié trois domaines
principaux où les méthodes d'arbres de décision peuvent être améliorées pour
lequel nous fournissons et évaluons des solutions algorithmiques originales: (i)
apprentissage sur des espaces de sortie de haute dimension, (ii) apprentissage
avec de grands ensembles d'échantillons et des contraintes mémoires strictes au
moment de la prédiction et (iii) apprentissage sur des espaces d'entrée creux de
haute dimension.

Une première approche pour résoudre des \emph{tâches d'apprentissage avec un
espace de sortie de haute dimension}, appelée \guillemotleft{}binary
relevance\guillemotright{} ou \guillemotleft{}single target\guillemotright{},
est l’apprentissage d’un ensemble d'arbres de décision par  sortie. Toutefois,
cette approche néglige complètement les corrélations potentiellement existantes
entre les sorties. Une approche alternative, appelée \guillemotleft{}arbre de
décision multi-sorties\guillemotright{}, est l’apprentissage d’un seul ensemble
d'arbres de décision pour toutes les sorties, faisant l'hypothèse que toutes les
sorties sont corrélées. Cependant, les deux approches ont (i) exactement la même
complexité en temps de calcul et (ii) visent des structures de corrélation de
sorties extrêmes. Dans notre première contribution, nous montrons comment
combiner des projections aléatoires (une méthode de réduction de
dimensionnalité) de l'espace de sortie avec l'algorithme des forêts aléatoires
diminuant la complexité en temps de calcul de la phase d'apprentissage. La
précision est préservée, et peut même être améliorée en atteignant un compromis
biais-variance différent. Dans notre seconde contribution, nous adaptons d'abord
formellement la méthode d'ensemble \guillemotleft{}gradient
boosting\guillemotright{} à la régression multi-sorties et à la classification
multi-labels. Nous proposons ensuite de combiner une seule projection aléatoire
de l'espace de sortie avec l’algorithme de  \guillemotleft{}gradient
boosting\guillemotright{} sur de telles tâches afin de s'adapter automatiquement à
la structure des corrélations existant entre les sorties.

Les algorithmes de forêts aléatoires génèrent souvent de grands ensembles de
modèles complexes grâce à la disponibilité d'un grand nombre d'observations.
Toutefois, la complexité mémoire, proportionnelle au nombre total de noeuds, de
tels modèles est souvent prohibitive, et donc ces modèles ne sont pas adaptés à
des \emph{contraintes mémoires fortes lors de la phase de prédiction}. Dans
notre troisième contribution, nous proposons de compresser ces ensembles en
résolvant un problème de régularisation basé sur la norme $\ell_1$ sur
l'ensemble des fonctions indicatrices défini par tous leurs noeuds.

Certaines tâches d'apprentissage supervisé ont un \emph{espace d'entrée de haute
dimension mais creux}, où chaque observation possède seulement quelques
variables d'entrée avec une valeur non-nulle. Les implémentations standards des
arbres de décision ne sont pas adaptées pour traiter des espaces d'entrée creux,
contrairement à d'autres techniques d'apprentissage supervisé telles que les
machines à vecteurs de support ou les modèles linéaires. Dans notre quatrième
contribution, nous montrons comment exploiter algorithmiquement le creux de
l'espace d'entrée avec les méthodes d'arbres de décision. Notre implémentation
diminue significativement le temps de calcul sur des ensembles de données
synthétiques et réelles, tout en fournissant exactement le même modèle. Cela
permet aussi de réduire la mémoire nécessaire pour apprendre de tels modèles en
exploitant des méthodes de stockage appropriées pour la matrice des entrées.

\endgroup

\vfill


\pdfbookmark[1]{Acknowledgments}{acknowledgments}

\begingroup
\let\clearpage\relax
\let\cleardoublepage\relax
\let\cleardoublepage\relax
\chapter*{Acknowledgments}

This PhD thesis started with the trust granted by Prof. Louis Wehenkel,  joined
soon after by Prof. Pierre Geurts. I would like to express my sincere gratitude
for their continuous encouragements, guidance and support. I have without doubt
benefitted from their motivations, patience and knowledge. Our insightful
discussions and interactions definitely moved the thesis forward.

I would like to thank the University of Liège, the FRS-FNRS, Belgium, the EU Network of Excellence PASCAL2,
and the IUAP DYSCO, initiated by the Belgian State, Science Policy Office to
have funded this research. Computational resources have been provided by the
Consortium des Équipements de Calcul Intensif (CÉCI), funded by the Fonds de la
Recherche Scientifique de Belgique (F.R.S.-FNRS) under Grant No. 2.5020.11.

The presented research would not have been the same without my co-authors (here
in alphabetic order): Jean-Michel Begon, Mathieu Blondel, Lars Buitinck, Pierre
Damas, Céline Delierneux, Damien Ernst, Hedayati Fares, Alexandre Gramfort,
Pierre Geurts, André Gothot, Olivier Grisel, Jaques Grobler, Alexandre Hego,
Bryan Holt, Justine Huart, Vincent Fran{\c{c}}ois-Lavet, Nathalie Layios, Robert
Layton, Christelle Lecut, Gilles Louppe, Andreas Mueller, Vlad Niculae, Cécile
Oury, Panagiotis Papadimitriou, Fabian Pedregosa, Peter Prettenhofer, Zixiao
Aaron Qiu, Fran{\c{c}}ois Schnitzler, Antonio Sutera, Jake Vanderplas, Gael
Varoquaux, and Louis Wehenkel.

I would like to thank the members of the jury, who take interests in my work,
and took the time to read this  dissertation.

Diane Zander and Sophie Cimino have been of an invaluable help with all the
administrative procedures. I would like to thank them for their patience and
availability. I would also like to thank David Colignon and Alain Empain for their
helpfulness about anything related to super-computers.

I would like to thank my colleagues from the Montefiore Institute, Department of
Electrical Engineering and Computer Science from the University of Liège, whom
have created a pleasant, rich and stimulating environment (in alphabetic order):
Samir Azrour, Tom Barbette, Julien Beckers, Jean-Michel Begon, Kyrylo Bessonov,
Hamid Soleimani Bidgoli, Vincent Botta, Kridsadakorn Chaichoompu, Célia Châtel,
Julien Confetti, Mathilde De Becker, Renaud Detry, Damien Ernst, Ramouna
Fouladi, Florence Fonteneau, Raphaël Fonteneau, Vincent Fran{\c{c}}ois-Lavet,
Damien Gérard, Quentin Gemine, Pierre Geurts, Samuel Hiard, Renaud Hoyoux,
Fabien Heuze, Van Anh Huynh-Thu, Efthymios Karangelos, Philippe Latour, Gilles
Louppe, Francis Maes, Alejandro Marcos Alvarez, Benjamin Laugraud, Antoine
Lejeune, Raphael Liégeois, Quentin Louveaux, Isabelle Mainz, Raphael Marée,
Sébastien Mathieu, Axel Mathei, Romain Mormont, Frédéric Olivier, Julien
Osmalsky, Sébastien Pierard, Zixiao Aaron Qiu, Loïc Rollus, Marie
Schrynemackers, Oliver Stern, Benjamin Stévens, Antonio Sutera, David Taralla,
François Van Lishout, Rémy Vandaele, Philippe Vanderbemden, and Marie Wehenkel.

I would like to thank the \emph{scikit-learn} community who has shared with me
their passion about computer science, machine learning and Python. By
contributing to this open source project, I have learnt much since my first
contribution.

\emph{I also offer my regards and blessing to all the people near and dear to my
heart for their continuous support, and to all of those who supported in any
respect during the completion of this project.}

\endgroup

\pagestyle{scrheadings}
\refstepcounter{dummy}
\pdfbookmark[1]{\contentsname}{tableofcontents}
\setcounter{tocdepth}{2} 
\setcounter{secnumdepth}{3} 
\manualmark{}
\markboth{\spacedlowsmallcaps{\contentsname}}{\spacedlowsmallcaps{\contentsname}}
\tableofcontents
\automark[section]{chapter}
\renewcommand{\chaptermark}[1]{\markboth{\spacedlowsmallcaps{#1}}{\spacedlowsmallcaps{#1}}}
\renewcommand{\sectionmark}[1]{\markright{\thesection\enspace\spacedlowsmallcaps{#1}}}

\pagenumbering{arabic}

\chapter{Introduction}\label{ch:introduction}


Progress in information technology enables the acquisition and storage of
growing amounts of rich data in many domains including science (biology,
high-energy physics, astronomy, etc.), engineering (energy, transportation,
production processes, etc.), and society (environment, commerce, etc.).
Connected objects, such as smartphones, connected sensors or intelligent houses,
are now able to record videos, images, audio signals, object localizations,
temperatures, social interactions of the user through a social network, phone
calls or user to computer program interactions such as voice help assistant or
web search queries. The accumulating datasets come in various forms such as
images, videos, time-series of measurements, recorded transactions, text etc.
WEB technology often allows one to share locally acquired datasets, and
numerical simulation often allows one to generate low cost datasets on demand.
Opportunities exist thus for combining datasets from different sources to search
for generic knowledge and enable robust decision.


All these rich datasets are of little use without the availability of automatic
procedures able to extract relevant information from them in a principled way.
In this context, the field of machine learning aims at developing theory and
algorithmic solutions for the extraction of synthetic patterns of information
from all kinds of datasets, so as to help us to better understand the underlying
systems generating these data and hence to take better decisions for their
control or exploitation.

Among the machine learning tasks, supervised learning aims at modeling a system
by observing its behavior through samples of pairs of inputs and outputs. The
objective of the generated model is to predict with high accuracy the outputs of
the system given previously unseen inputs. A genomic application of supervised
learning would be to model how a DNA sequence, a biological code, is linked to
some genetic diseases. The samples used to fit the model are the input-output
pairs obtained by sequencing the genome, the inputs, of patients with known
medical records for the studied genetic diseases, the outputs. The objective is
here twofold: (i) to understand how the DNA sequence influences the appearing
of the studied genetic diseases and (ii) to use the predictive models to infer
the probability of contracting the genetic disease.

The emergence of new applications, such as image annotation, personalized
advertising or 3D image segmentation, leads to high dimensional data with a large
number of inputs and outputs. It requires scalable supervised learning
algorithms in terms of computational power and memory space without sacrificing
accuracy.

Decision trees~\cite{breiman1984classification} are supervised learning models
organized in the form of a hierarchical set of questions each one typically
based on one input variable leading to a prediction. Used in isolation, trees
are generally not competitive in terms of accuracy, but when combined into
ensembles~\cite{breiman2001random,friedman2001greedy}, they yield
state-of-the-art performances on standard
benchmarks~\cite{caruana2008empirical,fernandez2014we,madjarov2012extensive}.
They however suffer from several limitations that make them not always suited to
address modern applications of machine learning techniques in particular
involving high dimensional input and output spaces.

In this thesis, we identify three main areas where random forest methods could
be improved and for which we provide and evaluate original algorithmic
solutions: (i) learning over high dimensional output spaces, (ii) learning with
large sample datasets and stringent memory constraints at prediction time and
(iii) learning over high dimensional sparse input spaces. We discuss each one of
these solutions in the following paragraphs.

\paragraph{High dimensional output spaces} New applications of machine learning
have multiple output variables, potentially in very high
number~\cite{agrawal2013multi,dekel2010multiclass}, associated to the same set
of input variables. A first approach to address such multi-output tasks is the so-called binary
relevance / single target
method~\cite{tsoumakas2009mining,spyromitros2012multi}, which separately fits one decision tree
ensemble for each output variable, assuming that the different output variables
are independent. A second approach called multi-output decision
trees~\cite{blockeel2000top,geurts2006kernelizing,kocev2013tree} fits a single
decision tree ensemble targeting simultaneously all the outputs, assuming that all outputs are
correlated. However in practice, (i) the computational complexity is the same
for both approaches and (ii) we have often neither of these two extreme output
correlation structures. As our first contribution, we show how to make random
forest faster by exploiting random projections (a dimensionality reduction
technique) of the output space. As a second contribution, we show how to combine
gradient boosting of tree ensembles with single random projections of the output space to
automatically adapt to a wide variety of correlation structures.

\paragraph{Memory constraints on model size} Even with a large number of
training samples $n$, random forest ensembles have good computational complexity
($O(n\log{n})$) and are easily parallelizable leading to the generation of very
large ensembles. However, the resulting models are big as the model complexity is
proportional to the number of samples $n$ and the ensemble size. As our third
contribution, we propose to compress these tree ensembles by solving an
appropriate optimization problem.

\paragraph{High dimensional sparse input spaces} Some supervised learning tasks
have very high dimensional input spaces, but only a few variables have non zero
values for each sample. The input space is said to be ``sparse''. Instances of
such tasks can be found in text-based supervised learning, where each sample is
often mapped to a vector of variables corresponding to the (frequency of)
occurrence of all words (or multigrams) present in the dataset. The problem is sparse as the
size of the text is small compared to the number of possible words  (or multigrams).
Standard decision tree implementations are not well adapted to treat sparse input
spaces, unlike models such as support vector
machines~\cite{cortes1995support,scholkopf2001learning} or linear
models~\cite{bottou2012stochastic}. Decision tree implementations are indeed
treating these sparse variables as dense ones raising the memory needed. The
computational complexity also does not depend upon the fraction of non zero
values. As a fourth contribution, we propose an efficient decision tree
implementation to treat supervised learning tasks with sparse input spaces.

\section{Publications}

This dissertation features several publications about random forest algorithms:
\begin{itemize}

\item \cite{joly2014random}  A. Joly, P. Geurts, and L. Wehenkel. \textit{Random
forests with random projections of the output space for high dimensional
multi-label classification.} In Machine Learning and Knowledge Discovery in
Databases, pages 607–622. Springer Berlin Heidelberg, 2014.

\item \cite{joly2012l1} A. Joly, F. Schnitzler, P. Geurts, and L. Wehenkel.
\textit{L1-based compression of random forest models.} In European Symposium on
Artificial Neural Networks, Computational Intelligence and Machine Learning,
2012.

\item \cite{buitinck2013api} L. Buitinck, G. Louppe, M. Blondel, F. Pedregosa,
A. Mueller, O. Grisel, V. Niculae, P. Prettenhofer, A. Gramfort, J. Grobler, R.
Layton, J. Vanderplas, A. Joly, B. Holt, and G. Varoquaux. \textit{Api design
for machine learning software: experiences from the scikit-learn project.} arXiv
preprint arXiv:1309.0238, 2013.

\end{itemize}
\noindent and also the following submitted article:
\begin{itemize}

\item H. Fares, A. Joly, and P. Papadimitriou. \textit{Scalable Learning of
Tree-Based Models on Sparsely Representable Data.}

\end{itemize}

Some collaborations were made during the thesis, but are not discussed
within this manuscript:
\begin{itemize}

\item \cite{sutera2014simple} A. Sutera, A. Joly, V. Fran{\c{c}}ois-Lavet, Z. A.
Qiu, G. Louppe, D. Ernst, and P. Geurts. \textit{Simple connectome inference
from partial correlation statistics in calcium imaging.} In JMLR: Workshop and
Conference Proceedings, pages 1–12, 2014.

\item \cite{delierneux2015elevated} C. Delierneux, N. Layios, A. Hego, J. Huart,
A. Joly, P. Geurts, P. Damas, C. Lecut, A. Gothot, and C. Oury. \textit{Elevated basal
levels of circulating activated platelets predict icu-acquired sepsis and
mortality: a prospective study.} Critical Care, 19(Suppl 1):P29, 2015a.

\item \cite{delierneux2015prospective}  C. Delierneux, N. Layios, A. Hego, J.
Huart, A. Joly, P. Geurts, P. Damas, C. Lecut, A. Gothot, and C. Oury.
\textit{Prospective analysis of platelet activation markers to predict severe
infection and mortality in intensive care units.} In journal of thrombosis and
haemostasis, volume 13, pages 651–651.

\item \cite{begon2016joint} J.-M. Begon, A. Joly, and P. Geurts. \textit{Joint
learning and pruning of decision forests.} In Belgian-Dutch Conference On
Machine Learning, 2016.

\end{itemize}

The following article has been submitted:
\begin{itemize}

\item C. Delierneux, N. Layios, A. Hego, J. Huart, C. Gosset, C. Lecut, N. Maes,
P. Geurts, A. Joly, P. Lancellotti, P. Damas, A. Gothot, and C. Oury.
\textit{Incremental value of platelet markers to clinical variables for sepsis
prediction in intensive care unit patients: a prospective pilot study.}

\end{itemize}

\section{Outline}

In Part~\ref{part:background} of this thesis, we start by introducing in
Chapter~\ref{ch:supervised-learning} the key concepts about supervised learning:
(i) what are the most popular supervised learning models, (ii) how to assess the
prediction performance of a supervised learning model and (iii) how to optimize
the hyper-parameters of theses models. We also present some unsupervised
projection methods, such as random projections, which transform the original
space to another one. We describe more in detail the decision tree model classes
in Chapter~\ref{ch:tree}. More specifically, we describe the methodology to grow
and to prune such trees. We also show how to adapt decision tree growing and
prediction algorithms to multi-output tasks. In
Chapter~\ref{ch:bias-var-ensemble}, we show why and how to combine models into
ensembles either by learning models independently with averaging methods or
sequentially with boosting methods.

In Part~\ref{part:tree-rp}, we first show how to grow an ensemble of decision
trees on very high dimensional output spaces  by projecting the original output
space onto a random sub-space of lower dimension. In
Chapter~\ref{ch:rf-output-projections}, it turns out that for random forest
models, an averaging ensemble of decision trees, the learning time complexity
can be reduced without affecting the prediction performance. Furthermore, it may
lead to accuracy improvement~\cite{joly2014random}. In
Chapter~\ref{ch:gbrt-output-projection}, we propose to combine random projections
of the output space and the gradient tree boosting algorithm, while reducing learning
time and automatically adapting to any output correlation structure.

In Part~\ref{part:sparsity}, we leverage sparsity in the context of decision
tree ensembles. In Chapter~\ref{ch:rf-compression}, we exploit sparsifying
optimization algorithms to compress random forest models while retaining their
prediction performances~\cite{joly2012l1}. In Chapter~\ref{ch:tree-sparse}, we
show how to leverage input sparsity to speed up decision tree induction.

During the thesis, I made significant contributions to the open source
scikit-learn project~\cite{pedregosa2011scikit,buitinck2013api} and developed my
own open source libraries
random-output-trees\footnote{\url{https://github.com/arjoly/random-output-trees}},
containing the work presented in Chapter~\ref{ch:rf-output-projections} and
Chapter~\ref{ch:gbrt-output-projection}, and
clusterlib\footnote{\url{https://github.com/arjoly/clusterlib}}, containing the
tools to manage jobs on supercomputers.

\part{Background}

\label{part:background}

\chapter{Supervised learning}\label{ch:supervised-learning}

\begin{remark}{Outline}
In the field of machine learning, supervised learning aims at finding the best
function which describes the input-output relation of a system only from
observations of this relationship. Supervised learning problems can be broadly
divided into classification tasks with discrete outputs and into regression
tasks with continuous outputs. We first present major supervised learning
methods for both classification and regression. Then, we show  how to estimate
their performance and how to optimize the hyper-parameters of these models. We
also introduce unsupervised projection techniques used in conjunction with
supervised learning methods.
\end{remark}

Supervised learning aims at modeling an input-output system from observations
of its behavior. The applications of such learning methods encompass a wide
variety of tasks and domains ranging from image recognition to medical diagnosis
tools. Supervised learning algorithms analyze the input-output pairs and learn
how to predict the behavior of a system (see Figure~\ref{fig:system}) by
observing its responses, described by output variables $y_1,\ldots,y_d$, also
called targets, to its environment described by input variables
$x_1,\ldots,x_p$, also called features. The outcome of the supervised learning
is a function $f$ modeling the behavior of the system.

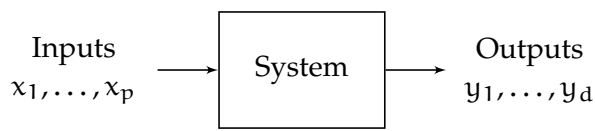
\begin{figure}[h]
\centering
\begin{tikzpicture}[auto, node distance=2cm, semithick]

\node [text centered, text width=5em,](inputs) {Inputs \\ $x_1,\ldots,x_p$};
\node [rectangle, draw, text width=5em, text centered,
       node distance=3cm, minimum height=4em,
       right of=inputs] (system) {System};
\node [right of=system, node distance=3cm,
       text width=5em, text centered] (outputs) {Outputs \\ $y_1,\ldots,y_d$};

\path [draw, -latex'] (inputs) -- (system);
\path [draw, -latex'] (system) -- (outputs);
\end{tikzpicture}

\caption{Input-output view of a system.}
\label{fig:system}
\end{figure}

Supervised learning has numerous applications in the multimedia,
in biology, in engineering or in the societal domain:
\begin{itemize}

\item Identification of digits from photos, such as house number from street
photos or digit post code from letters.

\item Automatic image annotation such as detecting tumorous cells or identifying
people in photos.

\item Detection of genetic diseases from DNA screening.

\item Disease diagnostic based on  clinical and  biological data of a
patient.

\item Automatic text translation from a source language to a target language such
as from French to English.

\item Automatic voice to text transcription from audio records.

\item Market price prediction on the basis of economical and performance indicators.

\end{itemize}

We introduce the supervised learning framework in Section~\ref{sec:sl-intro}. We
describe in Section~\ref{sec:supervised-learning-classes} the most common
classes of supervised learning models used to map the outputs of the system to
its inputs. We introduce how to assess their performances in
Section~\ref{sec:model-eval}, how to compare the model predictions to a ground
truth in Section~\ref{sec:model-diagnosis} and  how to select the best
hyper-parameters of such models in Section~\ref{sec:model-select}. We also show
some input space projection methods in
Section~\ref{sec:dimensionality-reduction}, often used in combination with
supervised learning models improving the computational time and / or the
accuracy of the model.

\section{Introduction}
\label{sec:sl-intro}

The goal of supervised learning is to learn the function $f$ mapping an input
vector $x= (x_1,\ldots,x_p)$ of a system to a vector of system outputs $y =
(y_1,\ldots,y_d)$, only from observations of input-output pairs. The set of
possible input (resp. output) vectors form the input space $\mathcal{X}$ (resp.
output space $\mathcal{Y}$).

Once we have identified the input and output variables, we start to collect
input-output pairs, also called samples. Table~\ref{tab:sl-dataset} displays 5
samples collected from a system with 4 inputs and 3 outputs. We distinguish
three types of variables: binary variables taking only two different values,
like the variables $x_1$ and $y_1$; categorical variables taking two or more
possible values, like variables $x_2$ and $y_2$, and numerical variables having
numerical values, like $x_2$, $x_4$ and $y_3$. A binary variable is also a
categorical variable. For simplicity, we will assume in the following without
loss of generality that binary and categorical variables have been mapped from
the set of their $k$ original values to a set of integers of the same cardinality
$\{0,\ldots,k-1\}$.

\begin{table}[h]
\caption{A dataset formed of samples of pairs of four inputs and three outputs.}
\label{tab:sl-dataset}
\centering
\begin{tabular}{@{}ccccccc@{}}
\toprule
$x_1$ & $x_2$ & $x_3$ & $x_4$ & $y_1$ & $y_2$ & $y_3$ \\
\midrule
$0$ & $0.25$ & A & $0.25$ & True  & Small   & $1.8$ \\
?   & $-2$   & B & $3.$   & True  & Average & $1.7$ \\
$0$ & $3$    & C & $2.$   & False & ?       & $1.65$ \\
$1$ & $10.7$ & ? & $-3.$  & False & Big     & $1.59$ \\
$1$ & $0.$   & A & $2.$   & False & Big     & ? \\
\bottomrule
\end{tabular}
\end{table}

When we collect data, some input and/or output values might be missing or
unavailable. Tasks with missing input values are said to have missing data.
Missing values are marked by a ``?'' in Table~\ref{tab:sl-dataset}.

We classify supervised learning tasks into two main families based on their
output domains. Classification tasks have either binary outputs as in disease
prediction ($y \in \{Healthy, sick\}$) or categorical outputs as in digits
recognition ($y \in \{0,\ldots,9\}$). Regression tasks have numerical outputs
($y \in \mathbb{R}$) such as in house price predictions. A classification task
with only one binary output (resp. categorical output) is called a binary
classification task (resp. multi-class classification task). A multi-class
classification task is assumed to have more than two classes, otherwise it is a
binary classification task. In the presence of multiple outputs, we further
distinguish multi-label classification tasks which associate multiple binary output
values to each input vector. In the multi-label context, the output variables
are also called ``labels'' and the output vectors are called ``label set''. From
a modeling perspective, multi-class classification tasks are multi-label
classification problems whose labels are mutually exclusive.
Table~\ref{tab:sl-tasks} summarizes the different supervised learning tasks.

\begin{table}[h]
\caption{The output domain determines the supervised learning task.}
\label{tab:sl-tasks}
\centering
\begin{tabular}{@{}ll@{}}
\toprule
Supervised learning task & Output domain \\
\midrule
Binary classification & $\mathcal{Y} = \{0, 1\}$ \\
Multi-class classification & $\mathcal{Y} = \{0, 1, \ldots, k-1\}$ with $k\hspace*{-1mm}>\hspace*{-1mm}2$ \\
Multi-label classification & $\mathcal{Y} = \{0, 1\}^d$ with $d\hspace*{-1mm}>\hspace*{-1mm}1$  \\
Multi-output multi-class classification & $\mathcal{Y} = \{0, 1, \ldots, k-1\}^d$\\
& with $k > 2, d > 1$ \\
Regression & $\mathcal{Y} = \mathbb{R}$ \\
Multi-output regression & $\mathcal{Y} = \mathbb{R}^d$ with $d\hspace*{-1mm}>\hspace*{-1mm}1$\\
\bottomrule
\end{tabular}
\end{table}

We will denote by $\mathcal{X}$ an input space, and by $\mathcal{Y}$ an output
space. We denote by $P_{{\cal X},{\cal Y}}$ the joint (unknown) sampling density
over $\mathcal{X} \times \mathcal{Y}$. Superscript indices ($x^{i}, y^{i}$)
denote (input, output) vectors of an observation $i \in \{1, \ldots , n\}$.
Subscript indices (e.g.\ $x_{j}, y_{k}$) denote components of vectors. With
these notations supervised learning can be defined as follows:
\begin{remark}{Supervised learning}
Given a learning sample $\left((x^i, y^i) \in \left(\mathcal{X} \times
\mathcal{Y}\right)\right)_{i=1}^n$ of $n$ observations in the form of
input-output pairs, a supervised learning task is defined as searching for a
function $f^{*} : \mathcal{X} \rightarrow \mathcal{Y}$ in a hypothesis space
$\mathcal{H} \subset \mathcal{Y}^\mathcal{X}$ that minimizes the expectation of
some loss function $\ell : \mathcal{Y} \times \mathcal{Y} \rightarrow
\mathbb{R}^+$ over the joint distribution of input / output pairs:
\begin{equation}
f^{*} = \arg \min_{f \in \mathcal{H}} \E_{P_{{\cal X},{\cal Y}}} \{ \ell(f(x), y)\}.
\end{equation}
\end{remark}

The choice of the loss function $\ell$ depends on the property of the supervised
learning task (see Table~\ref{tab:common-loss} for their definitions):
\begin{itemize}

\item In regression ($\mathcal{Y} = \mathbb{R}$), we often use the squared loss,
except when we want to be robust to the presence of outliers, samples with
abnormal output values, where we prefer other losses such as the absolute loss.

\item In classification tasks ($\mathcal{Y} = \{0,\ldots,k-1\}$), the reported
performance is commonly the average $0-1$ loss, called the error rate. However,
the model does not often directly minimize the $0-1$ loss as it leads to non
convex and non continuous optimization problems with often high computational
cost. Instead, we can relax the multi-class or the binary constraint by
optimizing a smoother loss such as the hinge loss or the logistic loss. To get a
binary or multi-class prediction, we can threshold the predicted value $f(x)$.

\end{itemize}

\begin{table}
\caption{Common losses to measure the discrepancy between a ground truth $y$ and
either a prediction or a score $y'$. In classification, we assume here that
the ground truth $y$ is encoded with $\{-1, 1\}$ values.}
\label{tab:common-loss}
\centering
\begin{tabular}{@{}ll@{}}
\toprule
Regression loss &  \\
\midrule
Square loss & $\ell(y, y') = \frac{1}{2} (y - y')^2$ \\
Absolute loss & $\ell(y, y') = |y - y'|$\\
\midrule
Binary classification loss &  \\
\midrule
0-1 loss &  $\ell(y, y') = 1(y\not=y')$\\
Hinge loss &  $\ell(y, y') = \max(0, 1 - y y')$\\
Logistic loss & $\ell(y, y') = \log(1 + \exp(- 2 y y'))$\\
\bottomrule
\end{tabular}
\end{table}

Figure~\ref{fig:loss} plots several loss discrepancies $\ell(1, y')$ whenever
the ground truth is $y=1$ d as a function of the value $y'$ predicted by the
model. The $0-1$ loss is a step function with a discontinuity at $y'=1$. The
hinge loss has a linear behavior whenever $y' \leq 1$ and is a constant with
$y' \geq 1$. The logistic loss strongly penalizes any mistake and is zero only
if the model is correct with an infinite score. The plot also highlights that we
can use regression losses for classification tasks. It shows that regression
losses penalize any predicted value $y'$ different from the ground truth $y$.
However, this is not always the desired behavior. For instance whenever $y=1$
(resp. $y=0$), regression losses penalize any score greater than $y'>1$ (resp.
smaller than $y' < 0$), while the model truly believes that the output is
positive (resp. negative). This is often the reason why regression losses are
avoided for classification tasks.

\begin{figure}
\centering
\includegraphics[width=0.75\textwidth]{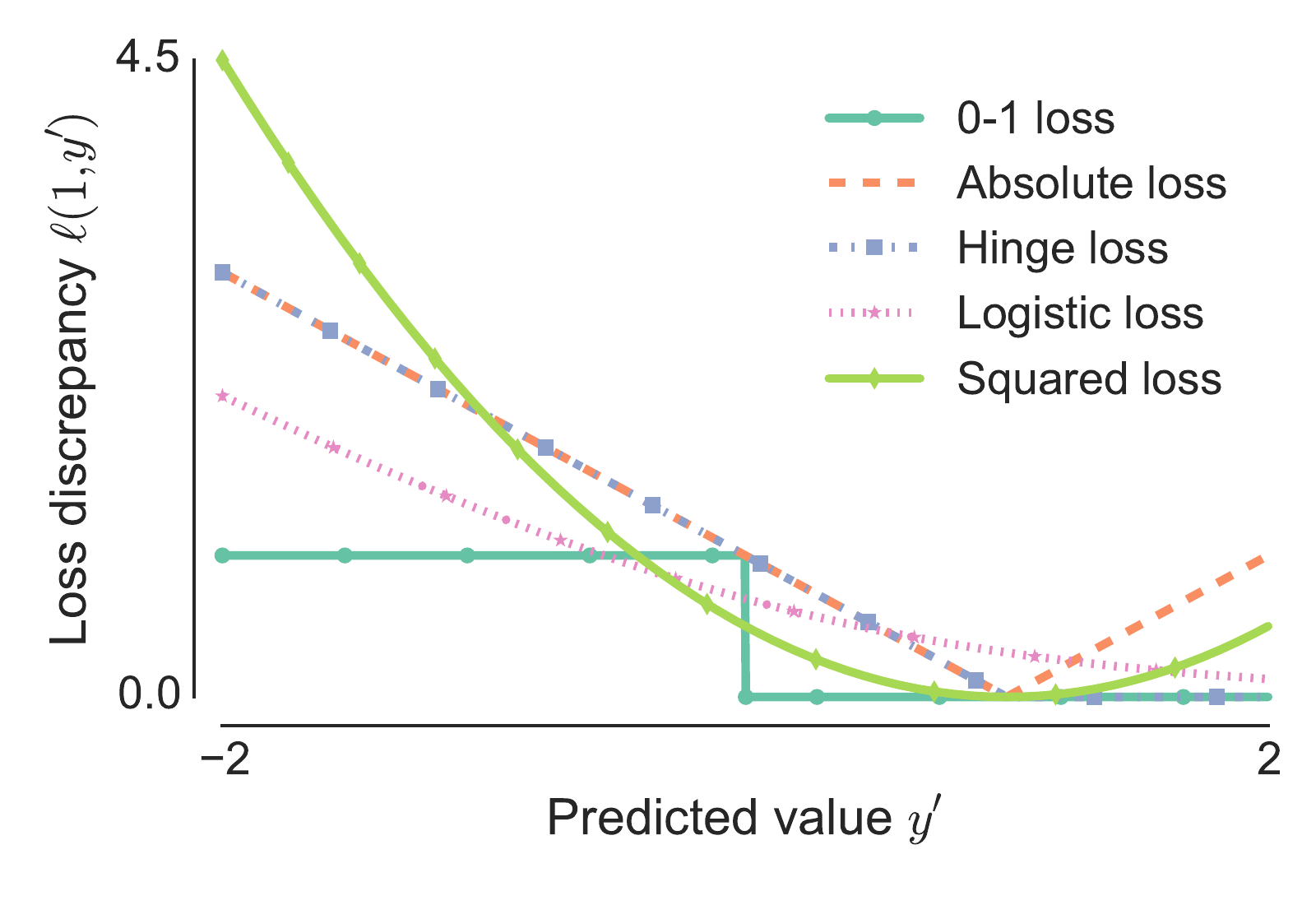}
\caption{Loss discrepancies $\ell(1, y')$ with $y=1$. (Adapted
from~\cite{friedman2001elements})}
\label{fig:loss}
\end{figure}

\section{Classes of supervised learning algorithms}
\label{sec:supervised-learning-classes}

Supervised learning aims at finding the best function $f$ in a hypothesis space
$\mathcal{H}$ to model the input-output function of a system. If there is no
restriction on the hypothesis space $\mathcal{H}$, the model $f$ can be any
function $f \in \mathcal{Y}^\mathcal{X}$.

Consider a binary function $f$ which has $p$ binary inputs. The binary function
is uniquely defined by knowing the output values of the $2^p$ possible input
vectors. The hypothesis space of all binary functions contains $2^{2^p}$ binary
functions. If we observe $n$ different input-output pair assignments, there
remain $2^{2^p - n}$ possible binary functions. For a binary function of $p=5$
inputs, we have $2^{2^5} = 4294967296$ possible binary functions. If we observe
$n = 16$ input-output pair assignments among the $2^5 = 32$ possible ones, we
still have $2^{2^5 - 16} = 65536$ possible binary functions. The number of
possible functions highly increases with the cardinality of each variable. The
hypothesis space will be even larger with a stochastic function, where
different output values are possible for each possible input assignment.

By making assumptions on the model class $\mathcal{H}$, we can largely reduce
the size of the hypothesis space. For instance in the previous example, if we
assume that 2 out of the 5 binary input variables are independent of the output,
there remain $2^{2^3}=256$ possible functions. The correct function would be
uniquely identified by observing the $8$ possible assignments.

Given the data stochasticity, those model classes can directly model the
input-output mapping $f: \mathcal{X} \rightarrow \mathcal{Y}$, but also the
conditional probability $P(y | x)$ and predictions are made through
\begin{equation}
f(x) = \arg\min_{\hat{y} \in \mathcal{Y}} E_{y|x}\left[ L(y, \hat{y}) \right]
= \arg\min_{\hat{y}} \int_{\mathcal{Y}} L(y, \hat{y}) dP(y|x).
\label{eq:model-cond-rule}
\end{equation}

We will present some of the most popular model classes: linear models in
Section~\ref{sub:linear-models-classes}; artificial neural networks in
Section~\ref{sub:neural-classes} which are inspired from the neurons in the
brain; neighbors-based models in Section~\ref{sub:neighbors-class} which find
the nearest samples in the training set; decision tree based-models in
Section~\ref{sub:decision-tree-class} (and in more details in
Chapter~\ref{ch:tree}). Note that we introduce ensemble methods in
Section~\ref{sub:decision-tree-class} and discuss them more deeply in Chapter~\ref{ch:bias-var-ensemble}.

\subsection{Linear models}
\label{sub:linear-models-classes}

Let us denote by $x \in \mathbb{R}^{p} = \begin{bmatrix}x_1& \ldots &
x_p \end{bmatrix}^T$ a vector of input variables. A linear model $\hat{f}$ is a
model having the following form
\[
\hat{f}(x) = \beta_0 + \sum_{j=1}^p \beta_j x_j
= \begin{bmatrix}1 & x^T\end{bmatrix} \beta
\]
\noindent where the vector $\beta \in \mathbb{R}^{1 + p}$ is a
concatenation of the intercept $\beta_0$ and the coefficients $\beta_j$.

Given a set of $n$ input-output pairs $\{(x^i, y^i) \in
(\mathcal{X},\mathcal{Y})\}_{i=1}^n$, we retrieve the coefficient vector $\beta$
of the linear model by minimizing a loss $\ell: \mathcal{Y} \times \mathcal{Y}
\rightarrow \mathbb{R}^+$:
\begin{equation}
\min_{\beta} \sum_{i=1}^n \ell(y^i, \hat{f}(x^i)) =
\min_{\beta} \sum_{i=1}^n \ell(y^i, \begin{bmatrix}1 & {x^i}^T\end{bmatrix} \beta).
\label{eq:optim-linear-model}
\end{equation}

With the square loss $\ell(y, y') = \frac{1}{2} (y - y')^2$, there exists an
analytical solution to Equation~\ref{eq:optim-linear-model} called ordinary
least squares. Let us denote by $\mathbf{X} \in \mathbb{R}^{n \times (1 +p)}$
the concatenation of the input vectors with a first column of $\mathbf{X}$ full
of ones to model the intercept $\beta_0$ and by $\mathbf{y}\in \mathbb{R}^{n}$
the concatenation of the output values. We can now express the sum of squares in
matrix notation:
\begin{align}
\sum_{i=1}^n \ell(y, \hat{f}(x^i))
&= \frac{1}{2} \sum_{i=1}^n (y^i - \hat{f}(x^i))^2\\
&= \frac{1}{2} (\mathbf{y} - \mathbf{X}\beta)^T(\mathbf{y} - \mathbf{X}\beta)
\end{align}
The first order differentiation of the sum of squares with respect to
$\beta$ yields to
\begin{equation}
\frac{\partial}{\partial\beta} \sum_{i=1}^n \ell(y, \hat{f}(x^i))
= \mathbf{X}^T (\mathbf{y} - \mathbf{X}\beta).
\end{equation}
\noindent The vector minimizing the square loss is thus
\begin{equation}
\beta = (\mathbf{X}^T\mathbf{X})^{-1}\mathbf{X}^T y.
\label{eq:ols-solution}
\end{equation}
\noindent The solution exists only if $\mathbf{X}^T\mathbf{X}$ is invertible.

Whenever the number of inputs plus one $p+1$ is greater than the number of
samples $n$, the analytical solution is ill posed as the matrix
$\mathbf{X}^T\mathbf{X}$ is rank deficient ($rank(\mathbf{X}^T\mathbf{X}) < p + 1$). To ensure a
unique solution, we can add a regularization penalty $R$ with a multiplying
constant $\lambda \in \mathbb{R}^{+}$ on the coefficients $\beta$ of the linear model:
\begin{equation}
\min_{\beta} \sum_{i=1}^n L(y^i, \beta_0 + \sum_{j=1}^p \beta_j x_j^i) + \lambda R(\beta_1, \ldots, \beta_p).
\end{equation}

With a $\ell_2$-norm constraint on the coefficients,
we transform the ordinary least square model into a ridge regression
model~\cite{hoerl1970ridge}:
\begin{equation}
\min_{\beta} \sum_{i=1}^n \left(y^i - \beta_0 - \sum_{j=1}^p \beta_j x_j^i\right)^2 +
\lambda \sum_{j=1}^p \beta_j^2.
\end{equation}
\noindent One can show (see Section 3.4.1 of~\cite{friedman2001elements}) that
the constant $\lambda$ controls the maximal value of all coefficients $\beta_j$
in the ridge regression solution.

With a $\ell_1$-norm constraint ($R(\beta_1, \ldots, \beta_p) = \sum_{j=1}^p
|\beta_j|$) on the coefficient $\beta_1,\ldots, \beta_p$, we have the Lasso
model~\cite{tibshirani1996regression}:
\begin{equation}
\min_{\beta} \sum_{i=1}^n \left(y^i - \beta_0 - \sum_{j=1}^p \beta_j x_j^i\right)^2 +
\lambda \sum_{j=1}^p |\beta_j|.
\end{equation}
\noindent Contrarily to the ridge regression, the Lasso has no closed formed
analytical solution  even though the resulting optimization problem remains
convex. However, we gain that the $\ell_1$-norm penalty sparsifies the
coefficients $\beta_j$ of the linear model. If the constant $\lambda$ tends
towards infinity, all the coefficients will be zero $\beta_j=0$. While with
$\lambda=0$, we have the ordinary least square formulation. With $\lambda$
moving from $+\infty$ to $0$, we progressively add variables to the linear model
with a magnitude $\sum_{j=1}^p |\beta_j|$ depending on $\lambda$. The monotone
Lasso~\cite{hastie2007forward} further restricts the coefficient to monotonous
variation with respect to $\lambda$ and has been shown to perform better
whenever the input variables are correlated.

A combination of the $\ell_1$-norm and the $\ell_2$-norm constraints on the
coefficients is called an elastic net penalty~\cite{zou2005regularization}. It
shares both the property of the Lasso and the ridge regression: sparsely
selecting coefficients as in Lasso and considering groups of correlated
variables together as in the ridge regression. With a careful design of the
penalty term $R$, we can enforce further properties such as selecting variables
in groups of pre-defined variables with the group
Lasso~\cite{yuan2006model,meier2008group} or taking into account the variable
locality in the coefficient vector $\beta$ while adding a new variable to the
linear model with the fused Lasso~\cite{tibshirani2005sparsity}.

By selecting an appropriate loss and penalty term, we have a wide variety of
linear models at our disposal with different properties. In regression, an
absolute loss leads to the least absolute deviation
algorithm~\cite{bloomfield2012least} which is robust to outliers. In
classification, we can use a logistic loss to model the class probability
distribution leading to the logistic regression model. With a hinge loss, we aim
at finding a hyperplane which maximizes the separations between the classes
leading to the support vector machine algorithm~\cite{cortes1995support}.

\begin{figure}
\centering
\subfloat[]{{\includegraphics[width=0.4\textwidth]{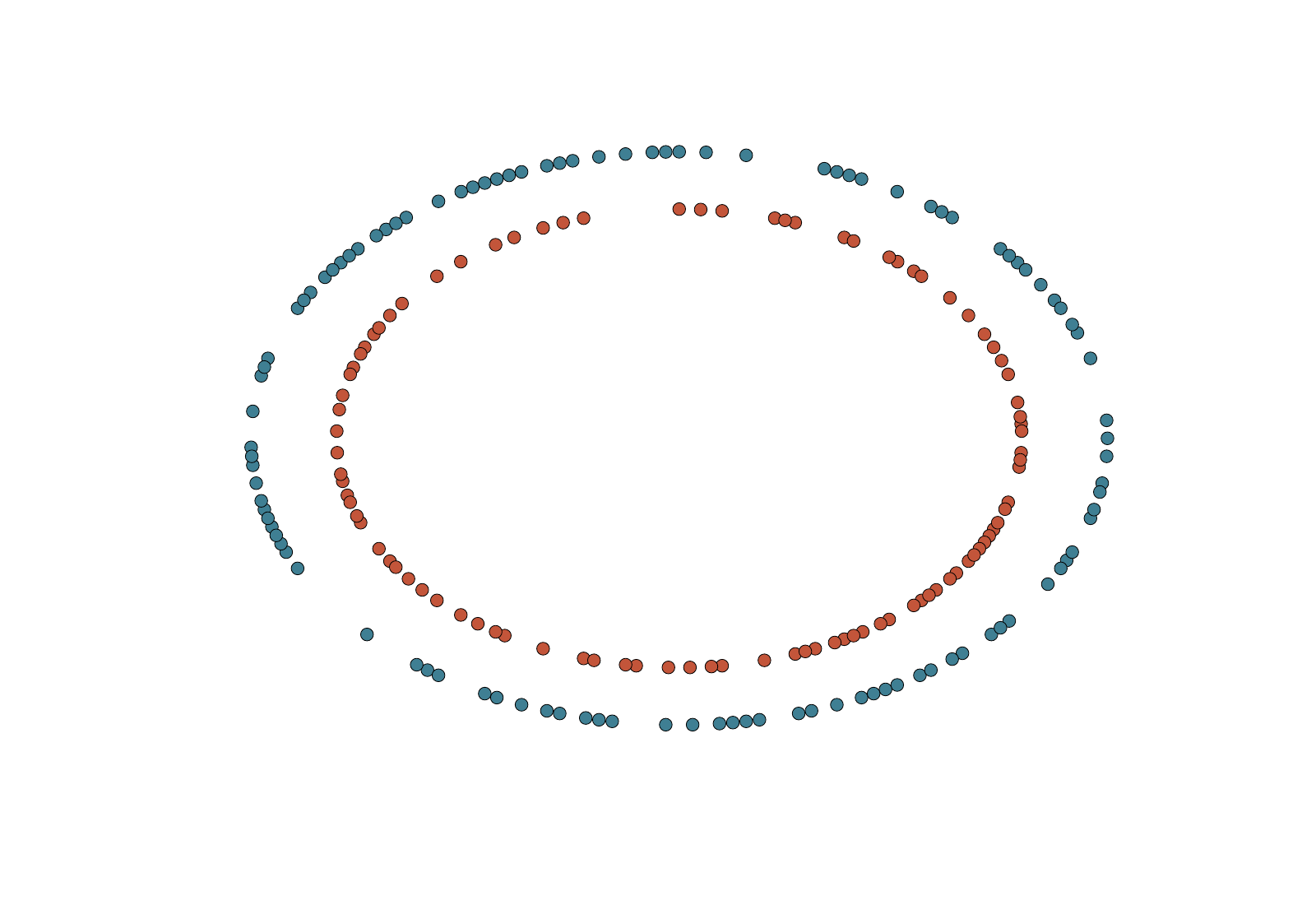}
\label{subfig:circle-data}}} \\
\subfloat[Original input space]{{\includegraphics[width=0.4\textwidth]{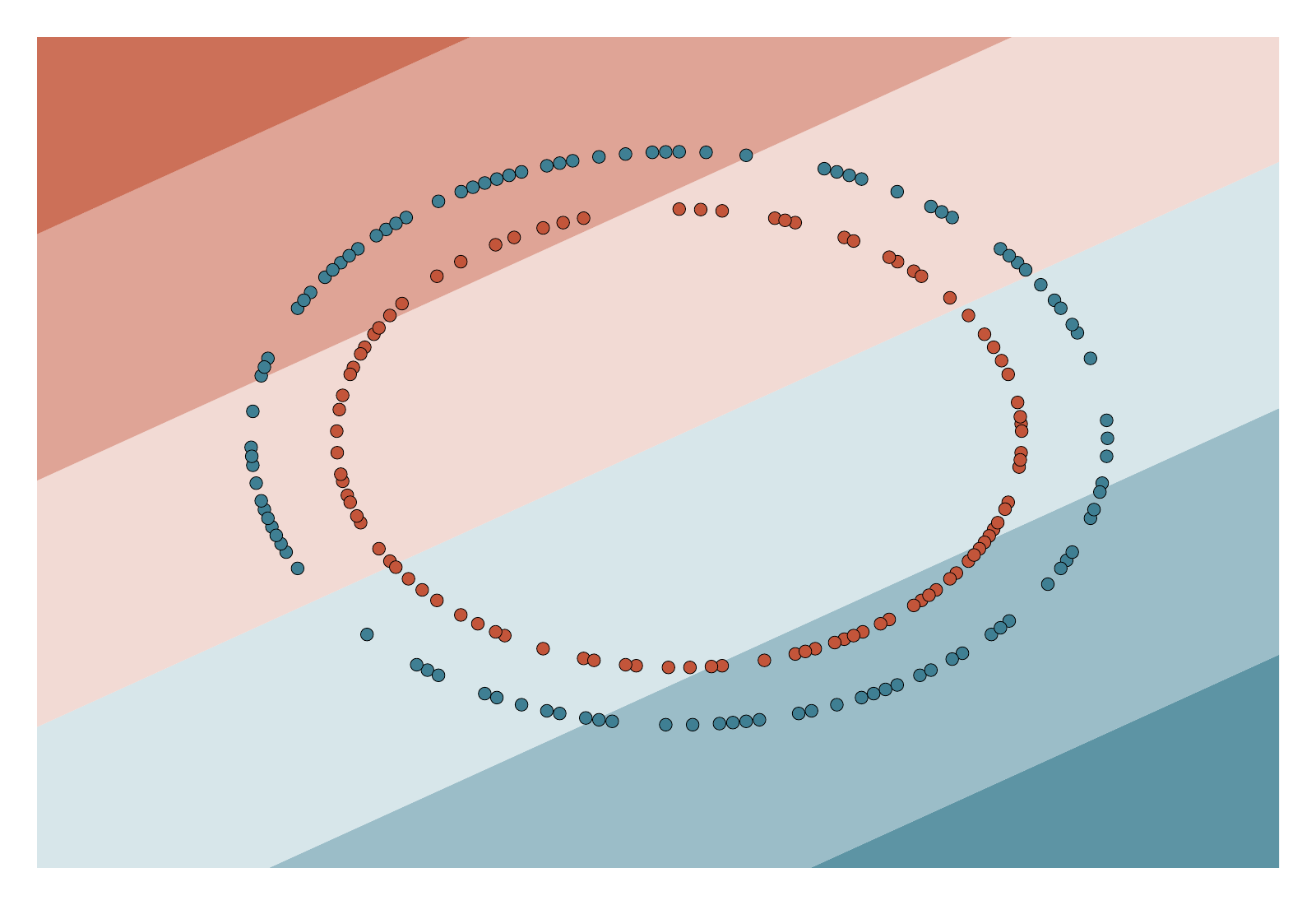}
\label{subfig:circle-linear}}}
\subfloat[Non-linear transformation]{{\includegraphics[width=0.4\textwidth]{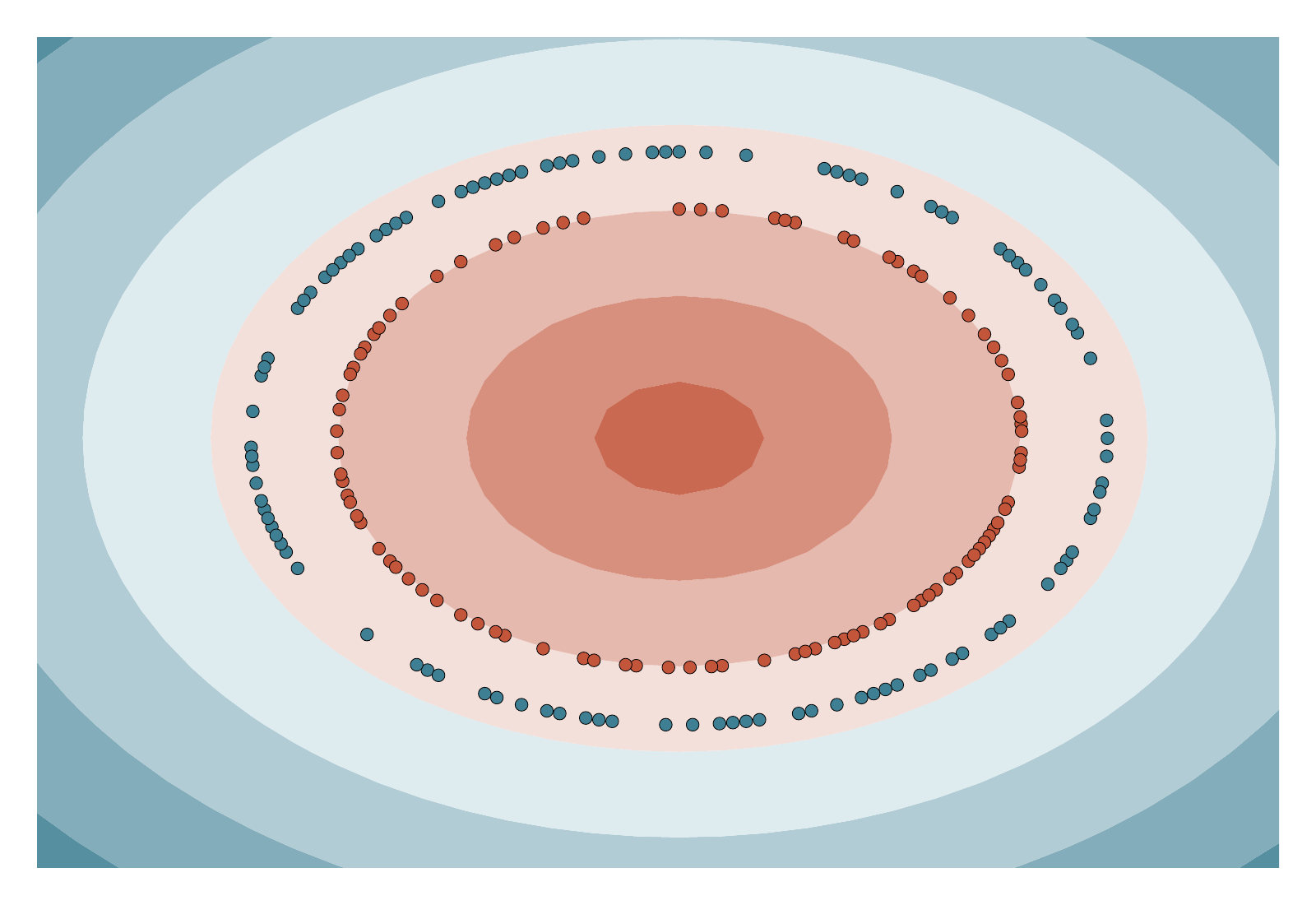}
\label{subfig:circle-nonlinear}}}
\caption{The logistic linear model on the bottom left is unable to find a separating
hyperplane. If we fit the linear model on the distance from the center of circle,
we separate perfectly both classes as shown in the bottom right.}
\label{fig:circle-model}
\end{figure}

A linear model can handle non linear problems by applying first a non linear
transformation to the input space $\mathcal{X}$. For instance, consider the
classification task of Figure~\ref{subfig:circle-data} where each class is
located on a concentric circle. Given the non linearity of the problem, we can
not find a straight line separating both classes in the cartesian plane as
shown in Figure~\ref{subfig:circle-linear}. If instead we fit a linear model on
the distance from the origin $\sqrt{x_1^2+x_2^2}$ as illustrated in
Figure~\ref{subfig:circle-nonlinear}, we find a model separating perfectly
both classes. We often use linear models in conjunction with kernel functions
(presented in Section~\ref{subsec:kernels}), which provide a range of ways to
achieve non-linear transformations of the input space.

\subsection{(Deep) Artificial neural networks}
\label{sub:neural-classes}

An artificial neural network is a statistical model mimicking the structure of
the brain and composed of artificial neurons. A neuron, as shown in
Figure~\ref{subfig:bio-neuron}, is composed of three parts: the \emph{soma}, the
cell body, processes the information from its \emph{dendrites} and transmits its
results to other neurons through the \emph{axon}, a nerve fiber. An artificial
neuron follows the same structure (see Figure~\ref{subfig:artificial-neuron})
replacing biological processing by numerical computations. The basic
neuron~\cite{rosenblatt1958perceptron} used for supervised learning consists in
a linear model of parameters $\beta \in \mathbb{R}^{p+1}$ followed by an
activation function $\phi$:
\[
\hat{f}_{neuron}(x) = \phi\left(\beta_0 + \sum_{j=1}^p \beta_j x_j\right).
\]

The activation function replicates artificially the non linear activation of
real neurons. It is a scalar function such as a hyperbolic tangent $\phi(x) =
\tanh(x)$, a sigmoid $\phi(x) = (1 + e^{-x})^{-1}$ or a rectified linear
function $\phi(x) = \max(0, x)$.

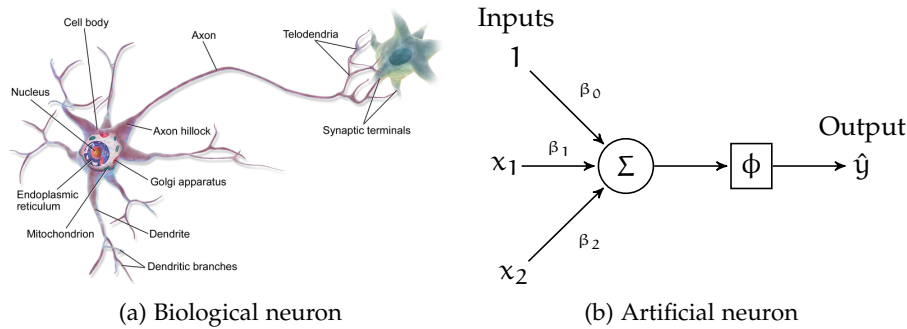
\begin{figure}[h]
\centering
\subfloat[Biological neuron]{{\includegraphics[width=0.5\textwidth]{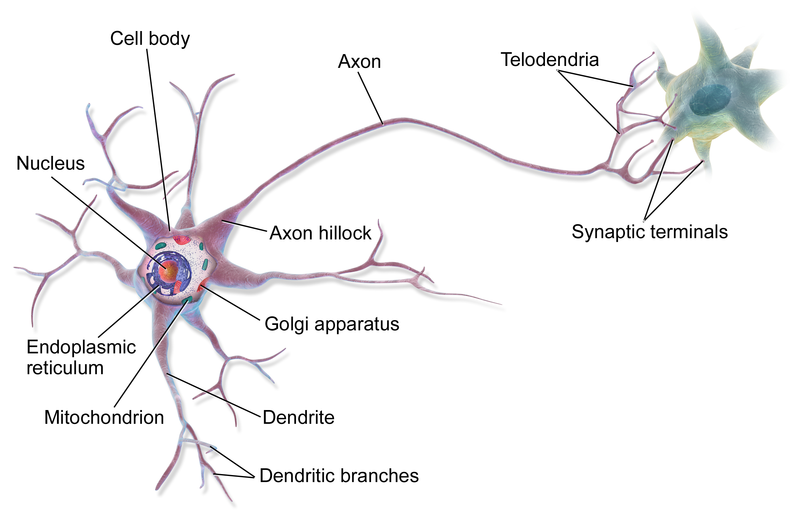}
\label{subfig:bio-neuron}}}
\subfloat[Artificial neuron]{{
\begin{tikzpicture}[auto]

\tikzstyle{arrow} = [->, shorten >=1pt,>=stealth',semithick];
\tikzstyle{func} = [minimum size=10pt,semithick];
\tikzstyle{inout} = [minimum size=10pt,inner sep=0pt];

\node[inout] (y) [label=above:\small{Output}] {$\hat{y}$};

\node[draw, rectangle, func] (activation)[left=of y]
        {$\phi$}
    edge [arrow] (y);
\node[draw, circle, func] (sum) [left=of activation]  {$\Sigma$}
    edge [arrow] (activation);

\node[inout] (bias)  [above left=of sum, label=above:\small{Inputs}] {$1$}
    edge [arrow] node[auto] {\tiny $\beta_0$} (sum);
\node[inout] (x1)  [left=of sum] {$x_1$}
    edge [arrow] node[auto] {\tiny $\beta_1$} (sum);
\node[inout] (x2)  [below left=of sum] {$x_2$}
    edge [arrow] node[auto,swap] {\tiny $\beta_2$} (sum);

\end{tikzpicture}
\label{subfig:artificial-neuron}}}
\caption{A biological neuron (on the left) and an artificial neuron (on the right).}
\label{fig:neuron-model}
\end{figure}

More complex artificial neural networks are often structured into layers of
artificial neurons. The inputs of a layer are the input variables or the outputs
of the previous layer. Each neuron of the layer has one output. The neural
network is divided into three parts as in Figure~\ref{fig:multilayer-network}: the
first and last layers are respectively the \emph{input layer} and the
\emph{output layer}, while the layers in between are the \emph{hidden layers}.
The hidden layer of Figure~\ref{fig:multilayer-network} is called a fully
connected layer as all the neurons (here the input variables) from the previous
layer are connected to each neuron of the layer. Other layer structures exist
such as convolutional layers~\cite{krizhevsky2012imagenet,lecun2004learning}
which mimic the visual cortex~\cite{hubel1968receptive}. A network is not
necessarily feed forward, but can have a more complex topology for example
recurrent neural networks~\cite{boulanger2012modeling,graves2013speech} mimic
the brain memory by forming internal cycles of neurons. Neural networks with
many layers are also known~\cite{lecun2015deep} as deep neural networks.

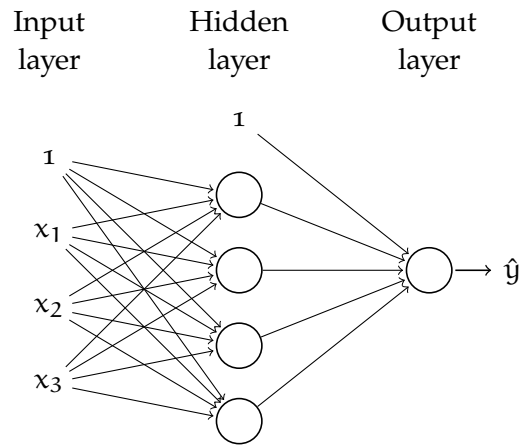
\begin{figure}[h]
\centering
\begin{tikzpicture}[shorten >=1pt,->, node distance=2.5cm]

\tikzstyle{every pin edge}=[<-,shorten <=1pt,semithick]
\tikzstyle{neuron}=[circle,minimum size=17pt,inner sep=0pt,semithick]
\tikzstyle{annotation} = [text width=4em, text centered]

\foreach \name / \y in {0,...,3}
    \node[neuron] (I-\name) at (0,-\y)
    {\ifthenelse{\NOT\name=0}{$x_\name$}{1} };

\node[neuron] (H-0) at (2.5cm, 0.5cm) {1};
\foreach \name / \y in {1,...,4}
    \path[yshift=0.5cm]
        node[draw, neuron] (H-\name) at (2.5cm,-\y cm) {};

\node[draw,neuron,pin={[pin edge={->}]right:$\hat{y}$}, right of=H-2] (O) {};

\foreach \source in {0,...,3}
    \foreach \dest in {1,...,4}
        \path (I-\source) edge (H-\dest);

\foreach \source in {0,...,4}
    \path (H-\source) edge (O);

\node[annotation,above of=H-0, node distance=1cm] (hl) {Hidden layer};
\node[annotation,left of=hl] {Input layer};
\node[annotation,right of=hl] {Output layer};
\end{tikzpicture}
\caption{A neural network with an input layer, a fully connected hidden layer
and an output layer.}
\label{fig:multilayer-network}
\end{figure}

Artificial neurons form a graph of variables. Through this representation, we
can learn such models by applying gradient based optimization
techniques~\cite{bengio2012practical,glorot2010understanding,lecun2012efficient}
to find the coefficient vector associated to each neuron minimizing a given loss
function.

\subsection{Neighbors based methods}
\label{sub:neighbors-class}

The $k$-nearest neighbors model is defined by a distance metric $d$ and a set
of samples. At learning time, those samples are stored in a database. We predict
the output of an unseen sample by aggregating the outputs of the $k$-nearest
samples in the input space according to the distance metric $d$, with $k$
being a user-defined parameter.

More precisely, given a training set $\left((x^i, y^i) \in
\left(\mathcal{X} \times \mathcal{Y}\right)\right)_{i=1}^n$ and a distance
measure $d : \mathcal{X} \times \mathcal{X}\rightarrow\mathbb{R}^{+}$, an unseen
sample with value in the input space $x$ is assigned a prediction through the
following procedure:
\begin{enumerate}

\item Compute the distances $d(x^i, x)$ in the input space, $\forall i= 1, \ldots, n$, between the training samples $x^i$ and
the input vector $x$.

\item Search for the $k$ samples in the training set which have the
smallest distance to the vector $x$.

\item In classification, compute the proportion of samples of each
class among these $k$-nearest neighbors: the final prediction is the class
with the highest proportion. This corresponds to a majority vote over the $k$
nearest neighbors. In regression, the prediction is the average output of the
$k$-nearest neighbors.

\end{enumerate}

The $k$-nearest neighbor method adapts to a wide variety of scenarios by selecting
or by designing a proper distance metric such as the euclidean distance or the
Hamming distance.

\subsection{Decision tree models}
\label{sub:decision-tree-class}

A decision tree model is a hierarchical set of questions leading to a
prediction. The internal nodes, also called test nodes, test the value of a
feature. In Figure~\ref{fig:dt-iris}, the starting node, also called root node,
tests whether the feature ``Petal width'' is bigger or smaller than $0.7cm$.
According to the answer, you follow either the right branch ($>0.7cm$) leading
to another test node or the left branch ($\leq 0.7cm$) leading to an external
node, also called a leaf. To predict an unseen sample, you start at the root
node and follow the tree structure until reaching a leaf labelled with a
prediction. With the decision tree of Figure~\ref{fig:dt-iris}, an iris with
petal width smaller than $0.7cm$ is an iris Setosa.

\begin{figure}[h]
\centering
\begin{tikzpicture}
[level 1/.style={sibling distance=60.0mm},
 level 2/.style={sibling distance=60.0mm},
 level 3/.style={sibling distance=30.0mm},
 level 4/.style={sibling distance=30.0mm},
 leaf/.style={rectangle},
 test/.style={rectangle},
 edge from parent/.style = {draw, -latex, font=\footnotesize},
 level distance=20mm]
\node[test] {Petal width?}
 child {node[leaf] {Setosa} edge from parent node[left] {$0.7cm\leq\,$}}
 child {node[test] {Petal width?}
     child {node[test] {Petal length?}
         child {node[leaf] {Versicolor} edge from parent node[left] {$5.25cm\leq\,$}}
         child {node[leaf] {Virginica} edge from parent node[right] {$\,>5.25cm$}}
         edge from parent node[left] {$1.65cm\leq\,$}}
     child {node[test] {Sepal length?}
         child {node[test] {Sepal length?}
             child {node[leaf] {Virginica} edge from parent node[left] {$5.85cm\leq\,$}}
             child {node[leaf] {Versicolor} edge from parent node[right] {$\,>5.85cm$}}
             edge from parent node[left] {$5.95cm\leq\,$}}
         child {node[leaf] {Virginica} edge from parent node[right] {$\,>5.95cm$}}
         edge from parent node[right] {$\,>1.65cm$}}
     edge from parent node[right] {$\,>0.7cm$}}
 ;
\end{tikzpicture}
\caption{A decision tree classifying iris flowers into its Setosa, Versicolor or
Virginica varieties according to the width and length of its petals and sepals.}
\label{fig:dt-iris}
\end{figure}
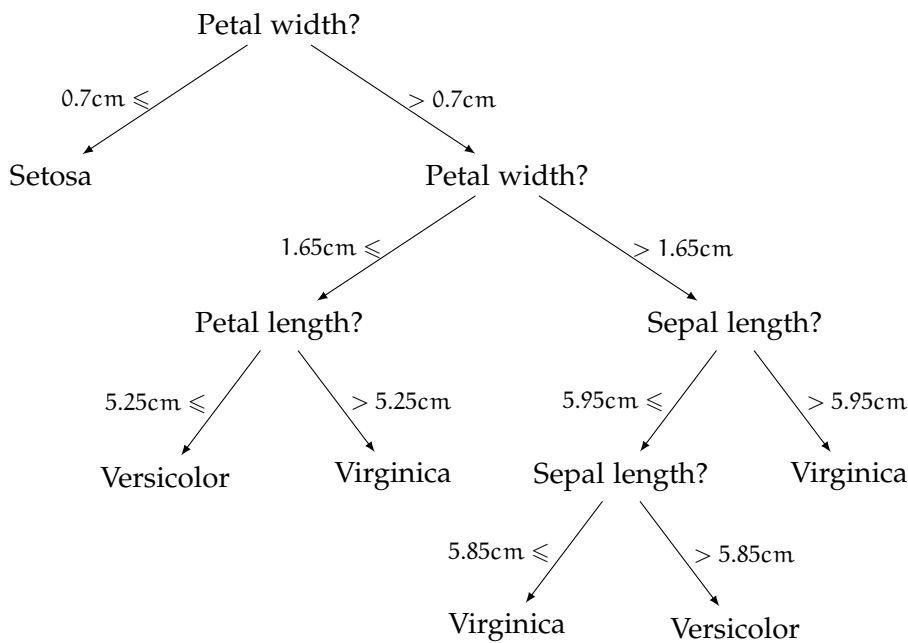

A classification or a regression tree~\cite{breiman1984classification} is built
using all the input-output pairs $((x^i, y^i) \in (\mathcal{X} \times
\mathcal{Y}))_{i=1}^n$ as follows: for each test node, the best split $(S_{r},
S_{l})$ of the local subsample $S$ reaching the node is chosen among the $p$
input features combined with the selection of an optimal cut point. The best
sample split $(S_{r}, S_{l})$ of $S$ minimizes the average reduction of impurity
\begin{align}
&\Delta I((y^i)_{i \in S}, (y^i)_{i \in S_l}, (y^i)_{i \in S_r}) \nonumber \\
&= I((y^i)_{i \in S})
-\frac{|S_l|}{|S|} I((y^i)_{i \in S_l})
-\frac{|S_r|}{|S|} I((y^i)_{i \in S_r}),
\label{eq:impurity_reduction}
\end{align}
\noindent where $I$ is the impurity of the output such as the entropy in
classification or the variance in regression. The decision tree growth continues
until we reach a stopping criterion such as no impurity $I((y^i)_{i \in
S})=0$.

To avoid over-fitting, we can stop earlier the tree growth by adding further
stopping criteria such as a maximal depth or a minimal number of samples to
split a node.

Instead of a single decision tree, we often train an ensemble of such models:
\begin{itemize}

\item Averaging-based ensemble methods grow an ensemble by randomizing the tree
growth. The random forest method~\cite{breiman2001random} trains
decision trees on bootstrap copies of the training set, i.e. by sampling
with replacement from the training dataset, and it randomizes the best split selection
by searching this split among $k$ out of the $p$ features at each nodes ($k \leq
p$).

\item Boosting-based methods~\cite{freund1997decision,friedman2001greedy} build
iteratively a sequence of weak models such as shallow trees which perform only
slightly better than random guessing. Each new model refines the prediction of
the ensemble by focusing on the wrongly predicted training input-output pairs.

\end{itemize}

We further discuss decision tree models in Chapter~\ref{ch:tree} and ensemble
methods in Chapter~\ref{ch:bias-var-ensemble}.

\subsection{From single to multiple output models}
\label{sec:multioutput-strategies}

With multiple outputs supervised learning tasks, we have to infer the values of
a set of $d$ output variables $y_1,\ldots,y_d$ (instead of a single one) from a
set of $p$ input variables $x_1,\ldots,x_p$. We hope to improve the accuracy and /
or computational performance by exploiting the correlation structure between the
outputs. There exist two main approaches to solve multiple output tasks: problem
transformation presented in Section~\ref{sec:mo-problem-transf} and algorithm
adaptation in Section~\ref{sec:mo-algo-adaptation}. We present here a non
exhaustive selection of both approaches. The interested reader will find a
broader review of the multi-label literature
in~\cite{zhang2014review,tsoumakas2009mining,madjarov2012extensive,gibaja2014multi} and of the
multi-output regression literature
in~\cite{spyromitros2012multi,borchani2015survey}.

\subsubsection{Problem transformation}
\label{sec:mo-problem-transf}

The problem transformation approach transforms the original multi-output task
into a set of single output tasks. Each of these single output tasks is then solved
by classical classifiers or regressors. The possible output correlations
are exploited through a careful reformulation of the original task.

\paragraph{Independent estimators} The simplest way to handle multi-output
learning is to treat all outputs in an independent way. We break the prediction of the $d$ outputs
into $d$ independent single output prediction tasks. A model is fitted on each output. At
prediction time, we concatenate the predictions of these $d$ models. This is called
the binary relevance method~\cite{tsoumakas2009mining} in multi-label
classification and the single target method~\cite{spyromitros2012multi} in
multi-output regression. Since we consider the outputs independently, we neglect
the output correlation structure. Some methods may however benefit from sharing identical computations needed for  the different outputs. For instance, the
$k$-nearest neighbor method can share the search for the $k$-nearest neighbors
in the input space, and the ordinary linear least squares  method can share the
computation of $(\mathbf{X}^T\mathbf{X})^{-1}\mathbf{X}^T$ in
Equation~\ref{eq:ols-solution}.

\paragraph{Estimator chain} If the outputs are dependent, the model of a single
output might benefit from the values of the correlated outputs. In the estimator
chain method, we sequentially learn a model for each output by providing the predictions of the
previously learnt models as auxiliary inputs. This is called a classifier
chain~\cite{read2011classifier} in classification and a regressor
chain~\cite{spyromitros2012multi} in regression.

More precisely, the estimator chain method first generates an order $o$ on the
outputs for instance based on prior knowledge, the output density, the output variance or at
random. Then with the training samples and the output order $o$, it sequentially
learns $d$ estimators: the $l$-th estimator $f_{o_l}$ aims at predicting the
$o_l$-th output using as inputs the concatenation of the input vectors with the predictions of the models learnt
for the $l-1$ previous outputs. To
reduce the model variance, we can generate an ensemble of estimator chains by
randomizing the chain order (and / or the underlying base estimator), and then
we average their predictions.

In multi-label classification, \citet{cheng2010bayes} formulates a Bayes optimal
classifier chain by modeling the conditional probability of
$P_{\mathcal{Y}|\mathcal{X}}(y|x)$. Under the chain rule, we have
\begin{equation}
P_{\mathcal{Y}|\mathcal{X}}(y|x) = P_{\mathcal{Y}_1|\mathcal{X}}(y_1|x) \prod_{j=2}^d
P_{\mathcal{Y}_j|\mathcal{X},\mathcal{Y}_1,\ldots,\mathcal{Y}_{j-1}}(y_j|x,y_1,\ldots,y_{j-1}).
\label{eq:chain-rule-cl-chain}
\end{equation}
\noindent Each estimator of the chain approximates a probability factor of the
chain rule decomposition. Using the estimation of $P_{\mathcal{Y}|\mathcal{X}}$
made by the chain and a given loss function $\ell$, we can perform Bayes optimal
prediction:
\begin{equation}
h^*(x) = \arg\min_{y'} E_{\mathcal{Y}|\mathcal{X}} \ell(y', y).
\end{equation}

\paragraph{Error correcting codes} Error correcting codes are techniques from
information and coding theory used to properly deliver a message through a noisy
channel. It first codes the original message, and then corrects the errors made
during the transmission at decoding time. This idea have been applied to
multi-class classification~\cite{dietterich1991error,guruswami1999multiclass},
multi-label
classification~\cite{ferng2011multi,zhang2011multi,kajdanowicz2012multi,kouzani2009multilabel,guo2008error,hsu2009multi,kapoor2012multilabel,cisse2013robust}
and multi-output regression~\cite{tsoumakas2014multi,yu2006multi} tasks by
viewing the predictions made by the supervised learning model(s) as a message
transmitted through a noisy channel. It transforms the original task by encoding
the output values with a binary error correcting code or output projections. One
classifier is then fitted for each bit of the code or output projection. At
prediction time, we concatenate the predictions made by each estimator and
decode them by solving the inverse problem. Note that the output coding might
also have for objective to reduce the dimensionality of the output
space~\cite{hsu2009multi,kapoor2012multilabel}.

\paragraph{Pairwise comparisons} In multi-label tasks, the ranking by pairwise
comparison approach~\cite{hullermeier2008label} aims to generate a ranking of
the labels by making all the pairwise label comparisons. The original tasks is
transformed into $d (d - 1) / 2$ binary classification tasks where we compare if
a given label is more likely to appear than another label. The datasets
comparing each label pair is obtained by collecting all the samples where only
one of the outputs is true, but not both. This approach is similar to the
one-versus-one approach~\cite{park2007efficient} in multi-class classification
task, however we can not directly transform the ranking into a prediction, i.e.
label set. To decrease the prediction time, alternative ranking construction
schemes have been proposed~\cite{mencia2008efficient,mencia2010efficient}
requiring less than $d (d-1)/2$ classifier predictions.

The Calibrated label ranking
method~\cite{brinker2006unified,furnkranz2008multilabel} extends the previous
approach by adding a virtual label which will serve as a split point between the
true and the false labels. For each label, we add a new tasks using all the
samples comparing the label $i$ to the virtual label whose value is the opposite
of the label $i$. To the $d (d-1)/2$ tasks, we effectively add $d$ tasks.

\paragraph{Label power set} For multi-label classification tasks, the label
power set method~\cite{tsoumakas2009mining} encodes each label set in the
training set as a class. It transforms the original task into a multi-class
classification task. At prediction time, the class predicted by the multi-class
classifier is decoded thanks to the one-to-one mapping of the label power set
encoding. The drawback of this approach is to generate a large number of classes
due to the large number of possible label sets. For $n$ samples and $d$ labels,
the maximal number of classes is $\max(2^d, n)$. This leads to accuracy issues
if some label sets are not well represented in the training set. To alleviate
the explosion of classes, rakel~\cite{tsoumakas2007random} generates an ensemble
of multi-class classifiers by subsampling the output space and then applying the
label power set transformation.

\subsubsection{Algorithm adaptation}
\label{sec:mo-algo-adaptation}

The algorithm adaptation approach modifies existing supervised learning
algorithms to handle multiple output tasks. We show here how to extend the
previously presented models classes to multi-output regression and to
multi-label classification tasks.

\paragraph{Linear-based models} Linear-based models have been adapted to
multi-output tasks by reformulated their mathematical formulation using
multi-output losses and (possibly) regularization constraints enforcing
assumptions on the input-output and the output-output correlation structures.
The proposed methods are based for instance on extending least-square
regression~\cite{dayal1997multi,breiman1997predicting,simila2007input,baldassarre2012multi,evgeniou2005learning,zhou2012multi}
(with possibly regularization), canonical correlation
analysis~\cite{izenman1975reduced,van1980multivariate}, support vector
machine~\cite{elisseeff2001kernel,jiang2008calibrated,xu2012efficient,evgeniou2004regularized,evgeniou2005learning},
support vector
regression~\cite{vazquez2003multi,sanchez2004svm,liu2009multi,xu2013multi},
and conditional random
fields~\cite{ghamrawi2005collective}.

\paragraph{(Deep) Artificial neural networks} Neural networks handles
multi-output tasks by having one node on the output layer per output variable.
The network minimizes a global error function defined over all the
outputs~\cite{specht1991general,zhang2006multilabel,ciarelli2009multi,zhang2009ml,nam2014large}.
The output correlation are taken into account by sharing the input and the
hidden layers between all the outputs.

\paragraph{Nearest neighbors} The $k$-nearest neighbors algorithm predicts an
unseen sample $x$ by aggregating the output value of the $k$ nearest neighbors
of $x$. This algorithm is adapted to multi-output tasks by sharing the nearest
neighbors search among all outputs. If we just share the search, this is called
binary relevance of $k$-nearest neighbors in classification and single target of
$k$-nearest neighbors in regression. Multi-output extensions of the $k$-nearest
neighbors modifies how the output values of the nearest neighbors are aggregated
for the predictions for instance it can utilize the maximum a posteriori
principle~\cite{zhang2007ml,younes2011dependent,cheng2009simple} or it can
re-interpret the output aggregation as a ranking
problem~\cite{chiang2012ranking,brinker2007case},

\paragraph{Decision trees} The decision tree model is a hierarchical structure
partitioning the input space and associating a prediction to each partition. The
growth of the tree structure is done by maximizing the reduction of an impurity
measure computed in the output space. When the tree growth is stopped at a leaf,
we associate a prediction to this final partition by aggregating the output
values of the training samples. We adapt the decision tree algorithm to
multi-output tasks in two
steps~\cite{segal1992tree,de2002multivariate,blockeel2000top,clare2001knowledge,zhang1998classification,vens2008decision,noh2004unbiased}:
(i) multi-output impurity measures are used to grow the structure as the sum
over the output space of the entropy or the variance; (ii) the leaf predictions
are obtained by computing a constant minimizing a multi-output loss function
such as the $\ell_2$-norm loss in regression or the Hamming loss in
classification. We discuss in more details how to adapt the decision tree
algorithm to multi-output tasks in Section~\ref{sec:mo-trees}.

Instead of growing a single decision tree, they are often combined together to
improve their generalization performance. Random forest
models~\cite{breiman2001random,geurts2006extremely} averages the predictions of
several randomized decision trees and has been studied in the context of
multi-output
learning~\cite{kocev2007ensembles,segal2011multivariate,kocev2013tree,madjarov2012extensive,joly2014random}.

\paragraph{Ensembles} Ensemble methods aggregate the predictions of multiple
models into a single one so to improve its generalization performance. We
discuss how the averaging and boosting approaches have been adapted to
multi-output supervised learning tasks.

Averaging ensemble methods have been straightforwardly adapted by averaging the
prediction of multi-output models. Instead of averaging scalar predictions, it
averages~\cite{kocev2007ensembles,segal2011multivariate,kocev2013tree,madjarov2012extensive,joly2014random}
the vector predictions of each model of the ensemble. If the learning algorithm
is not inherently multi-output, we could use one the problem transformation
techniques as in rakel~\cite{tsoumakas2007random}, which uses the label power
set transformation, or ensemble of estimator chain~\cite{read2011classifier}.

Boosting ensembles are grown by sequentially adding weak models minimizing a
selected loss, such as the Hamming loss~\cite{schapire2000boostexter}, the
ranking loss~\cite{schapire2000boostexter}, the $\ell_2$-norm
loss~\cite{geurts2006kernelizing} or any differentiable loss function~(see
Chapter~\ref{ch:gbrt-output-projection}).

\section{Evaluation of model prediction performance}
\label{sec:model-eval}

For a given supervised learning model $f$ trained on a set of samples
$\left((x^i, y^i) \in \left(\mathcal{X} \times
\mathcal{Y}\right)\right)_{i=1}^n$, we want a model having good generalization
able to predict unseen samples. Otherwise said, the model $f$ should
have minimal generalization error over the input-output pair distribution,
where the generalization error is defined as:
\begin{equation}
\text{Generalization error} = E_{P_{{\cal X},{\cal Y}}} \{ \ell(f(x), y)\}
\label{eq:model-eval}
\end{equation}
\noindent for a given loss function $\ell : \mathcal{Y} \times \mathcal{Y}
\rightarrow \mathbb{R}^+$.

Evaluating Equation~\ref{eq:model-eval} is generally unfeasible, except in the
rare cases where (i) the input-output distribution $P_{{\cal X},{\cal Y}}$ is
fully known and (ii) for restricted classes of models. In practice, neither of
these conditions are met. However, we still need a principle way to approximate
the generalization error.

A first approach to approximate Equation~\ref{eq:model-eval} is to evaluate the
error of the model $f$ on the training samples $\mathcal{L} = \left((x^i, y^i) \in
\left(\mathcal{X} \times \mathcal{Y}\right)\right)_{i=1}^n$ leading to the
resubstitution error:
\begin{equation}
\text{Resubstitution error} = \frac{1}{|\mathcal{L}|} \sum_{(x, y) \in \mathcal{L}}^n  \ell(f(x), y)
\end{equation}

A model with a high resubstitution error often has a high generalization error
and indeed underfits the data. The linear model shown in
Figure~\ref{subfig:under-fit} underfits the data as it is not complex enough to
fit the non linear data (here a second degree polynomial). Instead, we can fit a
high order polynomial model to have a zero resubstitution error as illustrated
in Figure~\ref{subfig:over-fit}. This complex model has poor generalization
error as it perfectly fits the noisy samples unable to retrieve the second order
parabola. Such overly complex models with zero resubstitution error and non zero
generalization error are said to overfit the data. Since a zero resubstitution
error does not imply a low generalization error, it is a poor proxy of the
generalization error.

\begin{figure}
\centering
\subfloat[Under-fitting]{{\includegraphics[width=0.5\textwidth]{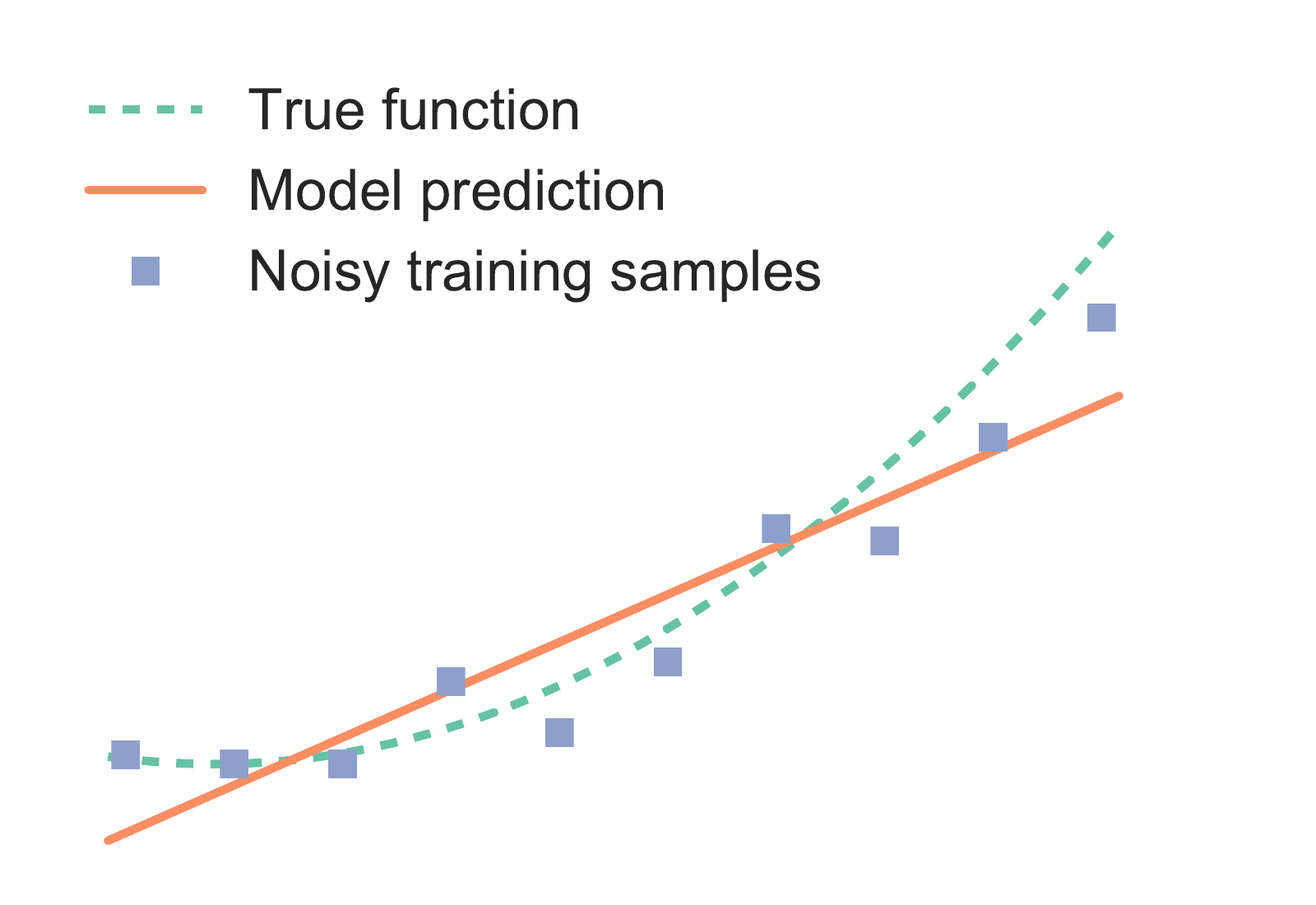}
\label{subfig:under-fit}}}
\subfloat[Over-fitting]{{\includegraphics[width=0.5\textwidth]{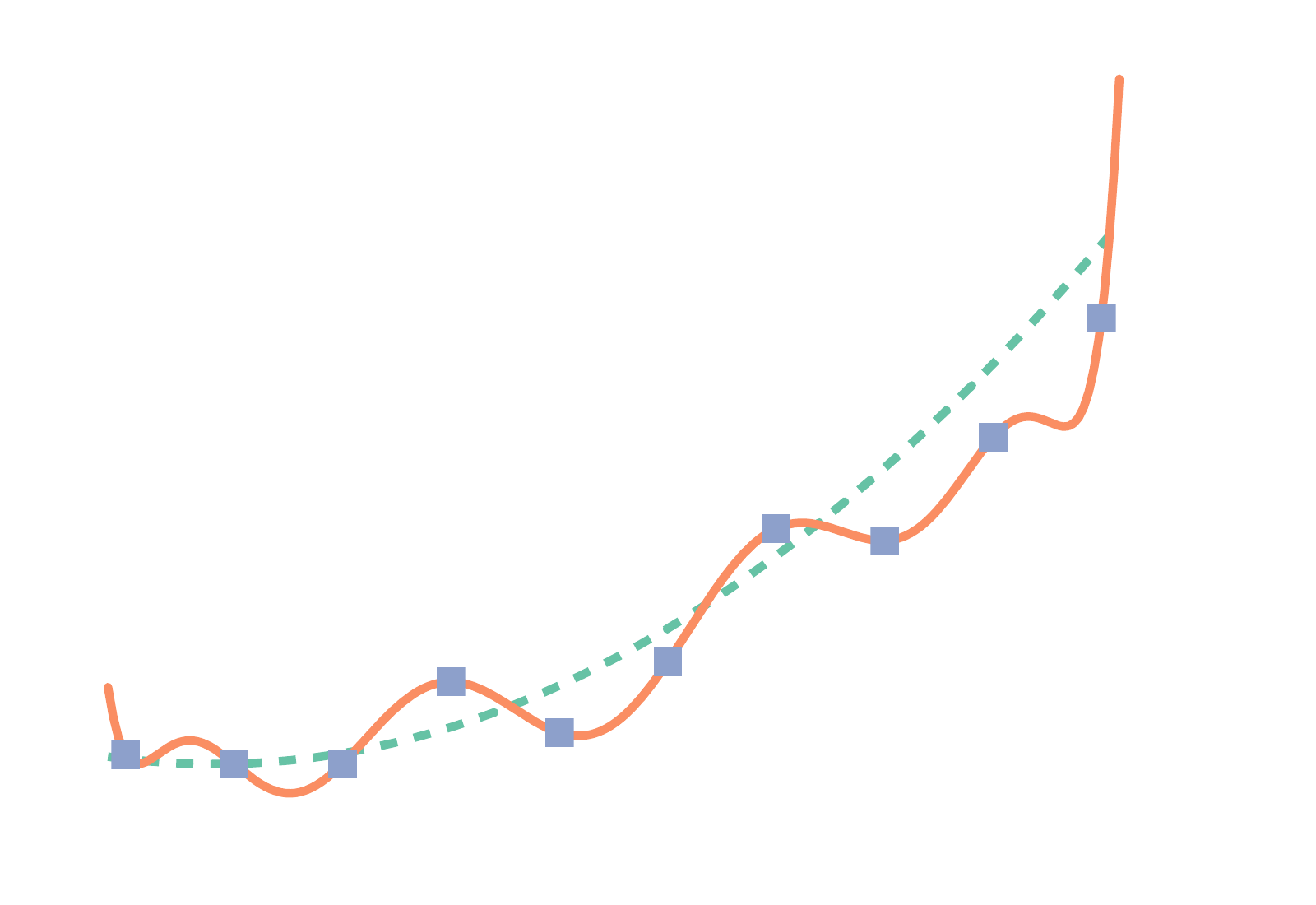}
\label{subfig:over-fit}}}
\caption{The linear model on the left figure underfits the training samples,
while the high order polynomial model on the right overfits the training
samples.}
\label{fig:under-over-fitting}
\end{figure}

Since we assess the quality of the model with the training samples, the
resubstitution error is optimistic and biased. Furthermore, it favors overly
complex models (as depicted in Figure \ref{fig:under-over-fitting}). To improve
the approximation of the generalization error, we need to use techniques which
avoid to use the training samples for performance evaluation. They are either
based on sample partitioning methods, such as hold out methods and cross
validation techniques, or sample resampling methods, such as bootstrap
estimation methods. Since the amount of available data and time are fixed for
both the model training and the model assessment, there is a trade-off between
(i) the quality of the error estimate, (ii) the number of samples available to
learn a model and (iii) the amount of computing time available for the whole
process.

The hold out evaluation method splits the samples into a training set $LS$,
also called learning set, and a testing set $TS$ commonly with a ratio of $2/3$ -
$1/3$. The hold out error is given by
\begin{equation}
\text{Hold out error} = \frac{1}{|TS|} \sum_{(x, y) \in TS}  \ell(f(x), y).
\end{equation}
\noindent This methods requires a high number of samples as a large part of the
data is devoted to the model assessment impeding the model training. If too few
samples are allocated to the testing set, the hold out estimate becomes
unreliable as its confidence intervals widen~\cite{kohavi1995study}. Since the
hold out error is a random number depending on the sample partition, we can
improve the error estimation by (i) generating $B$ random partitions $(LS_b,
TS_b)_{b=1}^B$ of the available samples, (ii) fitting a model $f^b$ on each
learning set $LS_b$ and (iii) averaging the performance of the $B$ models over
their respective testing sets $TS_b$:
\begin{equation}
\text{Random subsampling error} =
\frac{1}{B} \sum_{b=1}^B \frac{1}{|TS_b|} \sum_{(x, y) \in TS_b} \ell(f^b(x), y).
\end{equation}

To improve the data usage efficiency, we can resort to cross-validation methods,
also called rotation estimation, which split the samples into $k$ folds $\{TS_1,
\ldots, TS_k\}$  approximately of the same size. Cross validation methods
average the performance of $k$ models $(f^l)_{l=1}^k$ each tested on one of the
$k$ folds and trained using the $k-1$ remaining folds:
\begin{equation}
\text{CV error} =
\frac{1}{k} \sum_{l=1}^k \frac{1}{|TS_l|} \sum_{(x, y) \in TS_l}  \ell(f^l(x), y).
\end{equation}
\noindent The number of folds $k$ is usually $5$ or $10$. If $k$ is equal to the
number of samples ($k=n$), it is called leave-one-out cross validation.

Given that the folds do not overlap for cross validation methods, we are
tempted to assess the performance over the pooled cross validation estimates
with a given $metric$ obtained by concatenating the predictions made
by each model $f^l$ over each of the $k$-folds
\begin{equation}
\text{Pooled CV error}\hspace*{-1mm}=\hspace*{-1mm}
metric\hspace*{-1mm}\left(\hspace*{-0.5mm}(f^1\hspace*{-1mm}(x), y)_{\hspace*{-0.5mm}(x, y) \in TS_{1}}\hspace*{-1mm}\frown\hspace*{-1mm}\ldots\hspace*{-1mm}\frown\hspace*{-1mm}(f^k(x) ,y)_{\hspace*{-0.5mm}(x, y) \in TS_k}\right)\hspace*{-1mm},
\end{equation}
\noindent where $\frown$ is the concatenation operator. There is no difference
for sample-wise losses such as the square loss. However, this is not the case
for metrics comparing a whole set of predictions to their ground truth.
Depending on the metrics, it has been showed that pooling may or may not biase
the error
estimation~\cite{parker2007stratification,forman2010apples,airola2011experimental}.

We can improve the quality of the estimate by repeating the cross validation
procedures over $B$ different $k$-fold partitions, averaging the performance of
the models $f^{b,l}$ over each associated testing set $TS_{b,l}$:
\begin{equation}
\text{Repeated CV error} =
\frac{1}{Bk} \sum_{b=1}^B \sum_{l=1}^k \frac{1}{|TS_{b,l}|} \sum_{(x, y) \in TS_{b,l}}
\ell(f^{b,l}(x), y).
\end{equation}
\noindent If all combinations are tested exhaustively as in the leave-one-out
case, it is called complete cross validation. Since it is often too
expensive~\cite{kohavi1995study}, we can instead draw several sets of
folds at random.

The bootstrap method~\cite{efron1983estimating} draws $B$ bootstrap datasets
$\{B_1, \ldots, B_B\}$ by sampling with replacement $n$ samples from the
original dataset of size $n$. Each samples has a probability of $1 - (1 -
\frac{1}{n})^n$ to be selected in a bootstrap which is approximately $0.632$ for
large $n$. A first approach to estimate the error is to train a model $f^b$ on
each bootstrap dataset and use the original dataset as a testing set:
\begin{equation}
\text{Bootstrap error} =
\frac{1}{nB} \sum_{b=1}^B \sum_{(x, y) \in B_b} \ell(f^{b}(x), y).
\end{equation}
\noindent This leads to over optimistic results, given the overlap between
the training and the test data.

A better approach (discussed in Chapter~7.11 of~\cite{friedman2001elements}) is
to imitate cross validation methods by fitting on each bootstrap dataset a
model $f^{b}$ and using the unused samples as a testing set. This
approach is called bootstrap leave-one-out:
\begin{equation}
\text{LOO Bootstrap error} =
\frac{1}{n} \sum_{i=1}^n \frac{1}{|C^{-i}|} \sum_{b \in C^{-i}} \ell(f^{b}(x^i), y^i),
\end{equation}
\noindent where $C^{-i}$ gives the bootstrap indices where the sample $i$ was
not drawn. It is similar to a 2-fold repeated cross validation or random
subsampling error with a ratio of 2/3 - 1/3 for the training and testing set.
The estimation is thus biased as it uses approximately $0.632n$ training samples
instead of $n$. We can alleviate this bias due to the sampling
procedure through the ``0.632'' estimator which averages the training error and
LOO Bootstrap error:
\begin{align}
\text{0.632 estimator} =& 0.632 \times \text{LOO Bootstrap error} \nonumber \\
& + 0.368 \times \text{Resubstitution error}.
\end{align}
\noindent Note that with very low sample size, it has been
shown~\cite{braga2004cross} that the bootstrap approach yields better error
estimate than the cross validation approach.

Until now, we have assumed that the samples are independent and identically
distributed. Whenever this is no longer true, such as with time series of
measurements, we have to modify the assessment procedure to avoid biasing the
error estimation. For instance, the hold out estimate would train the model on
the oldest samples and test the model on the more recent samples. Similarly in
the medical context if we have several samples for one patient, we should keep
these samples together in the training set or in the testing set.

Partition-based methods (hold out, cross validation) break the assumption in
classification that the samples from the training set are independent from the
samples in the testing set as they are drawn without replacement from a pool of
samples. The representation of each class in the testing set is thus not
guaranteed to be the same as in the training set. It is
advised~\cite{kohavi1995study} to perform stratified splits by keeping the same
proportion of classes in each set.

\section{Criteria to assess model performance}
\label{sec:model-diagnosis}

Assessing the performance of a model requires evaluation metrics which will
compare the ground truth to a prediction, a score or a probability estimate. The
selection of an appropriate scoring or error measure is essential and is
dependent of the supervised learning task and the goal behind the modeling.

A first approach to assess a model is to define a goal for the model
and to quantify its realization. For instance, a company wants to maximize its
benefits and consider that the revenue must exceed the data analysis cost of
gathering samples, fitting a model and exploiting its predictions.
Unfortunately, this model optimization criterion is hardly expressible into
economical terms. We could instead consider the effectiveness of the model such
as the click-through-rate, used by online advertising companies, which counts
the number of clicks on a link to the number of opportunities that users have to
click on this link. However, it is hard to formulate a model optimizing directly
this score and it requires to put the model into a production setting (or at
least simulate its behavior). Other optimization criteria exit that are more
amenable to mathematical analysis and numerical computation such as the square
loss or the logistic loss. Knowing the properties of such criteria is necessary
to make a proper choice.

We present binary classification metrics in
Section~\ref{subsec:bin-clf-metrics}. Then, we show how to extend these metrics to
multi-class classification tasks in Section~\ref{subsec:multiclass-metrics} and
to multi-label classification tasks in Section~\ref{subsec:multilabel-metrics}.
We introduce metrics for regression tasks and multi-output regression tasks in
Section~\ref{subsec:regression-metrics}.

More details or alternative descriptions of these metrics can be found in the
following
references~\cite{sokolova2009systematic,hossin2015review,ferri2009experimental}.
Note that I made significant contributions to the implementations and the
documentations of these metrics in the scikit-learn
library~\cite{pedregosa2011scikit,buitinck2013api}.

\subsection{Metrics for binary classification }
\label{subsec:bin-clf-metrics}

Given a set of $n$ ground truth values $(y^i \in \{0, 1\})_{i=1}^n$ and their
associated model predictions $(\hat{y}^i \in \{0, 1\})_{i=1}^n$, we can
distinguish in binary classification four categories of predictions (as shown in
Table~\ref{tab:binary-label-kind}). We denote by true positives (TP) and true
negatives (TN) the predictions where the model accurately predicts the target
respectively as true or false:
\begin{align}
TP &= \sum_{i=1}^n 1(y^i=1;\hat{y}^i=1),\\
TN &= \sum_{i=1}^n 1(y^i=0;\hat{y}^i=0).
\end{align}

Whenever the model wrongly predicts the samples, we call
false positives (FP) samples predicted as true while their labels are false and
false negatives (FN) samples predicted as false while their labels are true:
\begin{align}
FN &= \sum_{i=1}^n 1(y^i=1;\hat{y}^i=0), \\
FP &= \sum_{i=1}^n 1(y^i=0;\hat{y}^i=1).
\end{align}

Together, the true positive, true negatives, false negatives and false positives
form the so called confusion or contingency matrix shown in
Table~\ref{tab:binary-label-kind}.

\begin{table}[t]
\caption{For a binary classification task, the prediction of a model is divided
into fours categories leading to a confusion matrix.}
\label{tab:binary-label-kind}
\centering
\begin{tabular}{@{}lcc@{}}
                        & \em{Truly positive} & \em{Truly negative} \\
\em{Predicted positive} & True positive  & False positive \\
\em{Predicted negative} & False negative & True negatives \\
\end{tabular}
\end{table}

Two common metrics to assess classification performance are the error rate, the
average of the $0-1$ loss, and its complement the accuracy:
\begin{align}
\text{Error rate} &= \frac{1}{n} \sum_{i=1}^n 1(y^i\not=\hat{y}^i), \\
\text{Accuracy} &= 1 - \text{Error rate} = \frac{1}{n} \sum_{i=1}^n 1(y^i=\hat{y}^i).
\end{align}

Both metrics can be expressed in term of the confusion matrix:
\begin{align}
\text{Error rate} &= \frac{FP + FN}{n}, \\
\text{Accuracy} &= \frac{TN + TP}{n}.
\end{align}
\noindent The error rate does not distinguish the false negatives from the false
positives. Similarly, the accuracy does not differentiate true positives from
true negatives. Thus, two classifiers may have exactly the same accuracy or
error rate, while leading to a totally different outcome by increasing either
the number of misses (false negatives) or the number of false alarms (false
positives). Furthermore, the error rate and the accuracy can be overly
optimistic whenever there is a high class imbalance. A classification task with
$99.99\%$ of samples in one of the classes would easily lead to an accuracy of
$99.99\%$ (and an error rate of $0.01\%$) by alway predicting the most common
class. The choice of an appropriate metric thus depends on the properties of the
classification task, such as the class imbalance.

To differentiate false positives from false negative, we can assess separately
the proportion of correctly classified positive and negative samples. This leads
to the \emph{true positive rate} (resp. \emph{true negative rate}) which
computes the proportion of correctly classified positive (resp. negative)
samples:
\begin{align}
\text{True positive rate} &= \frac{TP}{TP + FN} \\
\text{True negative rate} &= \frac{TN}{TN + FP}
\end{align}
\noindent The complement of the true positive rate (resp. true negative rate) is
the false negative rate (resp. false positive rate):
\begin{align}
\text{False negative rate} &= 1 - \text{True positive rate} = \frac{FN}{TP + FN},\\
\text{False positive rate} &= 1 - \text{True negative rate} = \frac{FP}{TN + FP}.
\end{align}

The true positive rate is also called \emph{sensitivy} and tests the ability of
the classifier to correctly classify all positive samples as true. A test with
$100\%$ sensitivity implies that all positive samples are correctly classified.
However, this does not imply that all samples are correctly classified.  A
classifier predicting all samples as true leads to $100\%$ sensitivity and
totally neglects false positives. We have to look to the true negative rate,
also called \emph{specifity}, which tests the ability of the classifier to
correctly classify all negative samples as negative. A perfect classifier should
thus have a high sensitivity and a high specifity. In the medical domain, the
sensitivity and the specificity are often used to characterize and to choose the
behavior of diagnosis tests such as pregnancy tests.

The average of the specifity and sensitivity is called the \emph{balanced
accuracy}:
\begin{align}
\text{Balanced accuracy}
&= \frac{\text{True positive rate} + \text{True negative rate}}{2} \\
&= \frac{\text{specifity} + \text{sensitivty}}{2} \\
&= \frac{1}{2} \frac{TP}{TP + FN} + \frac{1}{2} \frac{TN}{TN + FP}.
\end{align}

In the information retrieval context, a user sets a query to an information
system, e.g. a web search engine, to detect which documents are relevant among a
collection of such documents. In such systems, the collection of documents is
often extremely large with only a few relevant documents to a given query. Due
to the small proportion of relevant documents, we want to maximize the
\emph{precision}, the fraction of correctly predicted documents among the
predicted documents. Binary classification tasks with a high class imbalance can
be viewed as an information retrieval problems. In the context of binary
classification tasks, the precision is expressed as
\begin{equation}
\text{Precision} = \frac{TP}{TP + FP},
\end{equation}

To have a perfect precision, one could predict all documents or samples as
negative (as irrelevant documents). In parallel, we want also to maximize the
recall, the proportion of correctly predicted true samples among the true
samples. The recall is a synonym for true positive rate and sensitivity.

The precision and recall are often combined into a single number by computing the
$F_1$ score, the harmonic mean of the precision and recall,
\begin{equation}
F_1 = \frac{2}{\frac{1}{\text{Precision}} + \frac{1}{\text{Recall}}}.
\end{equation}

Some classifiers associate a score or a probability $\hat{f}(x)$ to a sample
instead of a class label. We can threshold these continuous predictions by a
constant $\tau$ to compute the number of true positives, false positives, false
negatives and true negatives:
\begin{align}
TP(\tau) &= \sum_{i=1}^n 1(y^i=1;f(x^i) \geq \tau),\\
TN(\tau) &= \sum_{i=1}^n 1(y^i=0;f(x^i) \leq \tau), \\
FN(\tau) &= \sum_{i=1}^n 1(y^i=0;f(x^i) \geq \tau), \\
FP(\tau) &= \sum_{i=1}^n 1(y^i=1;f(x^i) \leq \tau).
\end{align}

By varying $\tau$, we can first derive performance curves to analyze the
prediction performance of those more models and then select a classifier
performance point with pre-determined classification performance.

The receiver operating characteristic (ROC) curve~\cite{fawcett2006introduction}
plots the true positive rate as a function of the false positive rate by varying
the threshold $\tau$ as shown in Figure~\ref{subfig:roc-curve}. The receiver,
the model user, can indeed choose any point on the curve to operate at a given
model specifity / sensitivity tradeoff. A random estimator has its performance
on the line $((0, 0), (1, 1))$, while a perfect classifier has the points $(0,
1)$ with $0\%$ of false positive rate and $100\%$ of true positive rate on its
curve. Any curve below the random line can be reversed symmetrically to the line
$((0, 0), (1, 1))$ by flipping the classifier prediction. The ROC curve is often
used in the clinical domain~\cite{metz1978basic} and coupled to a cost analysis
to determine the proper threshold $\tau$. The area under the ROC curve can be
interpreted as~\cite{hanley1982meaning} the probability to rank with a higher
score one true sample than one false sample chosen at random.

The precision-recall (PR) curve is the precision as a function of the recall as
shown in Figure~\ref{subfig:pr-curve}. The ROC curve and the PR curves are
linked as there is a one to one mapping between points in the ROC space and in
the precision-recall space~\cite{davis2006relationship}. However conversely to
the ROC curve, the precision recall curve is sensitive to the class imbalance
between the positive and negative classes. Since both the precision and recall
do not take into account the amount of true negatives, the precision-recall
curve (compared to the ROC curve) focuses on how well the estimator is able to
classify correctly the positive class.

\begin{figure}
\centering
\subfloat[ROC curve]{{\includegraphics[width=0.5\textwidth]{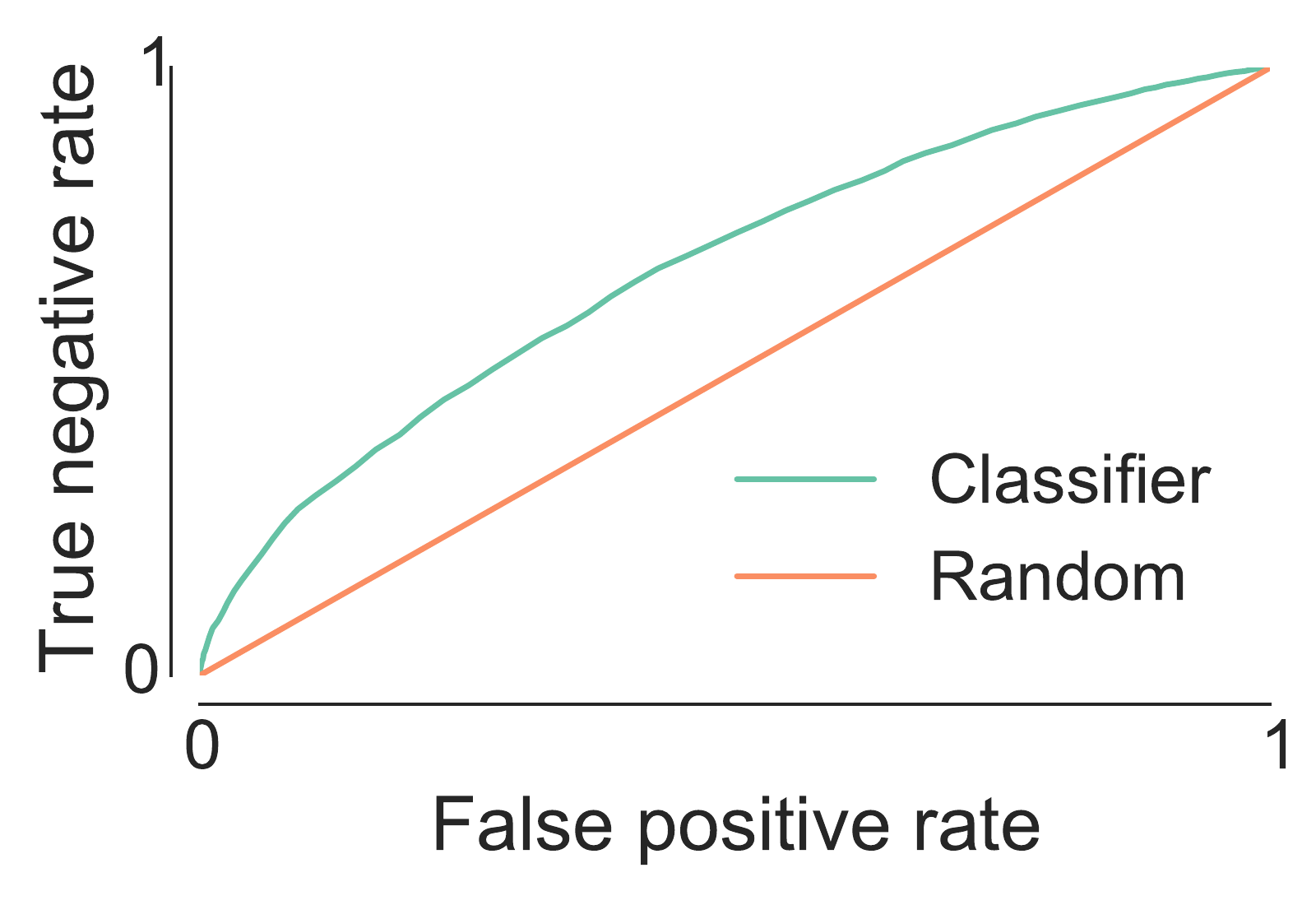}
\label{subfig:roc-curve}}}
\subfloat[PR curve]{{\includegraphics[width=0.5\textwidth]{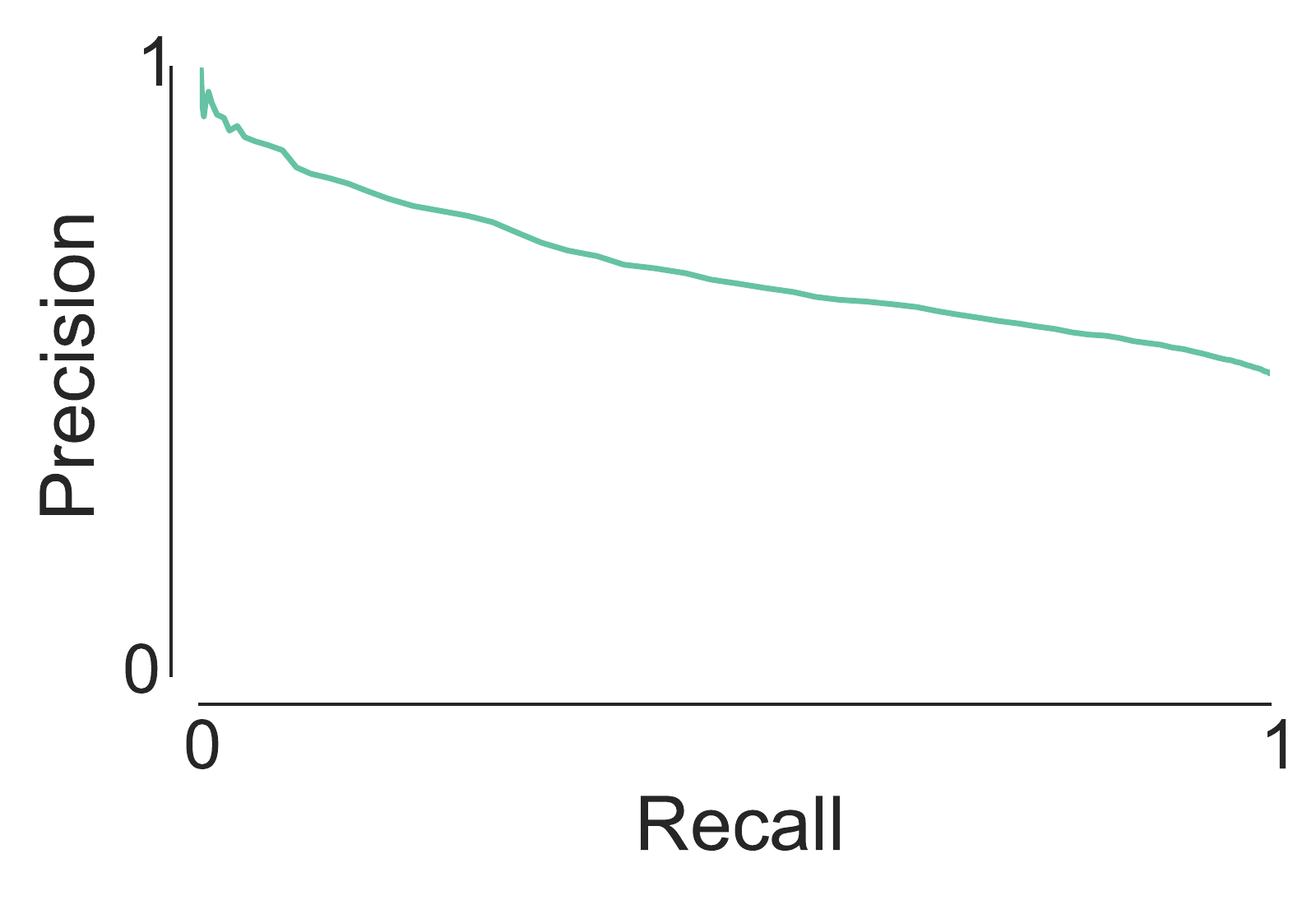}
\label{subfig:pr-curve}}}
\caption{A receiver operating characteristic curve and a precision-recall curve
of a classifier and a random model.}
\label{fig:metric-curves}
\end{figure}

\subsection{Metrics for multi-class classification}
\label{subsec:multiclass-metrics}

From binary classification to multi-class classification, the output value is
no more restricted to two classes and can go up to $k$-classes.
Given the ground truths $(y^i \in \{1, ..., k\})_{i=1}^n$ and the
associated model predictions $(\hat{y}^i \in \{1, ..., k\})_{i=1}^n$, we
can now divide the model predictions into $k^2$ categories leading
to a $k \times k$ confusion matrix:
\begin{equation}
c_{l_1,l_2} = \sum_{i=1}^n 1(y^i=l_1;\hat{y}^i=l_2) \quad \forall l_1, l_2 \in \{1, ..., k\}.
\end{equation}

Metrics such as the accuracy, the error rate or the log loss (see
Table~\ref{tab:common-loss}) naturally extend to multi-class classification
tasks. To extend other binary classification metrics (such as those developed in
Section~\ref{subsec:bin-clf-metrics}), we need to break the $k \times k$
confusion matrix into a set of $2 \times 2$ confusion matrices.

A first approach is to consider that each class $l$ is in turns the positive
class while the remaining labels form together the negative class. We thus have
$k$ confusion matrices whose true positives $TP_l$, true negatives $TN_l$, false
negatives $FN_l$ and false positives $FP_l$ for the class $l \in \{1,\ldots,k\}$
are
\begin{align}
TP_l &= c_{l,l}, \\
TN_l &= \sum_{j=1, j\not=l}^k c_{j,j}, \\
FN_l &= \sum_{j=1, j\not=l}^k c_{j,k}, \\
FP_l &= \sum_{j=1, j\not=l}^k c_{k,j}.
\end{align}

By averaging a metric $M$ computed on each derived confusion matrix, we have the
so called macro-averaged~\cite{sokolova2009systematic} of the corresponding
binary classification metric
\[
macro-M = \frac{1}{k} \sum_{l=1}^k M(TP_l, TN_l, FN_l, FP_l).
\]
\noindent Note that the balanced accuracy in binary classification is thus equal
to the macro-specificity or macro-sensitivity in multi-class classification.

Another useful averaging is the
micro-averaging~\cite{sokolova2009systematic}. It uses as true positives
$TP_\mu$ and true negatives $TN_\mu$ the sum of the diagonal elements of the
confusion matrix and as false negatives $FN_\mu$ (resp. false positives
$FP_\mu$) the sum of the lower (resp. upper) triangular part of the confusion
matrix:
\begin{align}
TP_\mu &= \sum_{l=1}^k c_{l,l}, \\
TN_\mu &= \sum_{l=1}^k c_{l,l}, \\
FN_\mu &= \sum_{l=1}^k \sum_{j=1: j < l}^k c_{l,j}, \\
FP_\mu &= \sum_{l=1}^k \sum_{j=1: j > l}^k c_{l,j}).
\end{align}

Each averaging has its own properties: the macro-averaging considers that each
class has the same importance and the micro-averaging reduces the importance
given to the minority classes.

\subsection{Metrics for multi-label classification and ranking}
\label{subsec:multilabel-metrics}

From binary to multi-label classification, the ground truths $(y^i \in \{0,
1\}^d)_{i=1}^n$ and the model predictions $(\hat{y}^i\in \{0, 1\}^d)_{i=1}^n$
are no longer scalars, but vectors of size $d$ or label sets. Both
representations are interchangeable. Usually, the number of labels associated to
a sample is small compared to the total number of labels.

The accuracy~\cite{ghamrawi2005collective}, also called subset accuracy, has a
direct extension in multi-label classification
\begin{equation}
\text{Accuracy} = \frac{1}{n} \sum_{i=1}^n 1(y^i=\hat{y}^i),
\end{equation}
\noindent and requires for each prediction that the predicted label set matches
exactly the ground truth. This is an overly pessimistic metric, especially for
high dimensional label space, as it penalizes any single mistake made for one
sample. The complement of the subset accuracy is called the subset 0-1
loss~\cite{}.

In information theory, the Hamming distance compares the number of differences
between two coded messages. The Hamming error
metric~\cite{schapire1999improved} averages the Hamming distance between the
ground truth and the model prediction over the samples
\begin{equation}
\text{Hamming error} = \frac{1}{n} \frac{1}{d} \sum_{i=1}^n \sum_{j=1}^d 1(y^i_j\not=\hat{y}^i_j).
\end{equation}
\noindent By contrast to the subset accuracy, the Hamming error is an optimistic
metric when the label space is sparse. For a sufficiently large number of
samples and a label density\footnote{The label density is the average number of
labels per samples on the ground truth divided by the size of the label space.}
$\epsilon \rightarrow 0$, a (useless) model predicting always the presence of a
label if its frequency of apparition is higher than $0.5$ in the training set
will roughly have a Hamming error of $\epsilon$. In some situations, the label
density $\epsilon$ is so small that (more useful) models have hardly an Hamming
error lower than $\epsilon$.

Both the Hamming error and the subset accuracy ignore the sparsity of the label
space leading to either overly optimistic or pessimistic error. Multi-label
metrics should be aware of the label space sparsity.

In statistics, the Jaccard index $J$ or Jaccard similarity coefficient
computes the similarity between two sets. Given two sets $A$ and $B$, the
Jaccard index is defined as
\begin{align}
J(A, B) = \frac{|A \cap B|}{|A \cup B|} & \text{ with } J(\emptyset,\emptyset) = 1.
\end{align}
With label sets encoded as boolean vectors  $x, y \in \{0, 1\}^d$,
the Jaccard index becomes
\begin{align}
J(x, y) = \frac{x^Ty}{1_d^Tx + 1_d^Ty - x^Ty},
\end{align}
\noindent where $1_d$ is a vector of ones of size $d$. The Jaccard similarity
score~\cite{godbole2004discriminative}, also sometimes called accuracy, averages
over the samples the Jaccard index between the ground truths and the model
predictions:
\begin{align}
\text{Jaccard similarity score} = \frac{1}{n} \sum_{i=1}^n J(y^i, \hat{y}^i).
\end{align}

By contrast to the Hamming loss, the Jaccard similarity score puts more emphasis
on the labels in the ground truth and the ones predicted by the models.
Moreover, it totally ignores all the negative labels. The Jaccard similarity
score can be viewed as ``local'' measure of similarity and the Hamming loss a
``global'' measure of distance.

A fitted model $f$ applied to an input vector $x$ can go beyond label prediction
and associate to each label $j$ a score or a probability estimate $f(x)_j$. When
the density of the label space $\epsilon$ is small and the size of the label
space $d$ is very high, it is often hard to correctly predict all labels.
Instead, the classifier can rank or score all the labels. We developed here
metrics for such classifiers with different possible goals, e.g. to predict
correctly the label with the highest score $f(x)_j$.

\emph{Note that in the following, we use indifferently the notation $|\cdot|$ to
express the cardinality of a set or the $\ell_1$-norm of a vector.}

If only the top scored label has to be correctly predicted, we are minimizing
the one error\cite{schapire1999improved} which computes the fraction of labels
with the highest score or probability that are incorrectly predicted:
\begin{equation}
\text{One error} = \frac{1}{n} \sum_{i=1}^n
I\left(y^i_j \not= 1 : j = \arg\max_{j \in \{1, \ldots, d\}} f(x^i)_j\right).
\end{equation}

If we want to discover all the true labels at the expense of some false labels,
the coverage error~\cite{schapire2000boostexter} is the metrics to minimize. It
counts the average number of labels with the highest scores or probabilities to
consider to cover all true labels:
\begin{equation}
\text{Coverage error} = \frac{1}{n} \sum_{i=1}^n
\max_{j: y^i_j \not=1} \left|\left\{k : f(x^i)_k \geq f(x^i)_j \right\}\right|.
\end{equation}
\noindent For a label density of $\epsilon$, the best coverage error is thus
$\epsilon d$ and the worst is $d$.

If we want to ensure that pairwise label comparisons are correctly made by the
classifier, we will minimize the (pairwise) ranking loss
metrics~\cite{schapire1999improved}. It counts for each sample the number of
wrongly made pairwise comparisons divided by the number of true labels and false
labels
\begin{equation}
\text{Ranking loss} = \frac{1}{n} \sum_{i=1}^n
\frac{1}{|y^i|}\frac{1}{d - |y^i|}
\left| \left\{
(k, l)\hspace*{-1mm}:\hspace*{-1mm}f(x^i)_k\hspace*{-1mm}<\hspace*{-1mm}f(x^i)_l, y^i_k\hspace*{-1mm}=\hspace*{-1mm}1, y^i_l\hspace*{-1mm}=\hspace*{-1mm}0
\right\}\right|
\end{equation}
\noindent The ranking loss is between $0$ and $1$. A ranking loss of $0$ (resp.
$1$) indicates that all pairwise comparisons are correct (resp. wrong).

If we want that the classifier gives on average a higher score to true labels,
we will use the label ranking average precision
metric~\cite{schapire2000boostexter} to assess the accuracy of the models. For
each samples $y^i$, it averages over each true labels $j$ the ratio between (i)
the number of true label (i.e. $y^i=1$) with higher scores or probabilities than
the label $j$ to (ii) the number of labels ($y^i$) with higher score $f(x^i)$
than the label $j$. Mathematically, we average the LRAP of all pairs of ground
truth $y^i$ and its associated prediction $f(x^i)$:
\begin{equation}
\text{LRAP}(\hat{f})
= \frac{1}{|TS|} \sum_{i=1}^n \frac{1}{|y^i|} \hspace*{-1mm}\sum_{j \in \{k : y^i_{k}=1\}} \hspace*{-1mm}
\frac{|\mathcal{L}_j^{i}(y^i)|}{|\mathcal{L}_j^{i}(1_d)|},
\end{equation}
\noindent where $$\mathcal{L}^i_{j}(q) = \left\{ k :
q_{k}=1 \mbox{~and~} \hat{f}(x^i)_k \geq \hat{f}(x^i)_j\right\}.$$ The best
possible average precision is thus 1. Note that the LRAP score is equal to
fraction of positive labels if all labels are predicted with the same score or
all negative labels have a score higher than the positive one.

Let us illustrate the computation of the previous metrics with a numerical
example. We compare the ground truth $\mathbf{y}$ of $n=2$ samples in a label of size
$d=5$ to the probability score $\mathbf{f(x)}$ given by the classifier:
\begin{align*}
\mathbf{y} &= \begin{bmatrix}
1 & 0 & 1 & 0 & 0 \\
1 & 0 & 0 & 0 & 0
\end{bmatrix}, \\
\mathbf{f(x)} &=  \begin{bmatrix}
0.75 & 0.6 & 0.1 & 0.8 & 0.15 \\
0.25 & 0.8 & 0.1 & 0.15 & 0.3
\end{bmatrix}.
\end{align*}
\noindent Thresholding $\mathbf{f(x)}$ at 0.5 yields the prediction
$\mathbf{\hat{y}}$ of the classifier:
\begin{align*}
\hat{y} &=  f(x) \leq 0.5 \begin{bmatrix}
1 & 1 & 0 & 1 & 0 \\
0 & 1 & 0 & 0 & 0
\end{bmatrix}
\end{align*}
Here, you will find the detailed computation of all previous metrics:
\begin{align*}
\text{Accuracy} &= 0 + 0 = 0, \\
\text{Hamming loss} &= \frac{1}{2}\frac{1}{5} 5 = 0.5, \\
\text{Jaccard similarity score} &= \frac{1}{2}\left(\frac{1}{4}+\frac{0}{2}\right) = 0.125, \\
\text{Top error} &= \frac{1}{2} \left(1 + 1\right) = 2, \\
\text{Coverage error} &= \frac{5 + 3}{2} = 4\\
\text{Ranking loss} &= \frac{1}{2} \left( \frac{1}{2}\frac{1}{3} (1 + 3) + \frac{1}{4}\frac{1}{1} 2 \right)  \approx 0.583, \\
\text{LRAP} &= \frac{1}{2} \left(
\frac{1}{2}(\frac{1}{2} + \frac{2}{5}) + \frac{1}{1}\frac{1}{3}
\right)\approx 0.392.
\end{align*}


While the previous metrics are suited to assess multi-label classification
models, we can complement these metrics with those developed for binary
classification tasks, e.g. specifity, precision, ROC AUC,\ldots (see
Section~\ref{subsec:bin-clf-metrics}). They are well understood in their
respective domains and have attractive properties such as a good handling of
class imbalance. We extend those metrics in three steps: (i) we break the ground
truth and the model prediction vectors into its elements, (ii) we concatenate
the elements into groups such as all predictions associated to a given sample or
all samples associated to a given label and (iii) we average the binary
classification metrics over each group. We will focus here on three averaging
methods: macro-averaging, micro-averaging and sample-averaging. Each averaging
method stems from a vision and different sets of assumptions.

If we view the multi-label classification task as a set of independent
binary classification tasks, we compute the metrics $M$ over each output
separately and average the performance over all $d$ labels leading the
\emph{macro-averaging} version~\cite{yang1999evaluation} of the metrics $M$:
\begin{equation}
\text{macro-}M((y^i)_{i=1}^n, (\hat{y}^i)_{i=1}^n)
= \frac{1}{d} \sum_{j=1}^d M((y^i_j)_{i=1}^n, (\hat{y}^i_j)_{i=1}^n)).
\end{equation}

If instead we view each sample as the result of a query (like in a search
engine), we want to evaluate the quality of each query (or sample) separately.
Under this perspective, the \emph{sample-averaging}
approach~\cite{godbole2004discriminative} computes and averages the metric $M$
over each sample separately:
\begin{equation}
\text{sample-}M((y^i)_{i=1}^n, (\hat{y}^i)_{i=1}^n)
= \frac{1}{n} \sum_{i=1}^n M(y^i, \hat{y}^i).
\end{equation}

The \emph{micro-averaging} approach~\cite{yang1999evaluation} views all label-sample
pairs as forming an unique binary classification task. It compute the metric $M$
as if all label predictions were independent:
\begin{equation}
M\text{-micro}((y^i)_{i=1}^n, (\hat{y}^i)_{i=1}^n)
= M((y^i_j)_{i,j=(1,\ldots,n)},
    (\hat{y}^i_k)_{i,j=(1,\ldots,n)})).
\end{equation}

\subsection{Regression metrics}
\label{subsec:regression-metrics}

Given a set of $n$ ground truths $(y^i \in \mathbb{R})_{i=1}^n$ and their
associated model predictions $(\hat{y}^i\in \mathbb{R})_{i=1}^n$, regression
tasks are often assessed using the mean square error (MSE), the average of
the square loss, expressed by
\begin{equation}
\text{MSE} = \frac{1}{n} \sum_{i=1}^n \left(y^i - \hat{y}^i\right)^2.
\end{equation}

From the mean square error, we can derive the $r^2$ score, also called
the coefficient of determination. It is the fraction of variance explained
by the model:
\begin{align}
r^2
&= 1 - \frac{\text{MSE}}{\text{Output variance}}\\
&= 1 - \frac{\sum_{i=1}^n (y^i - \hat{y}^i)^2}{\sum_{i=1}^n (y^i - \frac{1}{n}\sum_{l=1}^n y^l)^2}
\end{align}
\noindent The $r^2$ score is normally between 0 and 1. A $r^2$ score of zero
indicates that the models is no better than a constant, while a $r^2$ of one
indicates that the model perfectly explains the output given the
inputs. A negative $r^2$ score might occur and it indicates that the model is
worse than a constant model.

Square-based metrics are highly sensitive to the presence of outliers with
abnormally high prediction errors. The mean absolute error (MAE), the average of the
absolute loss, is often suggested as a robust replacement of the MSE:
\begin{equation}
\text{MAE} = \frac{1}{n} \sum_{i=1}^n |y^i - \hat{y}^i|.
\end{equation}

These single output metrics naturally extend to multi-output regression tasks.
The multi-output mean squared error and mean absolute error for an output space
size $d$ is given by
\begin{align}
\text{MSE} &= \frac{1}{n} \frac{1}{d} \sum_{i=1}^n ||y^i - \hat{y}^i||_{\ell_2}^2,\\
\text{MAE} &= \frac{1}{n} \frac{1}{d} \sum_{i=1}^n ||y^i - \hat{y}^i||_{\ell_1}.
\end{align}
\noindent These measures average the metrics over all outputs assuming they
are independent.

Similarly, averaging the $r^2$ score over each output leads to the macro-$r^2$
score:
\begin{align}
\text{macro-}r^2 = 1 -
\frac{1}{d} \sum_{j=1}^d
\frac{\sum_{i=1}^n (y^i_j - \hat{y}^i_j)^2}{\sum_{i=1}^n (y^i_j - \frac{1}{n}\sum_{l=1}^n y^l_j)^2}.
\end{align}

An alternative extension of the $r^2$ score is to consider the total fraction of
the output variance, or more strictly the sum of the variance over each output,
explained by the model
\begin{align}
\text{variance-}r^2
&= 1 - \frac{MSE}{\text{Total output variance}}  \\
&= 1 - \frac{\sum_{i=1}^n ||y^i - \hat{y}^i||_{\ell_2}^2}{\sum_{i=1}^n ||y^i - \frac{1}{n}\sum_{l=1}^n y^l||_{\ell_2}^2},
\end{align}
\noindent which is equal to 1 minus the fraction of explained
variance~\cite{bakker2003task}.

The variance-$r^2$ is a variance weighted average of the $r^2$ score. We can
reformulate the variance-$r^2$ as:
\begin{align}
\text{variance-}r^2
&= 1 - \frac{\sum_{i=1}^n \sum_{j=1}^d (y^i_j - \hat{y}^i_j)^2}{\text{Total output variance}} \\
&= 1 - \sum_{j=1}^d w_j \frac{\sum_{i=1}^n (y^i_j - \hat{y}^i_j)}{\sum_{i=1}^n (y^i_j - \frac{1}{n}\sum_{l=1}^n y^l)^2}
\label{eq:variance-weight-sum}
\end{align}
\noindent with
$w_j=\frac{\sum_{i=1}^n (y^i_j - \frac{1}{n}\sum_{l=1}^n y^l_j)^2}{\text{Total
output variance}}$. By contrast, the macro-$r^2$ score would have uniform
weights $w_j= \frac{1}{d}\, \forall j$ in Equation~\ref{eq:variance-weight-sum}.

\section{Hyper-parameter optimization}
\label{sec:model-select}

Supervised learning algorithms can be viewed as a function $A: (\mathcal{X}
\times \mathcal{Y})_{i=1}^n \times \mathcal{A} \rightarrow \mathcal{H}$ taking as
input a learning set $\mathcal{L} = \left((x^i, y^i) \in \left(\mathcal{X} \times
\mathcal{Y}\right)\right)_{i=1}^n$ and a set of hyper-parameters $\alpha \in
\mathcal{A}$ and outputting a function $f$ in a hypothesis space $\mathcal{H}
\subset \mathcal{Y}^\mathcal{X}$. The hypothesis space $\mathcal{A}$ is defined
through one or several hyper-parameter variables that can be either discrete,
like the number of neighbors for a nearest neighbors model, or continuous, like
the multiplying constant of a penalty loss in penalized linear models.

We need hyper-parameter tuning methods to find the best hyper-parameter set
$\alpha^* \in \mathcal{A}$ that minimizes the expectation of some loss function
$\ell : \mathcal{Y} \times \mathcal{Y} \rightarrow \mathbb{R}^+$ over the joint
distribution of input / output pairs $P_{{\cal X},{\cal Y}}$:
\begin{equation}
\alpha^* = \arg \min_{\alpha \in \mathcal{A}} \E_{P_{{\cal X},{\cal Y}}} \{ \ell(A(\mathcal{L},\alpha)(x), y) \}.
\label{eq:model-search}
\end{equation}

Directly optimizing Equation~\ref{eq:model-search} is in general not possible as
it consists in minimizing the generalization error over unseen samples. Thus, we
resort to validation techniques to split the samples into one (or more)
validation set(s) $S^{\text{valid}}$ to estimate the generalization error (see
Section~\ref{sec:model-eval}) and to select the best set of hyper-parameter
$\alpha^{*}$:
\begin{equation}
\alpha^* \approx  \arg \min_{\alpha \in \mathcal{A}} \sum_{(x, y) \in S^{\text{valid}}} \{ \ell(A(\mathcal{L},\alpha)(x), y) \}.
\label{eq:model-search-cv}
\end{equation}
\noindent Note that we can optimize a metric defined over a set of samples
instead a loss as the area under the ROC curve.

In its simplest form, the hyper-parameter search assesses all possible
hyper-parameter sets $\alpha \in \mathcal{A}$. While it is optimal on the
validation set(s), this is impractical as the size of the hyper-parameter space
$\mathcal{A}$ is often unbounded. The hyper-parameter space often consists of
continuous hyper-parameter variables leading to an infinite number of possible
hyper-parameter sets. Whenever the number of hyper-parameter sets is finite
($|\mathcal{A}|<\infty$), we are limited by computational budget constraints.
Instead, we resort to evaluate a subset of the hyper-parameter space
$\mathcal{A}^- \subset \mathcal{A}$:
\begin{equation}
\alpha^* \approx  \arg \min_{\alpha \in \mathcal{A}^-} \sum_{(x, y) \in S^{\text{valid}}} \{ \ell(A(\mathcal{L},\alpha)(x), y) \}.
\label{eq:model-search-cv-real}
\end{equation}

The classical approach to design a finite and reduced subspace $\mathcal{A}^-$
is to sample the hyper-parameter space $\mathcal{A}$ through a manually defined
grid. A too coarse grid will miss the optimum hyper-parameter set $\alpha^*$,
while a too fine grid will be very costly. In~\cite{hsu2003practical},
\citeauthor{hsu2003practical} suggests a two-stage approach: (i) a coarse
parameter grid first identifies regions of interest in the hyper-parameter
space, and then a finer grid locates the optimum.  Nevertheless, we might still
miss the optimal hyper-parameter set $\alpha^*$ since the objective function of
Equation~\ref{eq:model-search-cv-real} is not necessarily convex nor concave

\begin{remark}{How to wrongly optimize and / or to wrongly validate a model?}
Given a set of samples $S$ and a supervised learning algorithm $A$, one wants
simultaneously to find the hyper-parameter set $\alpha^*$ and estimate the
generalization error of the associated model $f$.

A wrong approach would be to use directly one of the validation techniques
presented in Section~\ref{sec:model-eval} dividing the sample set $S$ into
(multiple) pair(s) of a training set and a test set $(S^\text{train},
S^\text{test})$. If we select the best hyper-parameter set $\alpha^*$ based on
the test set(s) $S^\text{test}$, then the approximation of the generalization
error on $S^\text{test}$ is biased: the hyper-parameter set $\alpha^*$ has been
selected on the same test set(s). Another approach would be to repeat
independently the described process using different partitions of the sample set
$S$ to first select the best model and then to estimate the generalization
error. However, the generalization error is still biased: we might use the same
samples to train, to select or to validate the model.

The correct approach is to use \emph{nested validation techniques}. We first
divide the sample set into (multiple) pair(s) of a test set $S^\text{test}$ and
training-validation set $S^\text{train-valid}$. Then we again apply a validation
technique to split the training-validation set into (multiple) pair(s) of a
training set $S^\text{train}$ and a validation set $S^\text{valid}$. The models
$f$ with hyper-parameter $\alpha$ are first trained on $S^\text{train}$, then we
select the best hyper-parameter set $\alpha^*$ on $S^\text{valid}$. We finally
estimate the generalization error of the overall model training and selection
procedure by re-training a model on $S^\text{train-valid}$ using the best
hyper-parameter set $\alpha^*$ on the testing set $S^\text{test}$.

Proper validation is necessary and comes at the expense of the sample efficiency
and computing time. Note that nested validation methods are not needed if we
want solely either to select the best model or to estimate the generalization
error of a given model.
\end{remark}

In the grid search approach, we first sample each hyper-parameter variable and
then build all possible combinations of hyper-parameter sets. However, some of
these hyper-parameter variables have no or small influence on the performances
of the models. In these conditions, large hyper-parameter grids are doomed to
fail due to the explosion of hyper-parameter sets. Random search
techniques~\cite{solis1981minimization} tackles such optimization problems by
(i) defining a distribution over the optimization variables, (ii) drawing random
values from this distribution and (iii) selecting the best one out of these
values. \citep{bergstra2012random} have shown that random
hyper-parameter search scales better than grid search as the search is not
affected by the hyper-parameter variables having few or no influence on the
model performance. As an illustration, let us consider a model with one
parameter and one without impact on its generalization error. Sampling 9 random
hyper-parameter sets would yield more information than making a $3 \times 3$
grid as we evaluate 9 different values of the dependent variable in the random
search instead of 3 in the grid.

For a continuous loss and a continuous hyper-parameter space, Bayesian
hyper-parameter
optimization~\cite{snoek2012practical,bergstra2011algorithms,hutter2011sequential}
goes beyond random search and models the performance of a supervised learning
algorithm $A$ with hyper-parameters $\alpha$. Starting from an initial Gaussian
prior over the hyper-parameter space, it refines a posterior distribution of the
model error with each new tested sets of hyper-parameters. New hyper-parameter
sets are drawn to minimize the overall uncertainty and the model error.

\section{Unsupervised projection methods}
\label{sec:dimensionality-reduction}

Supervised learning aims at finding the best function $f$ which maps the input
space $\mathcal{X}$ to the output space $\mathcal{Y}$ given a set of $n$ samples
$\left((x^i, y^i) \in \left(\mathcal{X} \times
\mathcal{Y}\right)\right)_{i=1}^n$. However with very high dimensional input
space, we need a very high number of samples $n$ to find an accurate function
$f$. This is the so-called curse of dimensionality. Another problem arises if
the model $f$ is unable to model the input-output relationship because the model
classes $\mathcal{H}$ is too restricted, for instance a linear model will fail
to model quadratic data.

Unsupervised projection methods lift the original space $\mathcal{X}$ of size
$p$ to another space $\mathcal{Z}$ of size $q$. If the projection lowers the
size of the original space ($q<p$), this is a dimensionality reduction
technique. In the context of supervised learning, we hope to break the curse of
dimensionality with such projection methods while speeding up the model
training.  If the projections perform non linear transformations of the input
space, it might also improve the model performance. For instance, a linear
estimator will be able to fit quadratic data if we enrich the input variables
with their quadratic and bilinear forms. Note that projecting the input space to
two or three dimensions ($q \in \{2, 3\}$) is an opportunity to get insights on
the data through visualization.

We present three popular unsupervised projection methods and discuss their
properties: (i) the principal component analysis approach in
Section~\ref{subsec:pca}, which aims to find a subspace maximizing the total
variance of the data; (ii) random projection methods in Section~\ref{sub:rp},
which project the original space onto a lower dimensional space while
approximately preserving pairwise euclidean distances, and (iii) kernel
functions in Section~\ref{subsec:kernels}, which compute pairwise sample
similarities lifting the original space to a non linear one.

\subsection{Principal components analysis}
\label{subsec:pca}

The principal component analysis (PCA) method~\cite{jolliffe2002principal} is a
technique to find from a set of samples $(x^i \in \mathcal{X})_{i=1}^n$ an
orthogonal linear transformation $Z$ which maximizes the variance along each
axis of the transformed space as shown in Figure~\ref{fig:pca-cloud}.

\begin{figure}
\centering
\includegraphics[width=0.75\textwidth]{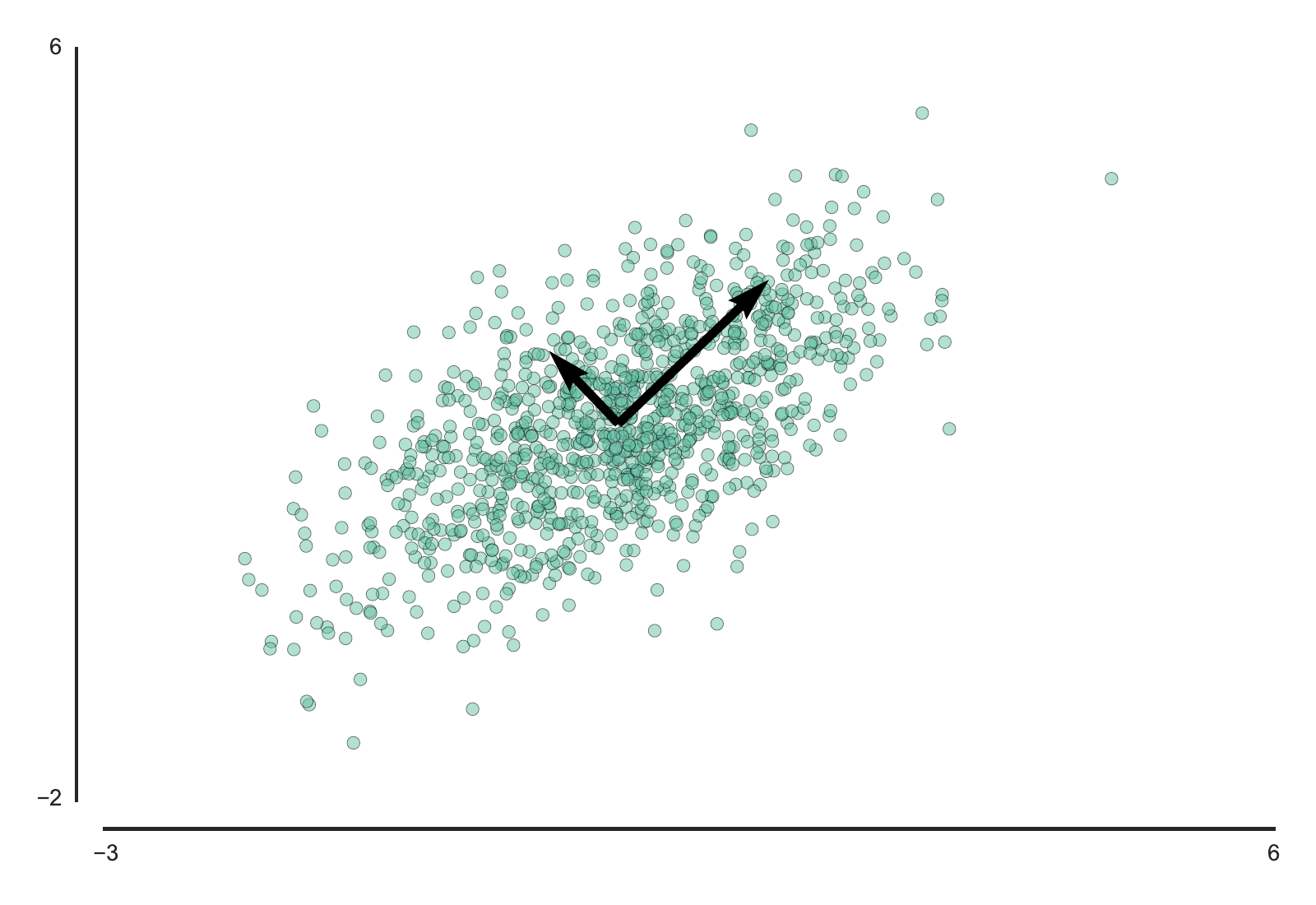}
\caption{The principal components, shown as black arrows, are the orthogonal
vectors maximizing the variance of the samples drawn here from a multivariate
Gaussian distribution.}
\label{fig:pca-cloud}
\end{figure}

Principal component analysis reduces the dimensionality of the dataset by
keeping a fraction of the principal components vectors which totalize a large
amount of the total variance. If we keep only two components, PCA allows to
visualize high dimensional datasets as illustrated in
Figure~\ref{fig:pca-cloud-digits} with digits.

\begin{figure}
\centering
\includegraphics[width=0.75\textwidth]{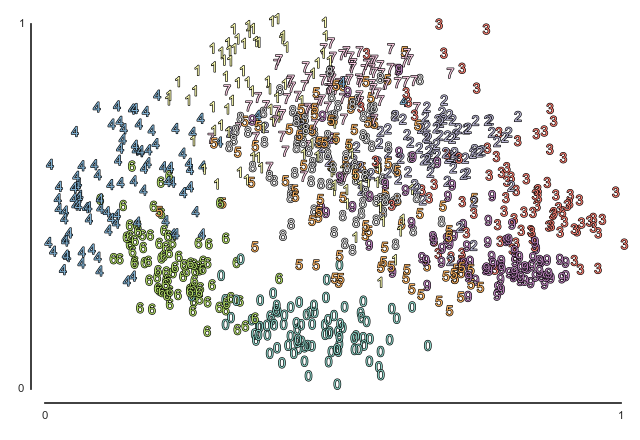}
\caption{We project the digits dataset~\cite{Lichman:2013} from its $8 \times 8$ pixel space on the
two principal components with the largest variance ($29\%$ of the total
variance). Digits such as $4$ or $0$ form well defined clusters on this two
dimensional space.}
\label{fig:pca-cloud-digits}
\end{figure}

Mathematically, we want to find the first principal component vector $u^1$ which
maximizes the variance along its direction:
\begin{align}
u^1 = &\arg\max_{u} \sum_{i=1}^n ||u^T x^i - u^T \sum_{l=1}^n  \frac{x^l}{n} ||^2_{\ell_2} \nonumber \\
\text{s.t. }& u^T u = 1.
\end{align}

Given that the covariance matrix $C$ is given by
\[
C = \sum_{i=1} (x^i - \sum_{l=1}^n \frac{x^l}{n})(x^i - \sum_{l=1}^n  \frac{x^l}{n})^T,
\] we have
\begin{equation}
u^1 = \arg\max_{u} u^T C u + \lambda_1 (1 -  u^T u),
\end{equation}
\noindent where $\lambda_1$ is the Lagrange multiplier of the normalization
constraint.

By derivating with respect to $u$ and setting the first derivative to zero, we
have that the maximum is indeed an eigen vector of the covariance matrix:
\begin{align}
C u^1  = \lambda_1 u^1
\end{align}
\noindent We also note that the variance along $u^1$ is given  by ${u^1}^T C u^1 =
\lambda_1$. Thus $u^1$ is the eigenvector with the highest eigen value.

The following vector $u^{m+1}$ maximizing the variance are obtained by
imposing that the $m+1$-th vector is orthogonal to the $m$ previous one:
\begin{align}
& u^{m+1} = \arg\max_{u^{m+1}} {u^{m+1}}^T C u^{m+1} \nonumber \\
\text{s.t. } & {u^{m+1}}^T u^{m+1} = 1,  \nonumber\\
             & {u^{m+1}}^T u^l = 0 \quad \forall l \in \{1,\ldots, m\},
\end{align}
\noindent or alternatively in Lagrangian form
\begin{equation}
\arg\max_{u^{m+1}} {u^{m+1}}^T C u^{m+1} +
\lambda_{m+1}(1 - {u^{m+1}}^T u^{m+1}) +
\sum_{l=1}^m \mu_{l} {u^{m+1}}^T u^{l}.
\end{equation}

By differentiating with respect to $u^l$ and multiplying by $u^{m+1}$, we have that
the Lagrangian constants of the orthogonality constraints are equal to zero
$\mu_l=0 \, \forall l \in \{1,\ldots, m\}$. And it follows that the $m+1$-th
principal component is the $m+1$-th eigen vector with the $m+1$-th largest eigen
value $\lambda_{m+1}$ since
\begin{equation}
C u^{m+1}  = \lambda_{m+1} u^{m+1},
{u^{m+1}}^T C u^{m+1} = \lambda_{m+1}.
\end{equation}

\subsection{Random projection}
\label{sub:rp}

Random projection is a dimensionality reduction method which projects the space
onto a smaller random space. For randomly projection, the Jonhson-Lindenstrauss
lemma gives the conditions of existence such that the distance between pairs of
points is approximately preserved.

\begin{remark}{Johnson-Lindenstrauss lemma~\cite{johnson1984extensions}}
\label{lemma:jl-lemma}
Given $\epsilon > 0$ and an integer $n$, let $q$ be a positive integer such
that  $q \geq 8 \epsilon^{-2} \ln {n}$. For any sample $(y^i)_{i=1}^{n}$ of $n$ points
in $\mathbb{R}^d$ there exists a matrix $\Phi \in \mathbb{R}^{q \times d}$
such that for all $i, j  \in \{1, \ldots , n\}$
\begin{equation}
(1 \hspace*{-0.3mm} - \hspace*{-0.3mm}\epsilon) ||y^i \hspace*{-0.3mm}- \hspace*{-0.3mm}y^j||^2 \leq || \Phi y^i \hspace*{-0.3mm}- \hspace*{-0.3mm}\Phi y^j ||^2
                           \leq (1\hspace*{-0.3mm} +\hspace*{-0.3mm} \epsilon) || y^i \hspace*{-0.3mm}- \hspace*{-0.3mm}y^{j}||^2.
\label{eqn:js}
\end{equation}
\end{remark}

Moreover, when $d$ is sufficiently large, several random matrices satisfy
Equation~\ref{eqn:js} with high probability. In particular, we can consider
Gaussian matrices whose elements are drawn {\em i.i.d.} in $\mathcal{N}(0, 1 /
q)$, as well as (sparse) Rademacher matrices whose elements are drawn in the
finite set $\left\{ -\sqrt{\frac{s}{q}}, 0, \sqrt{\frac{s}{q}} \right\}$ with
probability $\left\{ \frac{1}{2s}, 1 - \frac{1}{s} ,\frac{1}{2s}\right\}$, where
$1 / s \in (0,1]$ controls the sparsity of $\Phi$. If $s=3$, we will say that
those projections are Achlioptas random
projections~\cite{achlioptas2003database}. When the size of the original space
is $p$ and $s=\sqrt{p}$, then we will say that we have sparse random projection
as in~\cite{li2006very}. Note a random sub-space~\cite{ho1998random} is also a
random projection scheme~\cite{candes2011probabilistic}: the projection matrix
$\Phi$ is obtained by sub-sampling with or without replacement the identity
matrix. Theoretical results proving~\ref{eqn:js} with high probability for each
random projection matrix can be found in the relevant paper.

The choice of the number of random projections $q$ is a trade-off between the
quality of the approximation and the size of the resulting embedding as illustrated
in Figure~\ref{fig:rp-distortion}.

\begin{figure}
\centering
\includegraphics[width=0.75\textwidth]{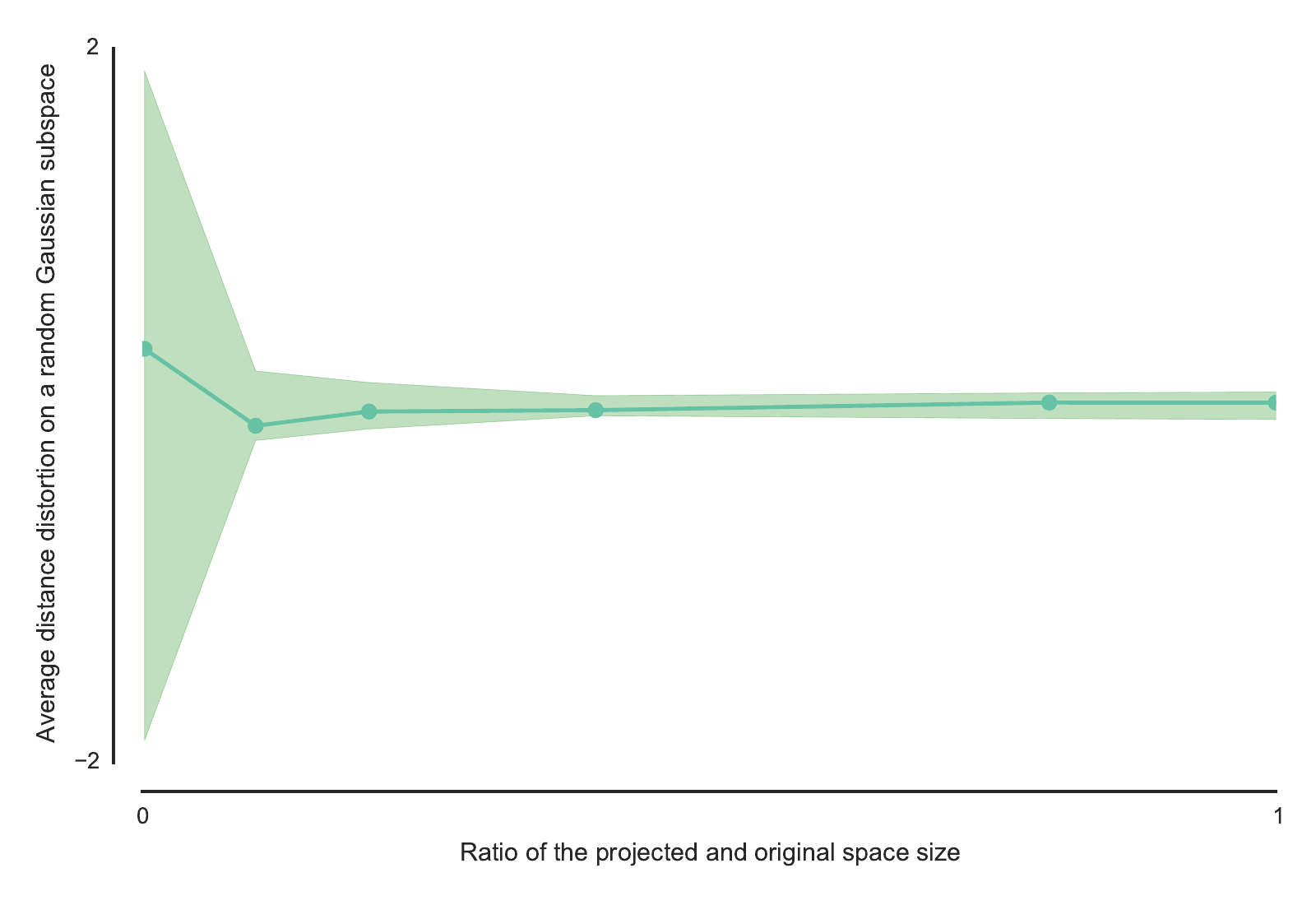}
\caption{Given a set of samples $(x^i)_{i=0}^{2000} \in \mathbb{R}^{500}$ drawn
from a Gaussian distribution, a few random projections preserve on average the
distance between pairs of points up to a distortion $\epsilon$.}
\label{fig:rp-distortion}
\end{figure}

\subsection{Kernel functions}
\label{subsec:kernels}

A kernel function $k: \mathcal{X} \times \mathcal{X} \rightarrow \mathbb{R}$
computes the similarity between pairs of samples (usually in the input space).
Machine learning algorithms relying solely on dot products, such as support
vector machine~\cite{cortes1995support} or principal components
analysis~\cite{jolliffe2002principal}, are indeed using the linear kernel
$k(x^i, x^j) = {x^i}^T x^j$. We can kernelize these algorithms by replacing
their dot product with a kernel presented in Table~\ref{tab:kernels}. This is
the so called kernel trick~\cite{scholkopf2001learning}. Kernel functions define
non linear projection schemes lifting the original space to the one defined by
the chosen kernel. It has been used in classification~\cite{hsu2003practical}, in
regression~\cite{jaakkola1999probabilistic} and in
clustering~\cite{scholkopf1997kernel}. Random
kernels~\cite{rahimi2007random,rahimi2009weighted} can be used to compress the
input space.

\begin{table}
\caption{Some common kernel functions between two vectors $x^i$ and $x^j$.}
\label{tab:kernels}
\centering
\begin{tabular}{@{}ll@{}}
\toprule
Kernels & \\
\midrule
Linear kernel & $k(x^i, x^j) = {x^i}^T x^j$ \\
Polynomial & $k(x^i, x^j) = ({x^i}^T x^j + r)^d$ \\
Gaussian radial basis function & $k(x^i, x^j) = \exp(-\gamma ||x^i - x^j||^2_2)$ \\
Hyperbolic tangent & $k(x^i, x^j) = \tanh(\kappa {x^i}^T x^j + r)$ \\
\bottomrule
\end{tabular}
\end{table}

The task shown in Figure~\ref{subfig:kernel-linear} requires to classify points
belonging to one of two interleaved moons. Given the non linearities, we can not
linearly separate both classes as illustrated in
Figure~\ref{subfig:kernel-linear} with a (linear) support vector machine. By
lifting the linear kernel to the radial basis function kernel, the support
vector machine algorithm now separates both classes as shown in
Figure~\ref{subfig:kernel-rbf}. Effectively, kernel functions enable machine
learning algorithms to handle a wide varieties of structured and unstructured
data such as sequences, graphs, texts or vectors through an appropriate choice
and design of kernel functions.

\begin{figure}
\centering
\subfloat[Classification task]{{\includegraphics[width=0.5\textwidth]{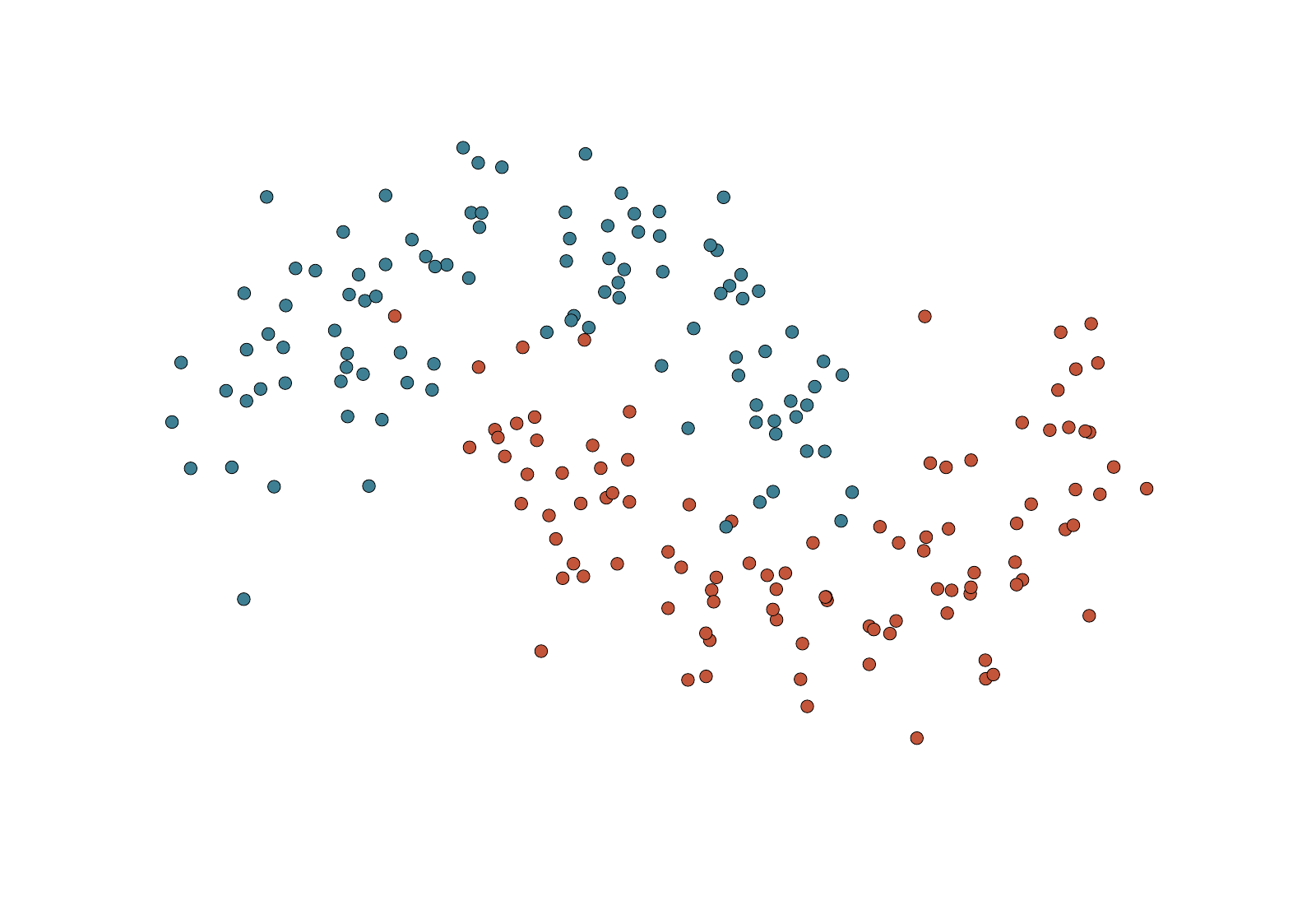}
\label{subfig:kernel-task}}} \\
\subfloat[Linear SVM]{{\includegraphics[width=0.5\textwidth]{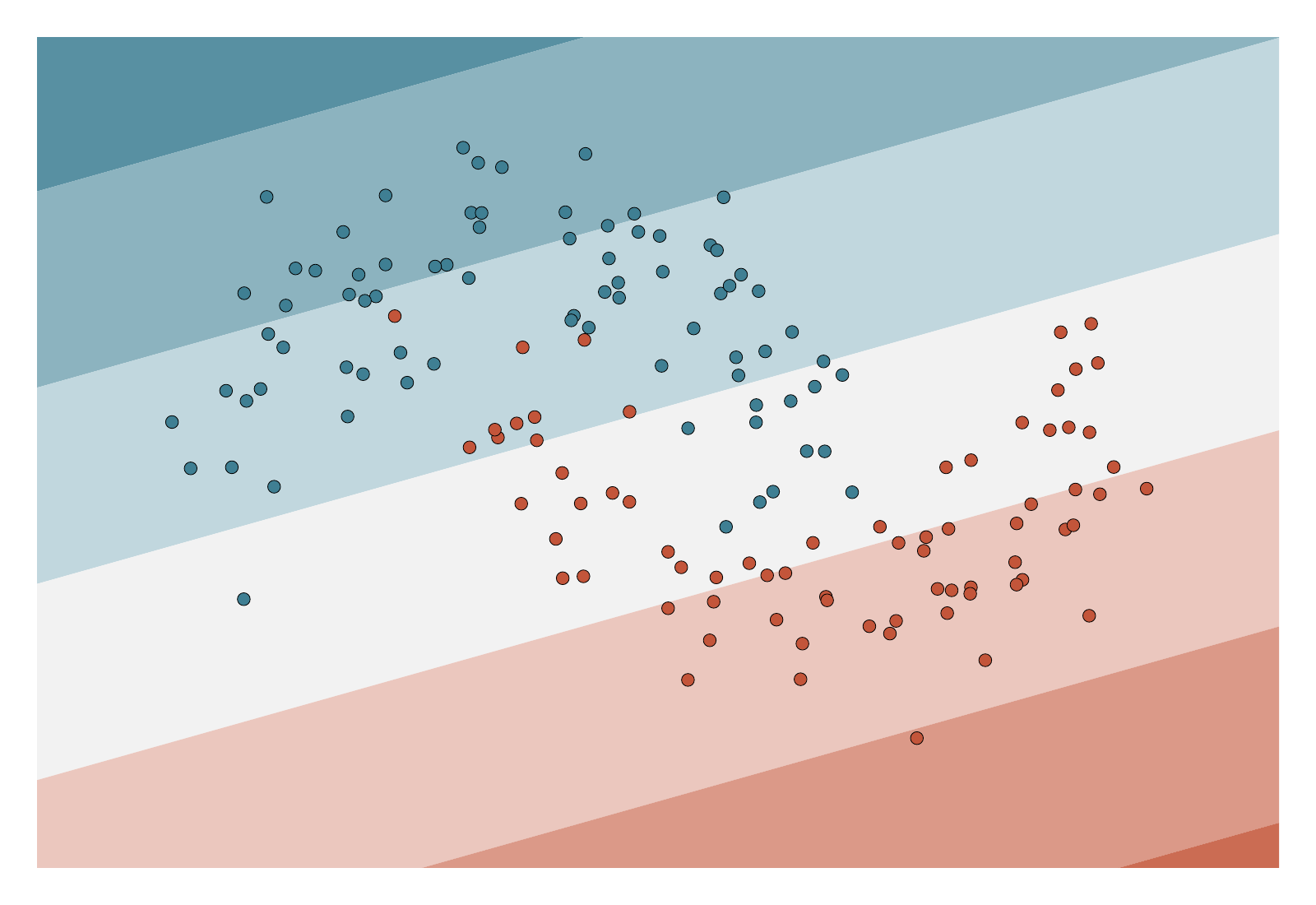}
\label{subfig:kernel-linear}}}
\subfloat[Radial basis function SVM]{{\includegraphics[width=0.5\textwidth]{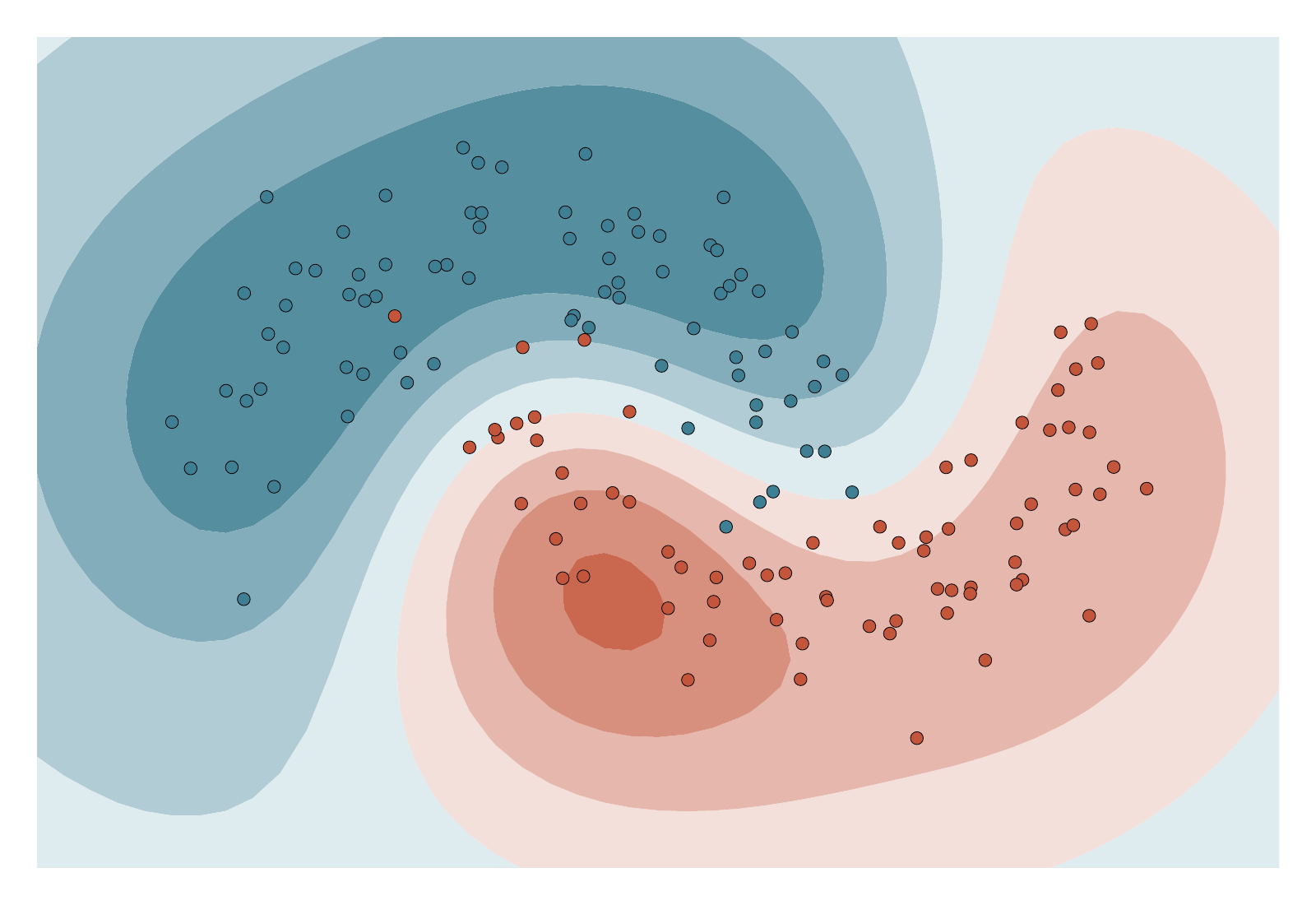}
\label{subfig:kernel-rbf}}}
\caption{The classification task consists in discriminating points belonging to
one of the two interleaved moons (the blue or the red dots). Given the non
linearities of the data, a linear support vector machine is not able to find a
hyperplane separating both classes as shown in
Figure~\ref{subfig:kernel-linear}. By lifting the original input
space through the radial basis function kernel, we are now able to separate both
classes as illustrated in Figure~\ref{subfig:kernel-rbf} retrieving the
interleaved moons.}
\label{fig:kernels}
\end{figure}


\chapter{Decision trees}\label{ch:tree}

\begin{remark}{Outline}
Decision trees are non parametric supervised learning models mapping the input
space to the output space. The model is a hierarchical structure made of test
and leaf nodes. Starting at the root node, the top of the tree, the test nodes
lead the reasoning through the tree structure until reaching a leaf node
outputting a prediction. In this chapter, we first describe the decision tree
model and show how to predict unseen samples. Then, we present the methods and
the techniques to grow and to prune these tree structures. We also introduce how
to interpret a decision tree model to gain insights on the studied systems and
phenomena.
\end{remark}

A decision tree is comparable to a human reasoning organized through a
hierarchical set of yes/no questions. As in medical diagnosis, an expert (here
the doctor) diagnoses the patient state (``Is the patient healthy or sick?'') by
screening the patient body and by retrieving important past patient history through
a directed set of questions. We can view each step of the reasoning as a branch
of a decision tree structure. Each answer leads either to another question
refining a set of hypotheses or finally to a conclusion, a prediction.

The binary questions at test nodes can target binary variables, like ``Do you
smoke?'', categorical variables, like ``Do you prefer pear or apple to orange or
lemon?'', or numerical variables, like ``Is the outside temperature below 25
degree Celsius (77 degree Fahrenheit)?''. Note that we can formulate multi-way
questions as a set of binary questions. For instance, the multi-way question
``Do you want to eat a pear, a peach or an apple?'' is equivalent to ask
sequentially ``Do you want to eat a pear or one fruit among peach and apple?'',
then you would also ask ``Do you want to eat a peach or an apple?'' if you
answered ``a peach or an apple''.

With only numerical input variables, questions that are typically asked are in
the form ``Is the value of variable $x$ lower than a given value?''. The
decision tree is then a geometric structure partionning the input space into a
recursive set of $p$-dimensional (hyper)rectangles (also called $p$-orthotopes). The root node first divides
the whole input space into two half-spaces. Each of those
may further be divided into smaller (hyper)rectangles. The
partition structure highlights the hierarchical nature of a decision tree. An
artistic example of such partitioning in a two dimensional space is the
``Composition II in Red, Blue, and Yellow'' by Piet Mondrian\footnote{Piet
Mondrian (1872-1944) is a painter famous for his grid-based paintings
partitioning the tableau through black lines into colored rectangles usually
blue, red, yellow and white} shown in Figure~\ref{fig:mondrian-comp-2}. Here,
Piet Mondrian divides hierarchically the whole painting into colored rectangles
through heavy thick black lines. Each black line is conceptually a testing node
of a decision tree, while the colored rectangles would be the predictions of
leaves node.

\begin{figure}
\centering
\begin{tikzpicture}[line width=0cm, scale=0.5]
\definecolor{mondrianblue}{RGB} {10,70,135}
\definecolor{mondrianred}{RGB} {210,20,25}
\definecolor{mondrianyellow}{RGB} {240,220,95}

\draw[fill=black] (0,0) rectangle (11.37, 11.37);

\draw[fill=mondrianblue] (0,0) rectangle (2.4, 2.87);
\draw[fill=white] (0,3.16) rectangle (2.4,7.3);
\draw[fill=white] (0,8.0) rectangle (2.4,11.37);

\draw[fill=mondrianred] (2.77,3.17) rectangle (11.37,11.37);

\draw[fill=white] (2.77,0) rectangle (10,2.87);

\draw[fill=mondrianyellow] (10.45,0) rectangle (11.37,1.21);
\draw[fill=white] (10.45,1.73) rectangle (11.37,2.87);
\end{tikzpicture}
\caption{Reproduction of ``Composition II in Red, Blue, and Yellow'' of Piet
Mondrian.}
\label{fig:mondrian-comp-2}
\end{figure}


Decision trees are popular machine learning model, because
of several nice properties:
\begin{itemize}

\item  The hierarchical nature of the model takes into account non linear
effects between the inputs and outputs, as well as conditional dependencies among inputs and outputs.

\item Growing a decision tree is computationally fast.

\item Decision tree techniques work with heterogeneous data combining binary,
categorical or numerical input variables.

\item Decision trees are interpretable models giving insights on the
relationship between the inputs and the outputs.

\item The tree training algorithm can be adapted to handle missing input
values~\cite{friedman1977recursive,breiman1984classification,quinlan1989unknown}.

\end{itemize}

In Section~\ref{sec:dt-model}, we present the structure of such models and how
to exploit these to predict unseen samples. We show in
Section~\ref{sec:dt-growth} how to train a decision tree model. In
Section~\ref{sec:dt-pruning}, we describe  the techniques used to prune a fully
grown decision tree to the right size: too shallow trees tend to under-fit the
data as they might lack predicting power, while large trees might overfit the
data as they are too complex. In Section~\ref{sec:dt-interpret}, we show how to
interpret a decision tree to gain insights over the input-output relationships:
(i) through input variable importance measures and (ii) through the conversion
of the tree structure to a set of rules. In Section~\ref{sec:mo-trees}, we show
how to extend decision trees to handle multi-output tasks.

\section{Decision tree model}\label{sec:dt-model}


The binary decision tree model is a  tree structure built by recursively
partitioning the input space. The root node is the node at the top of the tree.
We distinguish two types of nodes: (i) the test nodes, also called internal
nodes or branching nodes, and (ii) the leaves outputting predictions, also
called external nodes or terminal nodes. A test node $N_t$ has two children
called the left child and the right child; it furthermore has a splitting rule
$s_t$ testing whether or not a sample belongs to its left or right child. For a
continuous or categorical ordered input, the splitting rules are typically of
the form $s_t(x) = x_{F_t} \leq \tau_t$ testing whether or not the input value
$x_{F_t}$ is smaller or equal to a constant $\tau_t$. For a binary or
categorical input, the splitting rule is of the form $s_t(x) = x_{F_t} \in B_t$
testing whether or not the input value $x_{F_t}$ belongs to the subset of values
$B_t$.

The decision tree of Figure~\ref{fig:dt-structure} first splits  the input space
into two disjoint partitions $A_2 = \{x: x_{F_1} \leq \tau_1\}$ and $A_3 = \{x:
x_{F_2} > \tau_1\}$ at the root node $N_1$. The node $N_1$ has two children:
$N_2$ the left child and $N_3$ the right child. The node $N_2$ is a leaf and
thus a terminal partition. The test node $N_3$ further splits the input space
based on a categorical set $B_3$ into partitions $A_4 = \{x \in A_3, x_{F_3} \in
B_3\}$ and $A_5 = \{x \in A_3, x_{F_3} \in B_3\}$ with $A_3 = A_4 \cup A_5$. The
input space is further split with 3 more testing nodes with two continuous
splitting rules ($N_4$, $N_5$) or one categorical splitting rule $N_8$. The
remaining nodes $N_6$, $N_7$, $N_9$, $N_{10}$ and $N_{11}$ are leaf nodes. In
total, the decision tree defines a partition of the input space into 11
(hyper)rectangles ($A_1,\ldots,A_{11}$). A one to one relationship exists
between the leaf nodes and the subsets of this input space partition. Note that
all partitions of the input space are not expressible as a decision tree
structure.

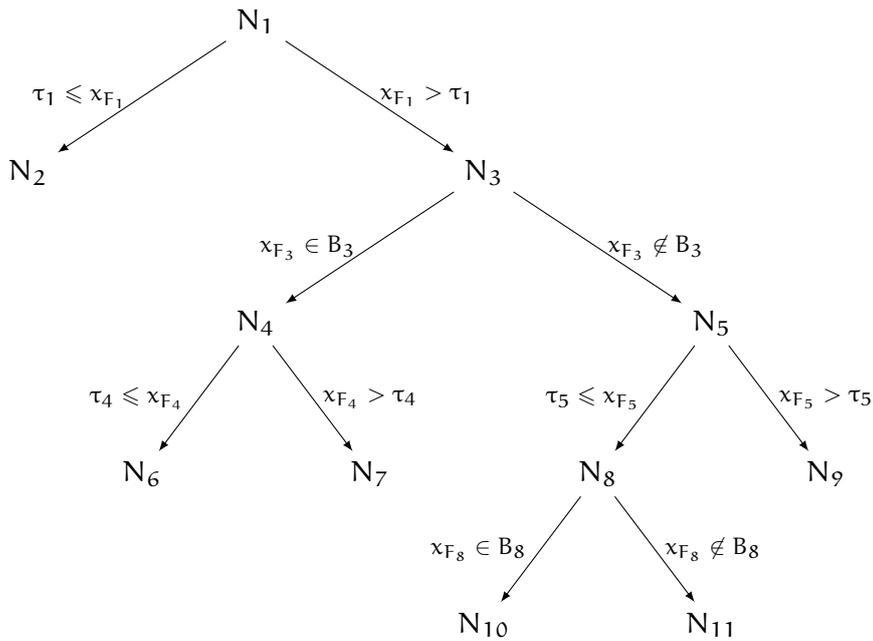
\begin{figure}[h]
\centering
\begin{tikzpicture}
[level 1/.style={sibling distance=60.0mm},
 level 2/.style={sibling distance=60.0mm},
 level 3/.style={sibling distance=30.0mm},
 level 4/.style={sibling distance=30.0mm},
 leaf/.style={rectangle},
 test/.style={rectangle},
 edge from parent/.style = {draw, -latex, font=\footnotesize},
 level distance=20mm]
\node[test] {$N_1$}
 child {node[leaf] {$N_2$} edge from parent node[left] {$\tau_1 \leq x_{F_1}\,$}}
 child {node[test] {$N_3$}
     child {node[test] {$N_4$}
         child {node[leaf] {$N_6$} edge from parent node[left] {$\tau_4 \leq x_{F_4}\,$}}
         child {node[leaf] {$N_7$} edge from parent node[right] {$x_{F_4} > \tau_4$}}
         edge from parent node[left] {$x_{F_3}\in B_3\,$}}
     child {node[test] {$N_5$}
         child {node[test] {$N_8$}
             child {node[leaf] {$N_{10}$} edge from parent node[left] {$x_{F_8} \in B_8\,$}}
             child {node[leaf] {$N_{11}$} edge from parent node[right] {$x_{F_8} \not\in B_8$}}
             edge from parent node[left] {$\tau_5 \leq x_{F_5}\,$}}
         child {node[leaf] {$N_9$} edge from parent node[right] {$x_{F_5} > \tau_5$}}
         edge from parent node[right] {$x_{F_3} \not\in B_3$}}
     edge from parent node[right] {$x_{F_1} > \tau_1$}}
 ;
\end{tikzpicture}
\caption{A binary decision tree structure containing 11 nodes: 5 test nodes
and 6 leaves nodes.}
\label{fig:dt-structure}
\end{figure}

A decision tree predicts the output of an unseen sample by following the tree
structure as described by Algorithm~\ref{algo:tree-predict}. The recursive
process starts at the root node, then the splitting rules of the testing nodes
send the sample further down the tree structure. We traverse the tree structure
until reaching a leaf, a terminal node, outputting its associated prediction
value. A decision tree model $\hat{f}: \mathcal{X} \rightarrow \mathcal{Y}$ is
then expressible as a sum of indicator functions, denoted by $1$, over the $|T|$
tree nodes:
\begin{equation}
\hat{f}(x) = \sum_{t=1}^{|T|} \beta_t 1(x \in A_t) 1(t\text{ is a leaf})
\end{equation}
\noindent where $\beta_t$ is the prediction associated to a node $N_t$.
The computational complexity of predicting an unseen sample is thus proportional
to the depth of the followed branch.

\begin{algorithm}
\caption{Predict a sample $x$ with a decision tree.}
\label{algo:tree-predict}
\begin{algorithmic}[1]
\Function{tree\_predict}{$tree$, $x$}
    \State $t \leftarrow $ Index of the root node of the tree.
    \While{the node $N_t$ is not a leaf.}
        \If{The splitting rule $s_t(x)$ of node $N_t$ is true}
            \State $t \leftarrow$ Index of the left child of node $N_t$.
        \Else
            \State $t \leftarrow$ Index of the right child node $N_t$.
        \EndIf
    \EndWhile
    \State \Return $\beta_t$.
\EndFunction
\end{algorithmic}
\end{algorithm}

\begin{remark}{Binary versus multi-way partitions}
Decision trees do not have to respect a binary tree structure. Each one of their internal nodes could have more
than two children with multi-way splitting rules. However such multi-way partitions
are equivalent to a set of consecutive binary partitions. Furthermore, multi-way
splits tends to quickly fragment the training data during the decision tree
growth impeding its generalization performance. In practice, decision trees
are therefore most of the time restricted to be binary.~\cite{friedman2001elements}
\end{remark}

\section{Growing decision trees}\label{sec:dt-growth}

We grow a decision tree using a set of samples of input-output pairs. Tree
growth starts at the root node and divides recursively the input space through
splitting rules until we reach a stopping criterion such as a maximal tree depth
or minimum sample size in a node. For each new testing node, we search for the
best splitting rule to divide the sample set at that node into two subsets. We
hope to make partitions ``purer'' at each new division. The decision tree
growing procedure has three main elements:
\begin{itemize}
\item a splitting rule search algorithm (see Section~\ref{subsec:node-split});
\item stop splitting criteria (see Section~\ref{subsec:stop-rules}) which dictate
whenever we stop the development of a branch;
\item a leaf labelling rule (see Section~\ref{subsec:leaves-labelling})
determining the output value of a terminal partition.
\end{itemize}
\noindent Putting all those key elements together leads to the decision tree
growing algorithm shown in Algorithm~\ref{algo:tree-grow}.

\begin{algorithm}
\caption{Grow a decision tree using the sample set
         $\mathcal{L}=\{(x^i,y^i)\in\mathcal{X}\times\mathcal{Y}\}_{i=1}^n$.}
\label{algo:tree-grow}
\begin{algorithmic}[1]
\Function{grow\_tree}{$\mathcal{L}$}
    \State $q = $\Call{EmptyQueue}{\null}
    \State Initialize the tree structure with the root node ($N_1$)
    \State $q$.\Call{enqueue}{$(1,\mathcal{L})$}.
    \While{$q$ is not empty \label{alg:line-stop-rules}}
        \State $(t, \mathcal{L}_t) \leftarrow q$.\Call{dequeue}{~}
        \If{Node $N_t$ satisfies one stopping criterion}
            \State Label node $t$ as a \emph{leaf} using samples $\mathcal{L}_t$
                   \label{alg-line:leaf-labelling}
        \Else
            \State Search for the best splitting rule $s_t$ using samples
                   $\mathcal{L}_t$. \label{alg-line:split-search}
            \State Split $\mathcal{L}_t$ into $\mathcal{L}_{t,r}$ and
                   $\mathcal{L}_{t,l}$ using the splitting rule $s_t$.
                   \label{alg-line:sample-partitionning}
            \State Label node $t$ as a \emph{test node} with the splitting rule $s_t$.
            \State $q$.\Call{enqueue}{($2t,\mathcal{L}_{t,l}$)}.
            \State $q$.\Call{enqueue}{($2t+1,\mathcal{L}_{t,r}$)}.
        \EndIf
    \EndWhile
    \State \Return The grown decision $tree$.
\EndFunction
\end{algorithmic}
\end{algorithm}

\subsection{Search among node splitting rules}
\label{subsec:node-split}

During tree growing, we recursively partition the input space $\mathcal{X}$ and
the sample set $\mathcal{L}=\{(x^i, y^i) \in \mathcal{X} \times
\mathcal{Y}\}_{i=1}^n$. At each testing node $t$, we split the sample set
$\mathcal{L}_t$ reaching node $t$ into two smaller subsets $\mathcal{L}_{t,l}$
and $\mathcal{L}_{t,r}$ using a binary splitting rule $s_t: \mathcal{X}
\rightarrow \{0, 1\}$ as shown in Figure~\ref{fig:node-splitting-rule}. This
raises two questions (i) what is the set of available binary and axis-wise
splitting rules $\Omega(\mathcal{L}_t)$ given a sample set $\mathcal{L}_t$ and
(ii) how to select the best one among all of them so as to make the descendants
``purer'' than the parent node.

\begin{figure}
\centering
\begin{tikzpicture}[auto, scale=0.7]

\tikzstyle{node} = [very thick, fill=white];
\tikzstyle{child} = [dashdotted, gray];
\tikzstyle{arrow} = [very thick, ->];
\tikzstyle{positive} = [fill=orange];
\tikzstyle{negative} = [fill=black];


\draw[arrow] (6,7.5) -- ({3.75}, {4.75});
\draw (3,5.5) node {$s_t(x)=1$};

\draw[arrow] (6,7.5) -- ({8.25}, {4.75});
\draw ({3 + 2*(6 - 3)}, 5.5) node {$s_t(x)=0$};

\draw[node] (6cm,7.5cm) circle (2.5cm);
\draw (5.5, 9.25) node {$\mathcal{L}_{t}$};

\foreach \i in {0,...,4}
    \foreach \j in {0,...,2}
        \draw[positive]
            ({4.25+0.75*\i},{8.375 - 0.75*\j})
            rectangle
            ({4.25+0.5+0.75*\i},{8.375 + 0.5 - 0.75*\j});

\foreach \i in {0,...,2}
    \foreach \j in {2,...,2}
        \draw[positive]
            ({4.25+0.75*\i},{8.375 - 0.75*\j})
            rectangle
            ({4.25+0.5+0.75*\i},{8.375 + 0.5 - 0.75*\j});

\foreach \i in {2,...,4}
    \foreach \j in {2,...,2}
        \draw[negative]
            ({4.25+0.75*\i},{8.375 - 0.75*\j})
            rectangle
            ({4.25+0.5+0.75*\i},{8.375 + 0.5 - 0.75*\j});

\foreach \i in {0,...,4}
    \foreach \j in {3,...,3}
        \draw[negative]
            ({4.25+0.75*\i},{8.375 - 0.75*\j})
            rectangle
            ({4.25+0.5+0.75*\i},{8.375 + 0.5 - 0.75*\j});

\draw[child,arrow] (2.5,2.5) -- ({2.5-(6-3.75)}, {2.5-(7.5-4.75)});
\draw[child,arrow] (2.5,2.5) -- ({2.5-(6-8.25)}, {2.5-(7.5- 4.75)});

\draw[child,node] (2.5cm,2.5cm) circle (2.5cm);
\draw ({2.5 - (6 - 5.5)}, {2.5 + (9.25 - 7.5)}) node {$\mathcal{L}_{t,l}$};

\foreach \i in {0,...,4}
    \foreach \j in {0,...,0}
        \draw[positive]
            ({0.75+0.75*\i},{3.375 - 0.75*\j})
            rectangle
            ({0.75+0.5+0.75*\i},{3.375 + 0.5 - 0.75*\j});

\foreach \i in {0,...,3}
    \foreach \j in {1,...,1}
        \draw[positive]
            ({0.75+0.75*\i},{3.375 - 0.75*\j})
            rectangle
            ({0.75+0.5+0.75*\i},{3.375 + 0.5 - 0.75*\j});

\foreach \i in {4,...,4}
    \foreach \j in {1,...,1}
        \draw[negative]
            ({0.75+0.75*\i},{3.375 - 0.75*\j})
            rectangle
            ({0.75+0.5+0.75*\i},{3.375 + 0.5 - 0.75*\j});

\foreach \i in {0,...,0}
    \foreach \j in {2,...,2}
        \draw[negative]
            ({0.75+0.75*\i},{3.375 - 0.75*\j})
            rectangle
            ({0.75+0.5+0.75*\i},{3.375 + 0.5 - 0.75*\j});

\draw[child,arrow] (9.5,2.5) -- ({9.5-(6-3.75)}, {2.5-(7.5-4.75)});
\draw[child,arrow] (9.5,2.5) -- ({9.5-(6-8.25)}, {2.5-(7.5- 4.75)});

\draw[child,node] (9.5cm,2.5cm) circle (2.5cm);
\draw ({9.5 - (6 - 5.5)}, {2.5 + (9.25 - 7.5)}) node {$\mathcal{L}_{t,r}$};

\foreach \i in {0,...,2}
    \foreach \j in {0,...,0}
        \draw[positive]
            ({7.75+0.75*\i},{3.375 - 0.75*\j})
            rectangle
            ({7.75+0.5+0.75*\i},{3.375 + 0.5 - 0.75*\j});

\foreach \i in {2,...,4}
    \foreach \j in {0,...,0}
        \draw[negative]
            ({7.75+0.75*\i},{3.375 - 0.75*\j})
            rectangle
            ({7.75+0.5+0.75*\i},{3.375 + 0.5 - 0.75*\j});

\foreach \i in {0,...,2}
    \foreach \j in {1,...,1}
        \draw[negative]
            ({7.75+0.75*\i},{3.375 - 0.75*\j})
            rectangle
            ({7.75+0.5+0.75*\i},{3.375 + 0.5 - 0.75*\j});

\end{tikzpicture}
\caption{During tree growing (here for a binary classification task with
orange and black classes), we search for the best splitting rule $s_t$ to divide
the sample set $\mathcal{L}_t$ reaching node $t$ into a left $\mathcal{L}_{t,l}$
and a right $\mathcal{L}_{t,r}$ subsets.}
\label{fig:node-splitting-rule}
\end{figure}
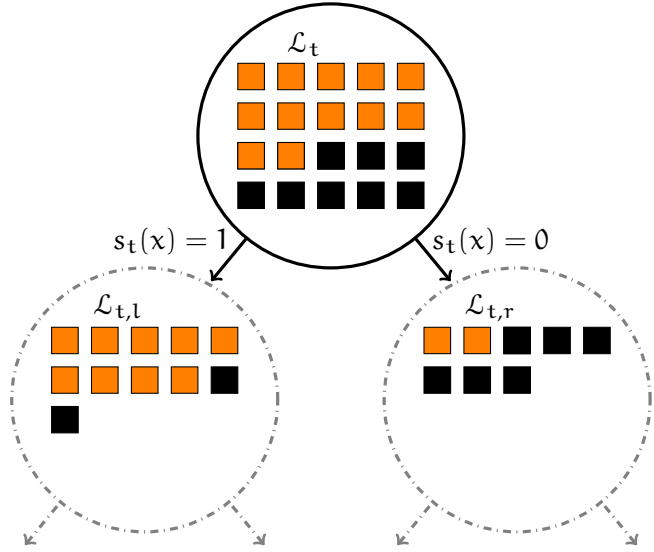

For a variable $x_j \in \mathcal{X}_{j}$ of cardinality $c_j$, the associated family  $Q(x_j)$ of
splitting rules consists of all possible subsets of
 $\mathcal{X}_j$:
\begin{equation}
Q(x_j) = \left\{s_{t}(x) \equiv 1(x_{j} \in \mathcal{X}') :  \mathcal{X}' \subset \mathcal{X}_j\right\}.
\end{equation}
\noindent The size of the splitting rule family is increasing exponentially with
the total number of possible values ($|Q(x_j)| = 2^{|\mathcal{X}_j|-1}$).

If the possible values of the variable $x_j$ are also ordered, we can reduce the
size of the splitting rule family $Q(x_j)$ from an exponential number of
candidates to a linear number of splitting rules ($|Q(x_j)| =
|\mathcal{X}_j|-1$):
\begin{align}
Q(x_j) =&\Bigg\{s_{t}(x) \equiv  1(x_j \leq \tau) : \tau \in \mathcal{X}_j \Bigg\}.
\label{eq:family-split-rules}
\end{align}

With a numerical variable $x_j \in \mathbb{R}$, the number of possible
splitting rules is infinite. However, the training set is of finite size. We
consider a family of splitting rules similar to
Equation~\ref{eq:family-split-rules} with the possible values $\tilde{\mathcal{X}}_j$ available in
the training set.

The selected splitting rule $s_t$ should split the sample set $\mathcal{L}_t$
such that the following conditions hold: the sample sets $\mathcal{L}_{t,r}$ and
$\mathcal{L}_{t,l}$ are non empty ($\mathcal{L}_{t,r} \not= \emptyset$,
$\mathcal{L}_{t,l} \not= \emptyset$) forming a disjoint ($\mathcal{L}_{t,l} \cap
\mathcal{L}_{t,r} = \emptyset$) and non overlapping partition ($\mathcal{L}_t =
\mathcal{L}_{t,l} \cup \mathcal{L}_{t,r}$) of the original sample set
$\mathcal{L}_t$. During the expansion of a test node $t$ into a left child and a
right child, we thus select a splitting rule $s_t$ among all possible splitting
rules $\Omega(\mathcal{L}_t)$:
\begin{align}
s_t \in \Omega(\mathcal{L}_t) = &
\Bigg\{s : s \in \bigcup_{j\in\{1,\ldots,p\}} Q(x_j), \nonumber \\
& \quad \mathcal{L}_{t,l} = \{(x,y)\in\mathcal{L}_{t}: s(x) = 1\}, \nonumber \\
& \quad \mathcal{L}_{t,r} = \{(x,y)\in\mathcal{L}_{t}: s(x) = 0\}, \nonumber \\
& \quad \mathcal{L}_{t,l}\not=\emptyset, \mathcal{L}_{t,r}\not=\emptyset \Bigg\}.
\end{align}

We strive to select the ``best'' possible local splitting rule $s_t$ for the
split at node $t$ leading ideally to good generalization performance. However,
it is impossible to minimize directly the generalization error. Thus instead, we
are going to minimize the resubstitution error, the error over the training set.
However, obtaining such a tree is trivial and it has poor generalization
performance. A more meaningful criterion is to search for the smallest tree
minimizing the resubstitution error. However, this optimization problem is a
NP-complete~\cite{hyafil1976constructing}. Instead, we greedily grow the tree by
maximizing the reduction of an impurity measure function $I : (\mathcal{X}
\times \mathcal{Y}) \times \ldots \times (\mathcal{X} \times \mathcal{Y})
\rightarrow \mathbb{R}$. Mathematically, we define the impurity reduction
$\Delta I$ obtained by dividing the sample set $\mathcal{L}_t$ into two
partitions $(\mathcal{L}_{t,r}, \mathcal{L}_{t,l})$ as

\begin{equation}
\Delta I(\mathcal{L}_t, \mathcal{L}_{t,l}, \mathcal{L}_{t,r}) =
I(\mathcal{L}_t)
- \frac{|\mathcal{L}_{t,l}|}{|\mathcal{L}_t|} I(\mathcal{L}_{t,l})
- \frac{|\mathcal{L}_{t,r}|}{|\mathcal{L}_t|} I(\mathcal{L}_{t,r}).
\label{eq:impurity-reduction-2}
\end{equation}

The splitting rule selection problem (line \ref{alg-line:split-search} of the
tree growing Algorithm~\ref{algo:tree-grow}) is thus written as
\begin{equation}
s_t = \arg\max_{s \in \Omega(\mathcal{L}_t)}
\Delta I(\mathcal{L}_t, \mathcal{L}_{t,l}, \mathcal{L}_{t,r})
\end{equation}

Intuitively, the impurity $I$ should be minimal whenever all samples have the
same output value. The node is then said to be ``pure''

Given the additivity of the impurity reduction, the best split at node $t$
according to the impurity reduction computed locally is also the best split at
this node in terms of global impurity. The remaining tree impurity $Imp(T)$ of a
tree $T$ is the sum of the remaining impurities of all leaf nodes:
\begin{equation}
Imp(T) = \sum_{t \in T} p(t) I(\mathcal{L}_t) 1(t \text{ is a leaf})
\end{equation}
\noindent with $p(t) = |\mathcal{L}_t| / |\mathcal{L}|$ the proportion of
samples reaching node $t$. If we develop the leaf node $t$ into a test node, it
leads to a new tree $T'$ with a new splitting rule $s_t$  having a left $t_l$
and a right $t_r$ children node. The overall impurity decreases from the
original tree $T$ to the bigger tree $T'$ and the impurity decrease is given by
\begin{align}
Imp(T) - Imp(T') &= p(t) I(\mathcal{L}_t) - p(t_l) I(\mathcal{L}_{t,l})
- p(t_r) I(\mathcal{L}_{t,r}) \\
&= p(t) \Delta I(\mathcal{L}_t,\mathcal{L}_{t,l},\mathcal{L}_{t,r})
\end{align}
\noindent Thus, the decision tree growing procedure is a repeated process aiming
at decreasing the total impurity as quickly as possible by suitably choosing the local splitting rules.

In classification, a node is pure if all samples have the same class. Given a
sample set $\mathcal{L}_t$ reaching node $t$, we will denote by
$p(y=l|t)=\frac{1}{|\mathcal{L}_t|}\sum_{(x,y)\in \mathcal{L}_t} 1(y = l)$ the
proportion of samples reaching node $t$ having the class $l$. A node will be
pure if $p(y=l|t)$ is equal to 1 for a class $l$ and zero for the others. The
node impurity should increase whenever samples with different classes are mixed
together. We require that the impurity measure $I$ in classification satisfies
the following properties:
\begin{enumerate}
\item $I$ is minimal only whenever the node is pure $p(y=l|t)=1$ and
$p(y=m|t)=0 \quad \forall m \in \{1,\ldots,l-1, l+1,\ldots, k\}$;
\item $I$ is maximal only whenever $p(y=l|t)=1/k \quad \forall l \in \{1,\ldots,k\}$;
\item $I$ is a symmetric function with respect to the class proportions
$p(y=1|t),\ldots, p(y=k|t)$ so as not to favor any class.
\end{enumerate}

A first function satisfying those three properties is the misclassification
error rate:
\begin{equation}
\text{Error rate}(\mathcal{L}_t) = 1 - \max_{l \in \{1,\ldots,k\}} p(y=l|t).
\end{equation}
\noindent However, this is not an appropriate criterion. In practice, many
candidate splitting rules have often the same error rate reduction, especially
in the multi-class classification where only the number of samples of the
majority class matters.

Consider the following split selection problem, we have a binary classification
task with 500 negative and 500 positive samples. The first splitting rule leads
to a left child with 125 positive and 375 negative samples, while the right
child has 375 positive and 125 negative samples. The misclassification error
reduction is thus given by:
\[
\frac{500}{1000}-2\frac{500}{1000}(1-\frac{375}{500})=0.25.
\]
Now, let's consider a second splitting rule leading to a pure node with 250
positive samples and another node with 250 positive and 500 negative samples.
This second split has the same impurity reduction score leading to a tie:
\[
\frac{500}{1000}-\frac{750}{1000}(1-\frac{500}{750})-\frac{250}{1000}(1-\frac{250}{250})= 0.25.
\]

The misclassification does not discriminate enough node purity as it varies
linearly with the fraction of the majority class. To solve this issue, we add a
fourth required properties to classification impurity
functions~\cite{breiman1984classification}:
\begin{enumerate}
\setcounter{enumi}{3}
\item $I$ must be a strictly concave function with respect to the class proportion
$p(y=l|t)$.
\end{enumerate}
\noindent This fourth property will increase the granularity of impurity
reduction scores leading fewer ties in splitting rule scores. It will reduce the
tree instability with respect to the training set.

Two more suitable classification impurity criteria satisfying all four
properties are the Gini index, a statistical dispersion measure, and the
entropy measure, an information theory measure.
\begin{align}
Gini(\mathcal{L}_t) &= \sum_{l=1}^k p(y=l|t) (1 - p(y=l|t) ) \\
Entropy(\mathcal{L}_t) &= - \sum_{l=1}^k p(y=l|t)  \log{p(y=l|t) }
\end{align}
\noindent By minimizing the Gini index, we minimize the class dispersion. While
selecting the splitting rule based on the entropy measure minimizes the
unpredictability of the target, the remaining unexplained information of the
output variable.

Given the strict concavity of the Gini index and entropy, we can now
discriminate the two splits of the previous example. The first split would have
an impurity reduction with the Gini index of $0.0625$ and the entropy of
$\approx0.131$. The second split with a pure node would have an impurity
reduction of $\approx0.083$ with the Gini index and of $\approx0.216$ with the
entropy. Based on these criteria, both measures would choose the second split.

In regression tasks, we consider a node as pure if the dispersion of the output
values is zero. We require the impurity regression criterion to be zero only if
all output values have the same value. A common dispersion measure used to grow
regression trees is the empirical variance
\begin{align}
Variance(\mathcal{L}_t) =
\frac{1}{|\mathcal{L}_t|}\sum_{(x,y)\in\mathcal{L}_t}
\left(y - \bar{y}\right)^2,
\text{ with } \bar{y} = \frac{1}{|\mathcal{L}_t|}\sum_{(x,y)\in\mathcal{L}_t} y
\end{align}
\noindent By maximizing the variance reduction, we are searching for a splitting
rule minimizing the square loss $\ell(y, y') = \frac{1}{2}(y - y')^2$.
Note that the Gini index and the empirical variance lead to the same impurity
measure for binary classification tasks and multi-label classification tasks
with output classes encoded with $\{0, 1\}$ numerical variables.

\subsection{Leaf labelling rules}
\label{subsec:leaves-labelling}

When tree growing is stopped by the activation of a stop splitting criterion,
the newly created leaf needs to be labeled by an output value (see
line~\ref{alg-line:leaf-labelling} of Algorithm~\ref{algo:tree-grow}). It is a
constant $\beta_t$ chosen to minimize a given loss function $\ell$ over the
samples $\mathcal{L}_t = \{(x, y) \in (\mathcal{X}, \mathcal{Y})\}$ reaching
the node $t$:
\begin{align}
\beta_{t} = \arg\min_{\beta}  \sum_{(x, y) \in \mathcal{L}_t} \ell(y, \beta).
\end{align}

In regression tasks, we want to find the constant $\beta_t$ minimizing
the square loss in single output regression :
\begin{align}
\beta_{t} = \arg\min_{\beta} \sum_{(x, y) \in \mathcal{L}_t} \frac{1}{2}(y - \beta_t)^2
\end{align}
\noindent By setting the first derivative to zero, we have
\begin{align}
\sum_{(x, y) \in \mathcal{L}_t} (y - \beta) = 0 \\
\beta_t = \frac{1}{|\mathcal{L}|_t} \sum_{(x, y) \in \mathcal{L}_t} y.
\end{align}
\noindent The constant leaf model minimizing the square loss is
the average output value of the samples reaching node $t$.

In classification tasks, the constant $\beta_t$ minimizing the
$0-1$ loss ($\ell_{0-1}(y, y')) = 1(y\not=y')$) is
the most frequent class. The decision tree can also be a class probability
estimator by outputting the proportion $p(l|t)$ of samples of class $l$
reaching node $t$ from the sample set $\mathcal{L}_t$.

\begin{remark}{Beyond constant leaf modeling}
The leaf labelling rule can go beyond a constant model with supervised learning
models such as linear
models~\cite{quinlan1992learning,wang1996induction,frank1998using,landwehr2005logistic},
kernel-based methods~\cite{torgo1997functional}, probabilistic
models~\cite{kohavi1996scaling} or even tree-ensemble
models~\cite{matthew2015bayesian}. It increases the modeling power of the
decision tree at the expense of computing time and new hyper-parameters.
\end{remark}

\subsection{Stop splitting criterion}
\label{subsec:stop-rules}

The tree growth at a node $t$ naturally stops if all the samples $\mathcal{L}_t$
reaching the node $t$ (i) share the same output value (zero impurity) or (ii)
share the same input values (but not necessarily the same output) as in this
case we can not find a valid split of the data. In both cases, we can not find a
splitting rule to grow the tree further due to a lack of data.

A tree developed in such ways is then said to be fully developed. The question
is ``Should we stop sooner the tree growth?''. A testing node $t$ splits the
data $\mathcal{L}_t$ into two partitions $(\mathcal{L}_{t,l},
\mathcal{L}_{t,r})$ leading to a left child node $t_l$ and a right child node
$t_r$. If we denote by $\hat{f}_t$, $\hat{f}_{t,r}$ and $\hat{f}_{t,l}$ the leaf
models that would be assigned to the nodes $t$, $t_{r}$ or $t_{l}$, we have that
the resubstition error reduction $\Delta{}Err$ associated to a loss function
$\ell$ is given by:
\begin{align}
\Delta{}Err&=
\sum_{(x,y)\in\mathcal{L}_t} \ell(y, \hat{f}_t(x)) -
\sum_{(x,y)\in\mathcal{L}_{t,r}} \ell(y, \hat{f}_{t,r}(x)) \nonumber\\
&\quad - \sum_{(x,y)\in\mathcal{L}_{t,l}} \ell(y, \hat{f}_{t,l}(x))  \\
&=\sum_{(x,y)\in\mathcal{L}_{t,r}} \left[
\ell(y, \hat{f}_t(x)) - \ell(y, \hat{f}_{t,r}(x))
\right] \nonumber\\
&\quad + \sum_{(x,y)\in\mathcal{L}_{t,l}} \left[
\ell(y, \hat{f}_t(x)) - \ell(y, \hat{f}_{t,l}(x))
\right]
\end{align}
\noindent Since we choose $\hat{f}_{t,r}$ and $\hat{f}_{t,l}$ so as to minimize
the resubstitution error on their respective training data ($\mathcal{L}_{t,l}$
and $\mathcal{L}_{t,r}$), the resubstitution error never increases through node
splitting. With only ``natural'' splitting rule, decision trees are optimally
fitting the training data.

Stop splitting criteria avoid over-fitting by stopping earlier the tree growth.
They are either based on (i) structural properties or on (ii) data statistics.
Criteria computed on the left and the right children can also discard splitting
rules, for example requiring a minimal number of samples in the left and right
children to split a node.

Structural-based stop splitting criteria regularize the tree growth by explicitly
limiting the tree complexity, by restricting for example:
\begin{itemize}
\item branch depths or
\item the total number of nodes.
\end{itemize}
\noindent In the second case, the order in which the tree nodes are split starts
to matter and can be chosen so as to maximize the total impurity reduction.

Data-based stop splitting criteria stop the tree growth if some statical
properties computed on the data used to split the node are below a threshold
such as
\begin{itemize}
\item the number of samples reaching the node,
\item the number of samples reaching the left and right children obtained after
splitting,
\item the impurity reduction or
\item the p-value of a significance test, such as a Chi-square test, testing
the independence of the split and the output variable.
\end{itemize}

\section{Right decision tree size}
\label{sec:dt-pruning}

To find the right decision tree size, there are two main families of complexity
reduction techniques, also called pruning techniques: (i) pre-pruning techniques
stop the tree growth before the tree is fully developed
(line~\ref{alg:line-stop-rules} of Algorithm~\ref{algo:tree-grow} and presented
in Section~\ref{subsec:stop-rules}) and (ii) post-pruning techniques remove tree
nodes a posteriori setting a trade-off between the tree size and the
resusbstitution error. Both approaches lead to smaller decision trees aiming to
improve generalization performance and to simplify decision tree interpretation.

Pre-pruning criteria are straightforward tools to control the decision tree
size. However, it is unclear which pruning level (or hyper-parameter values)
leads to the best generalization performance. Too ``strict'' stop splitting
criteria will grow shallow trees under-fitting the data. While too ``loose''
stop splitting criteria have the opposite effect, i.e., growing overly complex
trees over-fitting the data.

While pre-pruning techniques select the tree complexity a priori, post-pruning
techniques select the optimal complexity a posteriori. A naive approach to
post-pruning would be to build independently a sequence of decision trees with
different complexity by varying the stop splitting criteria, and then to select
the one minimizing an approximation of the generalization error such as the
error on a hold out sample set. However, this is not computationally efficient
as it re-grows each time a (new) decision tree.

Post-pruning techniques first grow a single decision tree $T$ with very loose or
no stop splitting criterion. This decision tree clearly overfits the training
data. Then, they select a posteriori a subtree $T^* \subseteq T$ among all
possible subtrees of $T$. The original decision tree is thus pruned by
collapsing nodes from the original tree into new leaf nodes. The post-pruning
method minimizes a tradeoff between the decision tree error over a sample set
$\mathcal{S}$ and a function measuring the decision tree complexity such as the
number of nodes:
\begin{equation}
T^*(\lambda) = \arg\min_{\check{T} \subset T}
\text{Error}(\mathcal{\mathcal{S}}|\check{T}) + \lambda \text{Complexity}(\check{T}).
\label{eq:tree-post-pruning}
\end{equation}

The cost complexity pruning method~\cite{breiman1984classification}, also known
as the weakest link pruning, implements Equation~\ref{eq:tree-post-pruning}
through a complexity coefficient $C_\alpha$ measuring a tradeoff between the
resubstitution error of a tree $\check{T}$ and its complexity $|\check{T}|$
defined by the number of leaves:
\begin{equation}
C_{\alpha}(\check{T}) = \text{resubstitution error}(\mathcal{L}|\check{T}) +
\alpha |\check{T}|.
\end{equation}
\noindent For each $\alpha$, there exists a unique tree $\check{T}_\alpha$
minimizing the cost complexity coefficient $C_\alpha$. Large values of $\alpha$
lead to small trees, while conversely small values of $\alpha$ allow bigger
sub-trees. For the extreme case $\alpha=0$ (resp. $\alpha=\infty$), we have the
original decision tree $T_0$ (resp. the subtree containing only the root node).
One can show~\cite{breiman1984classification} that we can sequentially obtain
the $\check{T}_\alpha$ from the original tree $T$ by removing the node with the
smallest increase in resubtitution error. We select the optimal subtree $T^*$
among all subtrees $\check{T}_\alpha \subseteq T$ by minimizing an approximation
of the generalization error using for instance cross-validation methods.

The reduced error pruning method~\cite{quinlan1987simplifying}, another post
pruning technique, splits the learning set into a training set and a
pruning set. It grows on the training set a large decision tree. During
the pruning phase, it first computes the error reduction of pruning each node and
its descendants on the pruning set, then removes greedily the node reducing
the most the error. It repeats this two steps procedure until the error on the
pruning set starts increasing.

Other post pruning methods have been developed with pruning criteria based on
statistical procedures~\cite{quinlan1987simplifying,quinlan1993c4.5} or on cost
complexity criteria based on information
theory~\cite{quinlan1989inferring,mehta1995mdl,wallace1993coding}. Instead of
relying on greedy processes,
authors~\cite{bohanec1994trading,almuallim1996efficient} have proposed dynamic
programming algorithms to find an optimal sub-tree sequence minimizing the
resubtistution error with increasing tree complexity at the expense of
computational complexity.

\section{Decision tree interpretation}\label{sec:dt-interpret}

A strength of the decision tree model is its interpretability. A closer
inspection reveals that we can convert a decision tree model to a set of
mutually exclusive classification or regression rules.  We get these rules by
following the path from each leaf to the root node. We have converted the
decision tree shown in Figure~\ref{fig:dt-iris-chap3} to three sets of
predicting rules, one for each class of iris flower (Versicolor, Virginica and
Setosa):
\begin{enumerate}
\item ``If Petal width $\leq$ 0.7cm, then Setosa''
\item ``If Petal width > 0.7cm and Petal width $\leq$ 1.65cm and Petal length
$\leq$ 5.25cm, then Versicolor.''
\item ``If Petal width > 0.7cm and Petal width $>$ 1.65cm and Sepal length
$\leq$ 5.95cm and Sepal length $>$ 5.85cm, then Versicolor.''
\item ``If Petal width > 0.7cm and Petal width $\leq$ 1.65cm and Petal length
$>$ 5.25cm, then Virginica.''
\item ``If Petal width > 0.7cm and Petal width $>$ 1.65cm and Sepal length
$\leq$ 5.95cm and Sepal length $\leq$ 5.85, then Virginica.''
\item ``If Petal width > 0.7cm and Petal width $>$ 1.65cm and Sepal length
$>$ 5.95cm, then Virginica.'' \label{en:redundant-cond}
\end{enumerate}
\noindent Remark that given the binary hierarchical structure, some rules are
redundant and can be further simplified. For instance, we can collapse the
constraints ``Petal width > 0.7cm and Petal width $>$ 1.65cm'' into ``Petal
width $>$ 1.65cm'' for the \ref{en:redundant-cond}-th rule.

\begin{figure}[h]
\centering
\begin{tikzpicture}
[level 1/.style={sibling distance=60.0mm},
 level 2/.style={sibling distance=60.0mm},
 level 3/.style={sibling distance=30.0mm},
 level 4/.style={sibling distance=30.0mm},
 leaf/.style={rectangle},
 test/.style={rectangle},
 edge from parent/.style = {draw, -latex, font=\footnotesize},
 level distance=20mm]
\node[test] {Petal width?}
 child {node[leaf] {Setosa} edge from parent node[left] {$0.7cm\leq\,$}}
 child {node[test] {Petal width?}
     child {node[test] {Petal length?}
         child {node[leaf] {Versicolor} edge from parent node[left] {$5.25cm\leq\,$}}
         child {node[leaf] {Virginica} edge from parent node[right] {$\,>5.25cm$}}
         edge from parent node[left] {$1.65cm\leq\,$}}
     child {node[test] {Sepal length?}
         child {node[test] {Sepal length?}
             child {node[leaf] {Virginica} edge from parent node[left] {$5.85cm\leq\,$}}
             child {node[leaf] {Versicolor} edge from parent node[right] {$\,>5.85cm$}}
             edge from parent node[left] {$5.95cm\leq\,$}}
         child {node[leaf] {Virginica} edge from parent node[right] {$\,>5.95cm$}}
         edge from parent node[right] {$\,>1.65cm$}}
     edge from parent node[right] {$\,>0.7cm$}}
 ;
\end{tikzpicture}
\caption{A decision tree is an interpretable set of rules organized
hierarchically. For instance, we can recognize a Versicolor iris with two sets
of rules. Among the four input  variables (petal width, petal length, sepal
width and sepal length), the decision tree shows that only one input
variable is necessary to classify the Setosa iris variety.}
\label{fig:dt-iris-chap3}
\end{figure}
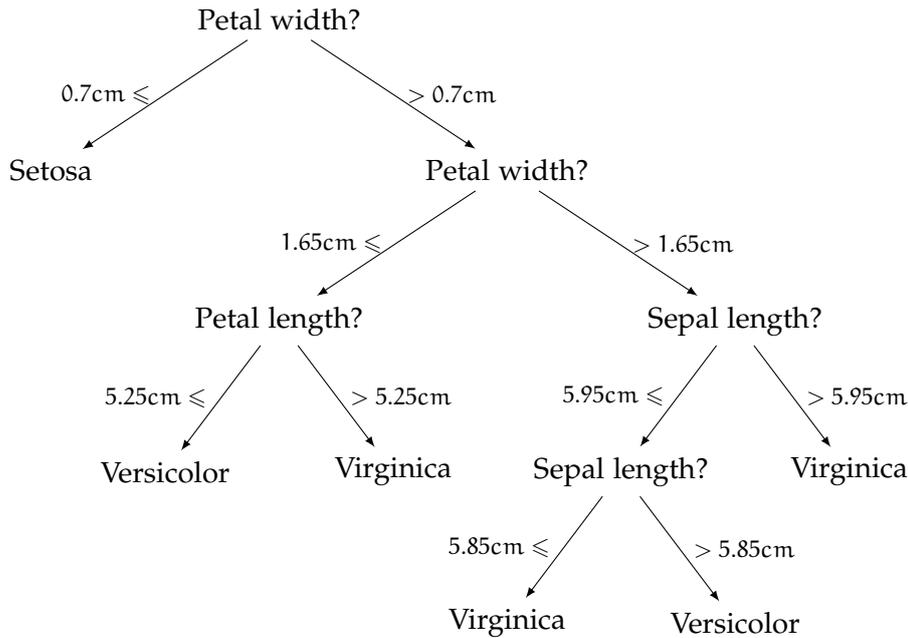

The decision tree model also shows which input variables $x_1, \ldots,
x_p$ are important to predict the output variable(s) $y$.
During the decision tree growth, we select at each node $t$ an axis-wise
splitting rules $s_t$  minimizing the reduction of an impurity measure
$\Delta{}I$ dividing the samples $\mathcal{L}_t$ reaching the node into two
subsets $(\mathcal{L}_{t,l}, \mathcal{L}_{t,r})$. The mean decrease of impurity
(MDI) of a variable $x_j$~\cite{breiman1984classification} sums, over the nodes
of a decision tree $T$ where $x_j$ is used to split, the total reduction of
impurity associated to the split weighted by the probability of reaching that
node over the training set $\mathcal{L}$:
\begin{equation}
\text{MDI}(x_j) =
\sum_{\left\{t \in T: x_j\text{ is tested by }s_t\right\}}
\frac{|\mathcal{L}_t|}{|\mathcal{L}|}
\Delta{}I(\mathcal{L}_t, \mathcal{L}_{t,l}, \mathcal{L}_{t,r}).
\end{equation}

The mean decrease of impurity scores and ranks the input variables according to
their importances during the decision tree growth process as illustrated in
Figure~\ref{fig:cost-complexity}. It takes into account variable correlations,
multivariate and non linear effects. As such, decision trees are often used as
pre-processing tools to select a fraction of the top most important variables.

\begin{figure}
\centering
\includegraphics[width=0.9\textwidth]{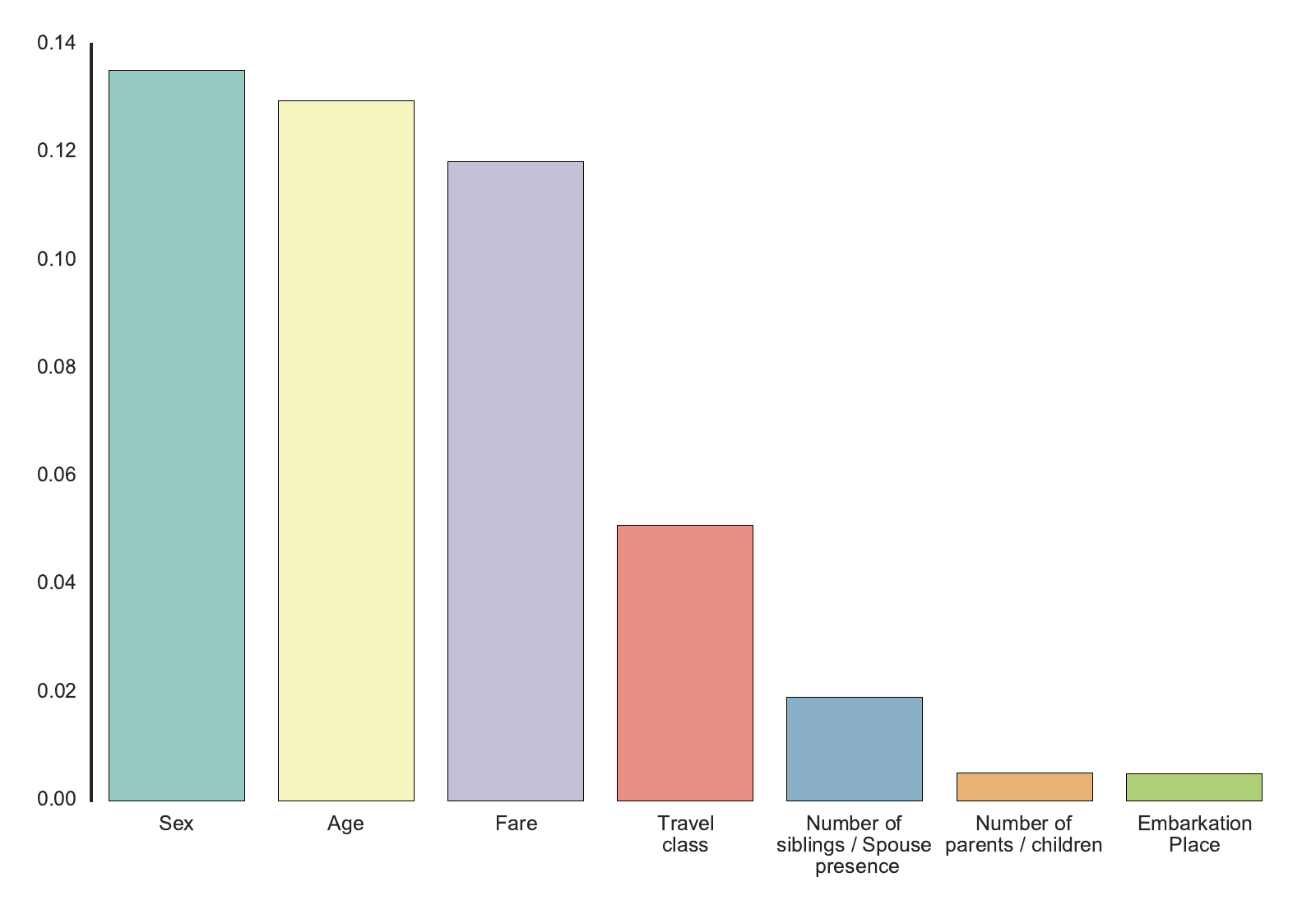}
\caption{Mean impurity decrease for each input variable of the Titanic
dataset for a decision tree whose complexity was chosen to maximize the accuracy
score on a hold out pruning sample set.}
\label{fig:cost-complexity}
\end{figure}

\section{Multi-output decision trees}
\label{sec:mo-trees}

Decision trees naturally extend from single output tasks to multiple output
tasks~\cite{segal1992tree,de2002multivariate,blockeel2000top,clare2001knowledge,zhang1998classification,vens2008decision,noh2004unbiased,siciliano2000multivariate}
such as multi-output regression, multi-label classification or multi-class
classification. No core modification are needed. Instead, we need appropriate
impurity measures and leaf labelling rules for the tree prediction
Algorithm~\ref{algo:tree-predict} and the tree growth
Algorithm~\ref{algo:tree-grow}. Note that a multi-output decision tree can
still be pruned~\cite{struyf2005constraint}.

\paragraph{Multi-output impurity measures} During the decision tree growth, we
aim to select a splitting rule dividing the sample set $\mathcal{L}_t$ reaching
the node $t$ into a left and a right sample sets $(\mathcal{L}_{t,l},
\mathcal{L}_{t,r})$. The best multi-output splitting rule is the one maximizing
the reduction of a multi-output impurity measure $I$. We can use native
multi-output impurity measures such as the variance in
regression~\cite{segal1992tree}:
\begin{align}
Variance(\mathcal{L}_t) =
\frac{1}{|\mathcal{L}_t|}\sum_{(x,y)\in\mathcal{L}_t}
\left|y - \bar{y}\right|^2_2
\text{ with } \bar{y} = \frac{1}{|\mathcal{L}_t|}\sum_{(x,y)\in\mathcal{L}_t} y.
\end{align}
\noindent or any impurity criterion derived from an appropriate distance
measure~\cite{blockeel2000top}.

We can also extend known impurity measures, such as the Gini index or the
entropy (see Section~\ref{subsec:node-split}), by summing the impurity measures
over each output~\cite{de2002multivariate}:
\begin{equation}
I_{mo}(\mathcal{S}) = \sum_{j=1}^d I(\{(x, y_j) \in \mathcal{S}\}),
\end{equation}
\noindent where $S = \{(x^i,y^i) \in (\mathcal{X} \times \mathcal{Y})\}_{i=1}^n$
is a sample set.

Since we can define an impurity measure on any set of outputs, we can derive the
mean decrease of impurity MDI (see Section~\ref{sec:dt-interpret}) either on all
or a subset of the outputs.

\paragraph{Leaf labelling and prediction rule} In the multi-output context, the
leaf prediction $\beta_t$ of a node $t$ is a constant vector of output values
chosen so as to minimize a multi-output loss function $\ell$ over the samples
$\mathcal{L}_t = \{(x^i, y^i) \in (\mathcal{X}, \mathcal{Y})\}_{i=1}^n$ reaching
the node:
\begin{align}
\beta_{t} = \arg\min_{\beta}  \sum_{(x, y) \in \mathcal{L}_t} \ell(y, \beta).
\end{align}

In multi-output regression, the loss $\ell$  is commonly the $\ell_2$-norm loss
$\ell_2(y, y') = \frac{1}{2} ||y - y'||_{2}^2$ the multi-output
extension of the square loss. The constant $\beta_t$ minimizing the
$\ell_2$-norm loss is the average output vector
\begin{equation}
\beta_t = \frac{1}{|\mathcal{L}_t|}\sum_{(x, y) \in \mathcal{L}_t} y.
\end{equation}

Whenever we extend the label rule assignment to multi-label and to multi-output
classification tasks, there are two common possibilities either minimizing the
subset $0-1$ loss which is equal to zero if only if all outputs are correctly
predicted $\ell_{\text{subset }0-1}(y, y') = 1(y\not=y')$ and
the Hamming loss which counts among the $d$ outputs the number of wrongly
predicted outputs $\ell_{\text{Hamming}}(y, y') = \sum_{j=1}^d 1(y_j
\not=y'_j)$.

Minimizing the subset $0-1$ loss takes into account output correlations. The
constant vector $\beta_t$ minimizing this loss is the most frequent output value
combination. Note that the constant $\beta_t$ might not be unique as we might
have several output combinations with the same frequency of appearance in the
sample set reaching the leaf.

The Hamming loss makes the assumption that all outputs are independent. The
constant $\beta_t$ minimizing this loss is the vector containing the most
frequent class of each output.

When the trees are fully developed with only pure leaves, minimizing the
Hamming loss or the subset $0-1$ loss leads to identical leaf prediction rules.


\chapter{Bias-variance and ensemble methods}\label{ch:bias-var-ensemble}

\begin{remark}{Outline}
Ensemble methods combine supervised learning models together so as to improve
generalization performance. We present two families of ensemble methods:
averaging methods and boosting methods. Averaging ensembles grow independent
unstable estimators and average their predictions. Boosting methods increase
sequentially their total complexity by adding biased and stable estimators. In
this chapter, we first show how to decompose the generalization error of
supervised learning estimators into their bias, variance and irreducible error
components. Then we show how to exploit averaging techniques to reduce variance
and boosting techniques to sequentially decrease bias.
\end{remark}

Ensemble methods fit several supervised learning models instead of a single one
and combine their predictions. The goal is to reduce the generalization error
by solving the same supervised learning task multiple times. We hope that the
errors made by the different models will compensate and thereby improve the overall accuracy
whenever we consider them together.

Real life examples of ``ensemble methods'' in the human society are democratic
elections. Each eligible person is asked to cast its vote for instance to choose
between political candidates. This approach considers each person of the
committee as an independent expert and averages simultaneously their opinions.
In supervised learning, these kinds of voting mechanism are called ``averaging
methods''.

Instead of querying all experts independently, we can instead collect their
opinions sequentially. We ask to each new expert to refine the predictions made
by the previous ones. The expert sequence is chosen so that each element of the
sequence improves the accuracy focusing on the unexplained phenomena. For
instance in medical diagnosis, a person itself is the first one to assess its
health status. The next expert in the line is the general practitioner followed
by a series of specialists. We call these ensemble methods ``boosting methods''.

The ``averaging'' approach aims to reduce the variability in the expert pool by
averaging their predictions. At the other end, the ``boosting'' approach
carefully refines its predictions by cumulating the predictions of each expert.

In Section~\ref{sec:bias-variance}, we show how to decompose the error made by
supervised learning models into three terms: a variance term due to the
variability of the model with respect the learning sample set, a bias term due
to a lack of modeling power of the estimator and an irreducible error term due
to the nature of the supervised learning task. Averaging methods presented in
Section~\ref{sec:averaging} are variance reducing techniques growing
independently supervised learning estimators. In Section~\ref{sec:boosting}, we
show how to learn sequentially a series of estimators through boosting methods
increasing the overall ensemble complexity and reducing the ensemble model bias.

\section{Bias-variance error decomposition}
\label{sec:bias-variance}

The expected error or generalization error $Err$ associated to a loss
$\ell:\mathcal{Y} \times \mathcal{Y} \rightarrow \mathbb{R}^+$ of a supervised
learning algorithm is a random variable depending on the learning samples
$\mathcal{L} = \{(x^i,y^i) \in \mathcal{X} \times \mathcal{Y}\}_{i=1}^n$ drawn
independently and identically from a distribution $P_{{\cal X},{\cal Y}}$ and
used to fit a model $f_\mathcal{L}:\mathcal{X}\rightarrow\mathcal{Y}$ in a
hypothesis space $\mathcal{H}\subset \mathcal{Y}^\mathcal{X}$. We want here to
analyze the expectation of the generalization error $Err$ over the distribution
of learning samples defined as follows:
\begin{align}
\E_{\mathcal{L}}(Err) &
= \E_{\mathcal{L}}\E_{P_{\mathcal{X},\mathcal{Y}}}\left\{\ell(f_\mathcal{L}(x), y)\right\}\\
&= \E_{\mathcal{L}}\E_{P_{\mathcal{X}}}
        \E_{P_{\mathcal{Y}|\mathcal{X}}}\left\{\ell(f_\mathcal{L}(x), y)\right\}
\end{align}

Let us denote the Bayes model by $f_\text{Bayes}(x) \in
\mathcal{Y}^\mathcal{X}$, the best model possible minimizing the generalization
error:
\begin{equation}
f_\text{Bayes}(x) = \arg\min_{f \in \mathcal{Y}^\mathcal{X}} \E_{P_{\mathcal{X},\mathcal{Y}}} \left\{\ell(f(x), y)\right\}.
\end{equation}

For the squared loss $\ell(y,y')=\frac{1}{2}(y - y')^2$, we can
decompose the expected error over the learning set $\E_{\mathcal{L}}(Err)$ into
three terms (see~\cite{geman1992neural} for the proof):
\begin{align}
\E_{\mathcal{L}}(Err)
& \triangleq \E_{\mathcal{L}}\E_{P_{\mathcal{X},\mathcal{Y}}}\left\{(f_\mathcal{L}(x) - y)^2\right\} \\
& = \E_{P_{\mathcal{X}}}\left\{
\Var_\mathcal{L}\left\{f_\mathcal{L}(x)\right\} +
\Bias^2(f_\mathcal{L}(x)) +
\Var_{P_{\mathcal{Y}|\mathcal{X}}}\left\{y\right\}\right\},
\label{eq:bias-var-decom}
\end{align}
\noindent where
\begin{align}
\Var_\mathcal{L}\left\{f_\mathcal{L}(x)\right\} &= \E_\mathcal{L}\left\{
\left(f_\mathcal{L}(x) - f_\text{avg}(x)\right)^2
\right\}, \\
\Bias^2(f_\mathcal{L}(x)) &= \left(\E_{\mathcal{L}}\left\{f_\mathcal{L}(x)\right\}(x) - f_\text{Bayes}(x)\right)^2.
\end{align}

We interpret each term of Equation~\ref{eq:bias-var-decom} as follows:
\begin{itemize}

\item The variance of a supervised learning algorithm
$\E_{P_{\mathcal{X}}}\Var_\mathcal{L}\left\{f_\mathcal{L}(x)\right\}$ describes
the variability of the model with a randomly drawn learning sample set
$\mathcal{L}$. Supervised learning algorithms with a high variance have often a
high complexity which makes them overfit the data.

\item  The square bias
$\E_{P_{\mathcal{X}}}\left\{\Bias^2(f_\mathcal{L}(x))\right\}$ is the distance
between the average model and the Bayes model. Biased models are not complex
enough to model the input-output function. They are indeed underfitting the
data.

\item The irreducible error
$\E_{P_{\mathcal{X}}}\Var_{P_{\mathcal{Y}|\mathcal{X}}}\left\{y\right\}$ is the
variance of the target around its true mean. It is the minimal attainable error
on a supervised learning problem.

\end{itemize}

The bias-variance decomposition allows to analyze and to interpret the effect of
hyper-parameters. It highlights and gives insights on their effects on the bias
and the variance. In Figure~\ref{fig:error-decomp}, we have fitted decision tree
models with an increasing number of leaves on the
Friedman1 dataset~\cite{friedman1991multivariate}, a simulated regression
dataset. We first assess the resubstitution error over 300 samples and an
approximation of the generalization error, the hold out error, computed on an
independent testing set of 20000 samples. We compute these errors (see
Figure~\ref{subfig:ch4-under-over}) by averaging the performance of decision
tree models over 100 learning sets $\mathcal{L}$ drawn from the same
distribution $P_{\mathcal{X}, \mathcal{Y}}$. By increasing the number of leaves,
the resubstitution error decreases up to zero with fully developed trees. On
the other hand, the hold out error starts increasing beyond 20 leaves indicating
that the model is under-fitting with less than 20 leaves and over-fitting with
more than 20 leaves. By increasing the number of leaves, we decrease the bias as
we grow more complex models as shown in Figure~\ref{subfig:ch4-bias-var}. It
also increases the variance as the tree structures become more unstable with the
learning set $\mathcal{L}$ .

\begin{figure}
\centering
\subfloat[]{{\includegraphics[width=0.5\textwidth]{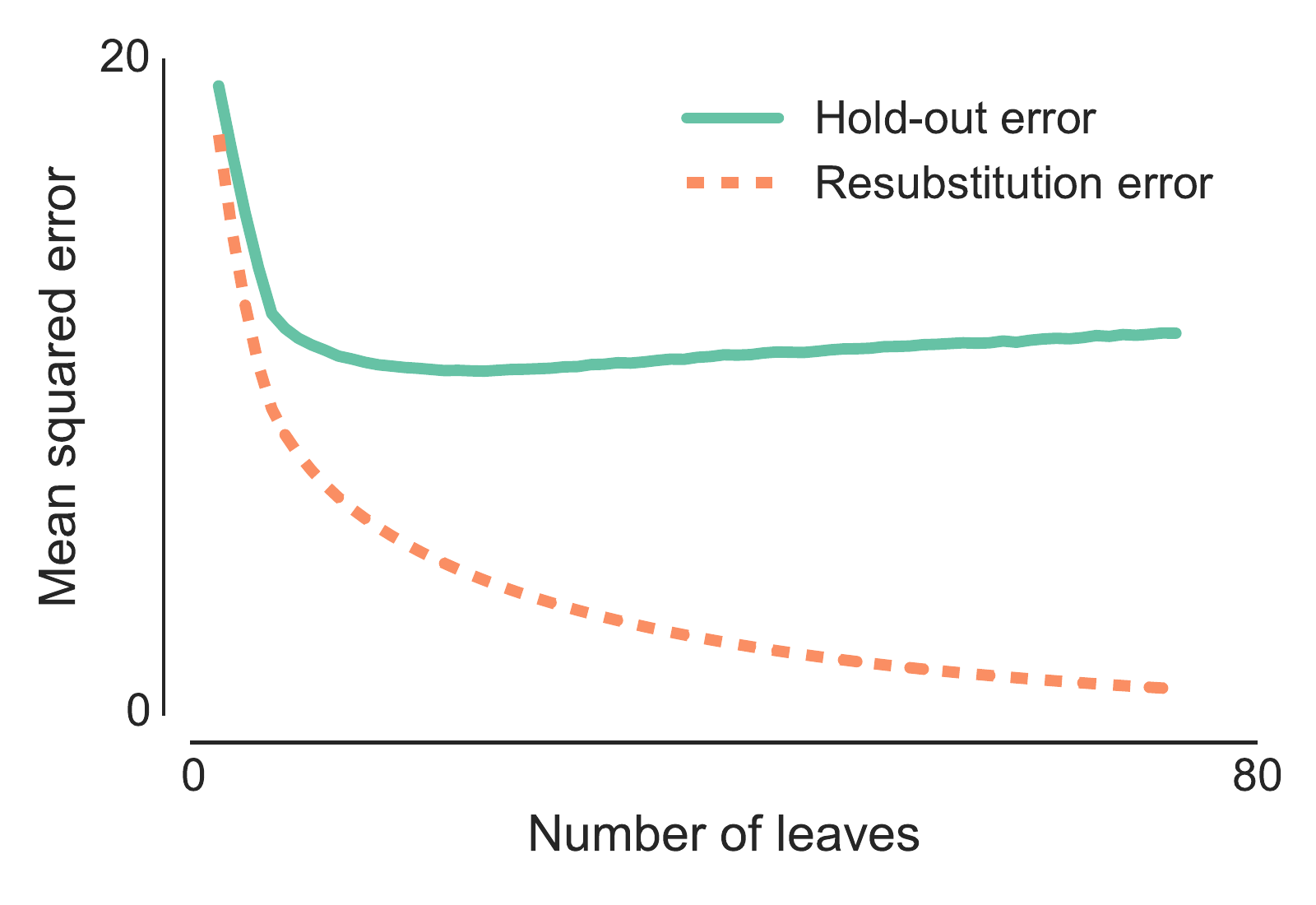}
\label{subfig:ch4-under-over}}}
\subfloat[]{{\includegraphics[width=0.5\textwidth]{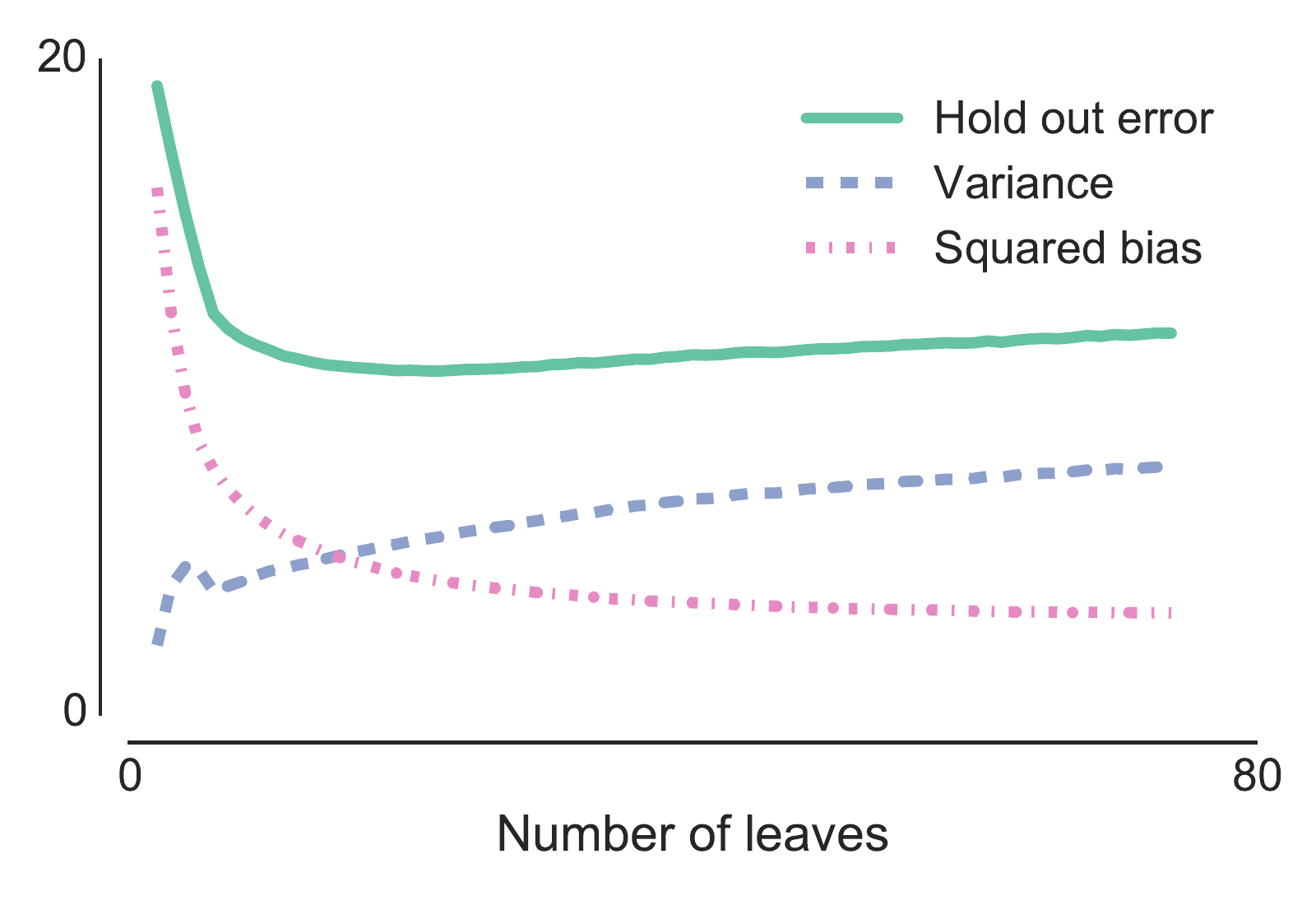}
\label{subfig:ch4-bias-var}}}
\caption{Bias-variance decomposition of decision tree models of increasing
complexity on the Friedman1 dataset.}
\label{fig:error-decomp}
\end{figure}

In general, we have the following trends for a single decision tree model. Large
decision trees overfit and are unstable with respect to the learning set
$\mathcal{L}$ which corresponds to a high variance and a small bias. Shallow
decision trees, on the other hand, underfit the learning set $\mathcal{L}$
and have stable structures, which corresponds to a small variance and a high
bias. The pruning technique presented in Section~\ref{sec:dt-pruning} allows to
select a tradeoff between the variance and the bias of the algorithm by
adjusting the tree complexity.

The bias-variance decomposition of the square loss is by far the most studied
decomposition, but there nevertheless exist similar decompositions for other
losses, e.g. see~\cite{domingos2000unified} for a decomposition of the
polynomial loss $\ell(y, y') = |y - y'|^p$,
see~\cite{friedman1997bias,kohavi1996bias,tibshirani1996bias,domingos2000unified}
for the $0-1$ loss, or see~\cite{james2003variance} for losses in general.

\section{Averaging ensembles}\label{sec:averaging}

An averaging ensemble model $f_{\theta_1,\ldots,\theta_M}$ builds a set of $M$
supervised learning models $\{f_{\theta_m}\}_{\theta_m=1}^M$, instead of a
single one. Each model $f_{\theta_m}$ of the ensemble is different as we
randomize and perturb the original supervised learning algorithm at fitting
time. We describe entirely the induced randomization of one model $f_{\theta_m}$
by drawing i.i.d. a random vector of parameters $\theta_m$ from a distribution
of model parameters $P_\theta$.

In regression, the averaging ensemble predicts an unseen sample by averaging the
predictions of each model of the ensemble:
\begin{equation}
f_{\theta_1,\ldots,\theta_M}(x) = \sum_{m=1}^M  f_{\theta_m}(x).
\end{equation}
\noindent It minimizes the square loss (or its extension the $\ell_2$-norm loss)
between the ensemble model and its members:
\begin{equation}
f_{\theta_1,\ldots,\theta_M}(x) = \arg\min_{y \in \mathcal{Y}} \sum_{m=1}^M (y - f_{\theta_m}(x))^2.
\end{equation}

In classification, the averaging ensemble combines the predictions of
its members to minimize the 0-1 loss by a majority vote of all its members:
\begin{equation}
f_{\theta_1,\ldots,\theta_M}(x) =
\arg\min_{c \in \mathcal{Y}} \sum_{m=1}^M 1(f_{\theta_m}(x) \neq c).
\end{equation}

An alternative approach, called soft voting, is to classify according to the average of the probability
estimates $\hat{P}_{f_{\theta_m}(x)}$ provided by the ensemble members:
\begin{equation}
f_{\theta_1,\ldots,\theta_M}(x) = \arg\max_{c \in \mathcal{Y}}
\sum_{m=1}^M \hat{P}_{f_{\theta_m}(x)}(Y=c).
\end{equation}

Both approaches have been studied and yield almost exactly the same result, but
soft voting provides smoother probability class estimates than majority
vote~\cite{breiman1996bagging,zhou2012ensemble}. The multi-output extension to
ensemble predictions often minimizes the Hamming loss applying either soft-voting
or majority voting to each output independently. Minimizing the subset $0-1$
loss for multi-label tasks would lead to predict the most frequent label set.

\begin{remark}{Ambiguity decomposition}
The ambiguity decomposition~\cite{krogh1995neural} of the square loss shows that
the generalization error of an ensemble $f_{\theta_1,\ldots,\theta_M}$ of $M$
models $\{f_{\theta_m}\}_{m=1}^M$ is always lower or equal than the average
generalization error $\bar{E}$ of its members:
\begin{align}
\E_{P_{\mathcal{X}}}\E_{P_{\mathcal{Y}|\mathcal{X}}}\left\{(y - f_{\theta_1,\ldots,\theta_M}(x))^2\right\}=
\bar{E}-\bar{A}\leq\bar{E}
\label{eq:ambiguity-decomp}
\end{align}
\noindent with
\begin{align}
\bar{E}&=\frac{1}{M}\sum_{m=1}^M
\E_{P_{\mathcal{X}}}\E_{P_{\mathcal{Y}|\mathcal{X}}}\left\{(y - f_{\theta_m}(x))^2\right\} \\
\bar{A}&=\frac{1}{M}\sum_{m=1}^M
\E_{P_{\mathcal{X}}}\left\{(f_{\theta_m}(x) - f_{\theta_1,\ldots,\theta_M}(x))^2\right\}
\end{align}
\noindent The ambiguity term $\bar{A}$ is the variance of the ensemble around
its average model $f_{\theta_1,\ldots,\theta_M}$. The equality occurs only if
all average models are identical $f_{\theta_1,\ldots,\theta_M} = f_{\theta_m}
\quad \forall m \in \{1,\ldots,M\}$.
\end{remark}

Averaging ensemble models are obtained by first perturbing supervised learning
models and then combining them. They aim to reduce the generalization error of
the ensemble compared to the original single model by reducing the variance of
the learning algorithm. Let us illustrate the effects of an averaging method
called bagging on the bias and variance of fully grown decision trees. The
bagging method fits each estimator of the ensemble on a bootstrap copy of the
learning set. In Figure~\ref{fig:error-decomp-bagging}, we show the variance and
the bias as function of the number of fully grown decision trees in the bagging
ensemble. With a single decision tree, the variance is the dominating error
component. By increasing the size of the Bagging ensemble, the variance and the
hold out error are reduced, while leaving the bias mostly unchanged.

\begin{figure}
\centering
\includegraphics[width=0.75\textwidth]{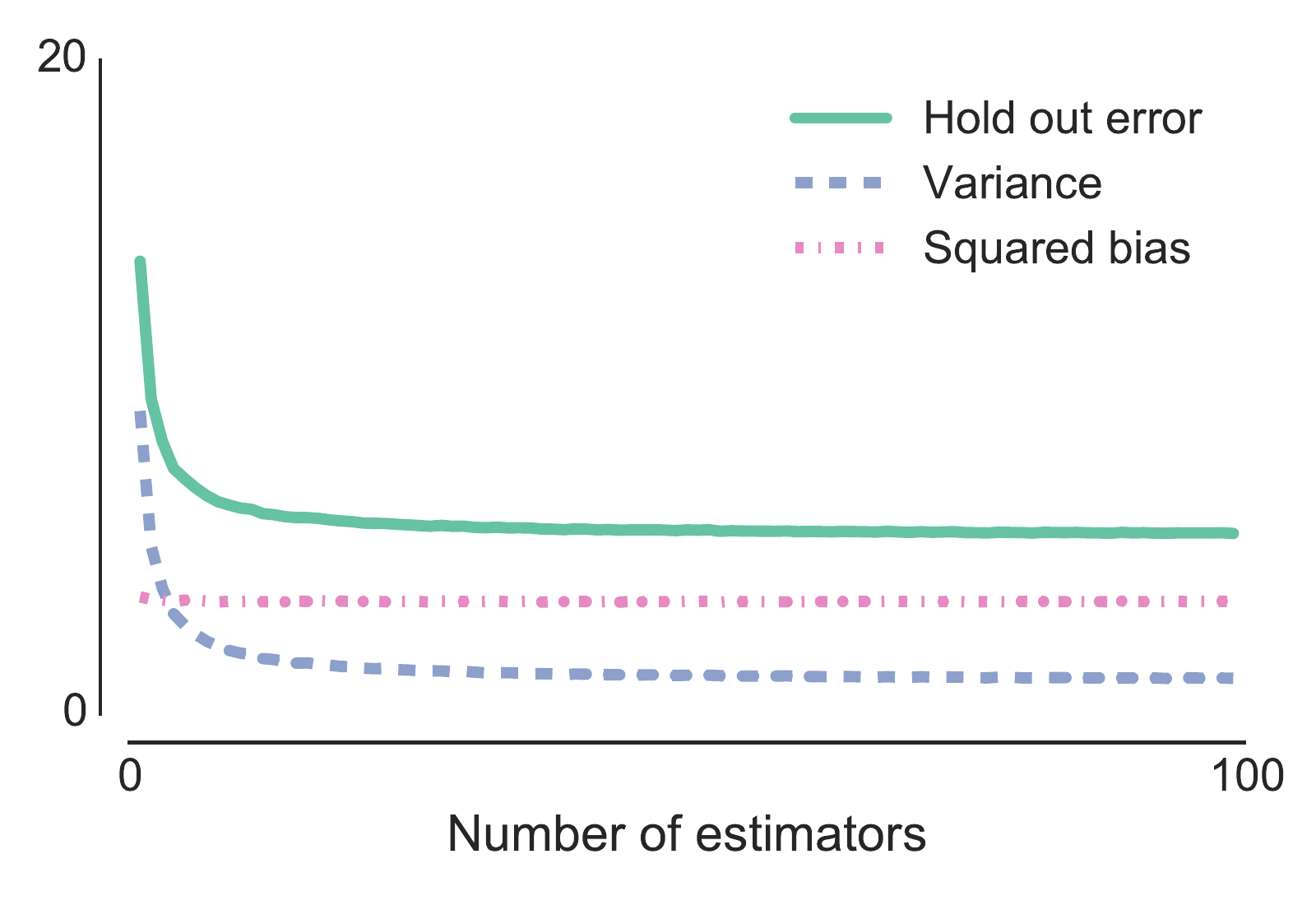}
\caption{Bias-variance decomposition of an ensemble of randomized and fully
developed decision tree models fitted on bootstrap copies of the learning set
(bagging method) for increasing ensemble size on the Friedman1 dataset. The hold
out error, variance and squared bias are averaged over 100 independent learning
and testing sets of respective size 300 and 20000 samples.}
\label{fig:error-decomp-bagging}
\end{figure}

In Section~\ref{subsec:variance-decomp-ensemble}, we show how the bias-variance
decomposition of a randomized supervised learning algorithm is affected by the
ensemble averaging method. In Section~\ref{subsec:randomization-induction}, we
present how to induce randomization without modifying the original supervised
learning algorithm. In Section~\ref{subsec:random-forest-model}, we describe
specific randomization schemes for decision tree methods leading to random
forest models.

\subsection{Variance reduction}
\label{subsec:variance-decomp-ensemble}

Let us first study the bias-variance decomposition for a model
$f_{\mathcal{L},\theta}$ trained on a learning set $\mathcal{L}$ whose
randomness is entirely captured by a random vector of parameters $\theta$. The
model $f_{\mathcal{L},\theta}$ is thus a function of two random variables
$\theta$ and the learning set $\mathcal{L}$. It admits the following
bias-variance decomposition of the generalization error for the average square
loss~\cite{geurts2002contributions}:
\begin{equation}
\E_{\mathcal{L},\theta}(Err) =\E_{P_{\mathcal{X}}}\left\{\hspace*{-0.8mm}
\Var_{\mathcal{L},\theta}\left\{f_{\mathcal{L},\theta}(x)\right\}\hspace*{-0.8mm}+\hspace*{-0.8mm}
\text{Bias}^2(f_{\mathcal{L},\theta}(x))\hspace*{-0.8mm}+\hspace*{-0.8mm}
\Var_{P_{\mathcal{Y}|\mathcal{X}}}\left\{y\right\}\hspace*{-0.8mm}\right\},
\label{eq:bias-var-random-model}
\end{equation}
\noindent where
\begin{align}
\Var_{\mathcal{L},\theta}\left\{f_{\mathcal{L},\theta}(x)\right\} &= \E_{\mathcal{L},\theta}\left\{
\left(f_{\mathcal{L},\theta}(x) - \E_{\mathcal{L}}\E_\theta\left\{f_\theta(x)\right\}\right)^2
\right\}, \\
\Bias^2(f_{\mathcal{L},\theta}(x)) &= \left(\E_{\mathcal{L}}\E_\theta\left\{f_{\mathcal{L},\theta}(x)\right\} - f_\text{Bayes}(x)\right)^2.
\end{align}

By comparison to the bias-variance decomposition of an unperturbed model (see
Equation~\ref{eq:bias-var-random-model}), we have two main differences:
\begin{enumerate}

\item The squared bias is now the distance between the Bayes model
$f_\text{Bayes}$ and the average model
$\E_{\mathcal{L}}\E_\theta\left\{f_{\mathcal{L},\theta}(x)\right\}$ over both
the learning set $\mathcal{L}$ and the randomization parameter $\theta$. Note
that the average model of the randomized algorithm is different from the
non-randomized model $\E_{\mathcal{L}}\left\{f_\mathcal{L}(x)\right\}$. The
randomization of the original algorithm might increase the squared bias.

\item The variance $\Var_{\mathcal{L},\theta}\left\{f_{\mathcal{L},\theta}(x)\right\}$ of the
algorithm now depends on the two random variables $\mathcal{L}$ and $\theta$.
With the law of total variance, we can further decompose the variance term into
two terms:
\begin{equation}
\hspace*{-6mm}\Var_{\mathcal{L},\theta}\left\{f_{\mathcal{L},\theta}(x)\right\}
= \Var_{\mathcal{L}}\left\{\E_{\theta|\mathcal{L}}\left\{f_{\mathcal{L},\theta}(x)\right\}\right\}
+ \E_{\mathcal{L}} \Var_{\theta|\mathcal{L}}\left\{f_{\mathcal{L},\theta}(x)\right\}
\label{eq:var-single-random-model-decomp}
\end{equation}
The first term is the variance brought by the learning sets of the average model over
all parameter vectors $\theta$. The second term describes the variance brought
by the parameter vector $\theta$ averaged over all learning sets $\mathcal{L}$.
\end{enumerate}

Now, we can study the bias-variance decomposition of the generalization error
for an ensemble model $f_{\mathcal{L},\theta_1,\ldots,\theta_m}$ whose
constituents $\{f_{\mathcal{L},\theta_m}\}_{m=1}^M$ depend each on the learning
set $\mathcal{L}$ and a random parameter vector $\{\theta_m\}_{m=1}^M$
capturing the randomness of the models. The bias-variance decomposition of the
ensemble model $f_{\theta_1,\ldots,\theta_m}$ is given by
\begin{align}
\E_{\mathcal{L},\theta_1,\ldots,\theta_M}(Err)=&
\E_{P_{\mathcal{X}}} \Var_{\mathcal{L}}\left\{\E_{\theta_1,\ldots,\theta_M|\mathcal{L}}\left\{f_{\mathcal{L},\theta_1,\ldots,\theta_M}(x)\right\}\right\} \nonumber\\
&+ \E_{P_{\mathcal{X}}} \E_{\mathcal{L}} \Var_{\theta_1,\ldots,\theta_M|\mathcal{L}}\left\{f_{\mathcal{L},\theta_1,\ldots,\theta_M}(x)\right\} \nonumber\\
&+ \E_{P_{\mathcal{X}}} \Bias^2(f_{\mathcal{L},\theta_1,\ldots,\theta_M}(x)) \nonumber\\
&+ \E_{P_{\mathcal{X}}}\Var_{P_{\mathcal{Y}|\mathcal{X}}}\left\{y\right\}.
\label{eq:decom-ensemble-model}
\end{align}

Let us compare the decomposition for a single random model
(Equations~\ref{eq:bias-var-random-model}-\ref{eq:var-single-random-model-decomp})
to the decomposition for an ensemble of random models
(Equation~\ref{eq:decom-ensemble-model}). We are going to expand the
bias-variance decomposition using the ensemble prediction formula
$f_{\theta_1,\ldots,\theta_m}(x)=\frac{1}{M}\sum_{m=1}^M f_{\theta_m}(x)$ (the
demonstration follows~\cite{geurts2002contributions}). As previously, the
variance $\Var_{P_{\mathcal{Y}|\mathcal{X}}}(y)$ is irreducible as this term
does not depend on the supervised learning model.

The average ensemble model of an ensemble of randomized models is equal to the
average model of a single model $f_{\mathcal{L},\theta}$ of random parameter
vector $\theta$:
\begin{align}
\E_{\mathcal{L},\theta_1,\ldots,\theta_M}
\left\{
f_{\mathcal{L},\theta_1,\ldots,\theta_M}(x)
\right\}
&= \frac{1}{M}\sum_{m=1}^M \E_\mathcal{L} \E_{\theta_m} \left\{f_{\mathcal{L},\theta_m}(x)\right\},\\
&= \E_\mathcal{L} \E_\theta \left\{f_{\mathcal{L},\theta}(x)\right\}.
\end{align}
\noindent The squared bias of the ensemble is thus unchanged compared to a
single randomized model.

Now, let us consider the two variance terms. The first one depends
on the variability of the learning set $\mathcal{L}$. With an ensemble
of randomized models, it becomes:
\begin{align}
&\E_{P_{\mathcal{X}}} \Var_{\mathcal{L}}\left\{\E_{\theta_1,\ldots,\theta_M|\mathcal{L}}\left\{f_{\mathcal{L},\theta_1,\ldots,\theta_M}(x)\right\}\right\} \nonumber\\
&= \E_{P_{\mathcal{X}}} \Var_{\mathcal{L}}\left\{
\frac{1}{M}\sum_{m=1}^M\E_{\theta_m|\mathcal{L}}\left\{f_{\mathcal{L},\theta_m}(x)\right\}
\right\}, \\
&= \E_{P_{\mathcal{X}}} \Var_{\mathcal{L}}\left\{
\E_{\theta|\mathcal{L}}\left\{f_{\mathcal{L},\theta}(x)\right\}.
\right\}.
\end{align}
\noindent The variance of the ensemble of randomized model with respect to the
learning set $\mathcal{L}$ drawn from the input-output pair distribution
$P_{\mathcal{X},\mathcal{Y}}$ is not affected by the averaging and is equal
to the variance of a single randomized model.

Let us developed the second variance term of the decomposition describing the
variance with respect to the set of random parameter vectors
$\theta_1,\ldots,\theta_M$:
\begin{align}
&\E_{P_{\mathcal{X}}} \E_{\mathcal{L}}
\Var_{\theta_1,\ldots,\theta_m|\mathcal{L}}\left\{f_{\mathcal{L},\theta_1,\ldots,\theta_M}(x)\right\}
\nonumber \\
&=\E_{P_{\mathcal{X}}} \E_{\mathcal{L}}
\Var_{\theta_1,\ldots,\theta_M|\mathcal{L}}\left\{
\frac{1}{M}\sum_{m=1}^M f_{\mathcal{L},\theta_m}(x)
\right\}, \\
&= \frac{1}{M^2} \E_{P_{\mathcal{X}}} \E_{\mathcal{L}} \left\{
\sum_{m=1}^M \Var_{\theta_1,\ldots,\theta_M|\mathcal{L}}\left\{f_{\mathcal{L},\theta_m}(x)\right\}
\right\}, \\
&= \frac{1}{M^2} \E_{P_{\mathcal{X}}} \E_{\mathcal{L}} \left\{
\sum_{m=1}^M \Var_{\theta_m|\mathcal{L}}\left\{f_{\mathcal{L},\theta_m}(x)\right\}
\right\}, \\
&= \frac{1}{M} \E_{P_{\mathcal{X}}} \E_{\mathcal{L}}
\Var_{\theta|\mathcal{L}}\left\{f_{\mathcal{L},\theta}(x)\right\},
\label{eq:var-ensemble-theta}
\end{align}
\noindent where we use the following properties: (i) $\Var\{a x\} =
a^2 \Var\{x\}$ where $a$ is a constant, (ii) at a fixed learning set
$\mathcal{L}$, the models $f_{\mathcal{L},\theta_m}\,\forall m$ are independent,
(iii) the variance of a sum of independent random variables is equal to the sum
of the variance of each independent random variables ($\Var\left\{\sum_{i=1}^n
x_i\right\} = \sum_{i=1}^n \Var\{x_i\}$).

Putting all together the bias-variance decomposition of
Equation~\ref{eq:decom-ensemble-model} becomes~\cite{geurts2002contributions}:
\begin{align}
\E_{\mathcal{L},\theta_1,\ldots,\theta_m}(Err)=
& \E_{P_{\mathcal{X}}} \Var_{\mathcal{L}}
    \left\{ \E_{\theta|\mathcal{L}}\left\{f_{\mathcal{L},\theta}(x)\right\}\right\} \nonumber \\
&+ \frac{1}{M} \E_{P_{\mathcal{X}}} \E_{\mathcal{L}}
\Var_{\theta|\mathcal{L}}\left\{f_{\mathcal{L},\theta}(x)\right\} \nonumber \\
&+ \E_{P_{\mathcal{X}}} \left(\E_\mathcal{L} \E_\theta \left\{f_{\mathcal{L},\theta}(x)\right\} - f_\text{Bayes}(x)\right)^2 \nonumber\\
&+ \E_{P_{\mathcal{X}}}\Var_{P_{\mathcal{Y}|\mathcal{X}}}\left\{y\right\}.
\label{eq:bias-var-decom-ensemb-final}
\end{align}

The bias variance decomposition of an ensemble of randomized models
$f_{\mathcal{L},\theta_1,\ldots,\theta_M}$ (see
Equation~\ref{eq:bias-var-decom-ensemb-final}) shows that averaging $M$ models
reduces the variance related to the randomization $\theta$ by a factor $1 / M$
over a single randomized model $f_{\mathcal{L},\theta}$ (see
Equation~\ref{eq:bias-var-random-model}) without modifying the other terms. Note
that we can not compare the bias variance decomposition of an ensemble of
randomized models $f_{\mathcal{L},\theta_1,\ldots,\theta_M}$ to its non
randomized counterparts $f_{\mathcal{L}}$(see Equation~\ref{eq:bias-var-decom}).
The bias and variance terms are indeed not comparable.

In practice, we first perturb the learning algorithm which increases the
variance of the models and then we combine them through averaging. The variance
reduction effect is expected to be higher than the added variance at training
time. The bias is either unaffected or increased through the randomization
induction. Perturbing the algorithm is thus a tradeoff between the reduction in
variance and the increase in bias. An ensemble of randomized models
$f_{\mathcal{L},\theta_1,\ldots,\theta_M}$  will have better performance than
its non perturbed counterparts $f_{\mathcal{L}}$ if the bias increase is
compensated by the variance reduction. In~\cite{louppe2014understanding}, the
authors have shown that the generalization error is reduced if the randomization
induction decorrelates the models of the ensemble.

The previous decomposition does not apply to the $0-1$ loss in classification.
However, the main conclusions remains
valid~\cite{breiman1996bagging,domingos2000unified,geurts2002contributions}.

\subsection{Generic randomization induction methods}
\label{subsec:randomization-induction}

In this section, we first discuss generic randomization methods to perturb a
supervised learning algorithm without modifying the original algorithm through
(i) the learning sample set $\mathcal{L} = \{(x^i, y^i) \in \mathcal{X} \times
\mathcal{Y}\}_{i=1}^n$ available at fitting time, (ii) the input space
$\mathcal{X}$ or (iii) the output space $\mathcal{Y}$. Those three perturbation
principles can be either applied separately or together.

We present in succession these three model agnostic randomization principles
(perturbing $\mathcal{L}$, $\mathcal{X}$ or $\mathcal{Y}$).

\subsubsection{Sampling-based randomization}

One of the earliest randomization method, called
bagging~\cite{breiman1996bagging}, fits independent models on bootstrap copies
of the learning set. A bootstrap copy~\cite{efron1979bootstrap} is obtained by
sampling with replacement $|\mathcal{L}|$ samples from the learning set
$\mathcal{L}$. The original motivation was a first theoretical development and
empirical experiments showing that bagging reduces the error of an unstable
estimator such as a decision tree. Bootstrap sampling totally ignores the class
distribution in the original sample set $\mathcal{L}$ and might lead to highly
unbalanced bootstraps. A partial solution is to use stratified bootstraps or to
bootstrap~\cite{chen2004using} separately the minority and majority classes. In
the bagging approach only a fraction of the dataset is provided as training set
to each estimator, the wagging approach~\citep{bauer1999empirical} fits instead
each estimator on the entire training set with random weights. Instead of using
bootstrap copies of training set, \citet{buchlmann2002analyzing} proposes to
subsample the training set, i.e. to sample without replacement the training set.

To take into account the input space structure, \citet{kuncheva2007classifier}
proposes to generate random input space partition with a random hyperplane. An
estimator is then build for each partition. To get an ensemble,
\citet{kuncheva2007classifier} repeats this process multiple times.

\subsubsection{Input-based randomization}

Input-based randomization techniques are often based on dimensionality reduction
techniques. The random subspace method~\cite{ho1998random} builds each model on
a random subset of the input space $\mathcal{X}$ obtained by sub-sampling inputs
without replacement. It was later combined with the bagging method
in~\cite{panov2007combining}, by bootstrapping the learning set $\mathcal{L}$
before learning an estimator, and with sub-sampling techniques with/without
replacement in~\cite{louppe2012ensembles} generating random (sample-input)
patches of the data. Note that while we reduce the input space size, we can also
over-sample the learning set. For instance, \citet{maree2005random} apply a
supervised learning algorithm on random sub-windows extracted from an image,
which effectively (i) increases the sample size available to train each model;
(ii) reduces the input space size, (iii) and takes into account spatial (pixel)
correlation in the images.

Since decision tree made their split orthogonally to the input space, authors
have proposed to randomize such ensembles by randomly projecting the input
space. Rotation forest~\cite{rodriguez2006rotation} is an ensemble method
combining bagging with principal component analysis. For each bootstrap copy, it
first slices the $p$ input variables into $q$ subsets, then projects each subset
of inputs of size $\frac{p}{q}$ on its principal components and finally grows
$q$ models (one on each subset). \citet{kuncheva2007experimental} further
compares three input dimensionality reduction techniques (described in
Section~\ref{sec:dimensionality-reduction}): (i) the PCA approach
of~\citet{rodriguez2006rotation}, (ii) Gaussian random projections and (iii)
sparse Gaussian random projections. On their benchmark, they find that the
PCA-based rotation matrices yield the best results and also that sparse random
projections are strictly better than dense random projections. The idea of using
dense Rademacher or Gaussian random projections was again re-discovered
by~\citet{schclar2009random}. Similarly, \citet{blaser2015random}
proposed to make ensembles through random rotation of the input space.

\subsubsection{Output-based randomization}

Output-based randomization methods directly perturb the output space
$\mathcal{Y}$ of each member of the ensemble.

In regression, we can induce randomization to an output variable through the
addition of an independent Gaussian noise~\cite{breiman2000randomizing}. We fit
each model of the ensemble on the perturbed output $y' = y + \epsilon_m$ with
$\epsilon_m \sim \mathcal{N}(0; \sigma)$.

In classification, we perturb the output of each model of the ensemble by
having a non zero probability to randomly switch the class associated to each
sample~\cite{breiman2000randomizing,martinez2005switching,martinez2008class}.

For multi-label tasks and multi-output tasks, supervised learning algorithms,
such as random k-label subset (see also
Section~\ref{sec:multioutput-strategies})~\cite{tsoumakas2007random}, randomizes
the ensemble by building each model of the ensemble on a subset of the output
space or the label sets present in the learning set.

\subsection{Randomized forest model}
\label{subsec:random-forest-model}

The decision tree algorithm has a high variance, due to the instability of its
structure. Large decision trees, such as fully developed trees, are often very
unstable, especially at the bottom of the tree. The selected splitting rules
depend on the samples reaching those nodes. Small changes in the learning set
might lead to very different tree structures. Authors have proposed
randomization schemes to perturb the search and selection of the best splitting
rule improving the generalization error through averaging methods.

One of the first propositions to perturb the splitting rule
search~\cite{dietterich1995machine} was to select randomly at each node one
splitting rule among the top $k$ splitting rules with the highest impurity
reduction. The variance of the algorithm increases with the number $k$ of
splitting rule candidates, leaving the bias unchanged. Later in the context of
digit recognition, \citet{amit1997joint} randomized the tree growth by
restricting the splitting rule search at each node to a random subset of $k$
input variables out of the $p$ available. The original motivation was to
drastically reduce the splitting rule search space as the number of input
variables is very high in digit recognition tasks. This randomization scheme
increases more the variance of the algorithm than the one
of~\citeauthor{dietterich1995machine}, at the expense of increasing the bias.
Note the similarity with the random subspace approach~\cite{ho1998random} which
subsamples the input space prior fitting a new estimator in an averaging
ensemble.

\citeauthor{breiman2001random}~got inspired by the work
of~\citeauthor{amit1997joint} and combined its bagging
method~\cite{breiman1996bagging} with the input variable sub-sampling leading to
the well known\footnote{The random forest method usually refers to the the
algorithm of \citeauthor{breiman2001random}, however any averaging ensemble of
randomized trees is also a random forest.} ``random forest''
algorithm~\cite{breiman2001random}. The combination of both randomization
schemes has led to one of the best of the shelf estimator for many supervised
learning tasks~\cite{caruana2008empirical,fernandez2014we}.

Later on, \citet{geurts2006extremely} randomized the cut point and input
variable selection of the splitting rules. At each test node, it draws one
random splitting rule for $k$ randomly selected input variables (without
replacement) and then selects the best one. This randomized tree ensemble is
called extremely randomized trees or extra trees. For a splitting rule $s(x) =
'x_j \leq \tau'$ associated to an ordered variable, the algorithm draws
uniformly at random the threshold $\tau$ between the minimum and maximum of the
possible cut point values. Similarly for an unordered variable, the algorithm
draws a non empty subset $B$ among the possible values to generate a splitting
rule of the form $s(x) = 'x_j \in B'$. Empirically, it has been
shown~\cite{geurts2002contributions} that the variance of the decision tree
algorithm is due to the variability of the cut point selection with respect to
the learning set. We can view the perturbation of the cut point selection as a
way to transfer the variance due to the learning set to the variance due to the
randomization of the cut point selection. The hyper-parameter $k$ controls the
trade-off between the bias and variance of the algorithm.

Besides perturbing the binary and axis wise splitting rules, they have been some
research to make splitting rule through random hyper-planes.
\citet{breiman2001random} proposed to select the best splitting rule obtained
from random sparse linear input combinations with non zero values drawn
uniformly in $[-1, 1]$. \citet{tomita2015randomer} proposed to use sparse random
projection where non zero elements are drawn uniformly in $[-1,1]$. Those
approaches increase the variance, while also trying to reduce the bias by
allowing random oblique splits. However, \citet{menze2011oblique} have shown
that those random sparse hyper-planes are inferior to deterministic linear
models such as ridge regressors or a linear discriminant analysis (LDA) models.

For supervised learning tasks with many outputs, we can also perturb the output
space of each decision tree by randomly projecting the output
space~\cite{joly2014random} onto a lower dimensional subspace or through random
output sub-sampling. The leaves are later re-labelled on the original output
space. This approach is developed in Chapter~\ref{ch:rf-output-projections} of
this thesis.

\section{Boosting ensembles}
\label{sec:boosting}

Boosting methods originate from the following question: ``How can we combine a
set of weak models together, each one doing slightly better than random
guessing, so as to get one stronger model having good generalization
performance?''. A boosting model $f$ answers this question through a weighted
combination of $M$ weak models $\{f_m\}_{m=1}^M$ leading to
\begin{equation}
f(x) = \sum_{m=1}^M \alpha_m f_m(x)
\end{equation}
\noindent where the coefficients $\{\alpha_m \in \mathbb{R}\}_{m=1}^M$ highlight
the contribution of each model $f_m(x)$ to the ensemble.

For a boosting ensemble, we usually want to minimize a loss function $\ell:
\mathcal{Y} \times \mathcal{Y} \rightarrow \mathbb{R}^+$ over a learning set
$\mathcal{L} = \{(x^i,y^i) \in \mathcal{X} \times \mathcal{Y}\}_{i=1}^n$:
\begin{equation}
\min_{\left\{(\alpha_m, f_m) \in \mathbb{R} \times \mathcal{H}\right\}_{m=1}^M}
\sum_{(x,y) \in \mathcal{L}} \ell\left(y, \sum_{m=1}^M \alpha_m f_m(x)\right)
\label{eq:boosting-general}
\end{equation}
\noindent where we select each model $f_m$  over a hypothesis space
$\mathcal{H}$. Solving this equation for many loss functions and models
is either intractable or numerically too intensive for practical purpose.
However, we can solve easily Equation~\ref{eq:boosting-general} for a single
model ($M=1$).

So boosting methods develop iterative and tractable schemes to solve
Equation~\ref{eq:boosting-general} by adding sequentially models to the
ensemble. A new model $f_m(x)$ builds over the work done by the previous
$m-1$ models to yield better predictions. It further minimizes the loss $\ell$
averaged over the training data:
\begin{equation}
\min_{\alpha_m, f_m \in \mathbb{R} \times \mathcal{H}}
\sum_{(x,y) \in \mathcal{L}}
\ell\left(y, \sum_{l=1}^{m-1} \alpha_l f_l(x) + \alpha_m f_m(x)\right).
\label{eq:boosting-optim}
\end{equation}
\noindent To improve the predictions made by the $m-1$ models, the new model
$f_m$ with coefficient $\alpha_m$ concentrates its efforts on the wrongly
predicted samples.

From a bias-variance perspective, each newly added model aims to reduce the bias
while leaving the variance term unmodified if possible. We choose the base model
so that it has a high bias and a small variance such as a stump, i.e. a decision
tree with only one testing node, or such as a linear model with only one non-zero
coefficient. In Figure~\ref{fig:error-decomp-gbrts}, we sequentially fit stumps
to decrease the least square loss $\ell(y, y') = \frac{1}{2} (y -y')^2$.
With a few stumps, the squared bias component of the generalization error
dominates with a low variance. By adding more stumps to the ensemble, we
drastically decrease the generalization error by diminishing the squared bias.
The best performance is a trade-off between the bias reduction and the increase
in variance.

\begin{figure}
\centering
\includegraphics[width=0.75\textwidth]{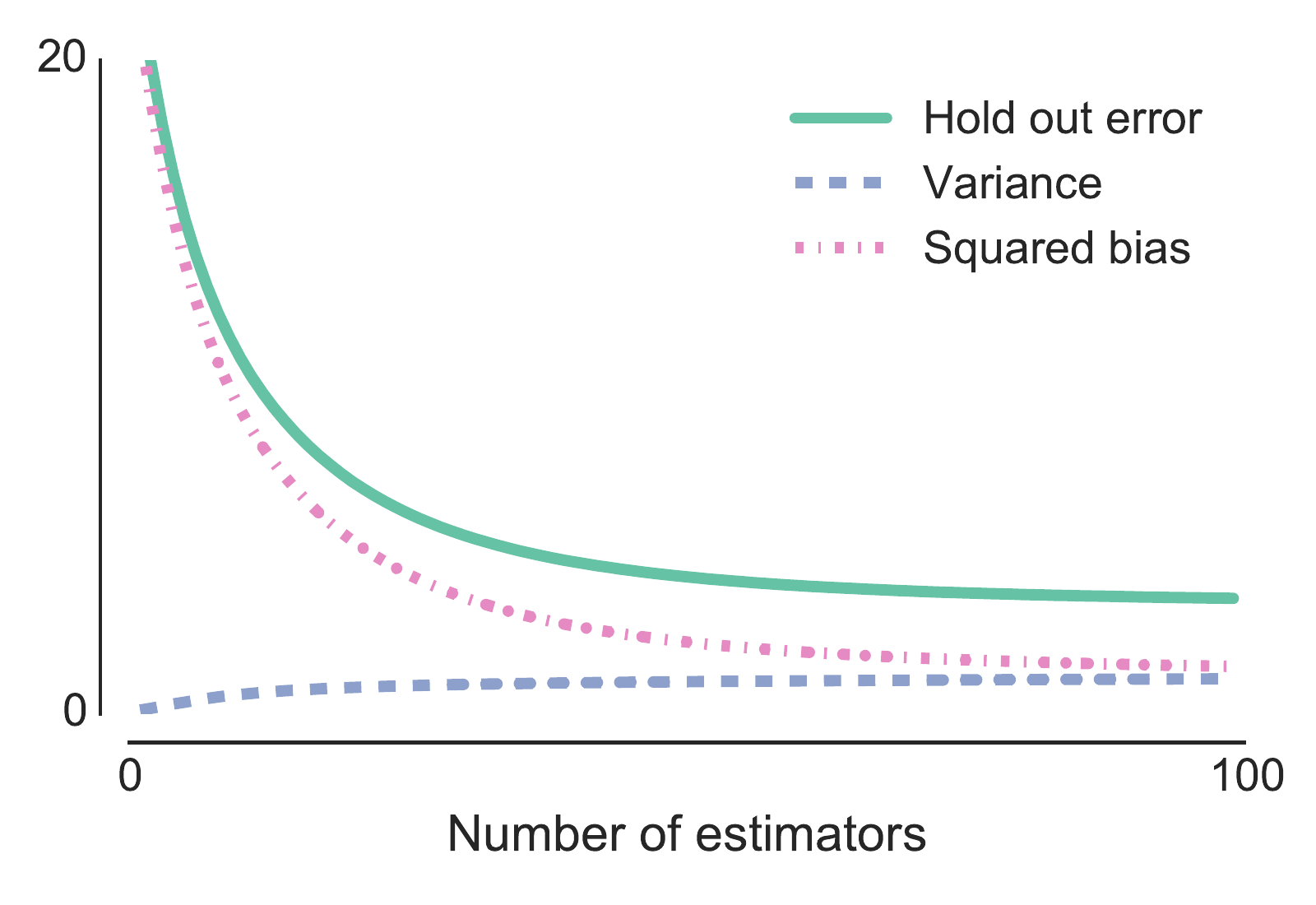}
\caption{Bias-variance decomposition of a boosting ensemble minimizing the
square loss with an increasing number of weak models on the Friedman1 dataset.
The hold out error, variance and squared bias are averaged over 100 independent
learning and testing sets of respective size 300 and 20000 samples.}
\label{fig:error-decomp-gbrts}
\end{figure}

We present the adaptive boosting and its variants in
Section~\ref{subsec:weight-boosting}, which directly solve
Equation~\ref{eq:boosting-optim}, and the functional gradient boosting approach
in Section~\ref{subsec:gradient-boosting}, which approximately solves
Equation~\ref{eq:boosting-optim} through the loss gradient.

\subsection{Adaboost and variants}
\label{subsec:weight-boosting}

One of the most popular and influential boosting algorithms is the ``AdaBoost''
algorithm~\cite{freund1997decision}. This supervised learning algorithm aims to
solve binary classification tasks with $\mathcal{Y} = \{-1, 1\}$. The algorithm
generates iteratively an ensemble of estimators $\{f_m\}_{m=1}^M$ by minimizing
the exponential loss function:
\begin{equation}
\ell_{\exp}(y, y') = \exp(-y y')),
\end{equation}
\noindent assuming a binary response of the weak models $f_m(x) \in \{-1, 1\}\,
\forall m$.

The prediction of an unseen sample $f(x)$ by an AdaBoost ensemble is a
majority vote from its members:
\begin{equation}
f(x) = \sign\left(\sum_{m=1}^M \alpha_m f_m(x)\right),
\end{equation}
\noindent where the $\{\alpha_m\}_{m=1}^M$ are constant weights indicating the
contribution of a model $f_m$ to solve the binary classification task. The
$\sign$ operator transforms the sum into an appropriate output value
($\mathcal{Y} = \{-1, 1\}$).

Given a learning set $\mathcal{L} = ((x^i, y^i) \in \mathcal{X} \times \{-1,
1\})_{i=1}^n$, we iteratively fit a weak model $f_m$ over the learning set
$\mathcal{L}$ by making the weak learner focuses on each sample with a weight
$(w^i)_{i=1}^n$. The higher the value of $w^i$, the more the algorithm will
concentrate to predict correctly the $i$-th sample. To design this algorithm, we
need to answer to the following questions: (i) how to assess the contribution
$\alpha_m$ of the $m$-th model $f_m$ to the ensemble and (ii) how to update the
weight $w^i$ to reduce iteratively the exponential loss.

We can write the resubtitution error of the exponential loss as
\begin{align}
&\frac{1}{n}\sum_{i=1}^n\ell_{\exp}
\left(y^i, \sum_{l=1}^{m-1} \alpha_l f_l(x^i) + \alpha_m f_m(x^i)\right) \nonumber \\
&=\frac{1}{n}\sum_{i=1}^n
\exp\left(-y^i \left(\sum_{l=1}^{m-1} \alpha_l f_l(x^i) + \alpha_m f_m(x^i)\right)\right) \\
&=\frac{1}{n}\sum_{i=1}^n
\exp\left(-y^i \left(\sum_{l=1}^{m-1} \alpha_l f_l(x^i))\right)\right) \exp\left(- y^i \alpha_m f_m(x^i)\right) \\
&=\frac{1}{n}\sum_{i=1}^n
\ell_{\exp}\left(y, \sum_{l=1}^{m-1} \alpha_l f_l(x^i)\right)
\exp(- y^i \alpha_m f_m(x^i)) \\
&=\sum_{i=1} w^i \exp(- y^i \alpha_m f_m(x^i))
\end{align}
\noindent with
\begin{equation}
w^i = \frac{1}{n} \ell_{\exp}\left(y, \sum_{l=1}^{m-1} \alpha_l f_l(x^i)\right) \forall i.
\label{eq:exp-loss-w-direct}
\end{equation}
\noindent Note that the weight computation is expressible as a recursive
equation starting with $w^i=1/n$:
\begin{equation}
w^i \leftarrow w^i \exp\left(\alpha_m 1(y^i \not= f_m(x^i)) \right).
\end{equation}

The sample weight $w^i$ highlights how well the $i$-th sample is predicted by
the $m-1$ first estimators of the boosting ensemble. A zero weight $w^i$ means
that the $i$-th sample is perfectly predicted. The $m$-th estimators should thus
focus on the sample with high weight $w^i$ to reduce the  resubtitution error.
Otherwise, it should minimize the weighted resubstitution error.

Let us now separate the resubtitution error of the correctly classified
points from the misclassified ones:
\begin{align}
&\sum_{i=1}^n w^i \exp(-y^i \alpha_m f_m(x^i)) \nonumber \\
=& \sum_{i=1}^n w^i  \exp(-\alpha_m) 1(y^i=f_m(x^i)) +  \exp(\alpha_m) 1(y^i\not=f_m(x^i))
\end{align}

By derivating the last equation with respect to $\alpha_m$ and setting the
derivative to zero, the $\alpha_m$ minimizing the resubtitution error of the
exponential loss is
\begin{equation}
\alpha_m = \log\left(\frac{1 - \text{err}_m}{\text{err}_m}\right)
\end{equation}
\noindent with
\begin{equation}
\text{err}_m = \frac{1}{2}\frac{\sum_{i=1}^n w^i 1(y^i \not= f_m(x^i))}{\sum_{i=1}^n w^i}.
\end{equation}
\noindent The optimization of the constant $\alpha_m$ means that the
resubstitution error is upper bounded and can not increase with the size of the
ensemble on the learning set.

Putting everything together, we obtain the AdaBoost  algorithm (see
Algorithm~\ref{alg:adaboost}). Many extensions and enhancements of this
fundamental idea have been proposed. If the weak model is able to predict a
probability estimate, \citet{friedman2000additive} have proposed an appropriate
extension called ``Real Adaboost'' by contrast to Algorithm~\ref{alg:adaboost}
which they call ``Discrete Adaboost''.

\begin{algorithm}
\caption{AdaBoost.M1 for binary classification $\mathcal{Y} = \{-1, 1\}$.}
\label{alg:adaboost}
\begin{algorithmic}[1]
\Function{Adaboost}{$\mathcal{L} = \{x^i, y^i\in\mathcal{X}\times\mathcal{Y}\}_{i=1}^n$}
\State Initialize the sample weights $w^i \leftarrow 1/n \, \forall i \in \{1,\ldots,n\}$.
\For{$m$ = 1 to $M$}
    \State Fit a model $f_m(x)$ to the learning set $\mathcal{L}$
           and $(w^i)_{i=1}^n$.
    \State Compute the weighted error rate
    \[ \qquad
        \text{err}_m \leftarrow
        \frac{\sum_{i=1}^n w^i 1(y^i \not= f_m(x^i))}{\sum_{i=1}^n w^i}.
    \]
    \State Compute $\alpha_m \leftarrow  \frac{1}{2}\log\left(\frac{1-\text{err}_m}{\text{err}_m}\right)$.
    \label{alg-line:ada-alpha}
    \State Update the weights
           \[ \qquad
           w^i \leftarrow w^i \exp\left(\alpha_m 1(y^i \not= f_m(x^i)) \right).
           \]
\EndFor
\State \Return $\hat{f}(x) = \sign\left(\sum_{m=1}^M \alpha_m f_m(x) \right)$
\EndFunction
\end{algorithmic}
\end{algorithm}

A direct multi-class extension, called AdaBoost.M1, of the AdaBoost algorithm is
to use a multi-class weak learner instead of a binary one. The AdaBoost.M1
ensemble predicts a new sample through:
\begin{equation}
f(x) = \arg\max_{k\in\mathcal{Y}} \sum_{m=1}^M \alpha_m 1(f_m(x)=k).
\end{equation}
\noindent An improvement over this approach is to directly minimize the
multi-class exponential loss as in the SAMME algorithm~\cite{zhu2009multi}.
It replaces the line~\ref{alg-line:ada-alpha} of Algorithm~\ref{alg:adaboost}
by
\begin{equation}
\alpha_m \leftarrow \frac{1}{2}\log\left(\frac{1-\text{err}_m}{\text{err}_m}\right) + \log(|\mathcal{Y}| - 1)
\end{equation}

We can minimize other losses than the exponential loss during the ensemble
growth, such as the logistic loss $\ell_{logistic}(y, y')= \log(1 +
\exp(-2yy'))$ with the LogitBoost algorithm~\cite{collins2002logistic} for
binary classification tasks; the Hamming loss $\ell_\text{Hamming}(y, y') =
\sum_{j=1}^n 1(y_j\not=y'_j)$ leading to the AdaBoost.MH
algorithm~\cite{schapire2000boostexter} and the pairwise ranking loss
$\ell_\text{ranking}$ \[
\ell_\text{ranking}(y, y')=\frac{1}{|y^i|}\frac{1}{d - |y^i|} \left| \left\{ (k, l) : y'^i_k < y'^i_l,
y^i_k=1, y^i_l=0 \right\}\right|
\] leading to the AdaBoost.MR
algorithm~\cite{schapire2000boostexter} for multi-label classification tasks and
also a wide range of regression losses as proposed
in~\cite{drucker1997improving} for regression tasks.

\begin{remark}{How to take into account sample weights in supervised
learning algorithms?}
For a set of learning samples $\left((x^i, y^i) \in \left(\mathcal{X} \times
\mathcal{Y}\right)\right)_{i=1}^n$ and a set of weights $(w^i \in
\mathbb{R}^+)_{i=1}^n$, the weighted resubtitution error is given by
\begin{equation}
\text{Resubstitution error} =
\frac{\sum_{i=1}^n w^i \ell(f(x^i), y^i)}{\sum_{i=1}^n w^i}.
\end{equation}

Extending supervised learning algorithms to support sample weights means that we
have to modify the learning algorithm so as to minimize the weighted resubtitution
error:
\begin{itemize}
\item For linear models, we will minimize the weighted average of a given loss $\ell:
\mathcal{Y} \times \mathcal{Y} \rightarrow \mathbb{R}^+$ over the learning set
$\mathcal{L} = \{(x^i, y^i) \in \mathcal{X} \times \mathcal{Y}\}_{i=1}^n$
\begin{equation}
\min_{\beta_0, \beta} \sum_{i=1}^n w^i L(y^i, \beta_0 + \beta^T x),
\end{equation}
\noindent where the $w^i \in \mathbb{R}^+$ are the weight associated to each
sample. There is an analytical solution in the case of a $\ell_2$ norm regularization
penalty and algorithms for a $\ell_1$ regularization penalty can be easily
extended to accommodate for the weights.

\item For a decision tree, we will use a weighted impurity criterion and a
weight-aware leaf labelling rule assignment procedure. It also allows new
stopping rule based on sample-weight, such as a minimal total weight to split
a node.

\item  For a $k$-nearest neighbors, we will store the sample weight during fit
and we will predict an unseen sample through a weighted aggregation of the
nearest neighbors.

\end{itemize}

Conversely, to support unweighted supervised learning task with a weight-aware
implementation, we can set the sample weights to a constant such as $w^i =
1/n\,\forall i$ prior the model training.
\end{remark}

\subsection{Functional gradient boosting}
\label{subsec:gradient-boosting}

The AdaBoost algorithm has an analytical and closed-form solution to
Equation~\ref{eq:boosting-optim} with the exponential loss. However, we would
like to build boosting ensembles when such closed-form solutions are not
available. Functional gradient boosting, a forward stagewise additive approach,
approximately solves Equation~\ref{eq:boosting-optim} for a given loss $\ell :
\mathcal{Y} \times \mathcal{Y} \rightarrow \mathbb{R}^+$ by sequentially adding
new basis function $f_m$, a regression model, with a weight $\alpha_m$ without
modifying the previous models.

If we want to add a new model $f_m$ to a boosting ensemble with $m-1$ models
while minimizing the square loss, the loss for a sample $(x, y)$ is given by
\begin{align}
\ell(y, f(x))
&= \frac{1}{2} (y - f(x))^2 \\
&= \frac{1}{2} (y - \sum_{l=1}^{m-1} \alpha_l f_l(x) - \alpha_m f_m(x))^2 \\
&= \frac{1}{2} (r_m(x) - \alpha_m f_m(x))^2
\end{align}
\noindent where $r_m(x) = y - \sum_{l=1}^{m-1} \alpha_l f_l(x)$ is the remaining
residual of the $m-1$ models to predict a sample $x$. Thus for the square loss,
we can add a new models $f_m$ by fitting the new model on the residuals left by
the $m-1$ previous models. This approach is called least square regression
boosting.

Solving Equation~\ref{eq:boosting-optim} is difficult for general loss
functions. It requires to be able to expand a new basis function $f_m$ while
minimizing the chosen loss function. For instance in the context of decision
trees, it would require a specific splitting criterion and a leaf labelling
rule minimizing the chosen loss.

Instead of solving Equation~\ref{eq:boosting-optim},
\citet{friedman2001greedy} proposed a fast approximate solution for arbitrary
differentiable losses inspired from numerical optimization. We can re-write the
loss function minimization as
\begin{equation}
\hat{f} = \arg\min_{f} \ell(f) = \min_{f} \sum_{(x,y)\in\mathcal{L}} \ell(y, f(x)).
\label{eq:boosting-numerical-optim}
\end{equation}
\noindent with the constraint that $f$ is a sum of supervised learning models.
Ignoring this constraint, the Equation~\ref{eq:boosting-numerical-optim} is an
unconstrained minimization problem with $f \in \mathbb{R}^n$ being a
$n$-dimensional vector. Iterative solvers solve such minimization problems
by correcting an initial estimate through a recursive equation.
The final solution is a sum of vectors
\begin{equation}
\hat{f} = \sum_{m=0}^M h_m, \quad h_m \in \mathbb{R}^n,
\end{equation}
\noindent where $h_0$ is the initial estimate. The construction of the sequence
of $h_0,\ldots,h_m$ depends on the chosen optimization algorithm.

The gradient boosting algorithm~\cite{friedman2001greedy} uses the same approach
as the gradient descent method. The update rule of the gradient descent
algorithm $h_m$ is of the form
\begin{equation}
h_m = - \rho_m g_m
\end{equation}
\noindent where $\rho_m$ is a scalar and $g_m \in \mathbb{R}^n$ is the
gradient of $L(f)$ with respect to $f$ evaluated at the current approximate
solution $\hat{f} = \sum_{l=1}^{m-1} \rho_l h_l$:
\begin{equation}
g^{i}_m = \left[
\frac{\partial}{\partial y'} \ell(y^i, y')
\right]_{y' = \hat{f}}.
\end{equation}
\noindent The scalar $\rho_m$ is the step length in the negative loss gradient direction
$-g_m$ chosen so as to minimize the objective function $\ell(f)$:
\begin{equation}
\rho_m = \arg\min_{\rho \in \mathbb{R}} \ell(\hat{f} - \rho_m g_m).
\end{equation}

Back to supervised learning, we can only compute the loss gradient for the
training samples. To generalize to unseen data, the idea is to approximate the
direction of the negative gradient using a regression model $g_m$ selected
within a hypothesis space $\mathcal{H}$ of weak base-learners minimizing the
square loss on the training data:
\begin{equation}
g_m = \arg\min_{g \in \mathcal{H}}\sum_{i=1}^n (-g_m^i - g(x^i))^2.
\end{equation}

The gradient boosting approach can be summarized as follows: start at an initial
constant estimate $\rho_0 \in \mathbb{R}$, then iteratively follows the negative
gradient of the loss $\ell$ as estimated by a regression model $g_m$ fitted over
the training samples and make an optimal step length $\rho_m$ minimizing the
loss $\ell$. The gradient boosting ensemble predicts an unseen sample through
\begin{equation}
f(x) = \rho_0 + \sum_{m=1}^M \rho_m g_m(x).
\end{equation}

The whole procedure is given in Algorithm~\ref{alg:so-gradient-boosting}. The
algorithm is completely defined once we have (i) a starting model, usually the
constant minimizing the chosen loss (line~\ref{alg-line:boosting-gradient}) and
(ii) the gradient of the loss (line~\ref{alg-line:gb-starting-model}).

\begin{algorithm}
\caption{Gradient boosting algorithm}\label{alg:so-gradient-boosting}
\begin{algorithmic}[1]
\Function{GradientBoosting}{$\mathcal{L} = \{(x^i, y^i)\in\mathcal{X}\times\mathcal{Y}\}_{i=1}^n; \ell; \mathcal{H}; M$}
\State $f_0(x) = \rho_0 = \arg\min_{\rho \in \mathbb{R}} \sum_{i=1}^n \ell(y^i, \rho)$.
       \label{alg-line:gb-starting-model}
\For{$m$ = 1 to $M$}
    \State Compute the loss gradient for the training set points
           \[\quad g_m^{i} =
            \left[
            \frac{\partial}{\partial y'}\ell(y^i, y')
            \right]_{y' = f_{m-1}(x)} \forall i \in \left\{1, \ldots, n\right\}.
            \]
            \label{alg-line:boosting-gradient}
    \State Find a correlated direction to the loss gradient
           \[\quad
           g_m = \arg\min_{g \in \mathcal{H}} \sum_{i=1}^n (- g_m^i  - g(x^i))^2.
           \]
    \State Find an optimal step length in the direction $g_m$
            \[
            \quad \rho_m = \arg\min_{\rho \in \mathbb{R}}
              \sum_{i=1}^n \ell\left(y^i, f_{m-1}(x^i) + \rho g_m(x^i)\right).
            \]
          \label{alg-line:gb-step-length}
    \State $f_m(x) = f_{m-1}(x)  + \mu \rho_m g_m(x)$.
\EndFor
\State \Return $f_M(x)$
\EndFunction
\end{algorithmic}
\end{algorithm}

We compute the optimal step length (line~\ref{alg-line:gb-step-length} of
Algorithm~\ref{alg:so-gradient-boosting}) either analytically as for the square
loss or numerically using, e.g., the Brent's method~\cite{brent2013algorithms},
a robust root-finding method allowing to minimize single unconstrained
optimization problem, as for the logistic loss. \citet{friedman2001greedy}
advises to use one step of the Newton–Raphson method. However, the
Newton–Raphson algorithm might not converge if the first and second derivative
of the loss are small. These conditions occurs frequently in highly imbalanced
supervised learning tasks.

A learning rate $\mu \in (0, 1]$ is often added to shrink the size of the
gradient step $\rho_m$ in the residual space in order to avoid overfitting the
the training samples. Another possible modification is to induce randomization,
e.g. by subsampling without replacement the samples available (from all learning
samples) at each iteration~\cite{friedman2002stochastic}.

Table~\ref{tab:common-loss-boosting} gives an overview of regression and
classification losses with their gradients, while
Table~\ref{tab:constant-loss-minimizer} gives the starting constant models
minimizing losses. The square loss in regression and the exponential loss in
classification leads to nice gradient boosting algorithm (respectively the least
square regression boosting algorithm and the exponential classification boosting
algorithm~\cite{zhu2009multi}). However, these losses are not robust to outlier.
More robust losses can be used such as the absolute loss in regression and the
logistic loss or the hinge loss in classification.

\begin{table}
\caption{Regression loss ($\mathcal{Y} = \mathbb{R}$) and binary classification
loss ($\mathcal{Y} = \{-1, 1\}$)  their derivative with respect to a basis function
$f(x)$.}
\label{tab:common-loss-boosting}
\centering
\begin{tabular}{@{}lcc@{}}
\toprule
Regression & $\ell(y, y')$ & $-\partial{}\ell(y, y')/\partial{}y'$  \\
\midrule
Square  & $\frac{1}{2} (y - y')^2$ & $y - y'$ \\
Absolute & $|y - y'|$ & $\sign(y - y')$ \\
\midrule
Classification & $\ell(y, y')$ & $-\partial{}\ell(y, y')/\partial{}y'$  \\
\midrule
Exponential & $\exp(-yy')$ & $-y \exp(-yy')$ \\
Logistic & $\log(1+\exp(-2 y y'))$ & $\frac{2y}{1+\exp(2yy')}$\\
Hinge & $\max(0, 1-yy')$ & $-y 1(yy' < 1)$\\
\bottomrule
\end{tabular}
\end{table}

\begin{table}
\caption{Constant minimizers of regression losses ($\mathcal{Y} = \mathbb{R}$) and
binary classification losses ($\mathcal{Y} = \{-1, 1\}$) given a set of samples
$\mathcal{L} = \{(x^i,y^i)\in\mathcal{X}\times\mathcal{Y}\}_{i=1}^n$.}
\label{tab:constant-loss-minimizer}
\centering
\begin{tabular}{@{}ll@{}}
\toprule
Regression &  \\
\midrule
Square & $f_0(x) = \frac{1}{n}\sum_{i=1}^n y^i$ \\
Absolute & $f_0(x) = \text{median}(\{y^i\}_{i=1}^n)$ \\
\midrule
Classification &  \\
\midrule
Exponential & $f_0(x) = \log\left(\frac{\sum_{i=1}^n 1(y^i=1)}{\sum_{i=1}^n 1(y^i=-1)}\right)$\\
Logistic & $f_0(x) = \log\left(\frac{\sum_{i=1}^n 1(y^i=1)}{\sum_{i=1}^n 1(y^i=-1)}\right)$\\
Hinge & $f_0(x) = \sign\left(\frac{1}{n}\sum_{i=1}^n 1(y^i=1) - \frac{1}{2}\right)$ \\
\bottomrule
\end{tabular}
\end{table}

\part{Learning in compressed space through random projections}
\label{part:tree-rp}


\chapter[Random projections of the output space]{Random forests with random projections of the output space for high
         dimensional multi-label classification}
\label{ch:rf-output-projections}

\begin{remark}{Outline}
We adapt the idea of random projections applied to the output space, so as
to enhance tree-based ensemble methods in the context of multi-label
classification. We show how learning time complexity can be reduced without
affecting computational complexity and accuracy of predictions. We also show
that random output space projections may be used in order to reach different
bias-variance tradeoffs, over a broad panel of benchmark problems, and that
this may lead to improved accuracy while reducing significantly the
computational burden of the learning stage.

\textit{This chapter is based on previous work published in}
\begin{quote}
Arnaud Joly, Pierre Geurts, and Louis Wehenkel. Random forests with random
projections of the output space for high dimensional multi-label classification.
In Machine Learning and Knowledge Discovery in Databases, pages 607–622.
Springer Berlin Heidelberg, 2014.
\end{quote}
\end{remark}

Within supervised learning, the goal of multi-label classification is to train
models to annotate objects with a subset of labels taken from a set of candidate
labels. Typical applications include the determination of topics addressed in a
text document, the identification of object categories present within an image,
or the prediction of biological properties of a gene. In many applications, the
number of candidate labels may be very large, ranging from hundreds to hundreds
of thousands~\cite{agrawal2013multi} and often even exceeding the
sample size~\cite{dekel2010multiclass}. The very large scale nature of the
output space in such problems poses both statistical and computational
challenges that need to be specifically addressed.

A simple approach to  multi-label classification problems, called binary
relevance, is to train independently a binary classifier for each label. Several
more complex schemes have however been proposed to take into account the
dependencies between the labels (see Section~\ref{sec:multioutput-strategies}).
In the context of tree-based methods, one way is to train multi-output trees
(see Section~\ref{sec:mo-trees}), i.e. trees that can predict multiple outputs
at once. With respect to binary relevance, the multi-output tree approach has
the advantage of building a single model for all labels. It can thus potentially
take into account label dependencies and reduce memory requirements for the
storage of the models. An extensive experimental
comparison~\cite{madjarov2012extensive} shows that this approach compares
favorably with other approaches, including non tree-based methods, both in terms
of accuracy and computing times. In addition, multi-output trees inherit all
intrinsic advantages of tree-based methods, such as robustness to irrelevant
features, interpretability through feature importance scores, or fast
computations of predictions, that make them very attractive to address
multi-label problems. The computational complexity of learning multi-output
trees is however similar to that of the binary relevance method. Both approaches
are indeed $O(p d n\log n)$, where $p$ is the number of input features, $d$ the
number of candidate output labels, and $n$ the sample size; this is a limiting
factor when dealing with large sets of candidate labels.

One generic approach to reduce computational complexity is to apply some
compression technique prior to the training stage to reduce the number of
outputs to a number $q$ much smaller than the total number $d$ of labels. A
model can then be trained to make predictions in the compressed output space and
a prediction in the original label space can be obtained by decoding the
compressed prediction. As multi-label vectors are typically very sparse, one can
expect a drastic dimensionality reduction by using appropriate compression
techniques. This idea has been explored for example in~\cite{hsu2009multi} using
compressed sensing, and in~\cite{cisse2013robust} using bloom filters, in both
cases using regularized linear models as base learners. The approach obviously
reduces computing times for training the model.
Random projections are also exploited in \cite{tsoumakas2014multi} for
multi-target regression. In this latter work however, they are not used to
improve computing times by compression but instead to improve predictive
performance. Indeed, more (sparse) random projections are computed than there are
outputs and they are used each as an output to train some single target
regressor. As in \cite{hsu2009multi,cisse2013robust}, the predictions of the
regressors need to be decoded at prediction time to obtain a prediction in the
original output space. This is achieved in \cite{tsoumakas2014multi} by solving
an overdetermined linear system.

In this chapter, we explore the use of random output space projections for
large-scale multi-label classification in the context of tree-based ensemble
methods. We first explore the idea proposed for linear models
in~\cite{hsu2009multi} with random forests: a (single) random projection of the
multi-label vector to a $q$-dimensional random subspace is computed and then a
multi-output random forest is grown based on score computations using the
projected outputs. We exploit however the fact that the approximation provided
by a tree ensemble is a weighted average of output vectors from the training
sample to avoid the decoding stage: at training time all leaf labels are
directly computed  in the original multi-label space. We show theoretically and
empirically that when $q$ is large enough, ensembles grown on such random output
spaces are equivalent to ensembles grown on the original output space. When $d$
is large enough compared to $n$, this idea hence may reduce computing times at
the learning stage without affecting accuracy and computational complexity of
predictions.

Next, we propose to exploit the randomization inherent to the projection of the
output space as a way to obtain randomized trees in the context of ensemble
methods: each tree in the ensemble is thus grown from a different randomly
projected subspace of dimension $q$. As previously, labels at leaf nodes are
directly computed in the original output space to avoid the decoding step. We
show, theoretically, that this idea can lead to better accuracy than the first
idea and, empirically, that best results are obtained on many problems with very
low values of $q$, which leads to significant computing time reductions at the
learning stage. In addition, we study the interaction between input
randomization (\`a la Random Forests) and output randomization (through random
projections), showing that there is an interest, both in terms of predictive
performance and in terms of computing times, to optimally combine these two ways
of randomization. All in all, the proposed approach constitutes a very
attractive way to address large-scale multi-label problems with tree-based
ensemble methods.

The rest of the chapter is structured as follows: Section~\ref{sec:methods}
presents the proposed algorithms and their theoretical properties;
Section~\ref{sec:biasvar} analyses the proposed algorithm from a bias-variance
perspective; Section~\ref{sec:rf-rp-experiments} provides the empirical validations,
whereas Section~\ref{sec:rf-rp-conclusions} discusses our work and provides further
research directions. 

\section{Methods}
\label{sec:methods}

We first present how we propose to exploit random projections to reduce the
computational burden of learning single multi-output trees in very
high-dimensional output spaces. Then we present and compare two ways to exploit
this idea with ensembles of trees.

\subsection{Multi-output regression trees in randomly projected output spaces}
\label{sec:theoretical-analysis}

The multi-output single tree algorithm described in
Chapter~\ref{ch:tree} requires the computation of the sum of impurity
criterion, such as the variance (or Gini), at each tree node and for each
candidate split. When $\mathcal{Y}$ is very high-dimensional, this computation
constitutes the main computational bottleneck of the algorithm. We thus propose
to approximate variance computations by using random projections of the output
space. The multi-output regression tree algorithm is modified as follows
(denoting by $\mathcal{L}$ the  learning sample $\mathcal{L} =
((x^i,y^i) \in \mathcal{X}\times\mathcal{Y})_{i=1}^n$):

\begin{itemize}

\item First, a projection matrix $\Phi$ of dimension $q\times d$ is randomly
generated.

\item A new dataset $\mathcal{L}_m=((x^i,\Phi y^i))_{i=1}^n$ is constructed by
projecting each learning sample output using the projection matrix $\Phi$.

\item A tree (structure) $\mathcal{T}_m$ is grown using the projected learning sample
$\mathcal{L}_m$.

\item Predictions $\hat{y}$ at each leaf of $\cal T$ are computed using the
corresponding outputs in the original output space.

\end{itemize}

The resulting tree is exploited in the standard way to make predictions: an
input vector $x$ is propagated through the tree until it reaches a leaf from
which a prediction $\hat{y}$ in the original output space is directly retrieved.

If $\Phi$ satisfies the Jonhson-Lindenstrauss lemma (Equation~\ref{eqn:js}), the
following theorem shows that variance computed in the projected subspace is an
$\epsilon$-approximation of the variance computed over the original space.

\begin{theorem}
\label{thm:var-jl-lemma}
Given $\epsilon > 0$, a sample $(y^i)_{i=1}^{n}$ of $n$ points $y \in
\mathbb{R}^d$, and a projection matrix $\Phi\in \mathbb{R}^{q\times d}$ such
that for all $i, j  \in \{1, \ldots , n\}$ the condition given by
Equation~\ref{eqn:js} holds, we have also:
\begin{equation}
(1 - \epsilon) \Var((y^i)_{i=1}^n)  \leq \Var((\Phi y^i)_{i=1}^n)
                                    \leq (1 + \epsilon) \Var((y^i)_{i=1}^n).
\label{eq:var-ineq}
\end{equation}
\end{theorem}
\begin{proof}
The sum of the variances of $n$ observations drawn from a random vector
$y \in \mathbb{R}^d$ can be interpreted as a sum of squared euclidean distances
between the pairs of observations
\begin{equation}
\Var((y^i)_{i=1}^n) = \frac{1}{2n^2} \sum_{i=1}^n \sum_{j=1}^n ||y^i - y^j||^2. \label{sum-of-dist}
\end{equation}

Starting from the defition of the variance, we have
\begin{align}
& \Var((y^i)_{i=1}^n) \nonumber \\
&\stackrel{\mbox{\tiny def}}{=} \frac{1}{n} \sum_{i=1}^n ||y^i - \frac{1}{n} \sum_{j=1}^n y^j||^2 \\
&= \frac{1}{n} \sum_{i=1}^n
      (y^i - \frac{1}{n} \sum_{j=1}^n y^j)^T
      (y^i - \frac{1}{n} \sum_{k=1}^n y^k) \\
&= \frac{1}{n} \sum_{i=1}^n
      \left(
      {y^i}^T y^i -
      \frac{2}{n}  \sum_{j=1}^n {y^i}^T y^j +
      \frac{1}{n^2} \sum_{j=1}^n \sum_{k=1}^n {y^j}^T y^k
      \right)\\
&=  \frac{1}{n} \sum_{i=1}^n {y^i}^T y^i
    - \frac{2}{n^2} \sum_{i=1}^n \sum_{j=1}^n {y^i}^T y^j
    + \frac{1}{n^2} \sum_{j=1}^n \sum_{k=1}^n {y^j}^T y^k \\
&= \frac{1}{n} \sum_{i=1}^n {y^i}^T y^i
   - \frac{1}{n^2} \sum_{i=1}^n \sum_{j=1}^n {y^i}^T y^j \\
&=   \frac{1}{2n} \sum_{i=1}^n {y^i}^T y^i
   + \frac{1}{2n} \sum_{j=1}^n {y^j}^T y^j
   - \frac{1}{n^2} \sum_{i=1}^n \sum_{j=1}^n {y^i}^T y^j \\
&=  \frac{1}{2 n^2} \sum_{i=1}^n \sum_{j=1}^n {y^i}^T\hspace*{-1mm} y^i
   \hspace*{-1mm}+ \hspace*{-1mm}\frac{1}{2 n^2} \sum_{i=1}^n \sum_{j=1}^n {y^j}^T\hspace*{-1mm} y^j
   \hspace*{-1mm}- \hspace*{-1mm}\frac{1}{n^2} \sum_{i=1}^n \sum_{j=1}^n {y^i}^T\hspace*{-1mm} y^j \\
&= \frac{1}{2n^2} \sum_{i=1}^n \sum_{j=1}^n
      \left(
      {y^i}^T y^i + {y^j}^T y^j - 2 {y^i}^T y^j
      \right) \\
&= \frac{1}{2n^2} \sum_{i=1}^n \sum_{j=1}^n ||y^i - y^j||^2.
\end{align}

From the Johnson-Lindenstrauss Lemma we have for any $i,j$
\begin{equation}
(1 - \epsilon) ||y^i - y^j||^2 \leq || \Phi y^i - \Phi y^j ||^2
                               \leq (1 + \epsilon) || y^i - y^j||^2. \label{ineq-vect}
\end{equation}

By summing the three terms of Equation~\ref{ineq-vect} over all pairs $i,j$ and
dividing by $1 / (2n^2)$ and by then using Equation~\ref{sum-of-dist}, we get
Equation~\ref{eq:var-ineq}.

\end{proof}

As a consequence, any split score approximated from the randomly projected
output space will be $\epsilon$-close to the unprojected scores in any subsample
of the complete learning sample. Thus, if the condition given by
Equation~\ref{eqn:js}) is satisfied for a sufficiently small $\epsilon$ then the
tree grown from the projected data will be identical to the tree grown from the
original data\footnote{Strictly speaking, this is only the case when the optimum
scores of test splits as computed over the original output space are isolated,
i.e. when there is only one single best split, no tie.}.

For a given size $q$ of the projection subspace, the complexity is reduced from
$O(d n)$ to $O(q n)$ for the computation of one split score and thus from $O(d p
n \log n)$ to $O(q p n\log n)$ for the construction of one full (balanced) tree,
where one can expect $q$ to be much smaller than $d$ and at worst of
$O(\epsilon^{-2} \log{n})$. The whole procedure requires to generate the
projection matrix and to project the training data. These two steps are
respectively $O(d q)$ and $O(n d q)$ but they can often be significantly
accelerated by exploiting the sparsity of the projection matrix and/or of the
original output data, and they are called only once before growing the tree.

All in all, this means that when $d$ is sufficiently large, the random
projection approach may allow us to significantly reduce tree building
complexity from $O(d p n \log n)$ to $O(q p n \log n + n d q)$, without
impact on predictive accuracy (see Section~\ref{sec:rf-rp-experiments}, for empirical
results).

\subsection{Exploitation in the context of tree ensembles}

The idea developed in the previous section can be directly exploited  in the
context of ensembles of randomized multi-output regression trees. Instead of
building a single tree from the projected learning sample, one can grow a
randomized ensemble of them. This ``shared subspace'' algorithm is described in
pseudo-code in Algorithm~\ref{alg:output-fix-subspace-tree-ensemble}.

\begin{algorithm}
\caption{Grow $t$ decision trees on a single shared subspace $\Phi$
using learning samples $\mathcal{L} = ((x^i, y^i) \in (\mathbb{R}^p \times
\mathbb{R}^d))_{i=1}^n$ }
\label{alg:output-fix-subspace-tree-ensemble}
\begin{algorithmic}[1]
\Function{GrowForestSharedOutputSubspace}{$\mathcal{L}$,$t$}
    \State{Generate a sub-space $\Phi \in \mathbb{R}^{q \times d}$;}
    \For{$j = 1$ to $t$}
        \State{Build a tree structure $\mathcal{T}_{j}$ using $((x^i, \Phi y^i))_{i=1}^n$;}
        \State{Label the leaves of $\mathcal{T}_{j}$ using $((x^i, y^i))_{i=1}^n$;}
        \State{Add the labelled tree $\mathcal{T}_{j}$ to the ensemble;}
    \EndFor
\EndFunction
\end{algorithmic}
\end{algorithm}

Another idea is to exploit the random projections used so as to introduce a
novel kind of diversity among the different trees of an ensemble. Instead of
building all the trees of the ensemble from a same shared output-space
projection, one could instead grow each tree in the ensemble from a different
output-space projection. Algorithm~\ref{alg:output-subspace-tree-ensemble}
implements this idea in pseudo-code. The randomization introduced by the output
space projection can of course be combined with any existing randomization
scheme to grow ensembles of trees. In this chapter, we will consider the
combination of random projections with the randomizations already introduced in
Random Forests and Extra Trees. The interplay between these different
randomizations will be discussed theoretically in the next subsection by a
bias/variance analysis and empirically in Section~\ref{sec:rf-rp-experiments}. Note
that while when looking at single trees or shared ensembles, the size $q$ of the
projected subspace should not be too small so that condition
(Equation~\ref{eqn:js}) is satisfied, the optimal value of $q$ when projections
are randomized at each tree is likely to be smaller, as suggested by the
bias/variance analysis in the next section.

\begin{algorithm}
\caption{Grow $t$ decision trees on individual random subspaces $(\Phi_{j})_{j=1}^t$
using learning samples $\mathcal{L} = ((x^i, y^i) \in (\mathbb{R}^p \times
\mathbb{R}^d))_{i=1}^n$}
\label{alg:output-subspace-tree-ensemble}

\begin{algorithmic}[1]
\Function{GrowForestOutputSubspace}{$\mathcal{L}$,$t$}
    \For{$j = 1$ to $t$}
        \State{Generate a sub-space $\Phi_{j} \in \mathbb{R}^{q \times d}$;}
        \State{Build a tree structure $\mathcal{T}_{j}$ using $((x^i, \Phi_{j} y^i))_{i=1}^n$;}
        \State{Label the leaves  of $\mathcal{T}_{j}$  using $((x^i, y^i))_{i=1}^n$;}
        \State{Add the labelled tree $\mathcal{T}_{j}$  to the ensemble;}
    \EndFor
\EndFunction
\end{algorithmic}
\end{algorithm}

From the computational point of view, the main difference between these two ways
of transposing random-output projections to ensembles of trees is that in the
case of Algorithm~\ref{alg:output-subspace-tree-ensemble}, the generation of the
projection matrix $\Phi$ and the computation of projected outputs is carried out
$t$ times, while it is done only once for the case of
Algorithm~\ref{alg:output-fix-subspace-tree-ensemble}. These aspects will be
empirically evaluated in Section~\ref{sec:rf-rp-experiments}.

\section{Bias/variance analysis}\label{sec:biasvar}

In this section, we adapt the bias/variance analysis carried out in
Section~\ref{subsec:variance-decomp-ensemble} to take into account random output
projections. The details of the derivations are reported in
Section~\ref{subsec:single-rp-tree} for a single tree and in
Section~\ref{subsec:single-rp-forest} for an ensemble of $t$ randomized trees.

Let us denote by $f_{\mathcal{L},\phi,\sigma}(.):{\cal X}\rightarrow \mathbb{R}^d$ a
single multi-output tree obtained from a projection matrix $\phi$ (below we use
$\Phi$ to denote the corresponding random variable), where $\sigma$ is the
value of a random variable $\varsigma$ capturing the random perturbation
scheme used to build this tree (e.g., bootstrapping and/or random input space
selection). The square error of this model at some point $x\in{\cal X}$ is
defined by:
$$Err(f_{\mathcal{L},\phi,\sigma}(x))\stackrel{\text{def}}{=}\E_{Y|x}\{||Y-f_{\mathcal{L},\phi,\sigma}(x))||^2\},$$
and its average can decomposed in its residual error, (squared) bias, and
variance terms denoted:
$$\E_{\mathcal{L},\Phi,\varsigma}\{Err(f_{\mathcal{L}, \phi, \varsigma}(x))\}=\sigma^2_R(x)+B^2(x)+V(x)$$
where the variance term $V(x)$ can be further decomposed as the sum of the
following three terms:
\begin{eqnarray*}
V_{\mathcal{L}}(x)&=&\Var_{\mathcal{L}}\{\E_{\Phi,\varsigma|\mathcal{L}}\{f_{\mathcal{L}, \phi, \varsigma}(x)\}\}\\
V_{Algo}(x)&=&\E_{\mathcal{L}}\{\E_{\Phi|\mathcal{L}}\{\Var_{\varsigma|\mathcal{L},\Phi}\{f_{\mathcal{L}, \phi, \varsigma}(x)\}\}\},\\
V_{Proj}(x)&=&\E_{\mathcal{L}}\{\Var_{\Phi|\mathcal{L}}\{\E_{\varsigma|\mathcal{L},\Phi}\{f_{\mathcal{L}, \phi, \varsigma}(x)\}\}\},
\end{eqnarray*}
\noindent that measure errors due to the randomness of, respectively, the learning sample,
the tree algorithm, and the output space projection
(see Section~\ref{subsec:single-rp-tree}).

Approximations computed respectively by
Algorithm~\ref{alg:output-fix-subspace-tree-ensemble} and
Algorithm~\ref{alg:output-subspace-tree-ensemble} take the following
forms:\vspace*{-2mm}
\begin{itemize}
\item $f_{1;\mathcal{L},\sigma^t,\phi}(x)=\frac{1}{t}\sum_{i=1}^t f_{\mathcal{L},\phi, \sigma_i}(x)$
\item $f_{2;\mathcal{L},\sigma^t,\phi^t}(x)=\frac{1}{t}\sum_{i=1}^t f_{\mathcal{L},\phi_i, \sigma_i}(x),$\vspace*{-2mm}
\end{itemize}
where $\sigma^t=(\epsilon_1,\ldots,\epsilon_{t})$ and
$\phi^t=(\phi_1,\ldots,\phi_t)$ are vectors of i.i.d. values of the random
variables $\varsigma$ and $\Phi$ respectively.

We are interested in comparing the average errors of these two algorithms,
where the average is taken over all random parameters (including the
learning sample). We show that these can be decomposed as
follows (see Section~\ref{subsec:single-rp-forest}):
\begin{eqnarray*}
&&\hspace*{-5mm}\E_{\mathcal{L},\Phi,\varsigma^t}\{Err(f_{1;\mathcal{L},\Phi,\varsigma^t}(x))\}\\
&&\hspace{-0.2cm}=\sigma^2_R(x)+B^2(x)+V_{\mathcal{L}}(x)+\frac{V_{Algo}(x)}{t}+V_{Proj}(x),\\
&&\hspace*{-5mm}\E_{\mathcal{L},\Phi^t,\varsigma^t}\{Err(f_{2;\mathcal{L},\Phi^t,\varsigma^t}(x))\}\\
&&\hspace{-0.2cm} =\sigma^2_R(x)+B^2(x)+V_{\mathcal{L}}(x)+\frac{V_{Algo}(x)+V_{Proj}(x)}{t}.
\end{eqnarray*}
From this result, it is hence clear that
Algorithm~\ref{alg:output-subspace-tree-ensemble} can not be worse, on the
average, than Algorithm~\ref{alg:output-fix-subspace-tree-ensemble}. If the
additional computational burden needed to generate a different random projection
for each tree is not problematic, then
Algorithm~\ref{alg:output-subspace-tree-ensemble} should always be preferred to
Algorithm~\ref{alg:output-fix-subspace-tree-ensemble}.

For a fixed level of tree randomization ($\varsigma$), whether the additional
randomization brought by random projections could be beneficial in terms of
predictive performance remains an open question that will be addressed
empirically in the next section. Nevertheless, with respect to an ensemble
grown from the original output space, one can expect that the
output-projections will always increase the bias term, since they disturb the algorithm
in its objective of reducing the errors on the learning sample. For small
values of $q$, the average error will therefore decrease (with a sufficiently
large number $t$ of trees) only if the increase in bias is
compensated by a decrease of variance.

The value of $q$, the dimension of the projected subspace, that will lead to the
best tradeoff between bias and variance will hence depend both on the level of tree
randomization and on the learning problem. The more (resp. less)
tree~randomization, the higher (resp. the lower) could be the optimal value of $q$,
since both randomizations affect bias and variance in the same direction.

\subsection{Single random trees.}
\label{subsec:single-rp-tree}

Let us denote by $f_{\mathcal{L}, \phi, \sigma} :{\cal X}\rightarrow
\mathbb{R}^d$ a single multi-output (random) tree obtained from a projection
matrix $\phi$ (below we use $\Phi$ to denote the corresponding random variable),
where $\sigma$ is the value of a random variable $\varsigma$ capturing the
random perturbation scheme used to build this tree (e.g., bootstrapping and/or
random input space selection). Denoting by $Err(f_{\mathcal{L}, \phi,
\sigma}(x))$ the square error of this model at some point $x\in{\cal X}$ defined
by:
\begin{equation}
\E_{P_{\mathcal{Y}|\mathcal{X}}}\{||y-f_{\mathcal{L}, \phi, \sigma}(x)||^2\}.
\end{equation}

The average of this square error can decomposed as follows:
\begin{eqnarray*}
&&\E_{\mathcal{L},\Phi,\varsigma}\{Err(f_{\mathcal{L}, \phi, \sigma}(x))\}\\
&=&\sigma^2_R(x)+||f_\text{Bayes}(x)-\bar{f}(x)||^2+\Var_{\mathcal{L},\Phi,\varsigma}\{f_{\mathcal{L}, \phi, \sigma}(x)\},
\end{eqnarray*}
where
\begin{align*}
\bar{f}(x)&\stackrel{\text{def}}{=}\E_{\mathcal{L},\Phi,\varsigma}\{f_{\mathcal{L}, \phi, \varsigma}(x)\} \\
f_\text{Bayes}(x)&=E_{Y|x}\{Y\} \\
\Var_{\mathcal{L},\Phi,\varsigma}\{f_{\mathcal{L}, \phi, \varsigma}(x)\}&\stackrel{\text{def}}{=} E_{\mathcal{L},\Phi,\varsigma}\{||f_{\mathcal{L}, \phi, \varsigma}(x)-\bar{f}(x)||^2\}.
\end{align*}
\noindent The three terms of this decomposition are respectively the residual error,
the bias, and the variance of this estimator (at $x$).

The variance term can be further decomposed as follows using the law of total
variance:
\begin{eqnarray}
\Var_{\mathcal{L},\Phi,\varsigma}\{f_{\mathcal{L}, \phi, \varsigma}(x)\}\hspace*{-4cm} &\hspace{1cm}&\notag\\
&=&\Var_{\mathcal{L}}\{E_{\Phi,\varsigma|\mathcal{L}}\{f_{\mathcal{L}, \phi, \varsigma}(x)\}\}\notag\\
&&+E_{\mathcal{L}}\{\Var_{\Phi,\varsigma|\mathcal{L}}\{f_{\mathcal{L}, \phi, \varsigma}(x)\}\}.\label{eqn:totalvar1}
\end{eqnarray}

The first term is the variance due to the learning sample randomization and the
second term is the average variance (over $\mathcal{L}$) due to both the random
forest randomization and the random output projection. By using the law of total
variance a second time, the second term of Equation~\ref{eqn:totalvar1}) can be
further decomposed as follows:
\begin{align}
&\E_{\mathcal{L}}\{\Var_{\Phi,\varsigma|\mathcal{L}}\{f_{\mathcal{L}, \phi, \varsigma}(x)\}\} \notag\\
&= \E_{\mathcal{L}}\{\Var_{\Phi|\mathcal{L}}\{\E_{\varsigma|\mathcal{L},\Phi}\{f_{\mathcal{L}, \phi, \varsigma}(x)\}\}\}\notag\\
&\quad + \E_{\mathcal{L}}\{\E_{\Phi|\mathcal{L}}\{\Var_{\varsigma|\mathcal{L},\Phi}\{f_{\mathcal{L}, \phi, \varsigma}(x)\}\}\}.
\label{eqn:totalvar2}
\end{align}

The first term of this decomposition is the variance due to the random choice of
a projection and the second term is the average variance due to the random
forest randomization.  Note that all these terms are non negative. In what
follows, we will denote these three terms respectively $V_{\mathcal{L}}(x)$,
$V_{Algo}(x)$, and $V_{proj}(x)$. We thus have:
$$\Var_{\mathcal{L},\Phi,\varsigma}\{f_{\mathcal{L}, \phi,
\varsigma}(x)\}=V_{\mathcal{L}}(x)+V_{Algo}(x)+V_{Proj}(x),$$ with
\begin{eqnarray*}
V_{\mathcal{L}}(x)&=&\Var_{\mathcal{L}}\{\E_{\Phi,\varsigma|\mathcal{L}}\{f_{\mathcal{L}, \phi, \varsigma}(x)\}\}\\
V_{Algo}(x)&=&\E_{\mathcal{L}}\{\E_{\Phi|\mathcal{L}}\{\Var_{\varsigma|\mathcal{L},\Phi}\{f_{\mathcal{L}, \phi, \varsigma}(x)\}\}\},\\
V_{Proj}(x)&=&\E_{\mathcal{L}}\{\Var_{\Phi|\mathcal{L}}\{\E_{\varsigma|\mathcal{L},\Phi}\{f_{\mathcal{L}, \phi, \varsigma}(x)\}\}\},
\end{eqnarray*}

\subsection{Ensembles of $t$ random trees.}
\label{subsec:single-rp-forest}

When the random projection is fixed for all $t$ trees in the ensemble
(Algorithm~\ref{alg:output-fix-subspace-tree-ensemble}), the algorithm computes
an approximation, denoted $f_1(x;\mathcal{L},\phi, \sigma^t)$, that takes the
following form:
$$f_{1;\mathcal{L},\phi,\sigma^t}(x)=\frac{1}{t}\sum_{i=1}^t
f_{\mathcal{L},\phi, \sigma_i}(x),$$
where $\sigma^t=(\sigma_1,\ldots,\sigma_{t})$ is a vector of i.i.d. values of the
random variable $\varsigma$. When a different random projection is chosen for
each tree (Algorithm~\ref{alg:output-subspace-tree-ensemble}), the algorithm
computes an approximation, denoted by $f_2(x;\mathcal{L},\phi^t,\sigma^t)$, of
the following form:
$$f_{2;\mathcal{L},\phi^t,\sigma^t}(x)=\frac{1}{t}\sum_{i=1}^t
f_{\mathcal{L},\phi_i, \sigma_i}(x),$$ where $\phi^t=(\phi_1,\ldots,\phi_t)$ is
also a vector of i.i.d. random projection matrices.

We would like to compare the average errors of these two algorithms with the
average errors of the original single tree method, where the average is taken
for all algorithms over their random parameters (that include the learning
sample).

Given that all trees are grown independently of each other, one can show that
the average models corresponding to each algorithm are equal:
\begin{eqnarray*}
\bar{f}(x)&=&\E_{\mathcal{L},\Phi,\varsigma^t}\{f_{1;\mathcal{L},\Phi,\varsigma^t}(x)\}\\
&=&\E_{\mathcal{L},\Phi^t,\varsigma^t}\{f_{2;\mathcal{L},\Phi^t,\varsigma^t}(x)\}.
\end{eqnarray*}
They thus all have the exact same bias (and residual error) and differ only in
their variance.

Using the same argument, the first term of the variance decomposition
in (\ref{eqn:totalvar1}), ie. $V_{\mathcal{L}}(x)$, is the same for all three
algorithms since:
\begin{eqnarray*}
&&\E_{\Phi,\varsigma|\mathcal{L}}\{f_{\mathcal{L}, \phi, \varsigma}(x)\}\\
&=&\E_{\Phi,\varsigma^t|\mathcal{L}}\{f_{1;\mathcal{L},\Phi,\varsigma^t}(x)\}\\
&=&\E_{\Phi^t,\varsigma^t|\mathcal{L}}\{f_{2;\mathcal{L},\Phi^t,\varsigma^t}(x)\}.
\end{eqnarray*}
Their variance thus only differ in the second term of
Equation~\ref{eqn:totalvar1}.

Again, because of the conditional independence of the ensemble terms given the
learning set $\mathcal{L}$ and the projection matrix $\phi$,
Algorithm~\ref{alg:output-fix-subspace-tree-ensemble}, which keeps the output
projection fixed for all trees, is such that
$$\E_{\sigma^t|\mathcal{L},\Phi}\{f_{1;\mathcal{L},\Phi,\varsigma^t}(x)\} =
\E_{\sigma|\mathcal{L},\Phi}\{f_{\mathcal{L}, \phi, \varsigma}(x)\}$$ and
\[
\Var_{\varsigma^t|\mathcal{L},\Phi}\{f_{1;\mathcal{L},\Phi,\varsigma^t}(x)\} =
\frac{1}{t}\Var_{\varsigma|\mathcal{L},\Phi}\{f_{\mathcal{L}, \phi, \varsigma}(x)\}. \]

It thus divides the second term of Equation~\ref{eqn:totalvar2} by the number $t$ of
ensemble terms. Algorithm~\ref{alg:output-subspace-tree-ensemble}, on the other
hand, is such that:
\[
\Var_{\Phi^t,\varsigma^t|\mathcal{L}}\{f_{2;\mathcal{L},\Phi,\varsigma^t}(x)\}
= \frac{1}{t}\Var_{\Phi,\varsigma|\mathcal{L}}\{f_{\mathcal{L}, \phi, \varsigma}(x)\},
\]
and thus divides the second term of Equation~\ref{eqn:totalvar1} by $t$.

Putting all these results together one gets that:
\begin{eqnarray*}
&&\hspace*{-5mm}\E_{\mathcal{L},\Phi,\varsigma}\{Err(f_{\mathcal{L},\Phi,\varsigma^t}(x))\}\\
&&\hspace{-0.2cm}=\sigma^2_R(x)+B^2(x)+V_{\mathcal{L}}(x)+V_{Algo}(x)+V_{Proj}(x),\\
&&\hspace*{-5mm}\E_{\mathcal{L},\Phi,\varsigma^t}\{Err(f_{1;\mathcal{L},\Phi,\varsigma^t}(x))\}\\
&&\hspace{-0.2cm}=\sigma^2_R(x)+B^2(x)+V_{\mathcal{L}}(x)+\frac{V_{Algo}(x)}{t}+V_{Proj}(x),\\
&&\hspace*{-5mm}\E_{\mathcal{L},\Phi^t,\varsigma^t}\{Err(f_{2;\mathcal{L},\Phi^t,\varsigma^t}(x))\}\\
&&\hspace{-0.2cm} =\sigma^2_R(x)+B^2(x)+V_{\mathcal{L}}(x)+\frac{V_{Algo}(x)+V_{Proj}(x)}{t}.
\end{eqnarray*}

Given that all terms are positive, this result clearly shows that
Algorithm~\ref{alg:output-subspace-tree-ensemble} can not be worse than
Algorithm~\ref{alg:output-fix-subspace-tree-ensemble}.

\section{Experiments} \label{sec:rf-rp-experiments}

\subsection{Effect of the size $q$ of the Gaussian output space}
\label{sec:empirical-convergence-delicious}

\begin{figure}[htb]
\centering
\includegraphics[width=0.75\textwidth]{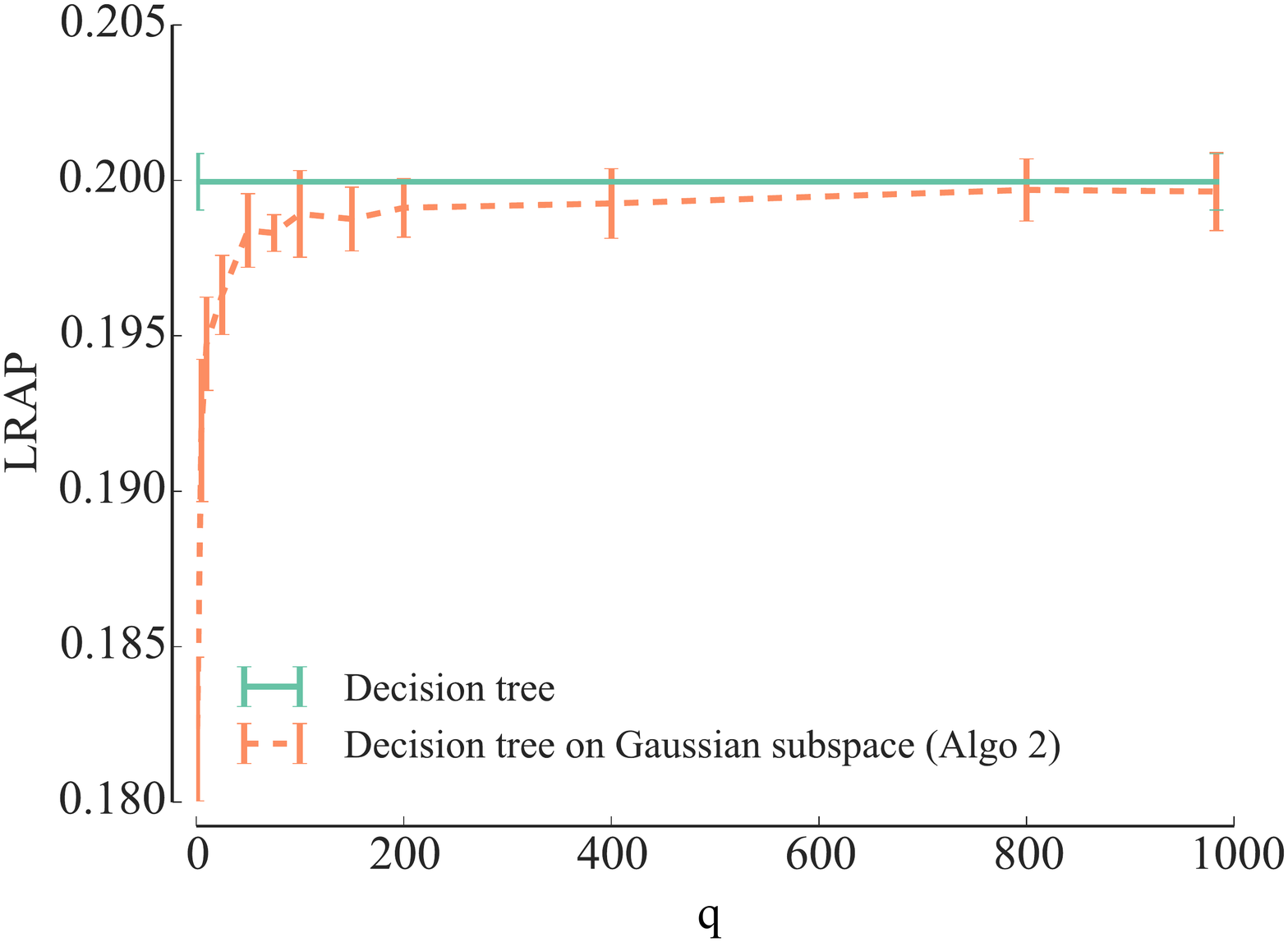}\\
\includegraphics[width=0.75\textwidth]{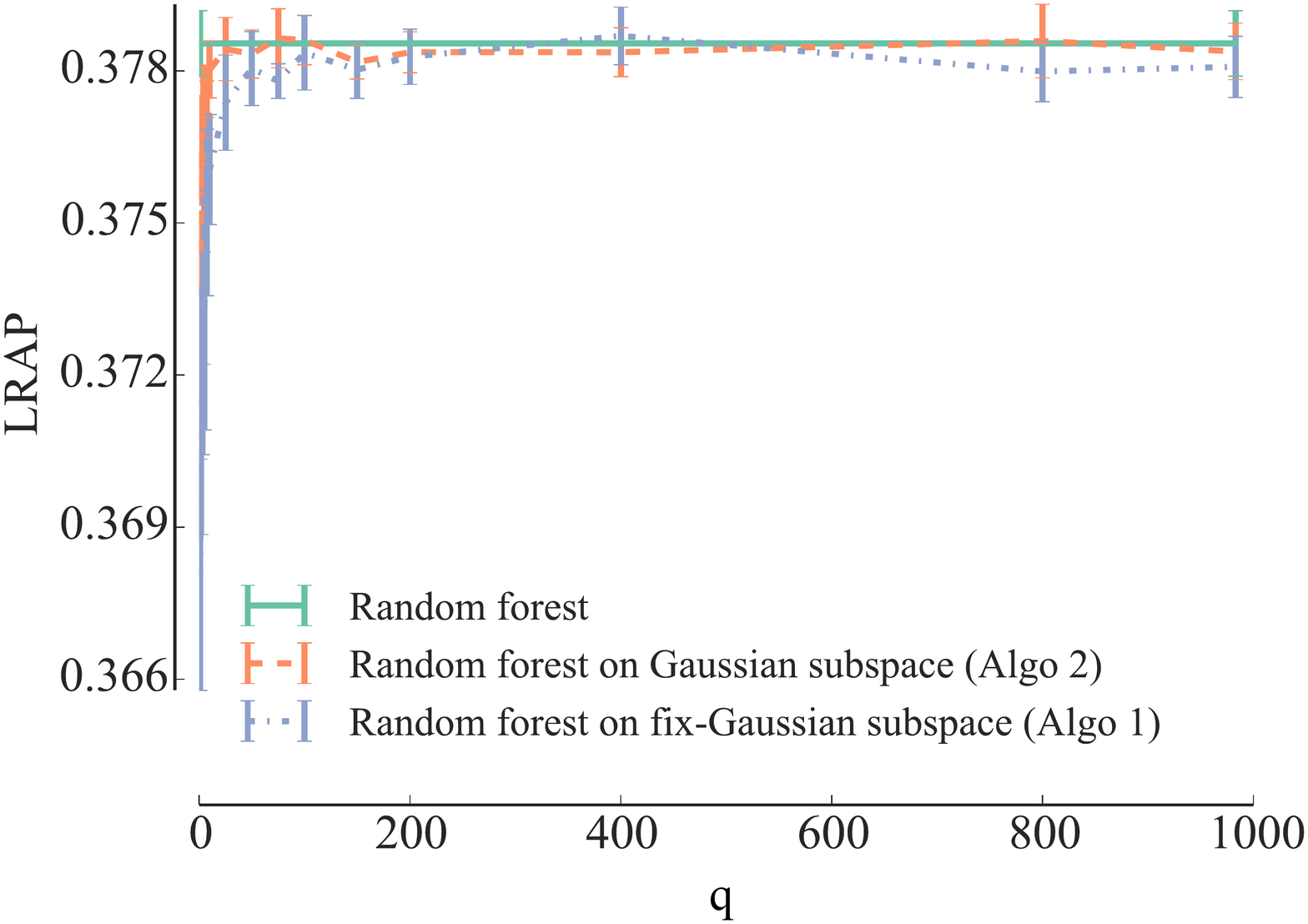}
\caption{Models built for the ``Delicious'' dataset ($d=983$) for growing
         numbers $q$ of Gaussian projections. Top: single unpruned CART
         trees ($n_{\min}=1$); Bottom: Random Forests ($k=\sqrt{p}$, $t=100$,
         $n_{\min}=1$). The curves represent average values (and standard
         deviations) obtained from 10 applications of the randomized algorithms
         over a same single $LS/TS$ split.}
\label{fig:delicious_vs_n_proj}
\end{figure}

To illustrate the behaviour of our algorithms, we first focus on the
``Delicious''  dataset \cite{tsoumakas2008effective}, which has a large number of labels ($d=983$), of input
features ($p=500$), and of training ($n_{LS} = 12920$) and testing
($n_{TS}=3185$) samples.

The top part of figure~\ref{fig:delicious_vs_n_proj} shows, when Gaussian
output-space projections are combined with the standard CART algorithm
building a single tree, how the precision converges
(cf Theorem~\ref{thm:var-jl-lemma}) when $q$ increases towards $d$.
We observe that in this case, convergence is reached around $q=200$
at the expense of a slight decrease of accuracy, so that a
compression factor of about 5 is possible with respect to the original output
dimension $d=983$.

The bottom part of figure~\ref{fig:delicious_vs_n_proj} shows, on the same
dataset, how the method behaves when combined with Random Forests. Let us first
notice that the Random Forests  grown on the original output space (green line)
are significantly more accurate than the single trees, their accuracy being
almost twice as high.
We also observe that Algorithm~\ref{alg:output-subspace-tree-ensemble}
(orange curve) converges much more rapidly than
Algorithm~\ref{alg:output-fix-subspace-tree-ensemble} (blue curve) and
slightly outperforms the Random Forest grown on the original output space.
It needs only about $q=25$ components to converge, while
Algorithm~\ref{alg:output-fix-subspace-tree-ensemble} needs about $q=75$ of
them. These results are in accordance with the analysis of
Section~\ref{sec:biasvar}, showing that
Algorithm~\ref{alg:output-subspace-tree-ensemble} can't be inferior to
Algorithm~\ref{alg:output-fix-subspace-tree-ensemble}. In the rest of this
chapter we will therefore focus on Algorithm~\ref{alg:output-subspace-tree-ensemble}.

\subsection{Systematic analysis over 24 datasets}
\label{sec:empirical-convergence-jl}

To assess our methods, we have collected 24 different multi-label
classification datasets from the literature (see Section~D of the supplementary
material, for more information and bibliographic references to these datasets)
covering a broad spectrum of application domains and ranges of the output
dimension ($d \in [6 ; 3993]$, see Table~\ref{tab:random-forest}). For 21 of
the datasets, we made experiments where the dataset is split randomly into a
learning set of size $n_{LS}$, and a test set of size $n_{TS}$, and are
repeated 10 times (to get average precisions and standard deviations), and for
3 of them we used a ten-fold cross-validation scheme (see Table
\ref{tab:random-forest}).

Table~\ref{tab:random-forest} shows our results on the 24 multi-label datasets,
by comparing Random Forests learnt on the original output space with those
learnt by Algorithm~\ref{alg:output-subspace-tree-ensemble} combined with
Gaussian subspaces of size $q\in \left\{1, d, \ln{d}\right\}$\footnote{$\ln d$
is rounded to the nearest integer value; in Table~\ref{tab:random-forest} the
values of $\ln{d}$ vary between 2 for $d=6$ and 8 for $d=3993$.}. In these
experiments, the three parameters of Random Forests are set respectively to
$k=\sqrt{p}$, $n_{\min}=1$ (default values, see
\cite{geurts2006extremely}) and $t=100$ (reasonable computing budget).
Each model is learnt ten times on a different shuffled train/testing split,
except for the 3 EUR-lex datasets where we kept the original 10 folds of
cross-validation.

We observe that for all datasets (except maybe SCOP-GO), taking $q=d$ leads to
a similar average precision to the standard Random Forests, i.e. no difference
superior to one standard deviation of the error.  On 11 datasets, we see that
$q=1$ already yields a similar average precision (values not underlined in
column $q=1$). For the 13 remaining datasets, increasing $q$ to $\ln{d}$
significantly decreases the gap with the Random Forest baseline and 3 more
datasets reach this baseline. We also observe that on several datasets such as
``Drug-interaction'' and ``SCOP-GO'', better performance on the Gaussian
subspace is attained with high output randomization
($q=\left\{1,\ln{d}\right\}$) than with $q=d$. We thus conclude that the
optimal level of output randomization (i.e. the optimal value of the ratio
$q/d$) which maximizes accuracy performances, is dataset dependent.

While our method is intended for tasks with very high dimensional output
spaces, we however notice that even with relatively small numbers of labels,
its accuracy remains comparable to the baseline, with suitable $q$.

\begin{table}[h!]
\centering
\caption{High output space compression ratio is possible, with no or negligible
  average precision reduction ($t=100$, $n_{\min}=1$, $k=\sqrt{p}$). Each
  dataset has $n_{LS}$ training samples, $n_{TS}$ testing samples, $p$ input
  features and $d$ labels. Label ranking average precisions are displayed in
  terms of their mean values and standard deviations over 10 random $LS/TS$
  splits, or over the 10 folds of cross-validation. Mean scores in the last
  three columns are underlined if they show a difference with respect to the
  standard Random Forests of more than one standard deviation.}
\renewcommand{\tabcolsep}{1.3mm}
{\small \begin{tabular}{@{} l  r lll ll ll ll  @{}}
\toprule
Datasets                                       & \multicolumn{2}{c}{Random}      & \multicolumn{6}{c}{Random Forests on Gaussian sub-space } \\
\cmidrule(r){6-7}\cmidrule(r){8-9} \cmidrule(r){10-11}
Name     &      \multicolumn{2}{c}{Forests}              & \multicolumn{2}{c}{$q=1$} & \multicolumn{2}{c}{$q\hspace*{-0.9mm}=\hspace*{-0.9mm}\lfloor\hspace*{-0.4mm} 0.5\hspace*{-0.8mm} +\hspace*{-0.8mm} \ln{d}\hspace*{-0.4mm}\rfloor$} & \multicolumn{2}{c}{$q=d$}\\
\midrule
emotions            & $0.800   $ &$\hspace*{-1.5mm} \pm 0.014$ & {$0.800$} & $\hspace*{-1.5mm} \pm 0.010 $ & $0.810  $ & $\hspace*{-1.5mm} \pm 0.014$ & $0.810  $ & $\hspace*{-1.5mm} \pm 0.016$ \\
scene               & $0.870 $ & $\hspace*{-1.5mm} \pm 0.003$ & {$0.875$} & $\hspace*{-1.5mm} \pm 0.007$ & $0.872 $ & $\hspace*{-1.5mm} \pm 0.004$ & $0.872 $ & $\hspace*{-1.5mm} \pm 0.004$ \\
yeast               & $0.759 $ & $\hspace*{-1.5mm} \pm 0.008$ & \dotuline{$ 0.748 $} & $\hspace*{-1.5mm} \pm 0.006$ & {$0.755$} & $\hspace*{-1.5mm} \pm 0.004$ & $0.758 $ & $\hspace*{-1.5mm} \pm 0.005$ \\
tmc2017             & $0.756 $ & $\hspace*{-1.5mm} \pm 0.003$ & \dotuline{$ 0.741 $} & $\hspace*{-1.5mm} \pm 0.003$ & \dotuline{$ 0.748 $} & $\hspace*{-1.5mm} \pm 0.003$ & $0.757 $ & $\hspace*{-1.5mm} \pm 0.003$ \\
genbase             & $0.992 $ & $\hspace*{-1.5mm} \pm 0.004$ & {$0.994$} & $\hspace*{-1.5mm} \pm 0.002$ & $0.994 $ & $\hspace*{-1.5mm} \pm 0.004$ & $0.993 $ & $\hspace*{-1.5mm} \pm 0.004$ \\
reuters             & $0.865 $ & $\hspace*{-1.5mm} \pm 0.004$ & {$0.864$} & $\hspace*{-1.5mm} \pm 0.003$ & $0.863 $ & $\hspace*{-1.5mm} \pm 0.004$ & $0.862 $ & $\hspace*{-1.5mm} \pm 0.004$ \\
medical             & $0.848 $ & $\hspace*{-1.5mm} \pm 0.009$ & \dotuline{$ 0.836 $} & $\hspace*{-1.5mm} \pm 0.011$ & {$0.842$} & $\hspace*{-1.5mm} \pm 0.014$ & $0.841 $ & $\hspace*{-1.5mm} \pm 0.009$ \\
enron               & $0.683 $ & $\hspace*{-1.5mm} \pm 0.009$ & {$0.680$} & $\hspace*{-1.5mm} \pm 0.006$ & $0.685 $ & $\hspace*{-1.5mm} \pm 0.009$ & $0.686 $ & $\hspace*{-1.5mm} \pm 0.008$ \\
mediamill           & $0.779 $ & $\hspace*{-1.5mm} \pm 0.001$ & \dotuline{$ 0.772 $} & $\hspace*{-1.5mm} \pm 0.001$ & {$0.777$} & $\hspace*{-1.5mm} \pm 0.002$ & $0.779 $ & $\hspace*{-1.5mm} \pm 0.002$ \\
Yeast-GO            & $0.420  $ & $\hspace*{-1.5mm} \pm 0.010 $ & \dotuline{$ 0.353 $} & $\hspace*{-1.5mm} \pm 0.008$ & \dotuline{$ 0.381 $} & $\hspace*{-1.5mm} \pm 0.005$ & $0.420  $ & $\hspace*{-1.5mm} \pm 0.010$ \\
bibtex              & $0.566 $ & $\hspace*{-1.5mm} \pm 0.004$ & \dotuline{$ 0.513 $} & $\hspace*{-1.5mm} \pm 0.006$ & \dotuline{$ 0.548 $} & $\hspace*{-1.5mm} \pm 0.007$ & $0.564 $ & $\hspace*{-1.5mm} \pm 0.008$ \\
CAL500              & $0.504 $ & $\hspace*{-1.5mm} \pm 0.011$ & {$0.504$} & $\hspace*{-1.5mm} \pm 0.004$ & $0.506 $ & $\hspace*{-1.5mm} \pm 0.007$ & $0.502 $ & $\hspace*{-1.5mm} \pm 0.010$ \\
WIPO                & $0.490  $ & $\hspace*{-1.5mm} \pm 0.010 $ & \dotuline{$ 0.430  $} & $\hspace*{-1.5mm} \pm 0.010 $ & \dotuline{$ 0.460  $} & $\hspace*{-1.5mm} \pm 0.010 $ & $0.480  $ & $\hspace*{-1.5mm} \pm 0.010$ \\
EUR-Lex (subj.)     & $0.840  $ & $\hspace*{-1.5mm} \pm 0.005$ & \dotuline{$ 0.814 $} & $\hspace*{-1.5mm} \pm 0.004$ & \dotuline{$ 0.828 $} & $\hspace*{-1.5mm} \pm 0.005$ & $0.840  $ & $\hspace*{-1.5mm} \pm 0.004$ \\
bookmarks           & $0.453$ & $\hspace*{-1.5mm} \pm 0.001$& \dotuline{$ 0.436 $} & $\hspace*{-1.5mm} \pm 0.002$ & \dotuline{$ 0.445 $} & $\hspace*{-1.5mm} \pm 0.002$ & $0.453 $ & $\hspace*{-1.5mm} \pm 0.002$ \\
diatoms             & $0.700   $ & $\hspace*{-1.5mm} \pm 0.010 $ & \dotuline{$ 0.650  $} & $\hspace*{-1.5mm} \pm 0.010 $ & \dotuline{$ 0.670  $} & $\hspace*{-1.5mm} \pm 0.010$  & $0.710  $ & $\hspace*{-1.5mm} \pm 0.020$ \\
corel5k             & $0.303 $ & $\hspace*{-1.5mm} \pm 0.012$ & {$0.309$} & $\hspace*{-1.5mm} \pm 0.011$ & $0.307 $ & $\hspace*{-1.5mm} \pm 0.011$ & $0.299 $ & $\hspace*{-1.5mm} \pm 0.013$ \\
EUR-Lex (dir.)      & $0.814 $ & $\hspace*{-1.5mm} \pm 0.006$ & \dotuline{$ 0.782 $} & $\hspace*{-1.5mm} \pm 0.008$ & \dotuline{$ 0.796 $} & $\hspace*{-1.5mm} \pm 0.009$ & $0.813 $ & $\hspace*{-1.5mm} \pm 0.007$ \\
SCOP-GO             & $0.811 $ & $\hspace*{-1.5mm} \pm 0.004$ & {$0.808$} & $\hspace*{-1.5mm} \pm 0.005$ & $0.811 $ & $\hspace*{-1.5mm} \pm 0.004$ & \dotuline{$ 0.806 $} & $\hspace*{-1.5mm} \pm 0.004$ \\
delicious           & $0.384 $ & $\hspace*{-1.5mm} \pm 0.004$ & {$0.381$} & $\hspace*{-1.5mm} \pm 0.003$ & $0.382 $ & $\hspace*{-1.5mm} \pm 0.002$ & $0.383 $ & $\hspace*{-1.5mm} \pm 0.004$ \\
drug-interaction    & $0.379 $ & $\hspace*{-1.5mm} \pm 0.014$ & {$0.384$} & $\hspace*{-1.5mm} \pm 0.009$ & $0.378 $ & $\hspace*{-1.5mm} \pm 0.013$ & $0.367 $ & $\hspace*{-1.5mm} \pm 0.016$ \\
protein-interaction & $0.330  $ & $\hspace*{-1.5mm} \pm 0.015$ & {$0.337$} & $\hspace*{-1.5mm} \pm 0.016$ & $0.337 $ & $\hspace*{-1.5mm} \pm 0.017$ & $0.335 $ & $\hspace*{-1.5mm} \pm 0.014$ \\
Expression-GO       & $0.235 $ & $\hspace*{-1.5mm} \pm 0.005$ & \dotuline{$ 0.211 $} & $\hspace*{-1.5mm} \pm 0.005$ & \dotuline{$ 0.219 $} & $\hspace*{-1.5mm} \pm 0.005$ & $0.232 $ & $\hspace*{-1.5mm} \pm 0.005$ \\
EUR-Lex (desc.)     & $0.523 $ & $\hspace*{-1.5mm} \pm 0.008$ & \dotuline{$ 0.485 $} & $\hspace*{-1.5mm} \pm 0.008$ & \dotuline{$ 0.497 $} & $\hspace*{-1.5mm} \pm 0.009$ & $0.523    $ & $\hspace*{-1.5mm} \pm 0.007   $ \\
\bottomrule
\end{tabular}}
\label{tab:random-forest}
\end{table}

To complete the analysis, let's carry out the same experiments with a
different base-learner combining Gaussian random projections (with
$q\in\{1,\ln{d},d\}$) with the Extra Trees method of~\cite{geurts2006extremely}.
Results on 23 datasets are compiled in
Table~\ref{tab:extra_trees_results_summary}.

Like for Random Forests, we observe that for all 23 datasets taking
$q=d$ leads to a similar average precision to the standard Random
Forests, ie. no difference superior to one standard deviation of the
error. This is already the case with $q=1$ for 12 datasets and with
$q=\ln{d}$ for 4 more datasets. Interestingly, on 3 datasets with
$q=1$ and 3 datasets with $q=\ln{d}$, the increased randomization
brought by the projections actually improves average precision with
respect to standard Random Forests (bold values in
Table~\ref{tab:extra_trees_results_summary}).

\begin{table*}[h!]
\renewcommand{\tabcolsep}{1.3mm}

\centering
\caption{Experiments with Gaussian projections and Extra Trees
  (($t=100$, $n_{\min}=1$, $k=\sqrt{p}$). Mean scores in the last three columns
  are underlined if they show a negative difference with respect to the
  standard Random Forests of more than one standard deviation. Bold
  values highlight improvement over standard RF of more than one standard
  deviation.}
{\small \begin{tabular}{@{} lll ll ll ll  @{}}
\toprule
Datasets            & \multicolumn{2}{c}{Extra trees}      & \multicolumn{6}{c}{Extra trees on Gaussian sub-space } \\
\cmidrule(r){4-5} \cmidrule(r){6-7}\cmidrule(r){8-9}
                   &          &             & \multicolumn{2}{c}{$q=1$} & \multicolumn{2}{c}{$q\hspace*{-0.9mm}=\hspace*{-0.9mm}\lfloor\hspace*{-0.4mm} 0.5\hspace*{-0.8mm} +\hspace*{-0.8mm} \ln{d}\hspace*{-0.4mm}\rfloor$} & \multicolumn{2}{c}{$q=d$}\\
\midrule
emotions            & $0.81  $ & $\hspace*{-1.5mm} \pm 0.01$  & $0.81  $ & $\hspace*{-1.5mm} \pm 0.014$ & $0.80  $ & $\hspace*{-1.5mm} \pm 0.013$ & $0.81 $ & $\hspace*{-1.5mm} \pm 0.014$ \\
scene               & $0.873 $ & $\hspace*{-1.5mm} \pm 0.004$ & $0.876 $ & $\hspace*{-1.5mm} \pm 0.003$ & $0.877 $ & $\hspace*{-1.5mm} \pm 0.007$ & $0.878 $ & $\hspace*{-1.5mm} \pm 0.006$ \\
yeast               & $0.757 $ & $\hspace*{-1.5mm} \pm 0.008$ & \dotuline{$0.746 $} & $\hspace*{-1.5mm} \pm 0.004$ & $0.752 $ & $\hspace*{-1.5mm} \pm 0.009$ & $0.757 $ & $\hspace*{-1.5mm} \pm 0.01$ \\
tmc2017             & $0.782 $ & $\hspace*{-1.5mm} \pm 0.003$ & \dotuline{$0.759 $} & $\hspace*{-1.5mm} \pm 0.004$ & \dotuline{$0.77  $} & $\hspace*{-1.5mm} \pm 0.002$ & $0.779 $ & $\hspace*{-1.5mm} \pm 0.002$ \\
genbase             & $0.987 $ & $\hspace*{-1.5mm} \pm 0.005$ & $0.991 $ & $\hspace*{-1.5mm} \pm 0.004$ & $0.992 $ & $\hspace*{-1.5mm} \pm 0.001$ & $0.992 $ & $\hspace*{-1.5mm} \pm 0.005$ \\
reuters             & $0.88  $ & $\hspace*{-1.5mm} \pm 0.003$ & $0.88  $ & $\hspace*{-1.5mm} \pm 0.003$ & $0.878 $ & $\hspace*{-1.5mm} \pm 0.004$ & $0.88  $ & $\hspace*{-1.5mm} \pm 0.004$ \\
medical             & $0.855 $ & $\hspace*{-1.5mm} \pm 0.008$ & $\mathbf{0.867}$ & $\hspace*{-1.5mm} \pm 0.009$ & $\mathbf{0.872}$ & $\hspace*{-1.5mm} \pm 0.006$ & $0.862 $ & $\hspace*{-1.5mm} \pm 0.008$ \\
enron               & $0.66  $ & $\hspace*{-1.5mm} \pm 0.01$  & $0.65  $ & $\hspace*{-1.5mm} \pm 0.01 $ & $0.663 $ & $\hspace*{-1.5mm} \pm 0.008$ & $0.66  $ & $\hspace*{-1.5mm} \pm 0.01$ \\
mediamill           & $0.786 $ & $\hspace*{-1.5mm} \pm 0.002$ & \dotuline{$0.778 $} & $\hspace*{-1.5mm} \pm 0.002$ & \dotuline{$0.781 $} & $\hspace*{-1.5mm} \pm 0.002$ & $0.784 $ & $\hspace*{-1.5mm} \pm 0.001$ \\
Yeast-GO            & $0.49  $ & $\hspace*{-1.5mm} \pm 0.009$ & \dotuline{$0.47  $} & $\hspace*{-1.5mm} \pm 0.01 $ & $0.482 $ & $\hspace*{-1.5mm} \pm 0.008$ & $0.48  $  & $\hspace*{-1.5mm} \pm 0.01$ \\
bibtex              & $0.584 $ & $\hspace*{-1.5mm} \pm 0.005$ & \dotuline{$0.538 $} & $\hspace*{-1.5mm} \pm 0.005$ & \dotuline{$0.564 $} & $\hspace*{-1.5mm} \pm 0.004$ & $0.583 $ & $\hspace*{-1.5mm} \pm 0.004$ \\
CAL500              & $0.5   $ & $\hspace*{-1.5mm} \pm 0.007$ & $0.502 $ & $\hspace*{-1.5mm} \pm 0.008$ & $0.499 $ & $\hspace*{-1.5mm} \pm 0.007$ & $0.503 $ & $\hspace*{-1.5mm} \pm 0.009$ \\
WIPO                & $0.52  $ & $\hspace*{-1.5mm} \pm 0.01 $ & \dotuline{$0.474 $} & $\hspace*{-1.5mm} \pm 0.007$ & \dotuline{$0.49  $} & $\hspace*{-1.5mm} \pm 0.01 $ & $0.515 $ & $\hspace*{-1.5mm} \pm 0.006$ \\
EUR-Lex (subj.)     & $0.845 $ & $\hspace*{-1.5mm} \pm 0.006$ & \dotuline{$0.834 $} & $\hspace*{-1.5mm} \pm 0.004$ & \dotuline{$0.838 $} & $\hspace*{-1.5mm} \pm 0.003$ & $0.845 $ & $\hspace*{-1.5mm} \pm 0.005$ \\
bookmarks           & $0.453 $ & $\hspace*{-1.5mm} \pm 0.002$ & \dotuline{$0.436 $} & $\hspace*{-1.5mm} \pm 0.002$ & \dotuline{$0.444 $} & $\hspace*{-1.5mm} \pm 0.003$ & $0.452 $ & $\hspace*{-1.5mm} \pm 0.002$ \\
diatoms             & $0.73  $ & $\hspace*{-1.5mm} \pm 0.01 $ & \dotuline{$0.69  $} & $\hspace*{-1.5mm} \pm 0.01$  & \dotuline{$0.71  $} & $\hspace*{-1.5mm} \pm 0.01 $ & $0.73  $ & $\hspace*{-1.5mm} \pm 0.01 $ \\
corel5k             & $0.285 $ & $\hspace*{-1.5mm} \pm 0.009$ & $\mathbf{0.313}$ & $\hspace*{-1.5mm} \pm 0.011$ & $\mathbf{0.309}$ & $\hspace*{-1.5mm} \pm 0.009$ & $0.285 $ & $\hspace*{-1.5mm} \pm 0.011$ \\
EUR-Lex (dir.)      & $0.815 $ & $\hspace*{-1.5mm} \pm 0.007$ & \dotuline{$0.805 $} & $\hspace*{-1.5mm} \pm 0.006$ & \dotuline{$0.807 $} & $\hspace*{-1.5mm} \pm 0.009$ & $0.815 $ & $\hspace*{-1.5mm} \pm 0.007$ \\
SCOP-GO             & $0.778 $ & $\hspace*{-1.5mm} \pm 0.005$ & $0.782 $ & $\hspace*{-1.5mm} \pm 0.004$ & $0.782 $ & $\hspace*{-1.5mm} \pm 0.006$ & $0.778 $ & $\hspace*{-1.5mm} \pm 0.005$ \\
delicious           & $0.354 $ & $\hspace*{-1.5mm} \pm 0.003$ & $\mathbf{0.36}$ & $\hspace*{-1.5mm} \pm 0.004$ & $\mathbf{0.358}$ & $\hspace*{-1.5mm} \pm 0.004$ & $0.355 $ & $\hspace*{-1.5mm} \pm 0.003$ \\
drug-interaction    & $0.353 $ & $\hspace*{-1.5mm} \pm 0.011$ & $\mathbf{0.375}$ & $\hspace*{-1.5mm} \pm 0.017$ & $0.364 $ & $\hspace*{-1.5mm} \pm 0.014$ & $0.355 $ & $\hspace*{-1.5mm} \pm 0.016$ \\
protein-interaction & $0.299 $ & $\hspace*{-1.5mm} \pm 0.013$ & $0.307 $ & $\hspace*{-1.5mm} \pm 0.009$ & $0.305 $ & $\hspace*{-1.5mm} \pm 0.012$ & $0.306 $ & $\hspace*{-1.5mm} \pm 0.017$ \\
Expression-GO       & $0.231 $ & $\hspace*{-1.5mm} \pm 0.007$ & \dotuline{$0.218 $} & $\hspace*{-1.5mm} \pm 0.005$ & $0.228 $ & $\hspace*{-1.5mm} \pm 0.005$ & $0.235 $ & $\hspace*{-1.5mm} \pm 0.005$ \\
\bottomrule
\end{tabular}}
\label{tab:extra_trees_results_summary}
\end{table*}

\clearpage

\subsection{Input vs output space randomization}
\label{sec:input-output-bias-variance-tradeoff}

We study in this section the interaction of the additional randomization of
the output space with that concerning the input space already built in the Random Forest method.

To this end, we consider the ``Drug-interaction'' dataset ($p=660$ input
features and $d=1554$ output
labels~\cite{yamanishi2011extracting}), and we study the
effect of parameter $k$ controlling the input space randomization of the Random
Forest method with the randomization of the output space by Gaussian
projections controlled by the parameter $q$. To this end, Figure
\ref{fig:max-features-drug-interaction} shows the evolution of the accuracy for
growing values of $k$ (i.e. decreasing strength of the input space
randomization), for three different quite low values of $q$ (in this case $q
\in \left\{1, \ln{d}, 2 \ln{d}\right\}$).  We observe that Random Forests
learned on a very low-dimensional Gaussian subspace (red, blue and pink curves)
yield essentially better performances than Random Forests on the original
output space, and also that their behaviour with respect to the parameter $k$
is quite different. On this dataset, the output-space randomisation makes the
method completely immune to the `over-fitting' phenomenon observed for high
values of $k$ with the baseline method (green curve).

\begin{figure}[htb]
\centering
\includegraphics[width=0.75\textwidth]{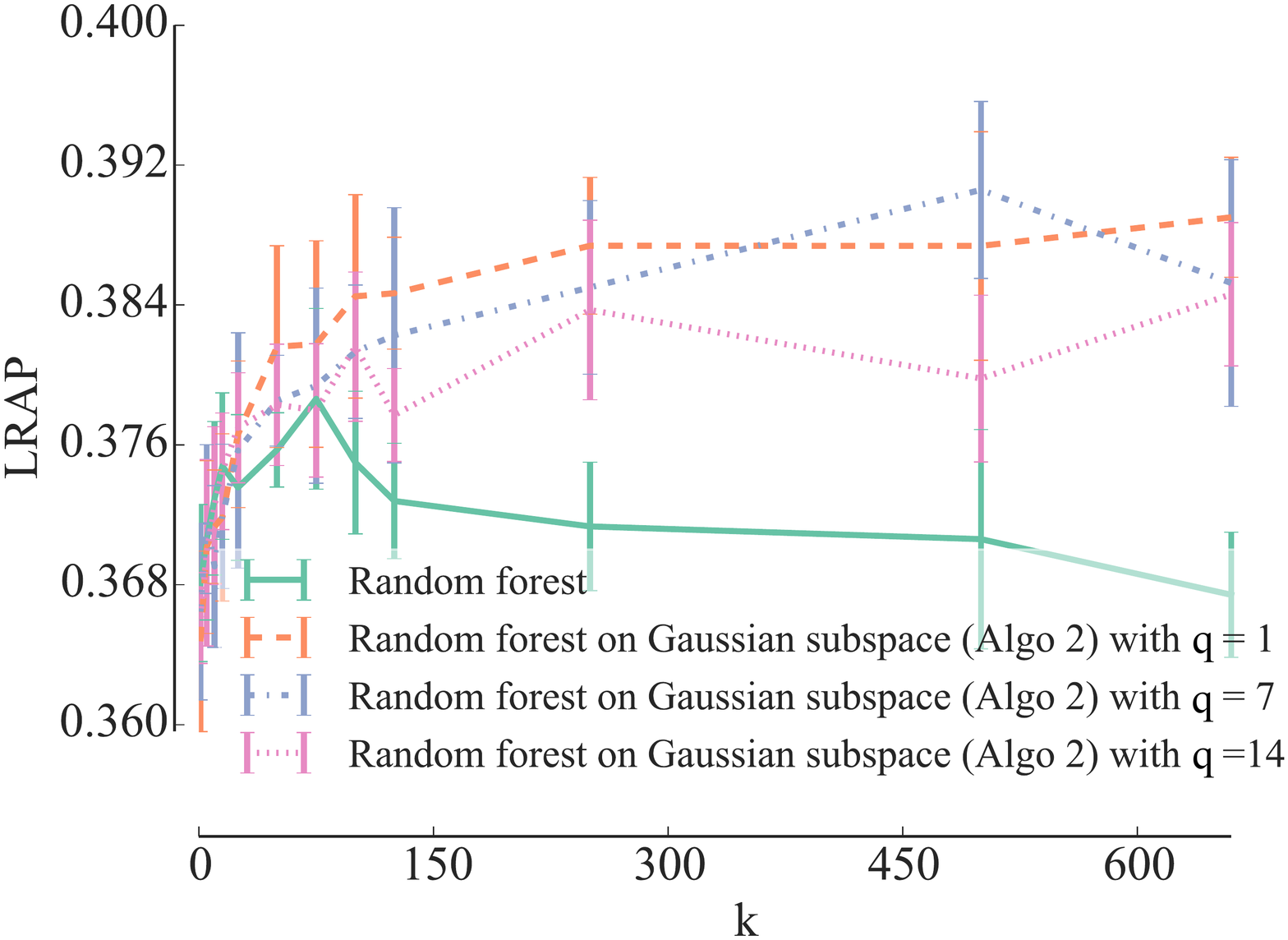}\vspace*{-3mm}
\caption{Output randomization with Gaussian projections yield better average
         precision than the original output space on the ``Drug-Interaction''
         dataset ($n_{\min}=1$ , $t=100$).}
\label{fig:max-features-drug-interaction}
\end{figure}

We carry out the same experiment, but on the ``Delicious''
dataset. Figure~\ref{fig:k_delicious_gaussian} shows the evolution of
the accuracy for growing values of $k$ (i.e. decreasing strength of
the input space randomization), for three different values of $q$ (in
this case $q \in \left\{1, \ln{d}, 2 \ln{d}\right\}$) on a Gaussian
output space.

Like on ``Drug-interaction'' (see
Figure~\ref{fig:max-features-drug-interaction}), using low-dimensional output
spaces makes the method more robust with respect to over-fitting as $k$
increases. However, unlike on ``Drug-interaction'', it is not really possible to
improve over baseline Random Forests by tuning jointly input and output
randomization. This shows that the interaction between $q$ and $k$ may be
different from one dataset to another.

\begin{figure}[htb]
\centering
\includegraphics[width=0.75\textwidth]{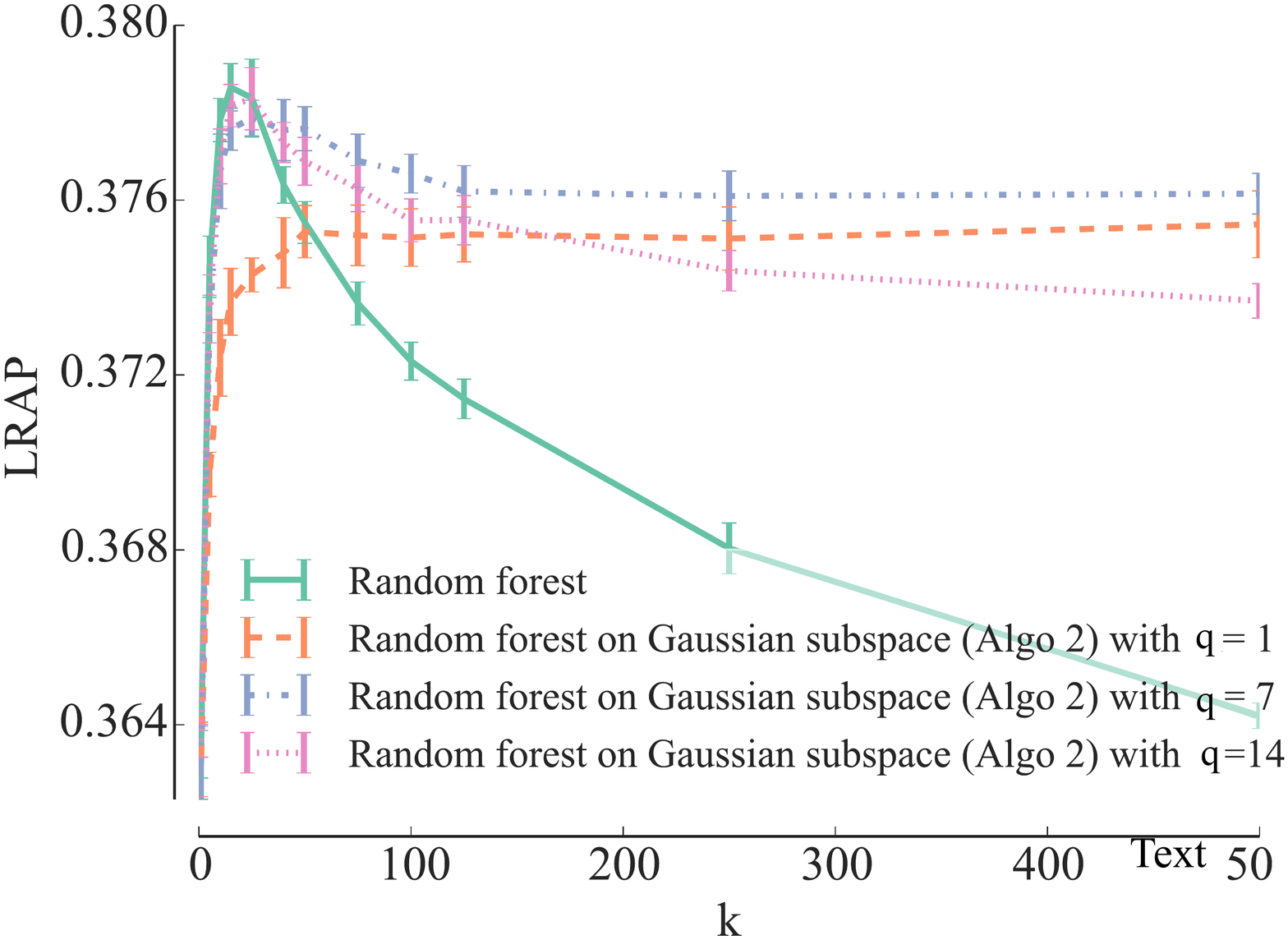}
\caption{``Delicious''  dataset: $n_{\min}=1$; $t=100$.}
\label{fig:k_delicious_gaussian}
\end{figure}

It is thus advisable to jointly optimize the value of $q$ and $k$, so as to
maximise the tradeoff between accuracy and computing times in a problem and
algorithm specific way.

\subsection{Alternative output dimension reduction techniques}
\label{sec:alternative-output-transformation}

In this section, we study Algorithm~\ref{alg:output-subspace-tree-ensemble}
when it is combined with alternative output-space dimensionality reduction
techniques. We focus again on the ``Delicious'' dataset,
but similar trends could be observed on other datasets.

Figure~\ref{fig:m_delicious_gaussian_rademacher_hadamard} first compares
Gaussian random projections with two other dense projections: Rademacher
matrices with $s=1$ (cf. Section~2.2) and compression matrices obtained by
sub-sampling (without replacement) Hadamard
matrices~\cite{candes2011probabilistic}.  We observe that Rademacher and
subsample-Hadamard sub-spaces behave very similarly to Gaussian random
projections.

In a second step, we compare Gaussian random projections with two (very) sparse
projections: first, sparse Rademacher sub-spaces obtained by setting the
sparsity parameter $s$ to $3$ and $\sqrt{d}$, selecting respectively about 33\%
and 2\% of the original outputs to compute each component, and second,
sub-sampled identity subspaces, similar to \cite{tsoumakas2007random}, where
each of the $q$ selected components corresponds to a randomly chosen original
label and also preserve sparsity. Sparse projections are very interesting from a
computational point of view as they require much less operations to compute the
projections but the number of components required for condition (\ref{eqn:js})
to be satisfied is typically higher than for dense projections
\cite{li2006very,candes2011probabilistic}.
Figure~\ref{fig:delicious_vs_n_proj_unitary} compares these three projection
methods with standard Random Forests on the ``delicious'' dataset. All three
projection methods converge to plain Random Forests as the number of components
$q$ increases but their behaviour at low $q$ values are very different.
Rademacher projections converge faster with $s=3$ than with $s=1$ and
interestingly, the sparsest variant ($s=\sqrt{d}$) has its optimum at $q=1$ and
improves in this case over the Random Forests baseline. Random output subspaces
converge slower but they lead to a notable improvement of the score over
baseline Random Forests. This suggests that although their theoretical
guarantees are less good, sparse projections actually provide on this problem a
better bias/variance tradeoff than dense ones when used in the context of
Algorithm~\ref{alg:output-subspace-tree-ensemble}.

Another popular dimension reduction technique is the principal component
analysis (PCA). In Figure~\ref{fig:delicious_vs_n_proj_pca}, we repeat the same
experiment to compare PCA with Gaussian random projections. Concerning PCA, the
curve is generated in decreasing order of eigenvalues, according to their
contribution to the explanation of the output-space variance. We observe that
this way of doing is far less effective than the random projection techniques
studied previously.

\begin{figure}[h!]
\centering
\subfloat[Computing the impurity criterion on a dense Rademacher or on a
           subsample-Hadamard output sub-space is another efficient way to learn
           tree ensembles.]{
              \includegraphics[width=0.75\textwidth]{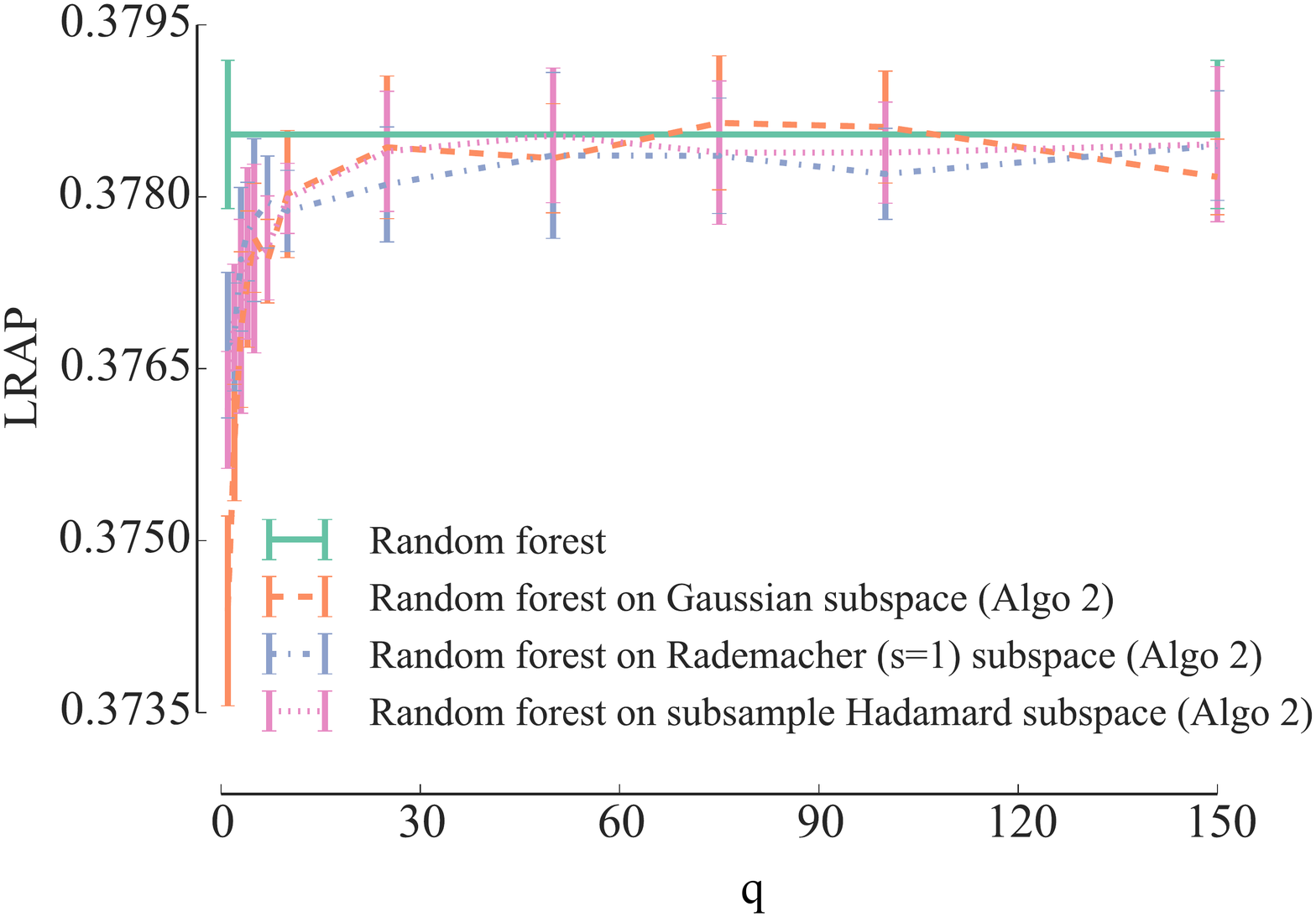}
              \label{fig:m_delicious_gaussian_rademacher_hadamard}
          } \\
\subfloat[Sparse random projections output sub-space yield better average precision
           than on the original output space.]{
              \includegraphics[width=0.75\textwidth]{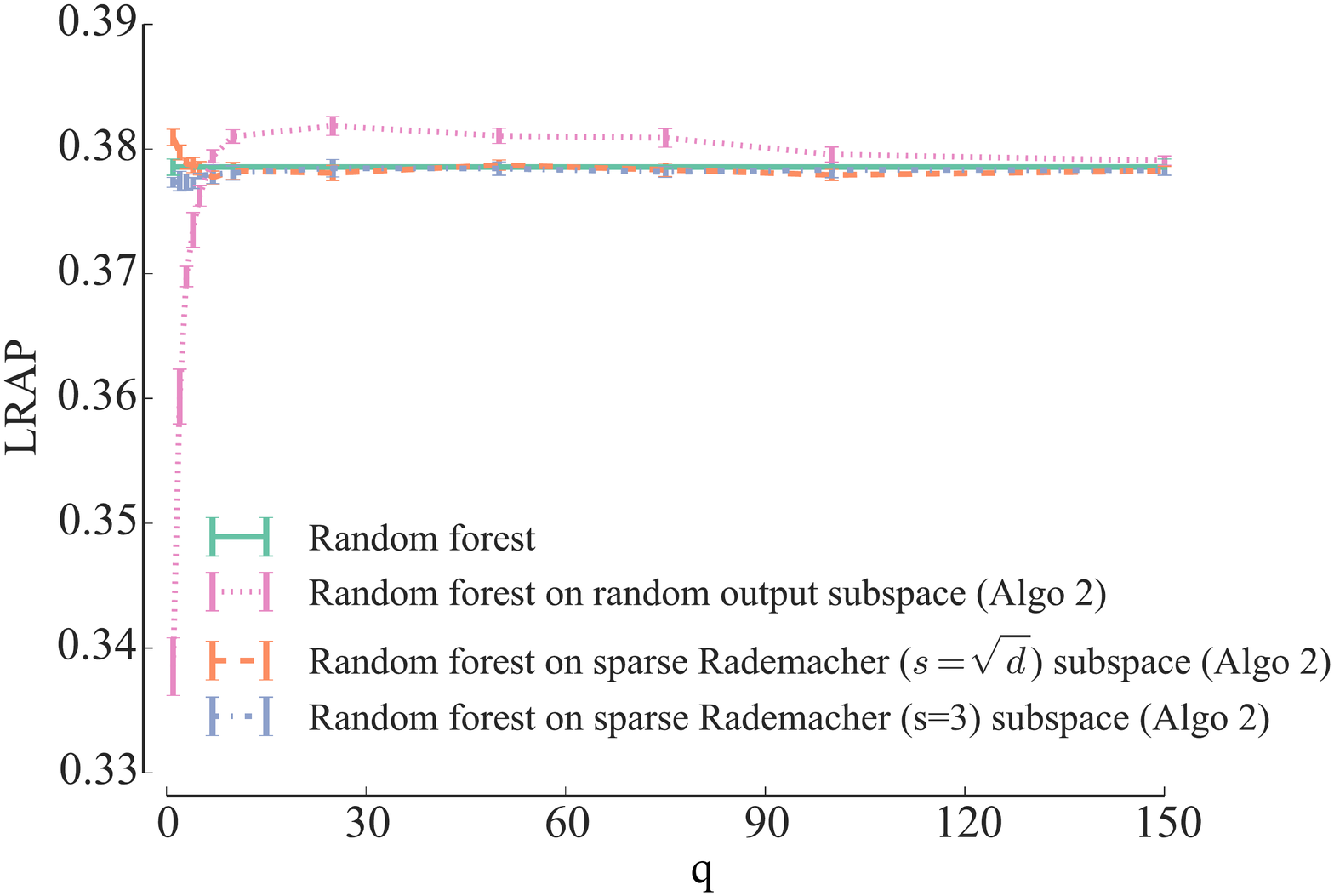}
              \label{fig:delicious_vs_n_proj_unitary}
          } \\
\subfloat[PCA compared with Gaussian subspaces.]{
              \includegraphics[width=0.75\textwidth]{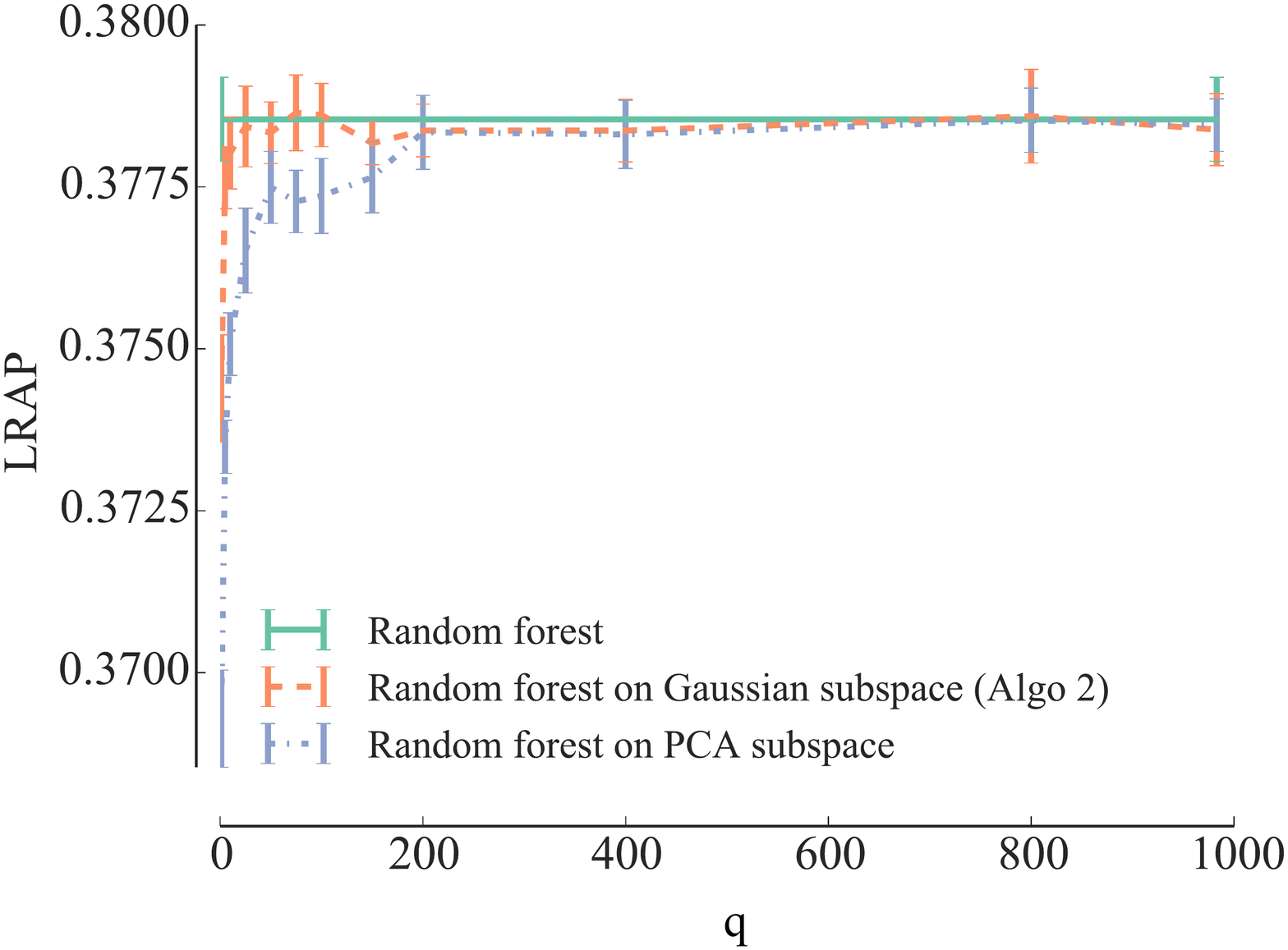}
              \label{fig:delicious_vs_n_proj_pca}
          }

\caption{``Delicious'' dataset, $t=100$, $k=\sqrt{p}$, $n_{\min}=1$.}
\end{figure}

\clearpage

\subsection{Learning stage computing times}\label{sec:benchmarking}

Our implementation of the learning algorithms  is based on the
\textit{scikit-learn} Python package version 0.14-dev
\cite{pedregosa2011scikit,buitinck2013api}. To fix ideas about computing times, we report
these obtained on a Mac Pro 4.1 with a dual Quad-Core Intel Xeon processor at
2.26 GHz, on the ``Delicious'' dataset.
Matrix operation, such as random projections, are performed with the
BLAS and the LAPACK from the Mac OS X \textit{Accelerate} framework. Reported
times are obtained by summing the user and sys time of the UNIX \textit{time}
utility.

The reported timings correspond to the following operation: (i) load
the dataset in memory, (ii) execute the algorithm. All methods use the
same code to build trees. In these conditions, learning a random
forest on the original output space ($t=100$, $n_{\min}=1$,
$k=\sqrt{d}$) takes 3348~ s; learning the same model on a Gaussian
output space of size $q=25$ requires 311~s, while $q=1$ and $q=250$
take respectively 236~s and 1088~s. Generating a
Gaussian sub-space of size $q=25$ and projecting the output data of
the training samples is done in less than 0.25~s, while $q=1$ and
$q=250$ takes around 0.07~s and 1~s respectively. The time needed to
compute the projections is thus negligible with respect to the time
needed for the tree construction.

We see that a speed-up of an order of magnitude could be obtained,
while at the same time preserving accuracy with respect to the
baseline Random Forests method. Equivalently, for a
fixed computing time budget, randomly projecting the output space
allows to build more trees and thus to improve predictive
performances with respect to standard Random Forests.

\section{Conclusions} \label{sec:rf-rp-conclusions}

This chapter explores the use of random output space projections combined with
tree-based ensemble methods to address large-scale multi-label classification
problems. We study two algorithmic variants that either build a tree-based
ensemble model on a single shared random subspace or build each tree in the
ensemble on a newly drawn random subspace. The second approach is shown
theoretically and empirically to always outperform the first in terms of
accuracy. Experiments on 24 datasets show that on most problems, using gaussian
projections allows to reduce very drastically the size of the output space, and
therefore computing times, without affecting accuracy. Remarkably, we also show
that by adjusting jointly the level of input and output randomizations and
choosing appropriately the projection method, one could also improve predictive
performance over the standard Random Forests, while still improving very
significantly computing times. As future work, it would be very interesting to
propose efficient techniques to automatically adjust these parameters, so as to
reach the best tradeoff between accuracy and computing times on a given
problem.

To the best of our knowledge, our work is the first to study random output
projections in the context of multi-output tree-based ensemble methods. The
possibility with these methods to relabel tree leaves with predictions in the
original output space makes this combination very attractive. Indeed, unlike
similar works with linear models~\cite{hsu2009multi,cisse2013robust},
our approach only relies on Johnson-Lindenstrauss lemma, and not on any output
sparsity assumption, and also does not require to use any output reconstruction
method. Besides multi-label classification, we would like to test our method on
other, not necessarily sparse, multi-output prediction problems.


\chapter[Random output space projections for gradient boosting]{Gradient boosting with random output projections for
         multi-label and multi-outputs regression tasks}
\label{ch:gbrt-output-projection}

\begin{remark}{Outline}
We first formally adapt the gradient boosting ensemble method for multi-output
supervised learning tasks such as multi-output regression and multi-label
classification. We then propose to combine single random projections of the
output space with gradient boosting on such tasks to adapt automatically to the
output correlation structure. The idea of this method is to train each weak
model on a single random projection of the output space and then to exploit the
predictions of the resulting model to approximate the gradients of all other
outputs. Through weak model sharing and random projection of the output space,
we implicitly take into account the output correlations. We perform extensive
experiments with these methods both on artificial and real problems using
tree-based weak learners. Randomly projecting the output space shows to
provide a better adaptation to different output correlation patterns and is
therefore competitive with the best of the other methods in most settings. Thanks to
the model sharing, the convergence speed is also faster, reducing the computing
times to reach a specific accuracy.

\emph{This contribution is a joint work with Pierre Geurts and Louis
Wehenkel from the University of Li\`ege.}
\end{remark}

\section{Introduction}\label{sec:introduction}

Multi-output supervised learning aims to model the input-output relationship of
a system from observations of input-output pairs whenever the output space is a
vector of random variables. Multi-output classification and regression tasks
have numerous applications in domains ranging from biology to multimedia.

The most straightforward way to address multi-output tasks is to apply standard
single output methods separately and independently on each output. Although
simple, this method, called binary relevance~\cite{tsoumakas2009mining} in
multi-label classification or single target~\cite{spyromitros2012multi} in
multi-output regression, is often suboptimal as it does not exploit potential
correlations that might exist between the outputs. For this reason, several
approaches have been proposed in the literature that improve over binary
relevance by exploiting output correlations. These approaches include for
example the explicit construction of the output dependency
graph~\cite{dembczynski2010label,gasse2015optimality,zhang2010multi} or the
sharing of models learnt for one output to the other
outputs~\cite{huang2012multi,yan2007model,read2011classifier}. Our contribution
falls into the latter category.

Classification and regression trees~\cite{breiman1984classification}
are popular supervised learning methods that provide state-of-the-art
accuracy when exploited in the context of ensemble methods, namely
Random forests~\cite{breiman2001random} and gradient
boosting~\cite{friedman2001greedy}. Classification and regression
trees have been extended by several authors to the joint prediction of
multiple outputs~\citep[see,
  e.g.,][]{segal1992tree,blockeel2000top}). These extensions build a
single tree to predict all outputs at once. They adapt the score
measure used to assess splits during the tree growth to take into
account all outputs and label each tree leaf with a vector of values,
one for each output (see Section~\ref{sec:mo-trees} for more
information). Like standard classification or regression trees,
multiple output trees can be exploited in the context of random
forests~\cite{barutcuoglu2006hierarchical,joly2014random,kocev2007ensembles,kocev2013tree,segal2011multivariate}
or boosting~\cite{geurts2007gradient} ensembles, which often offer
very significant accuracy improvements with respect to single
trees. Multiple output trees have been shown to be competitive with
other multiple output methods~\cite{madjarov2012extensive}, but, to
the best of our knowledge, it has not been studied as extensively in
the context of gradient boosting.

Binary relevance / single target of single output tree models and multiple
output tree models represent two extremes in terms of tree structure learning:
the former builds a separate tree ensemble structure for each output, while the
latter builds a single tree ensemble structure for all outputs. Building
separate ensembles for each output may be rather inefficient when the outputs
are strongly correlated. Correlations between the outputs could indeed be
exploited either to reduce model complexity (by sharing the tree structures
between several outputs) or to improve accuracy by regularization. Trying to fit
a single tree structure for all outputs seems however counterproductive when the
outputs are independent. Indeed, in the case of independent outputs,
simultaneously fitting all outputs with a single tree structure may require a
much more complex tree structure than the sum of the individual tree
complexities required to fit the individual outputs. Since training a more
complex tree requires a larger learning sample, multiple output trees are
expected to be outperformed by binary relevance / single target in this
situation.

In this chapter, we first formally adapt gradient boosting to multiple output
tasks. We then propose a new method that aims at circumventing the limitations
of both binary relevance / single target and multiple output methods, in the
specific context of tree-based base-learners. Our method is an extension of
gradient tree boosting that can adapt itself to the presence or absence of
correlations between the outputs. At each boosting iteration, a single
regression tree structure is grown to fit a single random projection of the
outputs, or more precisely, of their residuals with respect to the previously
built models. Then, the predictions of this tree are fitted linearly to the
current residuals of all the outputs (independently). New residuals are then
computed taking into account the resulting predictions and the process repeats
itself to fit these new residuals. Because of the linear fit, only the outputs
that are correlated with the random projection at each iteration will benefit
from a reduction of their residuals, while outputs that are independent of the random
projection will remain mostly unaffected. As a consequence, tree structures will
only be shared between correlated outputs as one would expect. Another variant
that we explore consists in replacing the linear global fit by a relabelling of
all tree leaves for each output in turn.

The chapter is structured as follows. We show how to extend the gradient
boosting algorithms to multi-output tasks in Section~\ref{sec:gb-mo-method}. We
provide for these algorithms a convergence proof on the training data and
discuss the effect of the random projection of the output space. We study
empirically the proposed approach in Section~\ref{sec:experiments}. Our first
experiments compare the proposed approaches to binary relevance / single target
on artificial datasets where the output correlation is known. We also highlight
the effect of the choice and size of the random projection space. We finally
carry out an empirical evaluation of these methods on 21 real-world multi-label
and 8 multi-output regression tasks. We draw our conclusions in
Section~\ref{sec:conclusions}.

\section{Gradient boosting with multiple outputs}
\label{sec:gb-mo-method}

Starting from a multi-output loss, we show in Section~\ref{subsec:gb-extend-mo}
how to extend the standard gradient boosting algorithm to solve multi-output
tasks, such as multi-output regression and multi-label classification, by
exploiting existing weak model learners suited for  multi-output prediction. In
Section~\ref{subsec:projections}, we then propose to combine single random
projections of the output space with gradient boosting to automatically adapt to
the output correlation structure on these tasks. We discuss and compare the
effect of the random projection of the output space in
Section~\ref{sec:effect-rp-gb}. We give a convergence proof on the training data
for the proposed algorithms in Section~\ref{sec:convergence-proof}.

\subsection{Standard extension of gradient boosting to multi-output tasks}
\label{subsec:gb-extend-mo}

A loss function $\ell(y,y') \in \mathbb{R}^{+}$ computes the difference between
a ground truth $y$ and a model prediction $y'$. It compares scalars with single
output tasks and vectors with multi-output tasks. The two most common regression
losses are the square loss $\ell_{square}(y, y')= \frac{1}{2} (y-y')^2$ and the
absolute loss $\ell_{absolute}(y, y) = |y - y'|$. Their multi-output extensions
are the $\ell_2$-norm and $\ell_1$-norm losses:
\begin{align}
\ell_2(y, y') &= \frac{1}{2}||y - y'||_{\ell_2}^2,\\
\ell_1(y, y') &= ||y - y'||_{\ell_1}.
\end{align}

In classification, the most commonly used loss to compare a ground truth $y$ to
the model prediction $f(x)$ is the $0-1$ loss $\ell_{0-1}(y,y') = 1(y\not=y')$,
where $1$ is the indicator function. It has two standard multiple output
extensions (i) the Hamming loss $\ell_{Hamming}$ and (ii) the subset $0-1$  loss
$\ell_{\text{subset }0-1}$:
\begin{align}
\ell_{Hamming}(y, y') &= \sum_{j=1}^d 1(y_j \not= y'_{j}), \\
\ell_{\text{subset }0-1} (y, y') &= 1(y\not=y').
\end{align}

Since these losses are discrete, they are
non-differentiable and difficult to optimize. Instead, we propose to extend the
logistic loss $\ell_{logistic}(y,y') =  \log(1 + \exp(-2 y y'))$ used for binary
classification tasks to the multi-label case, as follows:
\begin{equation}
\ell_{logistic}(y, y') = \sum_{j=1}^d \log(1 + \exp(-2 y_j y'_{j})),
\end{equation}
\noindent where we suppose that the $d$ components $y_{j}$ of the target output
vector belong to $\{-1, 1\}$, while the $d$ components $y'_{j}$ of the
predictions may belong to $\mathbb{R}$.

Given a training set  $\mathcal{L} = \left((x^i,y^i) \in \mathcal{X} \times
\mathcal{Y}\right)_{i=1}^n$ and one of these multi-output losses $\ell$, we want to learn
a model $f_{M}$ expressed in the following form
\begin{equation}
f_{M}(x) = \rho_0 + \sum_{m=1}^M \rho_m \odot g_m(x),
\label{eq:pred-mo-gbrt}
\end{equation}
\noindent where the terms $g_{m}$ are selected within a hypothesis space
$\mathcal{H}$ of weak multi-output base-learners, the coefficients $\{\rho_m \in
\mathbb{R}^d\}_{m=0}^M$ are $d$-dimensional vectors highlighting the
contributions of each term $g_m$ to the ensemble, and where the symbol $\odot$
denotes the Hadamard product.  Note that the prediction $f_{M}(x) \in
\mathbb{R}^d$ targets the minimization of the chosen loss $\ell$, but a
transformation might be needed to have a prediction in $\mathcal{Y}$, e.g. we
would apply the $\logit$ function to each output for the multi-output logistic
loss to get a probability estimate of the positive classes.

The gradient boosting method builds such a model in an iterative fashion, as
described in Algorithm~\ref{algo:gb-mo}, and discussed below.

\begin{algorithm}
\caption{Gradient boosting with multi-output regressor weak models.}
\label{algo:gb-mo}
\begin{algorithmic}[1]
\Function{GB-mo}{$\mathcal{L} = \left((x^i, y^i)\right)_{i=1}^n; \ell; \mathcal{H}; M$}
\State $f_{0}(x) = \rho_0 = \arg\min_{\rho \in \mathbb{R}^d} \sum_{i=1}^n \ell(y^i, \rho)$
\For{$m$ = 1 to $M$}
    \State Compute the loss gradient for the learning set samples
            \[g_m^i \in \mathbb{R}^d = \left[\nabla_{y'} \ell(y^i, y') \right]_{y' = f_{m-1}(x^i)}
            \forall i \in \left\{1, \ldots, n\right\}.\]
    \State Fit the negative loss gradient
           \[
           g_m = \arg\min_{g \in \mathcal{H}} \sum_{i=1}^n \left|\left|-g_m^i - g(x^i)\right|\right|_{\ell_2}^2.
           \]
    \State Find an optimal step length in the direction of $g_m$
            \[
            \rho_m = \arg\min_{\rho \in \mathbb{R}^d} \sum_{i=1}^n \ell\left(y^i, f_{m-1}(x^i) + \rho \odot g_m(x^i)\right).
            \]
    \State $f_{m}(x)= f_{m-1}(x)  + \rho_m \odot g_m(x)$
\EndFor
\State \Return $f_{M}(x)$
\EndFunction
\end{algorithmic}
\end{algorithm}

To build the ensemble model, we first initialize it with the constant model
defined by the vector $\rho_0 \in \mathbb{R}^d$ minimizing the multi-output loss
$\ell$ (line 2):
\begin{equation}
\rho_0 = \arg\min_{\rho \in \mathbb{R}^d} \sum_{i=1}^n \ell(y^i, \rho).
\end{equation}

At each subsequent iteration $m$, the multi-output gradient boosting approach
adds a new multi-output weak model $g_m(x)$ with a weight $\rho_m$ to the
current ensemble model by approximating the minimization of the
multi-output loss $\ell$:
\begin{equation}
(\rho_m, g_m) = \arg{}\hspace*{-3mm}\min_{(\rho, g) \in \mathbb{R}^{d}\times \mathcal{H}}
\sum_{i=1}^{n} \ell\left(y^{i}, f_{m-1}(x^i) + \rho \odot g(x^{i})\right).
\label{eq:boosting-exact-fsadm}
\end{equation}

To approximate Equation~\ref{eq:boosting-exact-fsadm}, it first fits a
multi-output weak model $g_m$ to model the negative gradient $g_m^i$
of the multi-output loss $\ell$
\begin{equation}
g_m^i \in\mathbb{R}^d = \left[ \nabla_{y'}  \ell(y^i, y') \right]_{y' = f_{m-1}(x^i)}
\end{equation}
\noindent associated to each sample $i \in \mathcal{L}$ of the training set, by minimizing
the $\ell_2$-loss:
\begin{equation}
g_m = \arg\min_{{g} \in \mathcal{H}} \sum_{i=1}^n \left|\left|-g_m^i - {g}(x^i)\right|\right|_{\ell_2}^2.
\end{equation}

It then computes an optimal step length vector $\rho_m \in \mathbb{R}^d$ in the
direction of the weak model $g_m$ to minimize the multi-output
loss $\ell$:
\begin{equation}
\rho_m = \arg\min_{\rho \in \mathbb{R}^d} \sum_{i=1}^n \ell\left(y^i, f_{m-1}(x^i) + \rho \odot g_m(x^i)\right).
\end{equation}

\subsection{Adapting to the correlation structure in the output-space} \label{subsec:projections}

Binary relevance / single target of gradient boosting models and gradient
boosting of multi-output models (Algorithm~\ref{algo:gb-mo}) implicitly target
two extreme correlation structures. On the one hand, binary relevance / single
target predicts all outputs independently,  thus assuming that outputs are not
correlated. On the other hand, gradient boosting of multi-output models handles
them all together, thus assuming that they are all correlated. Both approaches
thus exploit the available dataset in a rather biased way. To remove this bias,
we propose a more flexible approach that can adapt itself automatically to the
correlation structure among output variables.

Our idea is that a weak learner used at some step of the gradient boosting
algorithm could be fitted on a single random projection of the output space,
rather than always targeting simultaneously all outputs or always targeting a
single a priori fixed output.

We thus propose to first generate at each iteration of the boosting algorithm
one random projection vector of size $\phi_m \in \mathbb{R}^{1 \times d}$. The
weak learner is then fitted on the projection of the current residuals according
to $\phi_m$ reducing dimensionality from $d$ outputs to a single output. A
weight vector $\rho_m \in \mathbb{R}^d$ is then selected to minimize the
multi-output loss $\ell$. The whole approach is described in
Algorithm~\ref{algo:gbrt-rp}. If the loss is decomposable, non zero components
of the weight vector $\rho_m$ highlight the contribution of the current $m$-th
model to the overall loss decrease. Note that sign flips due to the projection
are taken into account by the additive weights $\rho_m$. A single output
regressor can now handle multi-output tasks through a sequence of single random
projections.

The prediction of an unseen sample $x$ by the model produced by Algorithm~\ref{algo:gbrt-rp}
is now given by
\begin{equation}
f(x) = \rho_0 + \sum_{m=1}^M \rho_m g_m(x),
\end{equation}
\noindent where $\rho_{0}\in\mathbb{R}^d$ is a constant prediction, and the
coefficients $\{\rho_m \in \mathbb{R}^d\}_{m=1}^M$ highlight the contribution of
each model $g_m$ to the ensemble. Note that it is different from
Equation~\ref{eq:pred-mo-gbrt} (no Hadamard product), since here the weak models
$g_{m}$ produce single output predictions.

\begin{algorithm}
\caption{Gradient boosting on randomly projected residual spaces.}
\label{algo:gbrt-rp}
\begin{algorithmic}[1]
\Function{GB-rpo}{$\mathcal{L} = \left((x^i, y^i)\right)_{i=1}^n; \ell; \mathcal{H}; M$}
\State $f_{0}(x) = \rho_0 = \arg\min_{\rho \in \mathbb{R}^d} \sum_{i=1}^n \ell(y^i, \rho)$
\For{$m$ = 1 to $M$}
    \State Compute the loss gradient for the learning set samples
    \[g_m^i \in \mathbb{R}^d = \left[\nabla_{y'} \ell(y^i, y') \right]_{y' = f_{m-1}(x^i)}
    \forall i \in \left\{1, \ldots, n\right\}.\]
    \State Generate a random projection $\phi_m \in \mathbb{R}^{1 \times d}$.
    \State Fit the projected loss gradient
    \[
    g_m = \arg\min_{{g} \in \mathcal{H}} \sum_{i=1}^n  \left(-\phi_m  g_m^i - {g}(x^i)\right)^2.
    \]
    \State Find an optimal step length in the direction of $g_m$.
    \[
    \rho_m = \arg\min_{\rho \in \mathbb{R}^d} \sum_{i=1}^n \ell\left(y^i, f_{m-1}(x^i) + \rho g_m(x^i)\right),
    \] \label{alg-line:rho-optim}
    \State $f_{m}(x) = f_{m-1}(x)  + \rho_m g_m(x)$
\EndFor
\State \Return $f_{M}(x)$
\EndFunction
\end{algorithmic}
\end{algorithm}

Whenever we use decision trees as models, we can grow the tree structure on any
output space and then (re)label it in another one as in
Chapter~\ref{ch:rf-output-projections} Section~\ref{sec:theoretical-analysis}
by (re)propagating the training samples in the tree structure. This idea of
leaf relabelling could be readily applied to Algorithm~\ref{algo:gbrt-rp}
leading to Algorithm~\ref{algo:gbrt-rp-relabel}. After fitting the decision
tree on the random projection(s) and before optimizing the additive weights
$\rho_m$, we relabel the tree structure in the original residual space (line
7). More precisely, each leaf is labelled by the average unprojected residual
vector of all training examples falling into that leaf.
The predition of an unseen sample is then obtained with
Equation~\ref{eq:pred-mo-gbrt} as for Algorithm~\ref{algo:gb-mo}. We will
investigate whether it is better or not to relabel the decision tree structure
in the experimental section. Note that Algorithm~\ref{algo:gbrt-rp-relabel} can
be straightforwardly used in a multiple random projection context ($q \geq 1$)
using a random projection matrix $\phi_m \in \mathcal{R}^{q \times d}$. The
resulting algorithm with arbitrary $q$ corresponds to the application to
gradient boosting of the idea explored in Chapter 5 in the context of random
forests. We will study in Section~\ref{subsec:multiple-rp} and
Section~\ref{sec:exp-rp-size} the effect of the size of the projected space
$q$.

\begin{algorithm}
\caption{Gradient boosting on randomly projected residual spaces
         with relabelled decision trees as weak models.}
\label{algo:gbrt-rp-relabel}
\begin{algorithmic}[1]
\Function{GB-relabel-rpo}{$\mathcal{L} = \left((x^i, y^i)\right)_{i=1}^n; \ell; \mathcal{H}; M; q$}
\State $f_{0}(x) = \rho_0 = \arg\min_{\rho \in \mathbb{R}^d} \sum_{i=1}^n \ell(y^i, \rho)$
\For{$m$ = 1 to $M$}
    \parState{Compute the loss gradient for the learning set samples
            \[g_m^i \in \mathbb{R}^d = \left[\nabla_{y'} \ell(y^i, y') \right]_{y' = f_{m-1}(x^i)}
            \forall i \in \left\{1, \ldots, n\right\}.\]}
    \parState{Generate a random projection $\phi_m \in \mathbb{R}^{q \times d}$.}
    \parState{Fit a single-output tree $g_{m}$ on the projected negative loss gradients
           \[
           g_m = \arg\min_{g \in \mathcal{H}} \sum_{i=1}^n \left|\left|-\phi_m g_m^i - {g}(x^i)\right|\right|_{\ell_2}^2.
           \]}
    \parState{Relabel each leaf of the tree $g_m$ in the original (unprojected) residual space, by
    averaging at each leaf the $g_m^i$ vectors of all examples falling into that leaf.}
    \parState{Find an optimal step length in the direction of $g'_{m}$.}
            \[
            \rho_m = \arg\min_{\rho \in \mathbb{R}^d} \sum_{i=1}^n \ell\left(y^i, f_{m-1}(x^i) + \rho \odot g'_m(x^i)\right).
            \] \label{alg-line:rho-optim2}
    \parState{$f_{m}(x) = f_{m-1}(x)  + \rho_m \odot g'_m(x)$}
\EndFor
\State \Return $f_{M}(x)$
\EndFunction
\end{algorithmic}
\end{algorithm}

To the three presented algorithms, we also add a constant learning rate $\mu \in (0,
1]$ to shrink the size of the gradient step $\rho_m$ in the residual space. Indeed, for
a given weak model space $\mathcal{H}$ and a loss $\ell$, optimizing both the
learning rate $\mu$ and the number of steps $M$ typically  improves
generalization performance.

\subsection{Effect of random projections}
\label{sec:effect-rp-gb}

Randomly projecting the output space in the context of the gradient boosting
approach has two direct consequences: (i) it strongly reduces the size of the
output space, and (ii) it randomly combines several outputs. We will consider
here the following random projection matrices $\phi \in \mathbb{R}^{q \times
d}$ ordered from the sparsest to the densest ones:
\begin{itemize}
\item \textbf{Random output subsampling matrices} is obtained by sampling random
lines from the identity matrix.

\item \textbf{(Sparse) Rademacher matrices} is obtained by drawing its elements
in $\left\{ -\sqrt{\frac{s}{q}}, 0, \sqrt{\frac{s}{q}} \right\}$ with
probability $\left\{ \frac{1}{2s}, 1 - \frac{1}{s} ,\frac{1}{2s}\right\}$, where
$1 / s \in (0,1]$ controls the sparsity of $\phi$. With $s=1$, we have (dense)
\textbf{Rademacher random projections}. If $s=3$, we will call them \textbf{Achlioptas random
projections}~\cite{achlioptas2003database}. When $s=\sqrt{d}$, we will say that
we have \textbf{sparse random projections} as in~\cite{li2006very}.

\item  \textbf{Gaussian matrices} are obtained by drawing their elements
\emph{i.i.d.} in $\mathcal{N}(0, 1 / q)$.

\end{itemize}

We discuss in more details the random sub-sampling projection in
Section~\ref{subsec:output-sampling} and the impact of the density of random
projection matrices in Section~\ref{subsec:rp-density}.  We study the benefit to
use more than a single random projection of the output space ($q > 1$) in
Section~\ref{subsec:multiple-rp}. We highlight the difference in model
representations between tree ensemble techniques, \emph{i.e.} the gradient tree
boosting approaches and the random forest approaches, in
Section~\ref{subsec:gb-rp-discussions}.

\subsubsection{$\ell_2$-norm loss and random output sub-sampling}
\label{subsec:output-sampling}


The gradient boosting method has an analytical solution when the loss is the
square loss or its extension the $\ell_2$-norm loss $\ell_2(y, y') =
\frac{1}{2}||y - y'||^2$:
\begin{itemize}

\item The constant model $f_0$ minimizing this loss is the average output value
of the training set given by
\begin{equation}
f_0(x) = \rho_0 = \arg\min_{\rho \in \mathbb{R}^d} \sum_{i=1}^n \frac{1}{2}||y^i - \rho||_{\ell_2}^2
= \frac{1}{n} \sum_{i=1}^n y^i.
\end{equation}

\item The gradient of the $\ell_2$-norm loss for the $i$-th sample is the
difference between the ground truth $y^i$ and the prediction of the ensemble
$f$ at the current step $m$ ($\forall i \in \left\{1, \ldots, n\right\}$):
\begin{align}
g_m^i &= \left[\nabla_{y'} \ell(y^i, y') \right]_{y' = f_{m-1}(x^i)}
 = y^i - f_{m-1}(x^i).
\end{align}

\item Once a new weak estimator $g_m$ has been fitted on the loss gradient
$g_m^i$ or the projected gradient $\phi_m g_m^i$ with or without relabelling, we
have to optimize the multiplicative weight vector $\rho_m$ of the new weak model
in the ensemble. For Algorithm~\ref{algo:gb-mo} and
Algorithm~\ref{algo:gbrt-rp-relabel} that exploit multi-output weak learners,
this amounts to
\begin{align}
\rho_m &= \arg\min_{\rho \in \mathbb{R}^d} \sum_{i=1}^n \frac{1}{2}\left|\left|y^i -  f_m(x^i) - \rho \odot g_m(x^i)\right|\right|^2 \\
&= \arg\min_{\rho \in \mathbb{R}^d} \sum_{i=1}^n \frac{1}{2}\left|\left|g_m^i - \rho \odot g_m(x^i)\right|\right|^2
\end{align}
\noindent
which has the following solution:
\begin{equation}
\rho_{m, j} = \frac{\sum_{i=1}^n g_{m,j}^i g_m(x^i)_j}{\sum_{i=1}^n g_m(x^i)_j} \forall j \in \{1,\ldots,d\}.
\label{eq:loss-weight-sub}
\end{equation}

For Algorithm~\ref{algo:gbrt-rp}, we have to solve
\begin{align}
\rho_m &= \arg\min_{\rho \in \mathbb{R}^d} \sum_{i=1}^n \frac{1}{2}\left|\left|y^i -  f_m(x^i) - \rho g_m(x^i)\right|\right|^2 \\
&= \arg\min_{\rho \in \mathbb{R}^d} \sum_{i=1}^n \frac{1}{2}\left|\left|g_m^i - \rho g_m(x^i)\right|\right|^2
\end{align}
\noindent which has the following solution
\begin{equation}
\rho_{m, j} = \frac{\sum_{i=1}^n g_{m,j}^i g_m(x^i)}{\sum_{i=1}^n g_m(x^i)} \forall j \in \{1,\ldots,d\}.
\label{eq:loss-weight-sub-2}
\end{equation}

\end{itemize}

From Equation~\ref{eq:loss-weight-sub} and Equation~\ref{eq:loss-weight-sub-2},
we have that the weight $\rho_{m, j}$ is proportional to the correlation
between the loss gradient of the output $j$ and the weak estimator $g_m$. If
the model $g_m$ is independent of the output $j$, the weight $\rho_{m,j}$
will be close to zero and $g_m$ will thus not contribute to the prediction of
this output. On the opposite, a high magnitude of $|\rho_{m,j}|$ means that the
model $g_m$ is useful to predict the output $j$.

If we subsample the output space at each boosting iteration
(Algorithm~\ref{algo:gbrt-rp} with random output sub-sampling), the weight
$\rho_{m, j}$ is then proportional to the correlation between the model fitted
on the sub-sampled output and the output $j$. If correlations exist between the
outputs, the optimization of the constant $\rho_m$ allows to share the trained
model at the $m$-th iteration on the sub-sampled output to all the other
outputs. In the extreme case where all outputs are independent given the
inputs, the weight $\rho_m$ is expected to be nearly zero for all outputs
except for the sub-sampled output, and Algorithm~\ref{algo:gbrt-rp} would be
equivalent to the binary relevance / single target approach. If all outputs are
strictly identical, the elements of the constant vector $\rho_m$ would have the
same value, and Algorithm~\ref{algo:gbrt-rp} would be equivalent to the
multi-output gradient boosting approach
(Algorithm~\ref{algo:gb-mo}). Algorithm~\ref{algo:gbrt-rp} would also produce
in this case the exact same model as binary relevance / single target approach
asymptotically but it would require $d$ times less trees to reach similar
performance, as each tree would be shared by all $d$ outputs.

Algorithm~\ref{algo:gbrt-rp-relabel} with random output sub-sampling is a gradient
boosting approach fitting one decision tree at each iteration on a random output
space and relabelling the tree in the original output space. The leaf
relabelling procedure minimizes the $\ell_2$-norm loss over the training
samples by averaging the output values of the samples reaching the corresponding
leaves. In this case, the optimization of the weight $\rho_m$ is unnecessary, as it would
lead to an all ones vector. For similar reasons if the multi-output gradient
boosting method (Algorithm~\ref{algo:gb-mo}) uses decision trees as weak
estimators, the weight $\rho_m$ is also an all ones vector as the leaf
predictions already minimize the $\ell_2$-norm loss. The difference between
these two algorithms is that Algorithm~\ref{algo:gbrt-rp-relabel} grows trees
using a random output at each iteration instead of all of them with
Algorithm~\ref{algo:gb-mo}.



\subsubsection{Density of the random projections}
\label{subsec:rp-density}


In Chapter~\ref{ch:rf-output-projections}, we have combined the random forest
method with a wide variety of random projection schemes. While the algorithms presented
 in this chapter were originally devised with random output
sub-sampling in mind (see Section~\ref{subsec:output-sampling}), it seems
natural to also combine the proposed approaches with random projection schemes such
as Gaussian random projections or (sparse) Rademacher random projections.

With random output sub-sampling, the projection matrix $\phi_m \in
\mathbb{R}^{1 \times d}$ is extremely sparse as only one element is non zero.
With denser random projections, the weak estimators of
Algorithm~\ref{algo:gbrt-rp} and Algorithm~\ref{algo:gbrt-rp-relabel} are fitted
on the projected gradient loss $\{(x^i, \phi_m g_m^i\}_{i=1}^n$. It means that
a weak estimator $g_m$ is trying to model the direction of a weighted
combination of the gradient loss.

Otherwise said, the weak model fitted at the $m$-th step approximates a
projection $\phi_m$ of the gradient losses given the input vector. We can
interpret the weight $\rho_{m,j}$ when minimizing the $\ell_2$-norm loss as the
correlation between the output $j$ and a weighted approximation of the output
variables $\phi_m$. With an extremely sparse projection having only one non zero
element, we have the situation described in the previous section. If we have two
non zero elements, we have the following extreme cases: (i) both combined
outputs are identical and (ii) both combined outputs are independent given the
inputs. In the first situation, the effect is identical to the case where we
sub-sample only one output. In the second situation, the weak model makes a
compromise between the independent outputs given by $\phi_m$. Between those two
extremes, the loss gradient direction $\phi_m$ approximated by the weak model is
useful to predict both outputs. The random projection of the output space will
indeed prevent over-fitting by inducing some variance in the learning process.
The previous reasoning can be extended to more than two output variables.

Dense random projection schemes, such as Gaussian random projection, consider a
higher number of outputs together and is hoped to speed up convergence by
increasing the correlation between the fitted tree in the projected space and
the residual space. Conversely, sparse random projections, such as random output
sub-sampling, make the weak model focus on few outputs.


\subsubsection{Gradient tree boosting and multiple random projections}
\label{subsec:multiple-rp}


The gradient boosting multi-output strategy combining random projections and tree
relabelling (Algorithm~\ref{algo:gbrt-rp-relabel}) can use random projection
matrices $\phi_m \in \mathbb{R}^{q \times d}$ with more than one line
($q\geq 1$).

The weak estimators are multi-output regression trees using the
variance as impurity criterion to grow their tree structures. With an increasing
number of projections $q$, we have the theoretical guarantee~(see
Chapter~\ref{ch:rf-output-projections}) that the variance computed in the
projected space is an approximation of the variance in the original output
space.

When the projected space is of infinite size $q\rightarrow\infty$, the decision
trees grown on the original space or on the projected space are identical as the
approximation of the variance is exact. We thus have that
Algorithm~\ref{algo:gbrt-rp-relabel} is equivalent to the gradient boosting with
multi-output regression tree method (Algorithm~\ref{algo:gb-mo}).

Whenever the projected space is of finite size ($q < \infty$),
Algorithm~\ref{algo:gbrt-rp-relabel} is thus an approximation of
Algorithm~\ref{algo:gb-mo}. We study empirically the effect of the number of
projections $q$ in Algorithm~\ref{algo:gbrt-rp-relabel} in
Section~\ref{sec:exp-rp-size}.

\subsubsection{Representation bias of decision tree ensembles}
\label{subsec:gb-rp-discussions}


Random forests and gradient tree boosting build an ensemble of  trees
either independently or sequentially, and thus offer different bias/variance tradeoffs.
The predictions of all these ensembles can be expressed as a weighted combination of
the ground truth outputs of the training set samples. In the present section, we discuss the differences
between single tree models, random forest models and gradient tree boosting
models in terms of the representation biases of the obtained models. We also
highlight the differences between single target models and multi-output tree
models.


\paragraph{Single tree models.}  The prediction of a regression tree learner can
be written as a weighted linear combination of the training samples
$\mathcal{L}=\{(x^i,y^i) \in \mathcal{X} \times \mathcal{Y}\}_{i=1}^n$. We
associate to each training sample $(x^i, y^i)$ a weight function $w^i:
\mathcal{X} \rightarrow \mathbb{R}$ which gives the contribution of a ground
truth $y^i$ to predict an unseen sample $x$. The prediction of a \emph{single output
 tree} $f$ is given by
\begin{equation}
f(x) = \sum_{i=1}^n w^i(x) y^i.
\label{eq:so-dt-pred}
\end{equation}
\noindent The weight function $w^i(x)$ is non zero if both the samples $(x^i,y^i)$
and the unseen sample $x$ reach the same leaf of the tree. If both
$(x^i,y^i)$ and $x$ end up in the same leaf of the  tree,
$w^i(x)$ is equal to the inverse of the number of training samples reaching that leaf.
The weight $w^i(x)$ can thus be rewritten as $k(x^i, x)$ and the function $k(\cdot, \cdot)$ is actually a positive semi-definite
kernel~\cite{geurts2006extremely}.

We can also express multi-output models as a weighted sum of the training
samples. With a \emph{single target regression tree}, we have an
independent weight function $w^i_j$ for each sample of the training set and
each output as we fit one model per output. The prediction of this model for output $j$ is
given by:
\begin{equation}
f(x)_j = \sum_{i=1}^n w^i_j(x) y^i_j.
\label{eq:br-dt-pred}
\end{equation}

With a \emph{multi-output regression tree}, the decision tree structure is shared
between all outputs so we have a single weight function $w^i$ for each training
sample:
\begin{equation}
f(x)_j = \sum_{i=1}^n w^i(x) y^i_j.
\label{eq:mo-dt-pred}
\end{equation}

\paragraph{Random forest models.} If we have a \emph{single target
random forest model}, the prediction of the $j$-th output combines the predictions
of the $M$ models of the ensemble in the following way:
\begin{equation}
f(x)_j = \frac{1}{M}\sum_{m=1}^M \sum_{i=1}^n w^{i}_{m,j}(x) y^i_j,
\label{eq:br-rf-pred}
\end{equation}
with one weight function $w^{i}_{m,j}$ per tree, sample and output. We note that
we can combine the weights of the individual trees into a single one per sample and per output
\begin{equation}
w^i_j(x) = \frac{1}{M} \sum_{m=1}^M w^{i}_{m,j}(x).
\end{equation}
\noindent The prediction of the $j$-th output for an ensemble of independent
models has the same form as a single target regression tree
model:
\begin{equation}
f(x)_j = \frac{1}{M}\sum_{m=1}^M \sum_{i=1}^n w^{i}_{m,j}(x) y^i_j = \sum_{i=1}^n w^i_j(x) y^i_j.
\end{equation}

We can repeat the previous development with a \emph{multi-output random forest
model}. The prediction for the $j$-th output of an unseen sample $x$ combines
the predictions of the $M$ trees:
\begin{align}
f(x)_j &= \frac{1}{M}\sum_{m=1}^M \sum_{i=1}^n w^i_m(x) y^i_j = \sum_{i=1}^n w^i(x) y^i_j
\label{eq:mo-rf-pred}
\end{align}
\noindent with
\begin{equation}
w^i(x) = \frac{1}{M} \sum_{m=1}^M w^i_m(x).
\end{equation}
\noindent With this framework, the prediction of an ensemble model has the same
form as the prediction of a single constituting tree.

\paragraph{Gradient tree boosting models.} The prediction of a \emph{single output
gradient boosting tree ensemble} is given by
\begin{equation}
f(x) = \rho_0 + \sum_{m=1}^M \mu \rho_m g_m(x),
\end{equation}
\noindent but also as
\begin{align}
f(x) = \sum_{m=1}^M \sum_{i=1}^n w^i_m(x) y^i = \sum_{i=1}^n w^i(x) y^i,
\end{align}
\noindent where the weight $w^i(x)$ takes into account the learning rate $\mu$,
the prediction of all tree models $g_m$ and the associated $\rho_m$. Given the
similarity between gradient boosting prediction and random forest model, we
deduce that the \emph{single target gradient boosting tree ensemble} has the
form of Equation~\ref{eq:br-dt-pred} and that \emph{multi-output gradient tree
boosting} (Algorithm~\ref{algo:gb-mo}) and \emph{gradient boosting tree with
projection of the output space and relabelling}
(Algorithm~\ref{algo:gbrt-rp-relabel}) has the form of
Equation~\ref{eq:mo-dt-pred}.

However, we note that the prediction model of the \emph{gradient tree boosting with
random projection of the output space} (Algorithm~\ref{algo:gbrt-rp}) is not
given by Equation~\ref{eq:br-dt-pred} and Equation~\ref{eq:mo-dt-pred} as the
prediction of a single output $j$ can combine the prediction of all $d$ outputs.
More formally, the prediction of the $j$-th output is given by:
\begin{equation}
f(x)_j = \sum_{m=1}^M \sum_{i=1}^n \sum_{k=1}^d w^i_{m,j,k}(x) y^i_k,
\end{equation}
\noindent where the weight function $w^i_{m,j,k}$ takes into account the contribution
of the $m$-th model fitted on a random projection $\phi_m$ of the output space
to predict the $j$-th output using the $k$-th outputs and the $i$-th sample. The triple
summation can be simplified by using a single weight to summarize the contribution of
all $M$ models:
\begin{align}
f(x)_j &= \sum_{m=1}^M \sum_{i=1}^n \sum_{k=1}^d w^i_{m,j,k}(x) y^i_k = \sum_{i=1}^n \sum_{k=1}^d w^i_{j,k}(x) y^i_k.
\end{align}

Between the studied methods, we can distinguish three groups of multi-output
tree models. The first one considers that all outputs are independent as with
binary relevance / single target trees, random forests or gradient
tree boosting models. The second group with multi-output random forests, gradient
boosting of multi-output tree and gradient boosting with random projection of
the output space and relabelling share the tree structures between all outputs,
but the leaf predictions are different for each output. The last and most
flexible group is the gradient tree boosting with random projection of the
output space sharing both the tree structures and the leaf predictions. We will
highlight in the experiments the impact of these differences in representation biases.

\subsection{Convergence when $M \rightarrow \infty$} 
\label{sec:convergence-proof}


Similarly to~\cite{geurts2007gradient}, we can prove the convergence of the
training-set loss of the gradient boosting with multi-output models
(Algorithm~\ref{algo:gb-mo}), and gradient boosting on randomly projected spaces
with (Algorithm~\ref{algo:gbrt-rp}) or without relabelling
(Algorithm~\ref{algo:gbrt-rp-relabel}).

Since the loss function is lower-bounded by $0$, we merely need to show that the loss $\ell$ is
non-increasing on the training set at each step $m$ of the gradient boosting
algorithm.


For Algorithm~\ref{algo:gb-mo} and Algorithm~\ref{algo:gbrt-rp-relabel}, we
note that
\begin{align}
\sum_{i=1}^n \ell(y^i, f_{m}(x^i))
&= \min_{\rho \in \mathbb{R}^{d}} \sum_{i=1}^n \ell\left(y^i, f_{m-1}(x^i) + \rho\odot g_m(x^i)\right) \nonumber \\
&\leq \sum_{i=1}^n \ell\left(y^i, f_{m-1}(x^i) \right).\label{thm:conv-rho}
\end{align}
and the learning-set loss is hence non increasing with $M$ if we use a learning
rate $\mu=1$. If the loss $\ell(y,y')$ is convex in its second
argument $y'$ (which is the case for those loss-functions that we use in
practice), then this convergence property actually holds for any value
$\mu\in(0;1]$ of the learning rate. Indeed, we have
\begin{eqnarray*}
& & \sum_{i=1}^n \ell\left(y^i, f_{m-1}(x^i) \right)\\
& \geq &(1-\mu) \sum_{i=1}^n \ell\left(y^i, f_{m-1}(x^i) \right) + \mu \sum_{i=1}^n \ell\left(y^i, f_{m-1}(x^i) + \rho_m \odot g_m(x^i)\right)\\
& \geq & \sum_{i=1}^n \ell\left(y^i, f_{m-1}(x^i) + \mu \rho_m \odot g_m(x^i)\right).
\end{eqnarray*}
\noindent given Equation~\ref{thm:conv-rho} and the convexity property.

For Algorithm~\ref{algo:gbrt-rp}, we have a weak estimator $g_m$ fitted on a
single random projection of the output space $\phi_m$ with a multiplying
constant vector $\rho_m \in \mathbb{R}^d$, and we have:
\begin{align}
\sum_{i=1}^n \ell(y^i, f_{m}(x^i))
&=  \min_{\rho \in \mathbb{R}^{d}} \sum_{i=1}^n \ell\left(y^i, f_{m-1}(x^i) + \rho g_m(x^i)\right) \nonumber \\
&\leq \sum_{i=1}^n \ell\left(y^i, f_{m-1}(x^i)\right).
\end{align}
and the error
is also non increasing for Algorithm~\ref{algo:gbrt-rp}, under the same conditions as above.


The previous development shows that Algorithm~\ref{algo:gb-mo},
Algorithm~\ref{algo:gbrt-rp} and Algorithm~\ref{algo:gbrt-rp-relabel} are
converging on the training set for a given loss $\ell$. The binary relevance /
single target of gradient boosting regression trees admits a similar convergence
proof. We expect however  the convergence speed of the binary relevance /
single target to be lower assuming that it fits one weak estimator for each
output in a round robin fashion.

\section{Experiments}
\label{sec:experiments}


We describe the experimental protocol in
Section~\ref{sec:experimental-protocol}. Our first experiments in
Section~\ref{subsec:synthetic-exp} illustrate the multi-output gradient boosting
methods on synthetic datasets where the output correlation structure is known.
The effect of the choice and / or the number of random projections of the output
space is later studied for Algorithm~\ref{algo:gbrt-rp} and
Algorithm~\ref{algo:gbrt-rp-relabel} in Section~\ref{sec:gb-exp-rp-effect}. We
compare multi-output gradient boosting approaches and multi-output random forest
approaches in Section~\ref{sec:systematic-analysis} over 29 real multi-label and
multi-output datasets.

\subsection{Experimental protocol}
\label{sec:experimental-protocol}

We describe the metrics used to assess the performance of the supervised
learning algorithms in Section~\ref{subsec:metrics-protocol}. The protocol used
to optimize hyper-parameters is given in
Section~\ref{subsec:hyperparam-protocol}.

Note that the datasets used in the following experiments are described in
Appendix~\ref{ch:datasets}. Whenever the number of testing samples is not given,
we use half of the data as training set and half of the data as testing set.

\subsubsection{Accuracy assessment protocol}
\label{subsec:metrics-protocol}

We assess the accuracy of the predictors on a test set using the  ``Label
Ranking Average Precision (LRAP)'' (defined in
Section~\ref{subsec:multilabel-metrics}) for multi-label classification tasks
and  the ``macro-$r^2$ score'' (defined in
Section~\ref{subsec:regression-metrics}) for multi-output regression tasks.

\subsubsection{Hyper-parameter optimization protocol}
\label{subsec:hyperparam-protocol}

The hyper-parameters of the supervised learning algorithms are optimized as
follows: we define an hyper-parameter grid and the best hyper-parameter set is
selected using $20\%$ of the training samples as a validation set. The results
shown are averaged over five random split of the dataset
while preserving the training-testing set size ratio.

For the boosting ensembles, we optimize the learning rate $\mu$ among $\{1.,
0.5, 0.2, 0.1, 0.05, 0.02, 0.01\}$ and use decision trees as weak models whose
hyper-parameters are also optimized: the number of features drawn at each node
$k$ during the tree growth is selected among $k \in \{\sqrt{p}, 0.1p, 0.2p,
0.5p, p\}$, the maximum number of tree leaves $n_{\max\_leaves}$ grown in
best-first fashion is chosen among $n_{\max\_leaves} \in \{2,\ldots,8\}$. Note
that a decision tree with $n_{\max\_leaves}=2$ and $k=p$ is called a stump. We
add new weak models to the ensemble by minimizing either the square loss or the
absolute loss (or their multi-output extensions) in regression and either the
square loss or the logistic loss (or their multi-output extensions) in
classification, the choice of the loss being an additional hyper-parameter
tuned on the validation set.

We also optimize the number of boosting steps $n_{iter}$ of each gradient
boosting algorithm over the validation set. However note that the number of
steps has a different meaning depending on the algorithm. For binary relevance /
single target gradient boosting, the number of boosting steps $n_{iter}$ gives
the number of weak models fitted per output. The implemented algorithm here fits
weak models in a round robin fashion over all outputs. For all other (multi-output) methods,
the number of boosting steps $n_{iter}$ is the total number of weak models for
all outputs as only one model is needed to fit all outputs. The computing time
of one boosting iteration is thus different between the approaches. We will set
the budget, the maximal number of boosting steps $n_{iter}$, for each algorithm
to $n_{iter}=10000$ on synthetic experiments (see
Section~\ref{subsec:synthetic-exp}) so that the performance of the estimator is
not limited by the computational power. On the real datasets however, this
setting would have been too costly. We decided instead to limit the computing
time allocated to each gradient boosting algorithm on each classification
(resp. regression) problem to $100\times T$ (resp. $500\times T$), where $T$ is
the time needed on this specific problem for one iteration of multi-output
gradient boosting (Algorithm~\ref{algo:gb-mo}) with stumps and the
$\ell_2$-norm loss. The maximum number of iterations, $n_{iter}$, is thus set
independently for each problem and each hyper-parameter setting such that this
time constraint is satisfied. As a consequence, all approaches thus receive
approximately the same global time budget for model training and
hyper-parameter optimization.

For the random forest algorithms, we use the default hyper-parameter setting
suggested in~\cite{friedman2001elements}, which corresponds in classification
to 100 totally developed trees with $k=\sqrt{p}$ and in regression to $100$
trees with $k=p/3$ and a minimum of 5 samples to split a node ($n_{\min{}}=5$).

The base learner implementations are based on the
random-output-trees\footnote{\url{https://github.com/arjoly/random-output-trees}}~\cite{joly2014random}
version 0.1 and on the scikit-learn~\cite{buitinck2013api,pedregosa2011scikit}
of version 0.16 Python package. The algorithms presented in this chapter will
be provided in random-output-trees version 0.2.

\subsection{Experiments on synthetic datasets with known output correlation structures}
\label{subsec:synthetic-exp}

We study here the proposed boosting approaches on synthetic datasets whose
output correlation structures are known. The datasets are first presented in
Section~\ref{sec:artificial-datasets}. We then compare on these datasets
multi-output gradient boosting approaches in terms of their convergence
speed in Section~\ref{sec:convergence-artificial} and in terms of their best
performance whenever hyper-parameters are optimized in
Section~\ref{sec:synthetic-exp}.

\subsubsection{Synthetic datasets}
\label{sec:artificial-datasets}

To illustrate multi-output boosting strategies, we use three synthetic datasets
with a specific output structure: (i) chained outputs, (ii) totally correlated
outputs and (iii) fully independent outputs. Those tasks are derived from the
\textbf{friedman1} regression dataset which consists in solving the following
single target regression task~\cite{friedman1991multivariate}
\begin{align}
f(x) &= 10 \sin(\pi x_1 x_2)  + 20 (x_3 - 0.5)^2 + 10 x_4 + 5 x_5
y &= f(x) + \epsilon
\end{align}
\noindent with  $x \in \mathbb{R}^5 \sim \mathcal{N}(0; I_5)$ and $\epsilon \sim \mathcal{N}(0;
1)$ where $I_5$ is an identity matrix of size $5 \times 5$.

The \textbf{friedman1-chain} problem consists in $d$ regression tasks
forming a chain obtained by cumulatively adding independent standard
Normal noise. We draw samples from the following distribution
\begin{align}
y_1 &= f(x) + \epsilon_1,\\
y_j &= y_{j-1} + \epsilon_j \quad \forall j \in \{2, \ldots, d \}
\end{align}
\noindent with $x \sim \mathcal{N}(0; I_5)$ and $\epsilon \sim \mathcal{N}(0;
I_{d})$. Given the chain structure, the output with the least amount of noise
is the first one of the chain and averaging a subset of the outputs would not
lead to any reduction of the output noise with respect to the first output,
since total noise variance accumulates more than linearly with the number of
outputs. The optimal multi-output strategy is thus to build a model using only
the first output and then to replicate the prediction of this model for all
other outputs.

The \textbf{friedman1-group} problem consists in solving $d$ regression tasks
simultaneously obtained from one friedman1 problem without noise where an
independent normal noise is added. Given $x \sim \mathcal{N}(0; I_5)$ and
$\epsilon \sim \mathcal{N}(0; I_{d})$, we have to solve the following task:
\begin{align}
y_j =& f(x) + \epsilon_j  \quad \forall j \in \{1, \ldots, d \}.
\end{align}
\noindent If the output-output structure is known, the additive noises
$\epsilon_j, \forall j \in \{1, \ldots, d \},$ can be filtered out by averaging
all outputs. The optimal strategy to address this problem is thus to train a
single output regression model to fit the average output. Predictions on unseen
data would be later done by replicating the output of this model for all
outputs.

The \textbf{friedman1-ind} problem consists in $d$ independent friedman1 tasks.
Drawing samples from  $x \sim \mathcal{N}(0; I_{5 d})$ and $\epsilon \sim
\mathcal{N}(0; I_{d})$, we have
\begin{align}
y_j =& f(x_{5j+1:5j+5}) + \epsilon_j \quad \forall j \in \{1, \ldots, d \}.
\end{align}
\noindent where $x_{5j+1:5j+5}$ is a slice of feature vector from feature $5j+1$
to $5j+5$. Since all outputs are independent, the best multi-output strategy
is single target: one independent model fits each output

For each multi-output friedman problem, we consider 300 training samples, 4000
testing samples and $d=16$ outputs.

\subsubsection{Convergence with known output correlation structure}
\label{sec:convergence-artificial}

We first study the macro-$r^2$ score convergence as a function of time (see
Figure~\ref{fig:gbrt-subsample}) for three multi-output gradient boosting
strategies: (i) single target of gradient tree boosting (st-gbrt), (ii)
gradient boosting with multi-output regression tree (gbmort,
Algorithm~\ref{algo:gb-mo}) and (iii) gradient boosting with output subsampling
of the output space (gbrt-rpo-subsample, Algorithm~\ref{algo:gbrt-rp}). We train
each boosting algorithm on the three friedman1 artificial datasets with the same
set of hyper-parameters: a learning rate of $\mu=0.1$  and stumps as weak
estimators (a decision tree with $k=p$, $n_{\max{}\_leaves}=2$) while minimizing
the square loss.


\begin{figure}
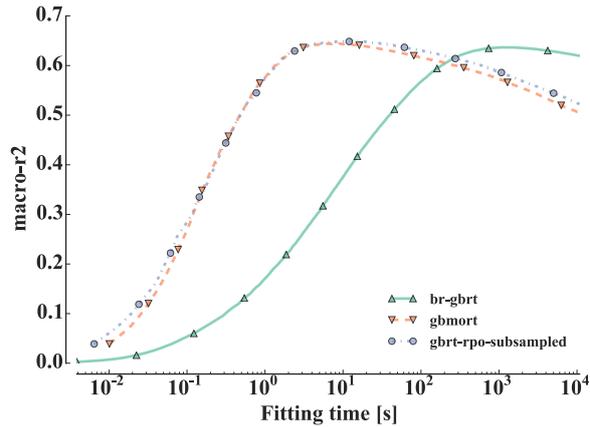
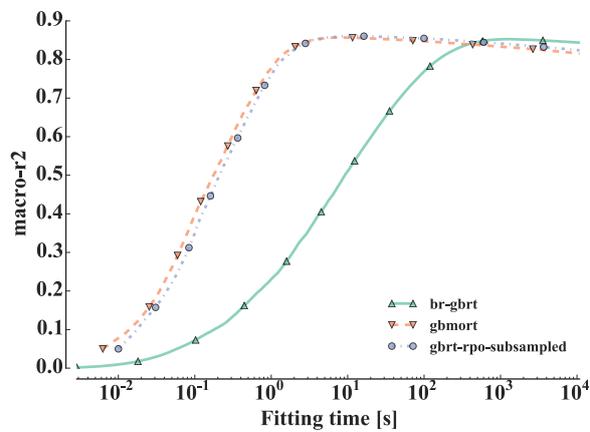
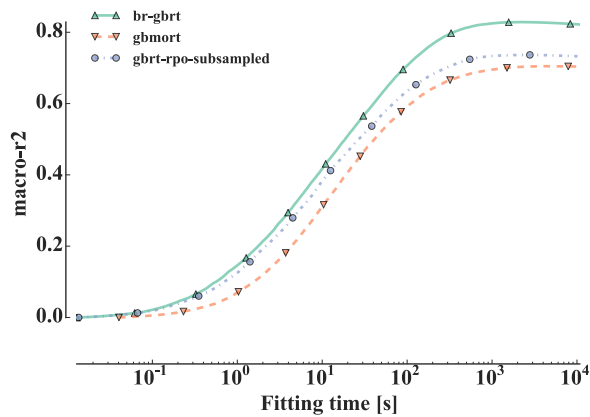

\centering
\subfloat[Friedman1-chain]{\label{subfig:subsample-chain}\includegraphics[width=0.7\textwidth]{{{ch06/friedman-subsample/friedman1_chain_v2_n_out_16_noise_1.0_log_time_macro-r2}}}} \\
\subfloat[Friedman1-group]{\label{subfig:subsample-gr}\includegraphics[width=0.7\textwidth]{{{ch06/friedman-subsample/friedman1_mo_v3_16_1_16_log_time_macro-r2}}}}\\
\subfloat[Friedman1-ind]{\label{subfig:subsample-ind}\includegraphics[width=0.7\textwidth]{{{ch06/friedman-subsample/friedman1_mo_16_16_1_log_time_macro-r2}}}}
\caption{The convergence speed and the optimum reached are affected by the
output correlation structure. The gbmort and gbrt-rpo-subsampled algorithms both exploit
output correlations, which yields faster convergence and slightly better performance than
st-gbrt on friedman1-chain and friedman1-group. However, st-gbrt
converges to a better optimum than gbmort and gbrt-rpo-subsample when there is no
output correlation as in friedman1-ind. (Model parameters: $k=p$,
$n_{\max{}\_leaves}=2$, $\mu=0.1$)}
\label{fig:gbrt-subsample}
\end{figure}

On the friedman1-chain (see Figure~\ref{subfig:subsample-chain}) and
friedman1-group (see Figure~\ref{subfig:subsample-gr}), gbmort and
gbrt-rpo-subsampled converge more than 100 times faster (note the logarithmic
scale of the abscissa) than single target. Furthermore, the optimal
macro-$r^2$ is slightly better for gbmort and gbrt-rpo-subsampled than st-gbrt.
Since all outputs are correlated on both datasets, gbmort and
gbrt-rpo-susbsampled are exploiting the output structure to have faster
convergence. The gbmort method exploits the output structure by filtering the
output noise as the stump fitted at each iteration is the one that maximizes the
reduction of the average output variance. By contrast, gbrt-rpo-subsample
detects output correlations by optimizing the $\rho_m$ constant and then shares
the information obtained by the current weak model with all other outputs.

On the friedman1-ind dataset (see Figure~\ref{subfig:subsample-ind}),
all three methods converge at the same rate. However, the single target
strategy converges to a better optimum than gbmort and gbrt-rpo-subsample. Since
all outputs are independent, single target enforces the proper correlation
structure (see Figure~\ref{subfig:subsample-ind}). The gbmort method has the
worst performance as it assumes the wrong set of hypotheses. The
gbrt-rpo-subsampled method pays the price of its flexibility by over-fitting the
additive weight associated to each output, \textbf{but less than gbmort.}

This experiment confirms that enforcing the right correlation structure yields
faster convergence and the best accuracy. Nevertheless, the output structure is
unknown in practice. We need flexible approaches such as gbrt-rpo-subsampled
that automatically detects and exploits the correlation structure.

\subsubsection{Performance and output modeling assumption}
\label{sec:synthetic-exp}

The presence or absence of structures among the outputs have shown to affect the
convergence speed of multi-output gradient boosting methods. As discussed
in~\cite{cheng2010bayes}, we talk about conditionally independent outputs when:
$$P(y_1,\ldots,y_q|x)=P(y_1|x)\cdots P(y_q|x)$$ and about unconditionally
independent outputs when:$$P(y_1,\ldots,y_q)=P(y_1)\cdots P(y_q).$$ When the
outputs are not conditionally independent and the loss function can not be
decomposed over the outputs (eg., the subset $0-1$ loss), one might need to
model the joint output distribution $P(y_1,\ldots,y_q|x)$ to obtain a Bayes
optimal prediction. If the outputs are conditionally independent however or if
the loss function can be decomposed over the outputs, then a Bayes optimal
prediction can be obtained by modeling separately the marginal conditional
output distributions $P(y_j|x)$ for all $j$. This suggests that in this case,
binary relevance / single target is not really penalized asymptotically with
respect to multiple output methods for not considering the outputs jointly. In
the case of an infinite sample size, it is thus expected to provide as good
models as all the multiple output methods. \emph{Since in practice we have to
deal with finite sample sizes, multiple output methods may provide better
results by better controlling the bias/variance trade-off.}

Let us study this question on the three synthetic datasets: friedman1-chain,
friedman1-group or friedman1-ind. We optimize the hyper-parameters with a
computational budget of 10000 weak models per hyper-parameter set. Five
strategies are compared (i) the artificial-gbrt method, which assumes that the
output structure is known and implements the optimal strategy on each problem as
explained in Section~\ref{sec:artificial-datasets}, (ii) single target of
gradient boosting regression trees (st-gbrt), (iii) gradient boosting with
multi-output regression tree (gbmort, Algorithm~\ref{algo:gb-mo}) and gradient
boosting with randomly sub-sampled outputs (iv) without relabelling
(gbrt-rpo-subsampled, Algorithm~\ref{algo:gbrt-rp}) and (v) with relabelling
(gbrt-relabel-rpo-subsampled, Algorithm~\ref{algo:gbrt-rp-relabel}). All
boosting algorithms minimize the square loss, the absolute loss or their
multi-outputs extension the $\ell_2$-norm loss.

We give the performance on the three tasks for each estimator in
Table~\ref{table:friedman1} and the p-value of Student's paired $t$-test
comparing the performance of two estimators on the same dataset in
Table~\ref{tab:tstudent-friedman1}.


\begin{table}
\caption{All  methods compared on the 3 artificial datasets. Exploiting the
output correlation structure (if it exists) allows beating single target
in a finite sample size, decomposable metric and conditionally independent
output.}\small
\label{table:friedman1}
\centering
\renewcommand{\tabcolsep}{1.1pt}
\begin{tabular}{llll}
\toprule
Dataset &          friedman1-chain &        friedman1-group &          friedman1-ind \\
\midrule
artificial-gbrt             &  $0.654\hspace*{-0.8mm}\pm\hspace*{-0.8mm}0.015$(1) &  $0.889\hspace*{-0.8mm}\pm\hspace*{-0.8mm}0.009$(1) &  $0.831\hspace*{-0.8mm}\pm\hspace*{-0.8mm}0.004$(1) \\
st-gbrt                     &  $0.626\hspace*{-0.8mm}\pm\hspace*{-0.8mm}0.016$(5) &  $0.873\hspace*{-0.8mm}\pm\hspace*{-0.8mm}0.008$(5) &   $0.830\hspace*{-0.8mm}\pm\hspace*{-0.8mm}0.003$(2) \\
gbmort                      &   $0.640\hspace*{-0.8mm}\pm\hspace*{-0.8mm}0.008$(4) &  $0.874\hspace*{-0.8mm}\pm\hspace*{-0.8mm}0.012$(4) &   $0.644\hspace*{-0.8mm}\pm\hspace*{-0.8mm}0.010$(5) \\
gbrt-relabel-rpo-subsampled &  $0.648\hspace*{-0.8mm}\pm\hspace*{-0.8mm}0.015$(2) &   $0.880\hspace*{-0.8mm}\pm\hspace*{-0.8mm}0.009$(2) &  $0.706\hspace*{-0.8mm}\pm\hspace*{-0.8mm}0.009$(4) \\
gbrt-rpo-subsampled         &  $0.645\hspace*{-0.8mm}\pm\hspace*{-0.8mm}0.013$(3) &  $0.876\hspace*{-0.8mm}\pm\hspace*{-0.8mm}0.007$(3) &  $0.789\hspace*{-0.8mm}\pm\hspace*{-0.8mm}0.003$(3) \\
\bottomrule
\end{tabular}
\end{table}

\begin{table}
\caption{P-values given by Student's paired $t$-test on the synthetic datasets.
We highlight p-values inferior to $\alpha=0.05$ in bold. Note that the sign $<$
(resp. $>$) indicates that the estimator in the row has better (resp.lower)
score than the column estimator.}
\label{tab:tstudent-friedman1}
\setlength{\tabcolsep}{1.2pt}
\centering
\begin{small}
\begin{tabular}{llllll}
{} & \rotatebox{90}{artificial-gbrt} & \rotatebox{90}{st-gbrt} & \rotatebox{90}{gbmort} & \rotatebox{90}{gbrt-relabel-rpo-subsampled} & \rotatebox{90}{gbrt-rpo-subsampled} \\
\midrule
Dataset friedman1-chain \\
\midrule
artificial-gbrt             &                                 &       \textbf{0.003}(>) &                   0.16 &                                        0.34 &                                0.24 \\
st-gbrt                     &               \textbf{0.003}(<) &                         &                   0.11 &                            \textbf{0.04}(<) &                    \textbf{0.03}(<) \\
gbmort                      &                            0.16 &                    0.11 &                        &                                        0.38 &                                0.46 \\
gbrt-relabel-rpo-subsampled &                            0.34 &        \textbf{0.04}(>) &                   0.38 &                                             &                                0.57 \\
gbrt-rpo-subsampled         &                            0.24 &        \textbf{0.03}(>) &                   0.46 &                                        0.57 &                                     \\
\midrule
Dataset friedman1-group \\
\midrule
artificial-gbrt             &                                 &       \textbf{0.005}(>) &      \textbf{0.009}(>) &                           \textbf{0.047}(>) &                   \textbf{0.006}(>) \\
st-gbrt                     &               \textbf{0.005}(<) &                         &                   0.56 &                           \textbf{0.046}(<) &                                0.17 \\
gbmort                      &               \textbf{0.009}(<) &                    0.56 &                        &                                        0.15 &                                0.63 \\
gbrt-relabel-rpo-subsampled &               \textbf{0.047}(<) &       \textbf{0.046}(>) &                   0.15 &                                             &                    \textbf{0.04}(>) \\
gbrt-rpo-subsampled         &               \textbf{0.006}(<) &                    0.17 &                   0.63 &                            \textbf{0.04}(<) &                                     \\
\midrule
Dataset friedman1-ind \\
\midrule
artificial-gbrt             &                                 &                    0.17 &      \textbf{2e-06}(>) &                           \textbf{2e-06}(>) &                   \textbf{1e-05}(>) \\
st-gbrt                     &                            0.17 &                         &      \textbf{2e-06}(>) &                           \textbf{4e-06}(>) &                   \textbf{4e-06}(>) \\
gbmort                      &               \textbf{2e-06}(<) &       \textbf{2e-06}(<) &                        &                           \textbf{9e-05}(<) &                   \textbf{6e-06}(<) \\
gbrt-relabel-rpo-subsampled &               \textbf{2e-06}(<) &       \textbf{4e-06}(<) &      \textbf{9e-05}(>) &                                             &                   \textbf{3e-05}(<) \\
gbrt-rpo-subsampled         &               \textbf{1e-05}(<) &       \textbf{4e-06}(<) &      \textbf{6e-06}(>) &                           \textbf{3e-05}(>) &                                     \\
\bottomrule
\end{tabular}
\end{small}
\end{table}

As expected, we obtain the best performance if the output correlation structure
is known with the custom strategies implemented with artifical-gbrt. Excluding
this artificial method, the best boosting methods on the two problems with
output correlations, friedman1-chain and friedman1-group, are the two gradient
boosting approaches with output subsampling (gbrt-relabel-rpo-subsampled and
gbrt-rpo-subsampled).

In friedman1-chain, the output correlation structure forms a chain as each new
output is the previous one in the chain with a noisy output. Predicting outputs
at the end of the chain, without using the previous ones, is a difficult task.
The single target approach is thus expected to be sub-optimal. And indeed, on
this problem, artificial-gbrt, gbrt-relabel-rpo-subsampled and
gbrt-rpo-subsampled are significantly better than st-gbrt (with
$\alpha=0.05$). All the multi-output methods, including gbmort, are
indistinguishable from a statistical point of view, but we note that gbmort is
however not significantly better than st-gbrt.

In friedman1-group, among the ten pairs of algorithms, four are not
significantly different, showing a p-value greater than 0.05 (see
Table~\ref{tab:tstudent-friedman1}). We first note that gbmort is not better
than st-gbrt while exploiting the correlation. Secondly, the boosting methods
with random output sub-sampling are the best methods. They are however not
significantly better than gbmort and significantly worse than artificial-gbrt,
which assumes the output structure is known. Note that gbrt-relabel-rpo-subsampled
is significantly better than gbrt-rpo-subsampled.

In friedman1-ind, where there is no correlation between the outputs, the best
strategy is single target which makes independent models for each output.  From
a conceptual and statistical point of view, there is no difference between
artificial-gbrt and st-gbrt. The gbmort algorithm, which is optimal when all
outputs are correlated, is here significantly worse than all other methods. The
two boosting methods with output subsampling (gbrt-rpo-subsampled and
gbrt-relabel-rpo-subsampled method), which can adapt themselves to the absence
of correlation between the outputs, perform better than gbmort, but they are
significantly worse than st-gbrt. For these two algorithms, we note that not
relabelling the leaves (gbrt-rpo-subsampled) leads to superior performance
than relabelling them (gbrt-relabel-rpo-subsampled). Since in
friedman1-ind the outputs have disjoint feature support, the test nodes of a
decision tree fitted on one output will partition the samples using these
features. Thus, it is not suprising that relabeling the trees leaves actually
deteriorates performance.

In the previous experiment, all the outputs were dependent of the inputs.
However in multi-output tasks with very high number of outputs, it is likely
that some of them have few or no links with the inputs, i.e., are pure
noise. Let us repeat the previous experiments with the main difference that we
add to the original 16 outputs 16 purely noisy outputs obtained through random
permutations of the original outputs. We show the results of optimizing each
algorithm in Table~\ref{table:raw-results-noisy-out} and the associated
p-values in Table~\ref{tab:tstudent-friedman1-noisy-out1}. We report the
macro-$r^2$ score computed either on all outputs (macro-$r^2$) including the noisy
outputs or only on the 16 original outputs (half-macro-$r^2$).  P-value were
computed between each pair of algorithms using Student's $t$-test on the macro
$r^2$ score computed on all outputs.

We observe that the gbrt-rpo-subsampled algorithm has the best performance on
friedman1-chain and friedman1-group and is the second best on the
friedman1-ind, below st-gbrt.  Interestingly on friedman1-chain and
friedman1-group, this algorithm is significantly better than all the others,
including gbmort. Since this latter method tries to fit all outputs
simultaneously, it is the most disadvantaged by the introduction of the noisy
outputs.


\begin{table}
\caption{Friedman datasets with noisy outputs.}\small
\label{table:raw-results-noisy-out}
\centering
\begin{tabular}{lrl}
\toprule
friedman1-chain  &  half-macro-$r^2$ &           macro-$r^2$ \\
\midrule
st-gbrt                     &       0.611 (4) &  $0.265 \pm 0.006$ (4) \\
gbmort                      &       0.617 (3) &  $0.291 \pm 0.012$ (3) \\
gbrt-relabel-rpo-subsampled &       0.628 (2) &  $0.292 \pm 0.006$ (2) \\
gbrt-rpo-subsampled         &       0.629 (1) &  $0.303 \pm 0.007$ (1) \\

\midrule
Friedman1-group &  half-macro-$r^2$ &           macro-$r^2$ \\
\midrule
st-gbrt                     &       0.840 (3) &  $0.364 \pm 0.007$ (4) \\
gbmort                      &       0.833 (4) &  $0.394 \pm 0.004$ (3) \\
gbrt-relabel-rpo-subsampled &       0.855 (2) &  $0.395 \pm 0.005$ (2) \\
gbrt-rpo-subsampled         &       0.862 (1) &  $0.414 \pm 0.006$ (1) \\

\midrule
Friedman1-ind &  half-macro-$r^2$ &           macro-$r^2$ \\
\midrule
st-gbrt                     &       0.806 (1) &  $0.3536 \pm 0.0015$ (1) \\
gbmort                      &       0.486 (4) &   $0.1850 \pm 0.0081$ (4) \\
gbrt-relabel-rpo-subsampled &       0.570 (3) &  $0.2049 \pm 0.0033$ (3) \\
gbrt-rpo-subsampled         &       0.739 (2) &  $0.3033 \pm 0.0021$ (2) \\

\bottomrule
\end{tabular}
\end{table}

\begin{table}
  \caption{P-values given by Student's paired $t$-test on the synthetic datasets.
We highlight p-values inferior to $\alpha=0.05$ in bold. Note that the sign $<$
(resp. $>$) indicates that the estimator in the row has better (resp.lower)
score than the column estimator.}
\label{tab:tstudent-friedman1-noisy-out1}
\setlength{\tabcolsep}{1.5pt}
\centering
\begin{small}
\begin{tabular}{lllll}
{} & \rotatebox{90}{st-gbrt} & \rotatebox{90}{gbmort} & \rotatebox{90}{gbrt-relabel-rpo-subsampled} & \rotatebox{90}{gbrt-rpo-subsampled} \\
\midrule
Dataset friedman1-chain \\
\midrule
st-gbrt                     &                         &     \textbf{0.0009}(<) &                          \textbf{0.0002}(<) &                  \textbf{0.0002}(<) \\
gbmort                      &      \textbf{0.0009}(>) &                        &                                        0.86 &                    \textbf{0.04}(<) \\
gbrt-relabel-rpo-subsampled &      \textbf{0.0002}(>) &                   0.86 &                                             &                    \textbf{0.03}(<) \\
gbrt-rpo-subsampled         &      \textbf{0.0002}(>) &       \textbf{0.04}(>) &                            \textbf{0.03}(>) &                                     \\
\midrule
Dataset friedman1-group \\
\midrule
st-gbrt                     &                         &     \textbf{0.0002}(<) &                          \textbf{0.0006}(<) &                  \textbf{0.0003}(<) \\
gbmort                      &      \textbf{0.0002}(>) &                        &                                        0.74 &                   \textbf{0.008}(<) \\
gbrt-relabel-rpo-subsampled &      \textbf{0.0006}(>) &                   0.74 &                                             &                   \textbf{0.002}(<) \\
gbrt-rpo-subsampled         &      \textbf{0.0003}(>) &      \textbf{0.008}(>) &                           \textbf{0.002}(>) &                                     \\
\midrule
Dataset friedman1-ind \\
\midrule
st-gbrt                     &                         &      \textbf{1e-06}(>) &                           \textbf{1e-07}(>) &                   \textbf{1e-06}(>) \\
gbmort                      &       \textbf{1e-06}(<) &                        &                            \textbf{0.02}(<) &                   \textbf{6e-06}(<) \\
gbrt-relabel-rpo-subsampled &       \textbf{1e-07}(<) &       \textbf{0.02}(>) &                                             &                   \textbf{2e-06}(<) \\
gbrt-rpo-subsampled         &       \textbf{1e-06}(<) &      \textbf{6e-06}(>) &                           \textbf{2e-06}(>) &                                     \\
\bottomrule
\end{tabular}
\end{small}
\end{table}

\subsection{Effect of random projection}
\label{sec:gb-exp-rp-effect}

With the gradient boosting and random projection of the output space approaches
(Algorithms~\ref{algo:gbrt-rp} and~\ref{algo:gbrt-rp-relabel}), we have
considered until now only sub-sampling a single output at each iteration as
random projection scheme. In Section~\ref{sec:exp-projections}, we show
empirically the effect of other random projection schemes such as Gaussian
random projection. In Section~\ref{sec:exp-rp-size}, we study the effect of
increasing the number of projections in the gradient boosting algorithm with
random projection of the output space and relabelling
(parameter $q$ of Algorithm~~\ref{algo:gbrt-rp-relabel}). We also show empirically the link
between Algorithm~~\ref{algo:gbrt-rp-relabel} and gradient boosting with
multi-output regression tree (Algorithm~\ref{algo:gb-mo}).

\subsubsection{Choice of the random projection scheme}
\label{sec:exp-projections}


Beside random output sub-sampling, we can combine the multi-output gradient boosting strategies
(Algorithms~\ref{algo:gbrt-rp} and~\ref{algo:gbrt-rp-relabel}) with other random
projection schemes. A key difference between
random output sub-sampling and random projections such as Gaussian and (sparse)
Rademacher projections is that the latter combines together several outputs.

We show in Figures~\ref{fig:mediamill-projections}, \ref{fig:delicious-projections} and
 \ref{fig:friedman1-ind-projections} the LRAP or macro-$r^2$ score
convergence of gradient boosting with randomly projected outputs (gbrt-rpo,
Algorithm~\ref{algo:gbrt-rp}) respectively on the mediamill, delicious, and Friedman1-ind
datasets with different random projection schemes.

The impact of the random projection scheme on convergence speed of gbrt-rpo
(Algorithm~\ref{algo:gbrt-rp}) is very problem dependent. On the mediamill
dataset, Gaussian, Achlioptas, or sparse random projections all improve
convergence speed by a factor of 10 (see
Figure~\ref{fig:mediamill-projections}) compared to subsampling randomly only
one output. On the delicious (Figure~\ref{fig:delicious-projections}) and
friedman1-ind (Figure~\ref{fig:friedman1-ind-projections}), this is the
opposite: subsampling leads to faster convergence than all other projections
schemes.  Note that we have the same behavior if one relabels the tree
structure grown at each iteration as in Algorithm~\ref{algo:gbrt-rp-relabel}
(results not shown).

Dense random projections, such as Gaussian random projections, force the weak
model to consider several outputs jointly and it should thus only improve when
outputs are somewhat correlated (which seems to be the case on mediamill). When
all of the outputs are independent or the correlation is less strong, as in
friedman1-ind or delicious, this has a detrimental effect. In this situation,
sub-sampling only one output at each iteration leads to the best performance.


\begin{figure}
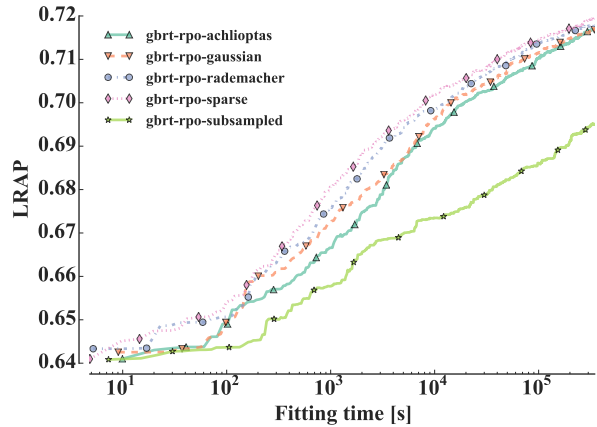

\centering
\includegraphics[width=0.7\textwidth]{{{ch06/projection-norelabel/mediamill_log_time_lrap}}}
\caption{
On the mediamill dataset, Gaussian, Achlioptas and sparse random projections with
gbrt-rpo (Algorithm \ref{algo:gbrt-rp}) show 10 times faster convergence in
terms of LRAP score, than sub-sampling one output variable at each iteration.
($k=p$, stumps, $\mu=0.1$, logistic loss)}
\label{fig:mediamill-projections}
\end{figure}

\begin{figure}
\centering
\includegraphics[width=0.7\textwidth]{{{ch06/projection-norelabel/delicious_log_time_lrap}}}
\caption{
On the delicious dataset, Gaussian, Achlioptas and sparse random projections with
gbrt-rpo (Algorithm~\ref{algo:gbrt-rp}) show 10 times faster convergence in
terms of LRAP score, than sub-sampling one output variable at each iteration.
($k=p$, stumps, $\mu=0.1$, logistic loss)}
\label{fig:delicious-projections}
\end{figure}

\begin{figure}
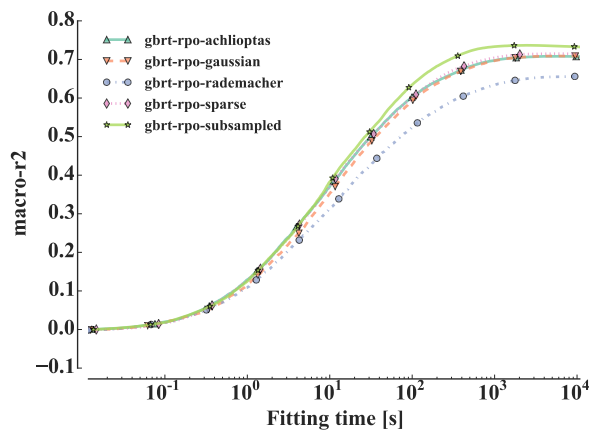

\centering
\includegraphics[width=0.7\textwidth]{{{ch06/projection-norelabel/friedman1_mo_16_16_1_log_time_macro-r2}}}
\caption{
On the friedman1-ind dataset where there is no output correlation, gbrt-rpo
(Algorithm~\ref{algo:gbrt-rp}) with one random subsampled output leads to a higher
macro-$r^2$ score than using Gaussian, Achlioptas or sparse random projections.
($k=p$, stumps, $\mu=0.1$, square loss)
\label{fig:friedman1-ind-projections}}
\end{figure}

\subsubsection{Effect of the size of the projected space}
\label{sec:exp-rp-size}

The multi-output gradient boosting strategy combining random projections and
tree relabelling (Algorithm~\ref{algo:gbrt-rp-relabel}) can use  more than one
random projection ($q\geq 1$) by using multi-output trees as base learners. In this section, we
study the effect of the size of the projected space $q$ in
Algorithm~\ref{algo:gbrt-rp-relabel}. This approach corresponds to the one
developed in Chapter~\ref{ch:rf-output-projections} for random forest.

Figure~\ref{fig:delicious-time-q} shows the LRAP score as a function of the
fitting time for gbmort (Algorithm~\ref{algo:gb-mo}) and gbrt-relabel-rpo
(Algorithm~\ref{algo:gbrt-rp-relabel}) with either Gaussian random projection
(see Figure \ref{subfig:delicious-time-gaussian}) or output subsampling (see
Figure~\ref{subfig:delicious-time-subsample}) for a number of projections $q
\in \left\{1, 98, 196, 491\right\}$ on the delicious dataset. In
Figure~\ref{subfig:delicious-time-gaussian} and
Figure~\ref{subfig:delicious-time-subsample}, one Gaussian random projection or
one sub-sampled output has faster convergence than their counterparts with a
higher number of projections $q$ and gbmort at fixed computational budget. Note
that when the number of projections $q$ increases, gradient boosting with
random projection of the output space and relabeling becomes similar to gbmort.


\begin{figure}
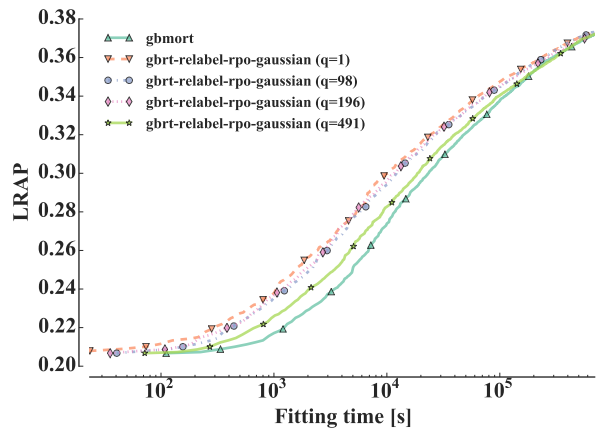
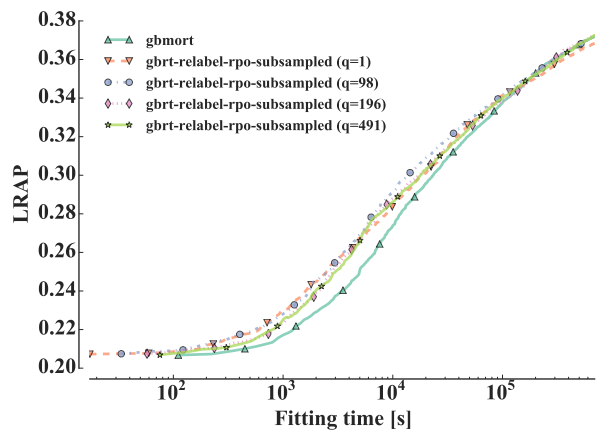

\centering
\subfloat[]{\label{subfig:delicious-time-gaussian}\includegraphics[width=0.7\textwidth]{{{ch06/n_components/gaussian_delicious_log_time_lrap_lr=0.1}}}} \\
\subfloat[]{\label{subfig:delicious-time-subsample}\includegraphics[width=0.7\textwidth]{{{ch06/n_components/subsample_delicious_log_time_lrap_lr=0.1}}}}
\caption{On the delicious dataset, LRAP score as a function of the boosting
ensemble fitting time for  gbrt-rpo-gaussian-relabel and
gbrt-rpo-subsampled-relabel with different number of projections $q$.
($k=p$, stumps, $\mu=0.1$, logistic loss)}
\label{fig:delicious-time-q}
\end{figure}

Instead of fixing the computational budget as a function of the training time,
we now set the computational budget to 100 boosting steps. On the delicious
dataset, gbrt-relabel-rpo (Algorithm~\ref{algo:gbrt-rp-relabel}) with Gaussian
random projection yields approximately the same performance as gbmort with
$q\geq20$ random projections as shown in Figure~\ref{subfig:rp-delicious-lrap}
and reduces computing times by \emph{a factor 7} at $q=20$ projections (see
Figure~\ref{subfig:rp-delicious-time}).


\begin{figure}
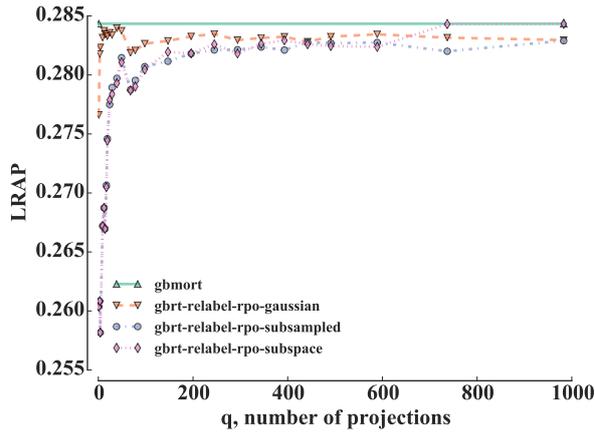
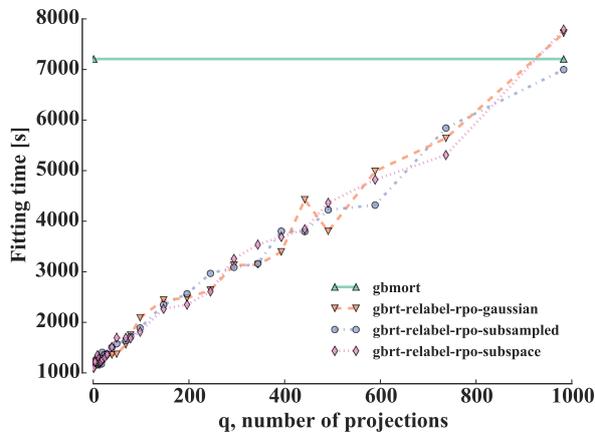

\centering
\subfloat[]{\label{subfig:rp-delicious-lrap}\includegraphics[width=0.7\textwidth]{{{ch06/n_components/delicious_n_components_lrap_loss=logistic-brent}}}} \\
\subfloat[]{\label{subfig:rp-delicious-time}\includegraphics[width=0.7\textwidth]{{{ch06/n_components/delicious_n_components_chrono-fit_loss=logistic-brent}}}}
\caption{On delicious, increasing the number of random projections $q$ allows
         to reach the same LRAP score as gbmort at a significantly
         reduced computational cost.
         ($k=p$, stumps, $\mu=0.1$, $M=100$, logistic loss)}
\label{fig:increasing-rp}
\end{figure}

These experiments show that gradient boosting with random projection and
relabelling (gbrt-relabel-rpo, Algorithm~\ref{algo:gbrt-rp-relabel}) is indeed an
approximation of gradient boosting with multi-output trees (gbmort,
Algorithm~\ref{algo:gb-mo}). The number of random projections $q$ influences
simultaneously the bias-variance tradeoff and the convergence speed of
Algorithm~\ref{algo:gbrt-rp-relabel}.

\subsection{Systematic analysis over real world datasets}
\label{sec:systematic-analysis}


We perform a systematic analysis over real world multi-label classification and multi-output
regression datasets. For this study, we evaluate the proposed algorithms:
gradient boosting of multi-output regression trees (gbmort,
Algorithm~\ref{algo:gb-mo}), gradient boosting with random projection of the
output space (gbrt-rpo, Algorithm~\ref{algo:gbrt-rp}), and gradient boosting
with random projection of the output space  and relabelling (gbrt-relabel-rpo,
Algorithm~\ref{algo:gbrt-rp-relabel}). For the two latter algorithms, we
consider two random projection schemes: (i) Gaussian random projection, a dense
random projection, and (ii) random output sub-sampling, a sparse random
projection. They will be compared to three common and well established
tree-based multi-output algorithms: (i) binary relevance / single target of
gradient boosting regression tree (br-gbrt / st-gbrt), (ii) multi-output random
forest (mo-rf) and (iii) binary relevance / single target of random forest
models (br-rf / st-rf).

We will compare all methods on multi-label tasks in
Section~\ref{subsec:gb-multilabel-exp} and on multi-output regression tasks
in Section~\ref{subsec:gb-mo-reg-exp}. Following the recommendations of
\citep{demvsar2006statistical}, we use the Friedman test and its associated
Nemenyi post-hoc test. Pairwise comparisons are also carried out using the
Wilcoxon signed ranked test.

\subsubsection{Multi-label datasets}
\label{subsec:gb-multilabel-exp}

Table~\ref{table:summary-multilabel} and Table~\ref{table:summary-multilabel-2}
show the performance of the random forest models and the boosting algorithms
over the 21 multi-label datasets. The critical distance diagram of
Figure~\ref{fig:cd-multilabel} gives the ranks of the algorithms and has an
associated Friedman test p-value of $1.36 \times 10^{-10}$ with a critical
distance of $2.29$ given by the Nemenyi post-hoc test ($\alpha=0.05$). Thus, we
can reject the null hypothesis that all methods are
equivalent. Table~\ref{tab:wilcoxon-multilabel} gives the outcome of the
pairwise Wilcoxon signed ranked tests.

The best average performer is gbrt-relabel-rpo-gaussian which is significantly
better according to the Nemenyi post-hoc test than all methods except
gbrt-rpo-gaussian and gbmort.

Gradient boosting with the Gaussian random projection has a significantly
better average rank than the random output sub-sampling projection. Relabelling
tree leaves allows to have better performance on the 21 multi-label dataset.
Indeed, both gbrt-relabel-rpo-gaussian and gbrt-relabel-rpo-subsampled are
better ranked and significantly better than their counterparts without
relabelling (gbrt-rpo-gaussian and gbrt-rpo-subsampled). These results somewhat
contrast with the results obtained on the artificial datasets, where
relabelling was always counterproductive.

Among all compared methods, br-gbrt has the worst rank and it is significantly
less good than all gbrt variants according to the Wilcoxon signed rank
test. This might be actually a consequence of the constant budget in time that
was allocated to all methods (see Section~\ref{sec:experimental-protocol}). All
methods were given the same budget in time but, given the very slow convergence
rate of br-gbrt, this budget may not allow to grow enough trees per output with
this method to reach competitive performance.

We notice also that both random forests based methods (mo-rf and br-rf) are
less good than all gbrt variants, most of the time significantly, except for
br-gbrt. It has to be noted however that no hyper-parameter was tuned for the
random forests. Such tuning could slightly change our conclusions, although
random forests often work well with default setting.

\begin{table}
\caption{LRAP scores over 21 multi-label datasets (part 1).}
\label{table:summary-multilabel}
\centering
\setlength{\tabcolsep}{1.5pt}
\begin{small}\footnotesize
\begin{tabular}{llll}
\toprule
                            &                   CAL500 &                   bibtex &                  birds \\
\midrule
br-gbrt                     &  $0.505 \pm 0.002$ (3.5) &    $0.587 \pm 0.007$ (6) &  $0.787 \pm 0.009$ (6) \\
br-rf                       &    $0.484 \pm 0.002$ (8) &    $0.542 \pm 0.005$ (8) &  $0.802 \pm 0.013$ (1) \\
gbmort                      &    $0.501 \pm 0.005$ (6) &  $0.595 \pm 0.005$ (4.5) &  $0.772 \pm 0.007$ (8) \\
gbrt-relabel-rpo-gaussian   &    $0.507 \pm 0.009$ (1) &    $0.607 \pm 0.005$ (1) &    $0.800 \pm 0.017$ (2) \\
gbrt-relabel-rpo-subsampled &    $0.499 \pm 0.008$ (7) &    $0.596 \pm 0.005$ (3) &   $0.790 \pm 0.016$ (4) \\
gbrt-rpo-gaussian           &  $0.505 \pm 0.006$ (3.5) &      $0.600 \pm 0.003$ (2) &  $0.793 \pm 0.017$ (3) \\
gbrt-rpo-subsampled         &    $0.506 \pm 0.006$ (2) &  $0.595 \pm 0.007$ (4.5) &  $0.779 \pm 0.018$ (7) \\
mo-rf                       &    $0.502 \pm 0.003$ (5) &    $0.553 \pm 0.005$ (7) &  $0.789 \pm 0.012$ (5) \\
\midrule
                            &              bookmarks   &                corel5k &              delicious \\
\midrule
br-gbrt                     &  $0.4463 \pm 0.0038$ (7) &    $0.291 \pm 0.006$ (7) &    $0.347 \pm 0.002$ (8) \\
br-rf                       &  $0.4472 \pm 0.0019$ (6) &    $0.273 \pm 0.012$ (8) &  $0.373 \pm 0.004$ (6.5) \\
gbmort                      &  $0.4855 \pm 0.0016$ (2) &  $0.312 \pm 0.009$ (3.5) &  $0.384 \pm 0.003$ (3.5) \\
gbrt-relabel-rpo-gaussian   &  $0.4893 \pm 0.0003$ (1) &  $0.315 \pm 0.007$ (1.5) &    $0.389 \pm 0.003$ (1) \\
gbrt-relabel-rpo-subsampled &  $0.4718 \pm 0.0034$ (4) &     $0.310 \pm 0.007$ (5) &  $0.384 \pm 0.003$ (3.5) \\
gbrt-rpo-gaussian           &  $0.4753 \pm 0.0022$ (3) &   $0.315 \pm 0.010$ (1.5) &    $0.386 \pm 0.004$ (2) \\
gbrt-rpo-subsampled         &  $0.4621 \pm 0.0026$ (5) &  $0.312 \pm 0.006$ (3.5) &    $0.377 \pm 0.003$ (5) \\
mo-rf                       &  $0.4312 \pm 0.0023$ (8) &     $0.294 \pm 0.010$ (6) &  $0.373 \pm 0.004$ (6.5) \\
\midrule
                            &                diatoms &         drug-interaction &               emotions \\
\midrule
br-gbrt                     &  $0.623 \pm 0.007$ (7.5) &  $0.271 \pm 0.018$ (8) &      $0.800 \pm 0.022$ (7) \\
br-rf                       &  $0.623 \pm 0.011$ (7.5) &   $0.310 \pm 0.009$ (5) &    $0.816 \pm 0.009$ (1) \\
gbmort                      &    $0.656 \pm 0.012$ (4) &  $0.304 \pm 0.005$ (7) &    $0.794 \pm 0.014$ (8) \\
gbrt-relabel-rpo-gaussian   &     $0.725 \pm 0.010$ (1) &  $0.326 \pm 0.008$ (1) &  $0.802 \pm 0.017$ (5.5) \\
gbrt-relabel-rpo-subsampled &    $0.685 \pm 0.012$ (3) &  $0.322 \pm 0.009$ (3) &    $0.808 \pm 0.021$ (3) \\
gbrt-rpo-gaussian           &    $0.702 \pm 0.014$ (2) &  $0.323 \pm 0.011$ (2) &    $0.804 \pm 0.009$ (4) \\
gbrt-rpo-subsampled         &  $0.653 \pm 0.013$ (5.5) &  $0.312 \pm 0.013$ (4) &  $0.802 \pm 0.007$ (5.5) \\
mo-rf                       &   $0.653 \pm 0.010$ (5.5) &  $0.308 \pm 0.007$ (6) &      $0.810 \pm 0.010$ (2) \\
\midrule
                            &                    enron &                  genbase &                mediamill \\
\midrule
br-gbrt                     &    $0.685 \pm 0.006$ (6) &  $0.989 \pm 0.009$ (8) &   $0.7449 \pm 0.0020$ (8) \\
br-rf                       &    $0.683 \pm 0.005$ (7) &  $0.994 \pm 0.005$ (2) &  $0.7819 \pm 0.0009$ (1) \\
gbmort                      &  $0.705 \pm 0.004$ (2.5) &   $0.990 \pm 0.004$ (6) &  $0.7504 \pm 0.0013$ (7) \\
gbrt-relabel-rpo-gaussian   &  $0.705 \pm 0.003$ (2.5) &  $0.993 \pm 0.006$ (3) &   $0.7660 \pm 0.0021$ (3) \\
gbrt-relabel-rpo-subsampled &    $0.697 \pm 0.004$ (5) &    $0.990 \pm 0.010$ (6) &  $0.7588 \pm 0.0013$ (5) \\
gbrt-rpo-gaussian           &    $0.706 \pm 0.004$ (1) &  $0.992 \pm 0.007$ (4) &  $0.7608 \pm 0.0008$ (4) \\
gbrt-rpo-subsampled         &    $0.699 \pm 0.005$ (4) &   $0.990 \pm 0.005$ (6) &  $0.7519 \pm 0.0006$ (6) \\
mo-rf                       &    $0.676 \pm 0.004$ (8) &  $0.995 \pm 0.004$ (1) &  $0.7793 \pm 0.0015$ (2) \\
\bottomrule
\end{tabular}
\end{small}
\end{table}

\begin{table}
\caption{LRAP scores over 21 multi-label datasets (part 2).}
\label{table:summary-multilabel-2}
\centering
\setlength{\tabcolsep}{1.5pt}
\begin{small}\footnotesize
\begin{tabular}{llll}
\toprule
                            &                medical &      protein-interaction &                  reuters \\
\midrule
br-gbrt                     &    $0.864 \pm 0.006$ (3) &    $0.294 \pm 0.007$ (6) &   $0.939 \pm 0.0033$ (7) \\
br-rf                       &    $0.821 \pm 0.007$ (8) &    $0.293 \pm 0.006$ (7) &  $0.9406 \pm 0.0016$ (6) \\
gbmort                      &  $0.867 \pm 0.011$ (1.5) &   $0.310 \pm 0.007$ (2.5) &  $0.9483 \pm 0.0014$ (3) \\
gbrt-relabel-rpo-gaussian   &  $0.867 \pm 0.019$ (1.5) &   $0.310 \pm 0.009$ (2.5) &  $0.9508 \pm 0.0009$ (1) \\
gbrt-relabel-rpo-subsampled &    $0.856 \pm 0.012$ (5) &  $0.303 \pm 0.003$ (4.5) &  $0.9441 \pm 0.0016$ (4) \\
gbrt-rpo-gaussian           &    $0.859 \pm 0.017$ (4) &    $0.311 \pm 0.007$ (1) &  $0.9486 \pm 0.0021$ (2) \\
gbrt-rpo-subsampled         &    $0.851 \pm 0.009$ (6) &  $0.303 \pm 0.003$ (4.5) &   $0.9430 \pm 0.0031$ (5) \\
mo-rf                       &    $0.827 \pm 0.006$ (7) &    $0.288 \pm 0.009$ (8) &  $0.9337 \pm 0.0021$ (8) \\
\midrule
                            &                    scene &                  scop-go &        sequence-funcat \\
\midrule
br-gbrt                     &     $0.880 \pm 0.003$ (4) &  $0.716 \pm 0.047$ (8) &  $0.678 \pm 0.008$ (6) \\
br-rf                       &    $0.876 \pm 0.003$ (6) &  $0.798 \pm 0.004$ (2) &  $0.658 \pm 0.008$ (7) \\
gbmort                      &    $0.886 \pm 0.004$ (1) &  $0.796 \pm 0.007$ (3) &  $0.699 \pm 0.005$ (3) \\
gbrt-relabel-rpo-gaussian   &  $0.884 \pm 0.006$ (2.5) &  $0.788 \pm 0.006$ (4) &  $0.703 \pm 0.007$ (2) \\
gbrt-relabel-rpo-subsampled &    $0.879 \pm 0.008$ (5) &    $0.770 \pm 0.010$ (6) &  $0.685 \pm 0.008$ (5) \\
gbrt-rpo-gaussian           &  $0.884 \pm 0.005$ (2.5) &  $0.775 \pm 0.018$ (5) &  $0.706 \pm 0.007$ (1) \\
gbrt-rpo-subsampled         &    $0.875 \pm 0.006$ (7) &  $0.723 \pm 0.016$ (7) &  $0.691 \pm 0.006$ (4) \\
mo-rf                       &    $0.865 \pm 0.003$ (8) &    $0.800 \pm 0.006$ (1) &  $0.643 \pm 0.003$ (8) \\
\midrule
                            &                   wipo &                    yeast &                 yeast-go \\
\midrule
br-gbrt                     &  $0.706 \pm 0.009$ (6) &    $0.756 \pm 0.009$ (8) &  $0.499 \pm 0.009$ (4.5) \\
br-rf                       &  $0.633 \pm 0.013$ (7) &   $0.760 \pm 0.008$ (3.5) &     $0.463 \pm 0.010$ (7) \\
gbmort                      &  $0.762 \pm 0.011$ (3) &   $0.760 \pm 0.007$ (3.5) &    $0.504 \pm 0.015$ (3) \\
gbrt-relabel-rpo-gaussian   &  $0.776 \pm 0.012$ (1) &    $0.762 \pm 0.007$ (2) &    $0.524 \pm 0.012$ (1) \\
gbrt-relabel-rpo-subsampled &  $0.751 \pm 0.017$ (4) &  $0.758 \pm 0.005$ (5.5) &    $0.496 \pm 0.013$ (6) \\
gbrt-rpo-gaussian           &   $0.763 \pm 0.010$ (2) &    $0.763 \pm 0.005$ (1) &    $0.522 \pm 0.012$ (2) \\
gbrt-rpo-subsampled         &  $0.724 \pm 0.011$ (5) &  $0.758 \pm 0.008$ (5.5) &  $0.499 \pm 0.011$ (4.5) \\
mo-rf                       &  $0.624 \pm 0.018$ (8) &    $0.757 \pm 0.008$ (7) &    $0.415 \pm 0.014$ (8) \\
\bottomrule
\end{tabular}
\end{small}
\end{table}

\begin{figure}
\caption{Critical difference diagram between algorithms on the multi-label datasets.}
\label{fig:cd-multilabel}
\centering
\includegraphics[width=\textwidth]{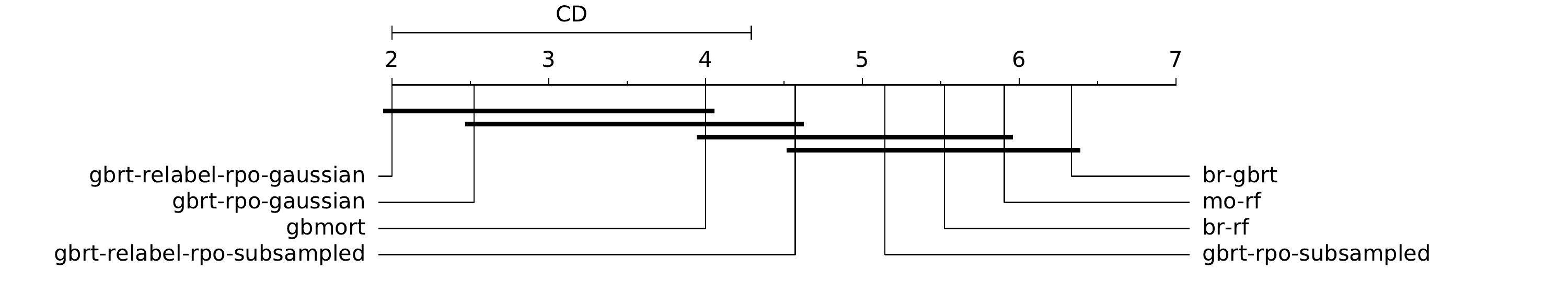}
\end{figure}

\begin{table}
\caption{P-values given by the Wilcoxon signed rank test on the multi-label datasets.
We bold p-values below $\alpha=0.05$. Note that the sign $>$ (resp. $<$) indicates that
the row estimator has superior (resp. inferior) LRAP score than the column estimator.}
\label{tab:wilcoxon-multilabel}
\centering
\setlength{\tabcolsep}{1pt}
\begin{footnotesize}
\begin{tabular}{lllllllll}
{} & \rotatebox{90}{br-gbrt} & \rotatebox{90}{br-rf} & \rotatebox{90}{gbmort} & \rotatebox{90}{gbrt-relabel-rpo-gaussian} & \rotatebox{90}{gbrt-relabel-rpo-subsampled} & \rotatebox{90}{gbrt-rpo-gaussian} & \rotatebox{90}{gbrt-rpo-subsampled} & \rotatebox{90}{mo-rf} \\
br-gbrt                     &                         &                  0.79 &      \textbf{0.001}(<) &                         \textbf{6e-05}(<) &                           \textbf{0.001}(<) &                \textbf{0.0001}(<) &                   \textbf{0.003}(<) &                  0.54 \\
br-rf                       &                    0.79 &                       &       \textbf{0.01}(<) &                         \textbf{0.002}(<) &                            \textbf{0.02}(<) &                 \textbf{0.005}(<) &                                0.07 &                  0.29 \\
gbmort                      &       \textbf{0.001}(>) &      \textbf{0.01}(>) &                        &                         \textbf{0.001}(<) &                                        0.36 &                  \textbf{0.04}(<) &                    \textbf{0.03}(>) &      \textbf{0.02}(>) \\
gbrt-relabel-rpo-gaussian   &       \textbf{6e-05}(>) &     \textbf{0.002}(>) &      \textbf{0.001}(>) &                                           &                          \textbf{0.0002}(>) &                 \textbf{0.005}(>) &                   \textbf{7e-05}(>) &    \textbf{0.0007}(>) \\
gbrt-relabel-rpo-subsampled &       \textbf{0.001}(>) &      \textbf{0.02}(>) &                   0.36 &                        \textbf{0.0002}(<) &                                             &                \textbf{0.0002}(<) &                    \textbf{0.02}(>) &     \textbf{0.008}(>) \\
gbrt-rpo-gaussian           &      \textbf{0.0001}(>) &     \textbf{0.005}(>) &       \textbf{0.04}(>) &                         \textbf{0.005}(<) &                          \textbf{0.0002}(>) &                                   &                   \textbf{7e-05}(>) &     \textbf{0.002}(>) \\
gbrt-rpo-subsampled         &       \textbf{0.003}(>) &                  0.07 &       \textbf{0.03}(<) &                         \textbf{7e-05}(<) &                            \textbf{0.02}(<) &                 \textbf{7e-05}(<) &                                     &      \textbf{0.04}(>) \\
mo-rf                       &                    0.54 &                  0.29 &       \textbf{0.02}(<) &                        \textbf{0.0007}(<) &                           \textbf{0.008}(<) &                 \textbf{0.002}(<) &                    \textbf{0.04}(<) &                       \\
\end{tabular}
\end{footnotesize}
\end{table}

\subsubsection{Multi-output regression datasets}
\label{subsec:gb-mo-reg-exp}

Table~\ref{table:summary-regression} shows the performance of the random forest
models and the boosting algorithms over the 8 multi-output regression datasets.
The critical distance diagram of Figure~\ref{fig:cd-regression} gives the rank
of each estimator. The associated Friedman test has a p-value of $0.3$. Given
the outcome of the test, we can therefore not reject the null hypothesis that
the estimator performances can not be distinguished.
Table~\ref{tab:wilcoxon-regression} gives the outcomes of the pairwise Wilcoxon
signed ranked tests. They confirm the fact that all methods are very close to
each other as only two comparisons show a p-value lower than 0.05 (st-rf is
better than st-gbrt and gbrt-rpo-subsampled). This lack of statistical power is
probably partly due here to the smaller number of datasets included in the
comparison (8 problems versus 21 problems in classification).

If we ignore statistical tests, as with multi-label tasks,
gbrt-relabel-rpo-gaussian has the best average rank and st-gbrt the worst
average rank. This time however, gbrt-relabel-rpo-gaussian is followed by the
random forest based algorithms (st-rf and mo-rf) and gbmort. Given the lack of
statistical significance, this ranking should however be intrepreted cautiously.


\begin{table}
\caption{Performance over 8 multi-output regression dataset}
\label{table:summary-regression}
\centering
\setlength{\tabcolsep}{1.5pt}
\begin{small}\footnotesize
\begin{tabular}{llll}
\toprule
                            &                  atp1d &                  atp7d &                    edm \\
\midrule
gbmort                      &   $0.80 \pm 0.03$(5.5) &  $0.63 \pm 0.03$(2) &  $0.39 \pm 0.16$(3) \\
gbrt-relabel-rpo-gaussian   &  $0.81 \pm 0.03$(3.5) &  $0.66 \pm 0.04$(1) &  $0.25 \pm 0.28$(8) \\
gbrt-relabel-rpo-subsampled &    $0.79 \pm 0.04$(7) &  $0.54 \pm 0.13$(7) &   $0.35 \pm 0.10$(5) \\
gbrt-rpo-gaussian           &   $0.80 \pm 0.04$(5.5) &   $0.54 \pm 0.20$(7) &  $0.36 \pm 0.04$(4) \\
gbrt-rpo-subsampled         &  $0.81 \pm 0.04$(3.5) &  $0.54 \pm 0.16$(7) &  $0.31 \pm 0.27$(7) \\
mo-rf                       &    $0.82 \pm 0.03$(2) &   $0.6 \pm 0.06$(4) &  $0.51 \pm 0.02$(1) \\
st-gbrt                     &    $0.78 \pm 0.05$(8) &  $0.59 \pm 0.08$(5) &  $0.34 \pm 0.14$(6) \\
st-rf                       &    $0.83 \pm 0.02$(1) &  $0.61 \pm 0.07$(3) &  $0.47 \pm 0.04$(2) \\
\midrule
                            &                  oes10 &                  oes97 &                    scm1d \\
\midrule
gbmort                      &  $0.77 \pm 0.05$(3.5) &    $0.67 \pm 0.07$(8) &  $0.908 \pm 0.003$(4.5) \\
gbrt-relabel-rpo-gaussian   &  $0.75 \pm 0.04$(7.5) &  $0.71 \pm 0.07$(2.5) &   $0.910 \pm 0.004$(2.5) \\
gbrt-relabel-rpo-subsampled &  $0.75 \pm 0.06$(7.5) &    $0.68 \pm 0.07$(6) &    $0.912 \pm 0.003$(1) \\
gbrt-rpo-gaussian           &  $0.77 \pm 0.03$(3.5) &    $0.68 \pm 0.08$(6) &   $0.910 \pm 0.004$(2.5) \\
gbrt-rpo-subsampled         &  $0.76 \pm 0.02$(5.5) &  $0.71 \pm 0.08$(2.5) &  $0.908 \pm 0.004$(4.5) \\
mo-rf                       &  $0.76 \pm 0.04$(5.5) &    $0.69 \pm 0.05$(4) &    $0.898 \pm 0.004$(8) \\
st-gbrt                     &  $0.79 \pm 0.03$(1.5) &    $0.68 \pm 0.07$(6) &    $0.905 \pm 0.003$(7) \\
st-rf                       &  $0.79 \pm 0.03$(1.5) &    $0.72 \pm 0.05$(1) &    $0.907 \pm 0.004$(6) \\
\midrule
                            &                   scm20d &          water-quality \\
\midrule
gbmort                      &    $0.856 \pm 0.006$(2) &  $0.14 \pm 0.01$(4.5) \\
gbrt-relabel-rpo-gaussian   &    $0.862 \pm 0.006$(1) &    $0.15 \pm 0.01$(2) \\
gbrt-relabel-rpo-subsampled &    $0.854 \pm 0.007$(3) &  $0.14 \pm 0.02$(4.5) \\
gbrt-rpo-gaussian           &    $0.852 \pm 0.006$(4) &  $0.14 \pm 0.01$(4.5) \\
gbrt-rpo-subsampled         &     $0.850 \pm 0.007$(5) &  $0.13 \pm 0.02$(7.5) \\
mo-rf                       &  $0.849 \pm 0.007$(6.5) &    $0.16 \pm 0.01$(1) \\
st-gbrt                     &    $0.836 \pm 0.006$(8) &  $0.13 \pm 0.02$(7.5) \\
st-rf                       &  $0.849 \pm 0.006$(6.5) &  $0.14 \pm 0.01$(4.5) \\
\bottomrule
\end{tabular}
\end{small}
\end{table}

\begin{figure}
\caption{Critical difference diagram between algorithm on the multi-output regression
datasets.}
\label{fig:cd-regression}
\centering
\includegraphics[width=\textwidth]{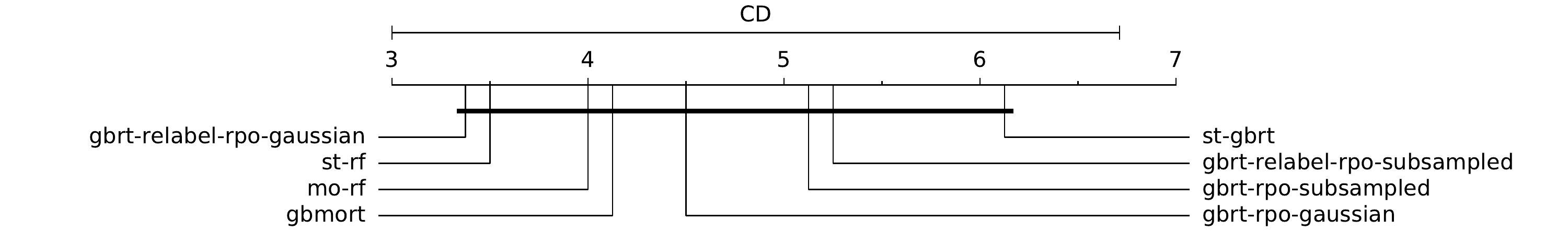}
\end{figure}

\begin{table}
\caption{P-value given by the Wilcoxon signed rank test on multi-output
regression datasets. We bold p-values below $\alpha=0.05$. Note that the sign $>$ (resp. $<$) indicates that
the row estimator has superior (resp. inferior) macro-$r^2$ score than the column estimator.}
\label{tab:wilcoxon-regression}
\centering
\setlength{\tabcolsep}{2pt}
\begin{footnotesize}
\begin{tabular}{lllllllll}
{} & \rotatebox{90}{st-gbrt} & \rotatebox{90}{st-rf} & \rotatebox{90}{gbmort} & \rotatebox{90}{gbrt-relabel-rpo-gaussian} & \rotatebox{90}{gbrt-relabel-rpo-subsampled} & \rotatebox{90}{gbrt-rpo-gaussian} & \rotatebox{90}{gbrt-rpo-subsampled} & \rotatebox{90}{mo-rf} \\
st-gbrt                     &                         &      \textbf{0.02}(<) &                   0.26 &                                      0.58 &                                        0.78 &                              0.67 &                                0.89 &                  0.16 \\
st-rf                       &        \textbf{0.02}(>) &                       &                   0.33 &                                      0.67 &                                        0.09 &                              0.09 &                    \textbf{0.04}(>) &                  0.58 \\
gbmort                      &                    0.26 &                  0.33 &                        &                                       0.4 &                                        0.16 &                              0.67 &                                 0.4 &                  0.48 \\
gbrt-relabel-rpo-gaussian   &                    0.58 &                  0.67 &                    0.4 &                                           &                                        0.16 &                               0.4 &                                0.48 &                  0.58 \\
gbrt-relabel-rpo-subsampled &                    0.78 &                  0.09 &                   0.16 &                                      0.16 &                                             &                              0.16 &                                   1 &                  0.07 \\
gbrt-rpo-gaussian           &                    0.67 &                  0.09 &                   0.67 &                                       0.4 &                                        0.16 &                                   &                                0.48 &                  0.26 \\
gbrt-rpo-subsampled         &                    0.89 &      \textbf{0.04}(<) &                    0.4 &                                      0.48 &                                           1 &                              0.48 &                                     &                   0.4 \\
mo-rf                       &                    0.16 &                  0.58 &                   0.48 &                                      0.58 &                                        0.07 &                              0.26 &                                 0.4 &                       \\
\end{tabular}
\end{footnotesize}
\end{table}

\clearpage

\section{Conclusions}
\label{sec:conclusions}

In this chapter, we have first formally extended the gradient boosting
algorithm to multi-output tasks leading to the ``multi-output gradient boosting
algorithm'' (gbmort). It sequentially minimizes a multi-output loss using
multi-output weak models considering that all outputs are correlated. By
contrast, binary relevance / single target of gradient boosting models fit one
gradient boosting model per output considering that all outputs are
independent. However in practice, we do not expect to have either all outputs
independent or all outputs dependent. So, we propose a more flexible approach
which adapts automatically to the output correlation structure called
``gradient boosting with random projection of the output space'' (gbrt-rpo). At
each boosting step, it fits a single weak model on a random projection of the
output space and optimize a multiplicative weight separately for each output.
We have also proposed a variant of this algorithm (gbrt-relabel-rpo) only valid
with decision trees as weak models: it fits a decision tree on the randomly
projected space and then it relabels tree leaves with predictions in the
original (residual) output space. The combination of the gradient boosting
algorithm and the random projection of the output space yields faster
convergence by exploiting existing correlations between the outputs and by
reducing the dimensionality of the output space. It also provides new
bias-variance-convergence trade-off potentially allowing to improve
performance.

We have evaluated in depth these new algorithms on several artificial and real
datasets. Experiments on artificial problems highlighted that gb-rpo with
output subsampling offers an interesting tradeoff between single target and
multi-output gradient boosting. Because of its capacity to automatically adapt
to the output space structure, it outperforms both methods in terms of
convergence speed and accuracy when outputs are dependent and it is superior to
gbmort (but not st-rt) when outputs are fully independent. On the 29 real
datasets, gbrt-relabel-rpo with the denser Gaussian projections turns out to be
the best overall approach on both multi-label classification and multi-output
regression problems, although all methods are statistically undistinguisable on
the regression tasks. Our experiments also show that gradient boosting based
methods are competitive with random forests based methods. Given that
multi-output random forests were shown to be competitive with several other
multi-label approaches in \cite{madjarov2012extensive}, we are confident that
our solutions will be globally competitive as well, although a broader
empirical comparison should be conducted as future work. One drawback of
gradient boosting with respect to random forests however is that its
performance is more sensitive to its hyper-parameters that thus require careful
tuning. Although not discussed in this chapter, besides predictive performance,
gbrt-rpo (without relabeling) has also the advantage of reducing model size
with respect to mo-rf (multi-output random forests) and gbmort, in particular
in the presence of many outputs. Indeed, in mo-rf and gbmort, one needs to
store a vector of the size of the number of outputs per leaf node. In gbrt-rpo,
one needs to store only one real number (a prediction for the projection) per
leaf node and a vector of the size of the number of outputs per tree
($\rho_m$). At fixed number of trees and fixed tree complexity, this could lead
to a strong reduction of the model memory requirement when the number of labels
is large. Note that the approach proposed in Chapter
\ref{ch:rf-output-projections} does not solve this issue because of leaf node
relabeling. This could be addressed by desactivating leaf relabeling and
inverting the projection at prediction time to obtain a prediction in the
original output space, as done for example in
\cite{hsu2009multi,kapoor2012multilabel,tsoumakas2014multi}. However, this
would be at the expense of computing times at prediction time and of accuracy
because of the potential introduction of errors at the decoding stage. Finally, while we
restricted our experiments here to tree-based weak learners,
Algorithms~\ref{algo:gb-mo} and \ref{algo:gbrt-rp} are generic and could
exploit respectively any multiple output and any single output regression
method. As future work, we believe that it would interesting to evaluate them with other weak
learners.


\part{Exploiting sparsity for growing and compressing decision trees}
\label{part:sparsity}


\chapter{$\ell_1$-based compression of random forest models}
\label{ch:rf-compression}

\begin{remark}{Outline}
Random forests are effective supervised learning methods applicable to
large-scale datasets. However, the space complexity of tree ensembles, in terms
of their total number of nodes, is often prohibitive, specially in the context
of problems with large sample sizes and very high-dimensional input spaces. We
propose to study their compressibility by applying a $\ell_1$-based
regularization to the set of indicator functions defined by all their nodes. We
show experimentally that preserving or even improving the model accuracy while
significantly reducing its space complexity is indeed possible.

\textit{This chapter extends on previous work published in}
\begin{quote}
Arnaud Joly, Fran{\c{c}}ois Schnitzler, Pierre Geurts, and Louis Wehenkel. L1-based
compression of random forest models. In European Symposium on Artificial Neural
Networks, Computational Intelligence and Machine Learning, 2012.
\end{quote}
\end{remark}

High-dimensional supervised learning problems, \emph{e.g.} in image exploitation
and bioinformatics, are more frequent than ever. Tree-based ensemble methods,
such as random forests~\cite{breiman2001random} and extremely randomized
trees~\cite{geurts2006extremely}, are effective variance reduction techniques
offering in this context a good trade-off between accuracy, computational
complexity, and interpretability. The number of nodes of a tree ensemble grows
as $n M$ ($n$ being the size of the learning sample and $M$ the number of trees
in the ensemble). Empirical observations show that the variance of individual
trees increases with the dimension $p$ of the original feature space used to
represent the inputs of the learning problem. Hence, the number $M(p)$ of
ensemble terms yielding near-optimal accuracy, which is proportional to this
variance, also increases with $p$.  The net result is that the space complexity
of these tree-based ensemble methods will grow as $n  M(p)$, which  may
jeopardize their practicality in large scale problems, or when memory is
limited.

While pruning of single tree models is a standard approach, less work has been
devoted to pruning ensembles of trees. On the one hand, \citet{geurts2000some}
proposes to transpose the classical cost-complexity pruning of individual trees
to ensembles. On the other hand,
\citet{meinshausen2010node,friedman2008predictive,meinshausen2009forest}
propose to improve model interpretability by selecting optimal rule subsets from
tree-ensembles. Another approach to reduce complexity and/or  improve accuracy
of ensembles of trees is to merely select an optimal subset of trees from a very
large ensemble generated in a random fashion at the first hand (see,
e.g.~\cite{bernard2009selection,martinez2009analysis}).

To further investigate the feasibility of reducing the space complexity of
tree-based ensemble models, we consider in this chapter  the following
method~\cite{joly2012l1}: (i) build an ensemble of trees; (ii) apply to this
ensemble a `compression step' by reformulating the tree-ensemble based model as
a linear model in terms of node indicator functions and by using an $\ell_1$-norm
regularization approach - \`a la Lasso~\cite{tibshirani1996regression} -  to
select a minimal subset of these indicator functions while maintaining
predictive accuracy.  We propose an algorithmic framework and an empirical
investigation of this idea, based on three complementary datasets, and we show
that indeed it is possible to so compress significantly tree-based ensemble
models, both in regression and in classification problems. We also observe that
the compression rate and the accuracy of the compressed models further increase
with the ensemble size $M$, even beyond the number  $M(p)$ of terms  required to
ensure convergence of the variance reduction effect.

The rest of this chapter is organized as follows:
Section~\ref{sec:ch07-algo_desc} introduces the $\ell_1q$-norm based compression
algorithm of random forests; Section~\ref{sec:ch07-empirical_experiments}
provides our empirical study and Section~\ref{sec:ch07-conclusion} concludes and
describes further perspectives.

\section{Compressing tree ensembles by $\ell_1$-norm regularization}
\label{sec:ch07-algo_desc}

From an ensemble of $M$ decision trees, one can extract a set of node indicator
functions as follows: each indicator function $1_{m,l}(x)$ is a binary variable
equal to 1 if the input vector $x$ reaches the $l${th} node in the $m${th} tree,
$0$ otherwise. Using these indicator functions, the output predicted by the
model may be rewritten as~\cite{geurts2006extremely,vens2011random}:
\begin{equation}
\hat{f}(x) = \frac{1}{M} \sum_{m=1}^M \sum_{l=1}^{N_m} w_{m,l}\,1_{m,l}(x),
\label{eq:prediction}
\end{equation}
\noindent where $N_m$ is the number of nodes in the $m$th tree and $w_{m,l}$ is
equal to the leaf-label if node $(m,l)$ is a leaf  and to zero if it is an
internal node.

We can therefore interpret a tree building algorithm as the (random) inference
of a new representation which lifts the original input space $\mathcal{X}$
towards the space $\mathcal{Z}$ of dimension $q = \sum_{m=1}^{M}N_{m}$ by \[ z(x) = \left(
1_{1,1}(x), \ldots, 1_{1,N_{1}}(x), \ldots{}, 1_{M,1}(x), \ldots, 1_{M,N_{M}}(x)
\right).\]

As an illustration, Figure~\ref{fig:sample-indicator-path} shows a set of
three decision trees with respective sizes 7, 7 and 5 nodes. The propagation of
a sample $x_s$ through the forest makes it pass trough nodes  1.1, 1.2, 1.5 in
the left tree, nodes 2.1, 2.2, 2.5 in the middle tree and nodes 3.1, 3.3 and 3.4
in the left tree (highlighted in orange).

\begin{figure}
\centering
\includegraphics[width=\textwidth]{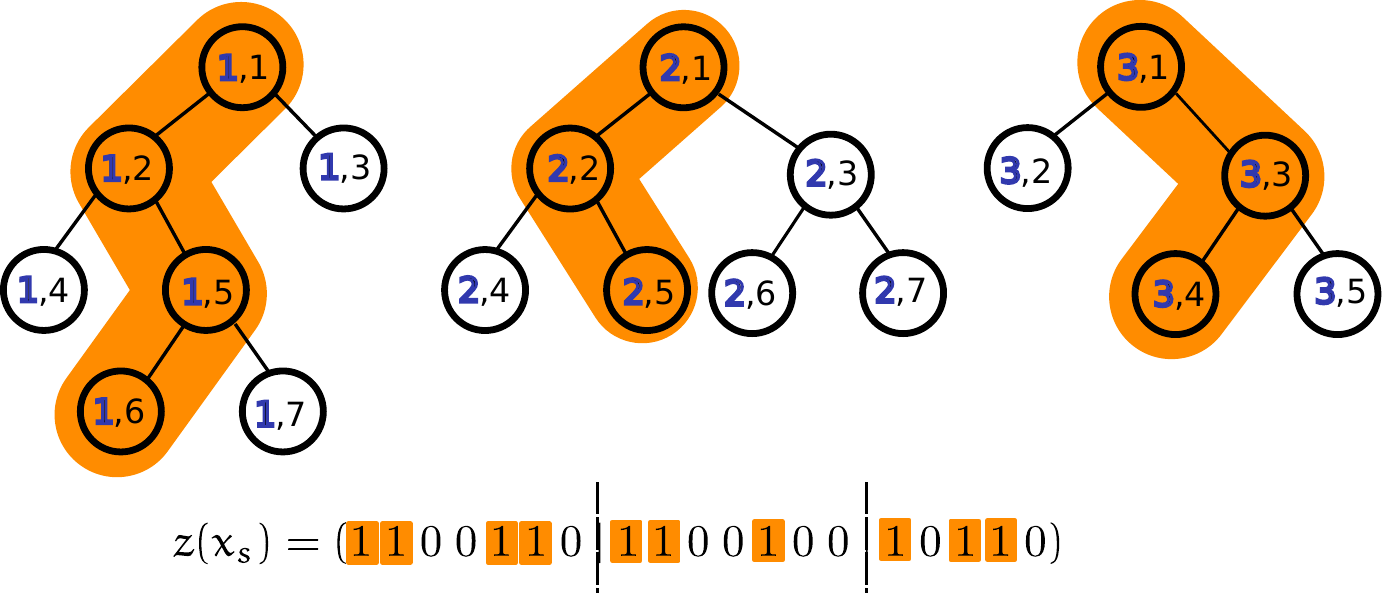}
\caption{From a random forest model, one can lift the original input space
representation of a sample $x_s$ toward the node indicator space $\mathcal{Z}$. }
\label{fig:sample-indicator-path}
\end{figure}

We propose to compress the tree ensemble by applying a variable selection method
to its induced feature space $\mathcal{Z}$. Namely, by $\ell_1$-regularization we
can search for a linear model by solving the following optimization problem:
\begin{align}
&\left(\beta_j^{*}(t)\right)_{j=0}^q = \arg\min_{\beta} \sum_{i=1}^n \left(y^i - \beta_0  - \sum_{j=1}^{q}  \beta_j \, z_{j}(x^i) \right)^2 \nonumber  \\
\mbox{s.t.}& \sum_{j=1}^q  | \beta_j | \leq t.
\label{eq:regularisation_L1_norm}
\end{align}

This optimization problem, also called Lasso~\cite{tibshirani1996regression}
(see Section~\ref{sub:linear-models-classes}), has received much attention in
the past decade and is particularly successful in high dimension. The $\ell_1$-norm
constraint leads to a sparse solution: only a few weights $\beta_j$ will be non
zero, and their number tends to zero with $t\rightarrow 0$; the optimal value
$t^{*}$ of $t$ is problem specific and is typically adjusted by
cross-validation.

In order to solve Equation~\ref{eq:regularisation_L1_norm} for growing values of
$t$, we use the `incremental forward stagewise regression'
algorithm~\cite{hastie2007forward} solving the monotone Lasso which imposes that
each $\beta^{*}_{j}(t)$ increases monotonically with $t$. This version deals
indeed better with many correlated variables, which is relevant in our setting,
since each node indicator function is highly correlated with those of its
neighbor nodes in the tree from which it originates. The final weights
$\beta_j^{*}(t^{*})$ may be exploited to prune the randomized tree ensemble: a
test node can be deleted if all its descendants correspond to
$\beta_j^{*}(t^{*})=0$.

Starting from a forest model $\hat{f}$, a value of parameter $t$, and a sample
$\mathcal{S}$, the tree ensemble compression procedure is described
in Algorithm~\ref{alg:l1-pruning}.

\begin{algorithm}
\caption{$\ell_1$-based compression of tree ensemble model $\hat{f}$ using a sample
$\mathcal{S}= \{x^i, y^i\in\mathcal{X}\times\mathcal{Y}\}_{i=1}^n$}
\label{alg:l1-pruning}
\begin{algorithmic}[1]
\Function{ForestCompression}{$\mathcal{S}$, $\hat{f}$, t}
\State Lift the sample  $\mathcal{S}$ to the random forest space $\mathcal{Z}$
       \[
       \mathcal{S}_z = \{(z(x^i), y^i) \in \mathcal{Z} \times \mathcal{Y}\}_{(x,y)\in\mathcal{S}}
       \] with the induced feature space by the forest model $\hat{f}$
       \[ z(x) = \left(
       1_{1,1}(x), \ldots, 1_{1,N_{1}}(x), \ldots{}, 1_{M,1}(x), \ldots, 1_{M,N_{M}}(x)
       \right).\]
\State Select weight vector $\beta^{*}(t)$ over $\mathcal{Z}$ through $\ell_1$ minimization
\begin{align}
&\left(\beta_j^{*}(t)\right)_{j=0}^q = \arg\min_{\beta} \sum_{i=1}^n \left(y^i - \beta_0  - \sum_{j=1}^{q}  \beta_j \, z_{j}(x^i) \right)^2 \nonumber  \\
\mbox{s.t.}& \sum_{j=1}^q  | \beta_j | \leq t.
\nonumber 
\end{align}
\State Compress the random forest model $\hat{f}$ using vector $\beta^{*}(t)$
\State \Return The compressed model.
\EndFunction
\end{algorithmic}
\end{algorithm}

Note that in practice both the forest construction and the generation of its
sequence of compressed versions for growing values of $t$ may use the same
sample (the learning set). A separate validation set is however required to
select the optimal value of parameter $t$. This is similar to what is done with
the pruning of a single decision tree (see Section~\ref{sec:dt-pruning}).

\section{Empirical analysis}
\label{sec:ch07-empirical_experiments}

In the following experiments, datasets are pre-whitened: input/output data are
translated to zero mean and rescaled to unit variance. All results shown are
averaged over $50$ runs in order to avoid randomization artifacts.

Each one of these runs consisted of first generating a training set, and a
testing set, and then working as follows. When using the monotone Lasso, we
apply the incremental forward stagewise algorithm with a  step size
$\epsilon=0.01$. The optimal number of steps $n_{step}^*$ or the optimal point
$t^{*} = n_{step}^*\epsilon$ was chosen by ten-fold cross-validation
$t_{cv}^{*}$ over the training set (to this end, we used a quadratic loss in
regression and a $0-1$ loss in classification). More precisely, the training set
is first divided ten times through cross-validation into a learning set, used
both to fit a forest model and to run the incremental forward stagewise
algorithm on it, and into a validation set, to estimate the losses of the
resulting sequence of compressed forests. For each fold, we assess the model
fitted over the training set using the validation set with increasing values of
$t$ by steps of $\epsilon$. For each value of $t$, the ten model losses are
averaged. The optimal value of $t^{*}$ and the
corresponding model compression level are those leading to the best average loss
over the ten folds. The model is then refitted using the entire training set
with $t=t^{*}$.

Below, we will apply our approach while using the extremely randomized trees
method~\cite{geurts2006extremely} to grow the forests (abbreviated by ``ET'')
and we denote their $\ell_1$-regularization-based compressed version ``rET''.

We present an overall performance analysis in
Section~\ref{subsec:l1-pruning-overall}. Later on, we enhance our comprehension
of the pruning algorithm by studying the effect of the regularization parameter
$t$ in Section~\ref{subsec:regularization} and of the complexity of the initial
forest model by varying the pre-pruning rule values $n_{\min}$, the minimum
number of samples to split, and $M$, the number of trees, in
Section~\ref{subsec:tree-complexity}. While in these last two sections, we focus
our analysis on models obtained on the Friedman1 problem, we notice that similar
conclusions can also be drawn for Two-norm and SEFTi datasets.

\subsection{Overall performances}
\label{subsec:l1-pruning-overall}

We have evaluated our approach on two regression datasets Friedman1 and SEFTi
and one classification dataset Two-norm (see Appendix~\ref{ch:datasets} for
their description).

We have used a set of representative meta-parameter values ($K$, $n_{\min}$ and
$M$) of the Extra-Trees algorithm (see Table~\ref{tab:pruning-perf}). Accuracies
are measured on the test sample and complexity is measured by the number of test
nodes of the ET and rET models (the compression factor being the ratio of the
former to the latter). We observe a compression factor between 9 and 34, a
slightly lower error for the rET model than for the ET model on the two
regression problems (Friedman1 and SEFTi) and the opposite on Two-norm. To
compare, we show the results obtained with the linear Lasso based on the
original features (its complexity is measured by the number of kept features):
it is much less accurate than  both ET and rET on the (non-linear) regression
problems (Friedman1 and SEFTi), but superior on the (linear) classification
problem (Two-norm).

\begin{table}[hbt]
\caption{Overall assessment (parameters of the Extra-Tree method: $M=100$;
$K=p$;  $n_{\min}=1$ on Friedman1 and Two-norm, $n_{\min}=10$ on SEFTi). }
\label{tab:pruning-perf}

\ra{1.1}
\renewcommand{\tabcolsep}{1.6mm}

\centering
\begin{small}
\begin{tabular}{@{} r c ccc c ccc c@{}}
\toprule
Datasets & &     \multicolumn{3}{c}{Error}             & &     \multicolumn{4}{c}{Complexity}          \\
\cmidrule{3-5} \cmidrule{7-10}
  & & ET &   $\text{rET}$ & Lasso && ET  &  $\text{rET}$ & ET/rET & Lasso\\

\midrule
Friedman1 & & $0.19587$ &  $0.18593$& $0.282441$ && $29900$  & $~885$ & $34$ & $~4$\\
Two-norm   & & $0.04177$ &  $0.06707$& $0.033500$ && $~4878$  & $~540$ & $~9$ & $20$\\
SEFTi 	  & & $0.86159$ &  $0.84131$& $0.988031$ && $39436$  & $2055$ & $19$ & $14$\\
\bottomrule
\end{tabular}
\end{small}
\end{table}

Side experiments (results not provided) show that changing the value of
parameter $K$ does not influence significantly the final accuracy and complexity
on the Two-norm and Friedman1 datasets, while for SEFTi,  accuracy increases
strongly with $K$ (presumably due to a large number of noisy and/or irrelevant
features) with however little impact on the final complexity.

\subsection{Effect of the regularization parameter $t$.}
\label{subsec:regularization}

The complexity of the regularized ET model is shrunk with the $\ell_1$-norm
constraint of Equation~\ref{eq:regularisation_L1_norm}  in a way depending on
the value of $t$. As shown in Figure~\ref{fig:tree-pruning-path}(a), an increase
of $t$ decreases the error of rET until $t=3$, leading to a complexity
(Figure~\ref{fig:tree-pruning-path}(b)) of about 900 test nodes. Notice that in
general the rET model eventually overfits when $t$ becomes large, although this
is not visible on the range of values displayed in
Figure~\ref{fig:tree-pruning-path}(a) as the algorithm stops before.

\begin{figure}[h]
\centering
\subfloat[Estimated risk]{\includegraphics[width=0.48\textwidth]{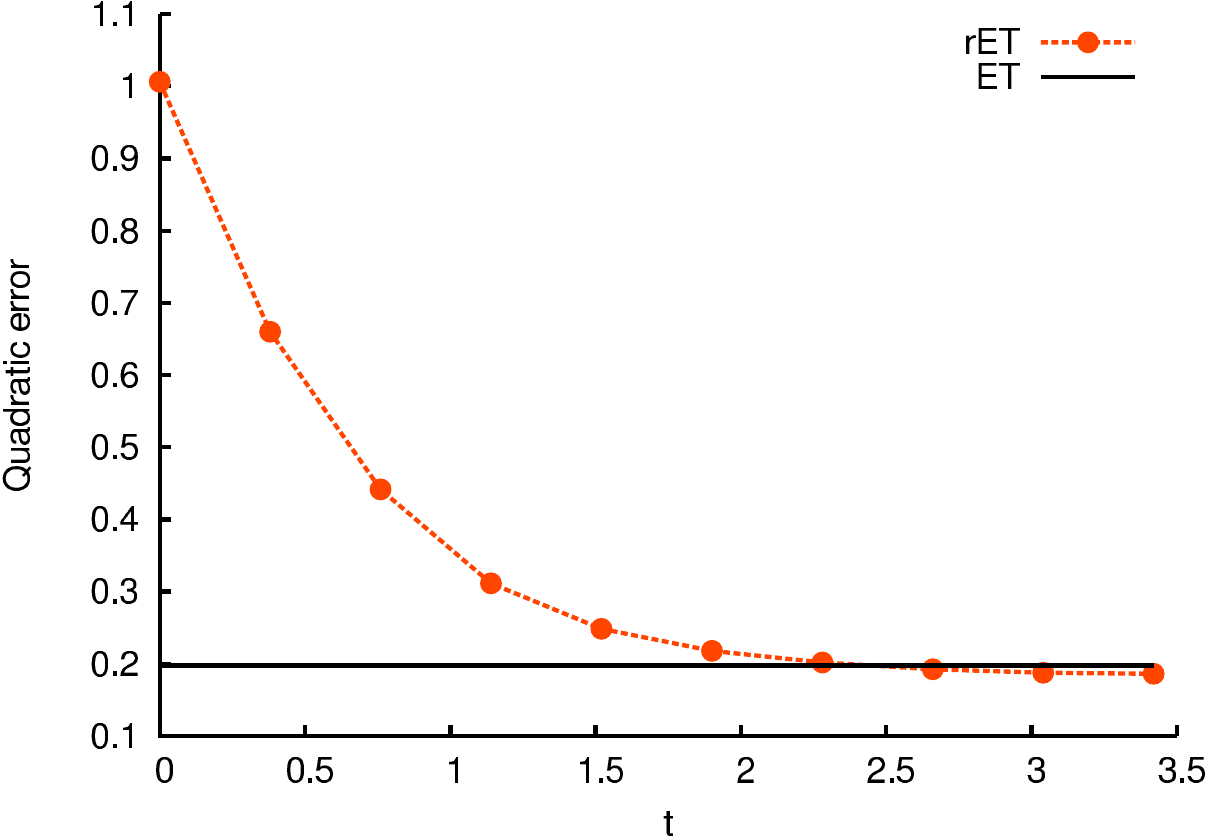}\label{fig:friedman_path_t_vs_RSS_M100_nmin1_ts}}\hspace*{3mm}
\subfloat[Complexity]{\includegraphics[width=0.48\textwidth]{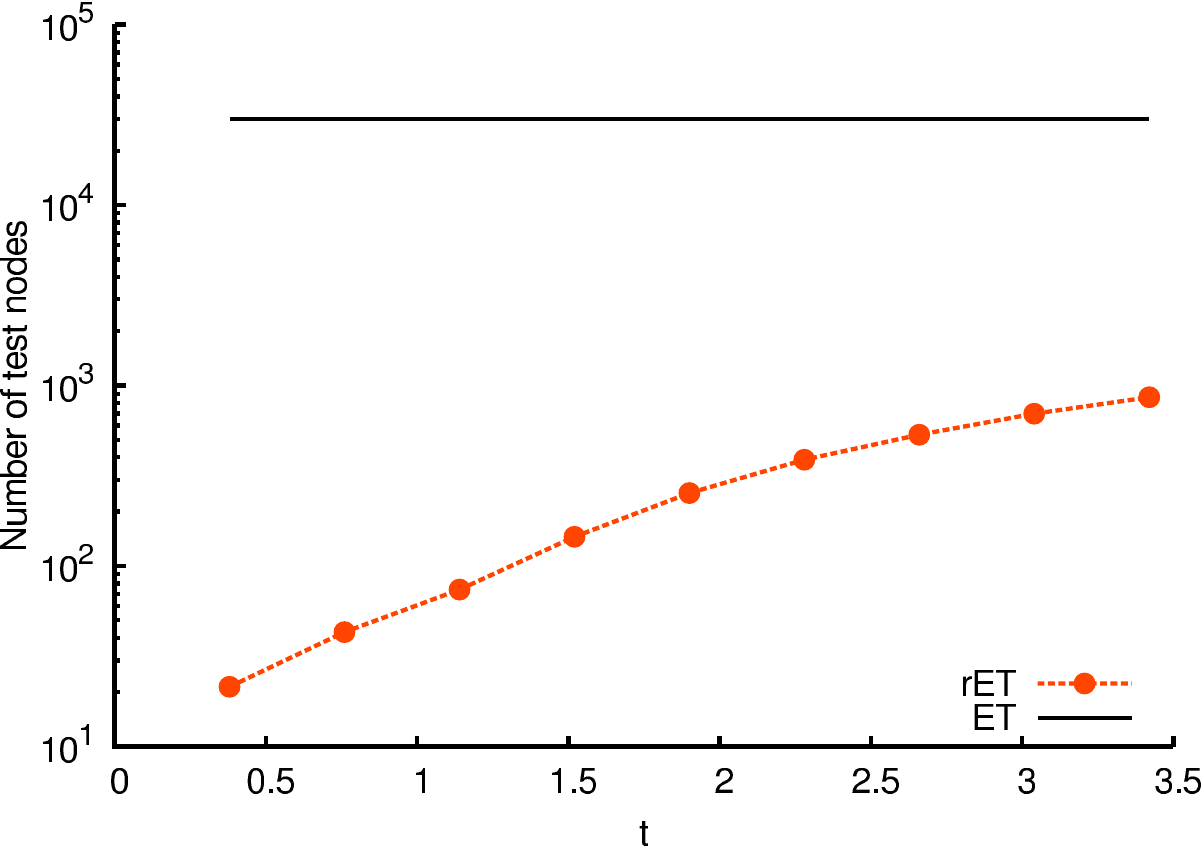}\label{fig:friedman_path_t_vs_complexity_M100_nmin1_nfolds_10.eps}}
\caption{An increase of $t$ decreases the error of rET until t = 3 with drastic
pruning (Friedman1, $M=100$, $K=p=10$ and $n_{\min}=1$).}
\label{fig:tree-pruning-path}
\end{figure}

\subsection{Influence of the Extra-Tree meta parameters $n_{\min}$ and $M$.}
\label{subsec:tree-complexity}

The complexity of an ET model grows (linearly) with the size of the ensemble $M$
and is inversely proportional to its pre-pruning parameter $n_{\min}$.

Figure~\ref{fig:nmin} shows the effect of $n_{\min}$ on both ET and rET.
Interestingly, the accuracy and the complexity of the rET model are both more
robust with respect to the choice of the precise value of $n_{\min}$ than those
of the ET model, specially for the smaller values of $n_{\min}$ ($n_{\min}\leq
10$, in Figures~\ref{fig:nmin}).

\begin{figure}[h]
\centering
\subfloat[Estimated risk]{\includegraphics[width=0.48\textwidth]{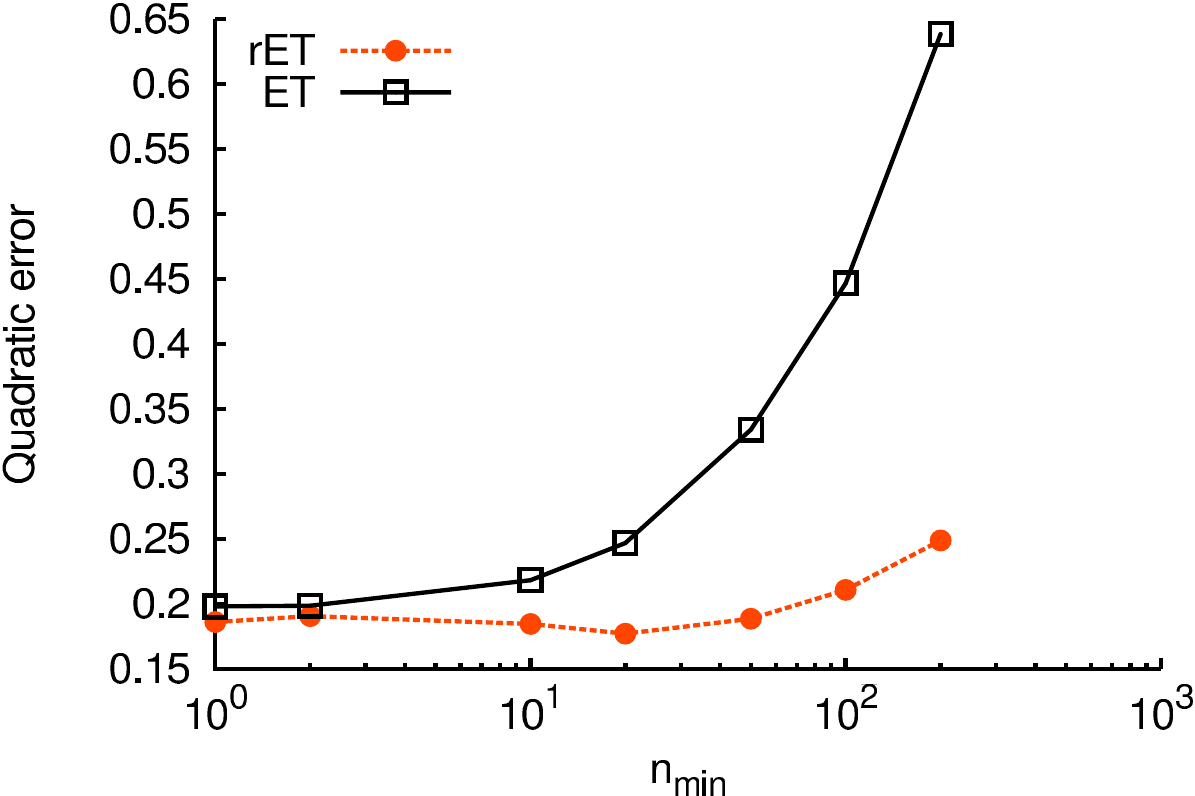}\label{fig:friedman_nmin_vs_RSS_M100_nfolds10_ts}}\hspace*{3mm}
\subfloat[Complexity]{\includegraphics[width=0.48\textwidth]{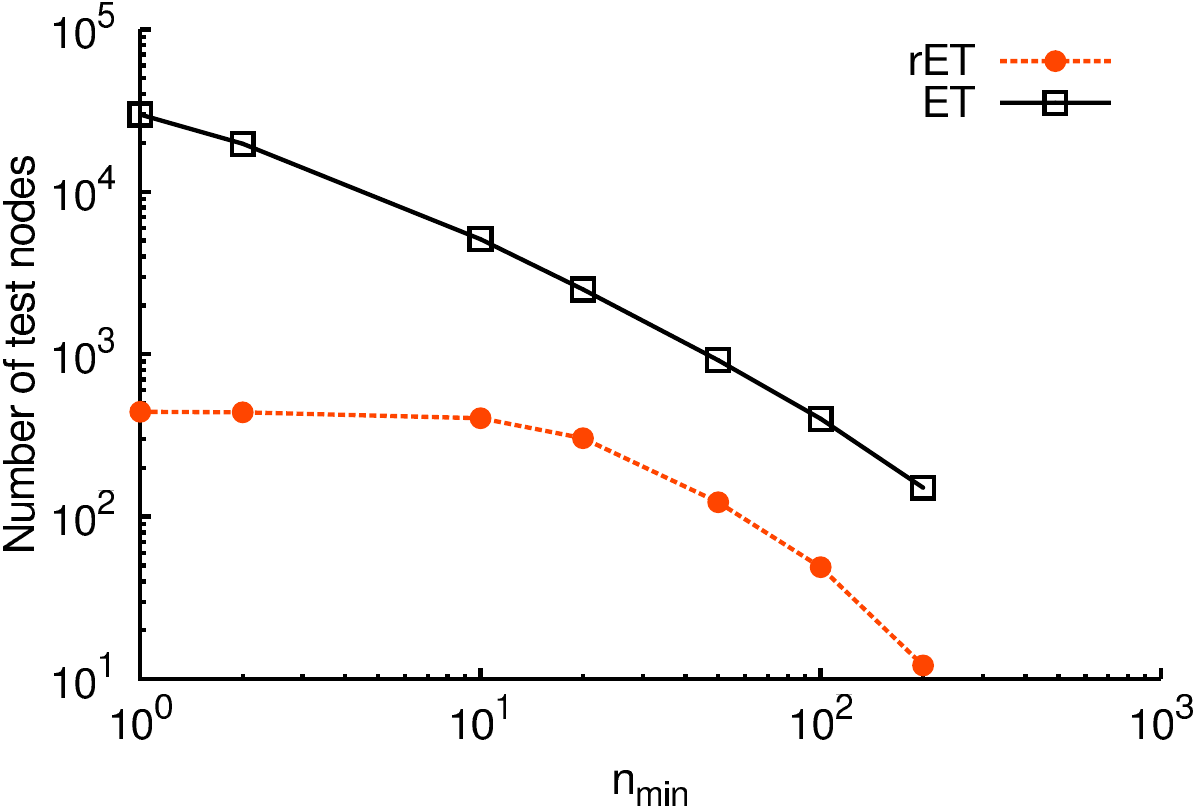}\label{fig:friedman_nmin_vs_complexity_M100_nfolds10}}
\caption{The accuracy and complexity of an rET model does not depend on
$n_{\min}$, for $n_{\min}$ small enough (Friedman1, $M=100$, $K=p=10$ and
$t=t_{cv}^*$).}
\label{fig:nmin}
\end{figure}

Figure \ref{fig:M} shows the effect of $M$ on both ET and rET models. We observe
that  increasing the value of $M$ beyond the value $M(p)$ where variance
reduction has stabilized ($M(p)\simeq100$ in Figure~\ref{fig:M}) allows to
further improve the accuracy of the rET model without increasing its complexity.

\begin{figure}[h]
\centering
\subfloat[Estimated risk]{\includegraphics[width=0.48\textwidth]{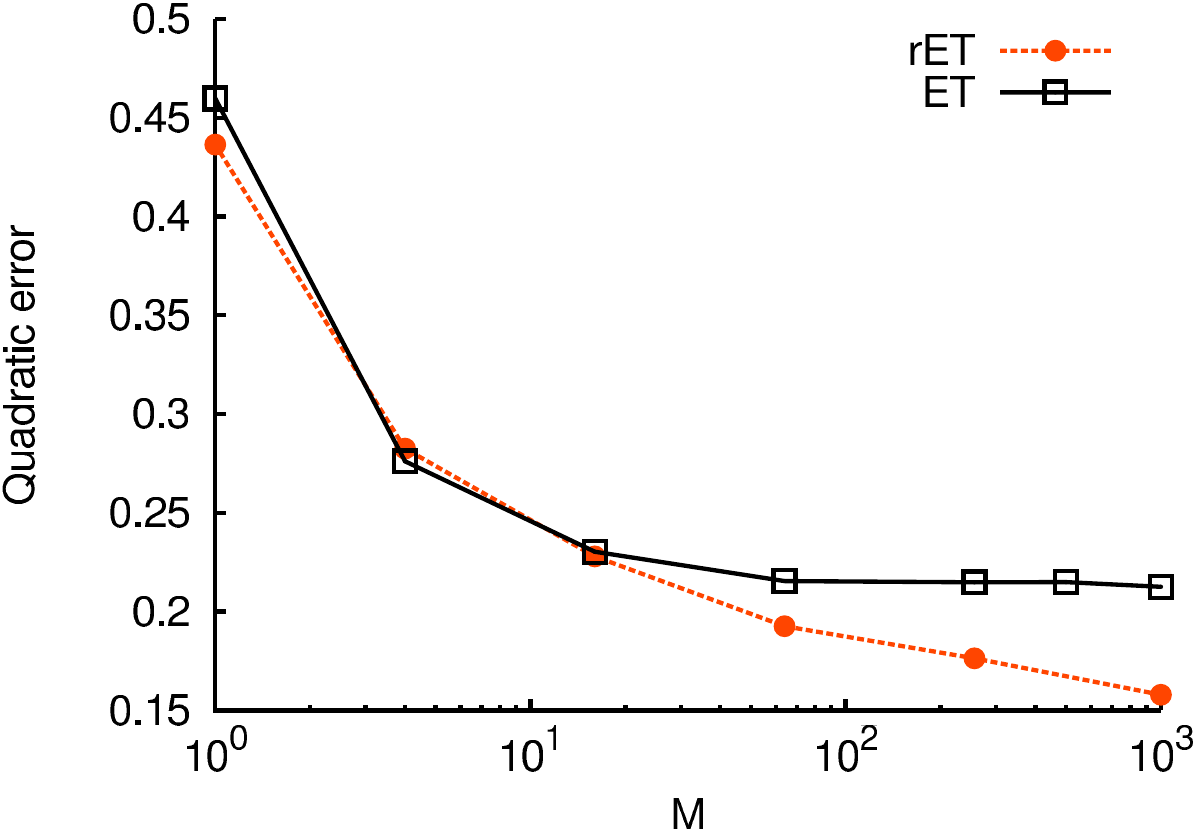}\label{fig:friedman_M_vs_RSS_nmin10_nfolds10_ts}}\hspace*{3mm}
\subfloat[Complexity]{\includegraphics[width=0.48\textwidth]{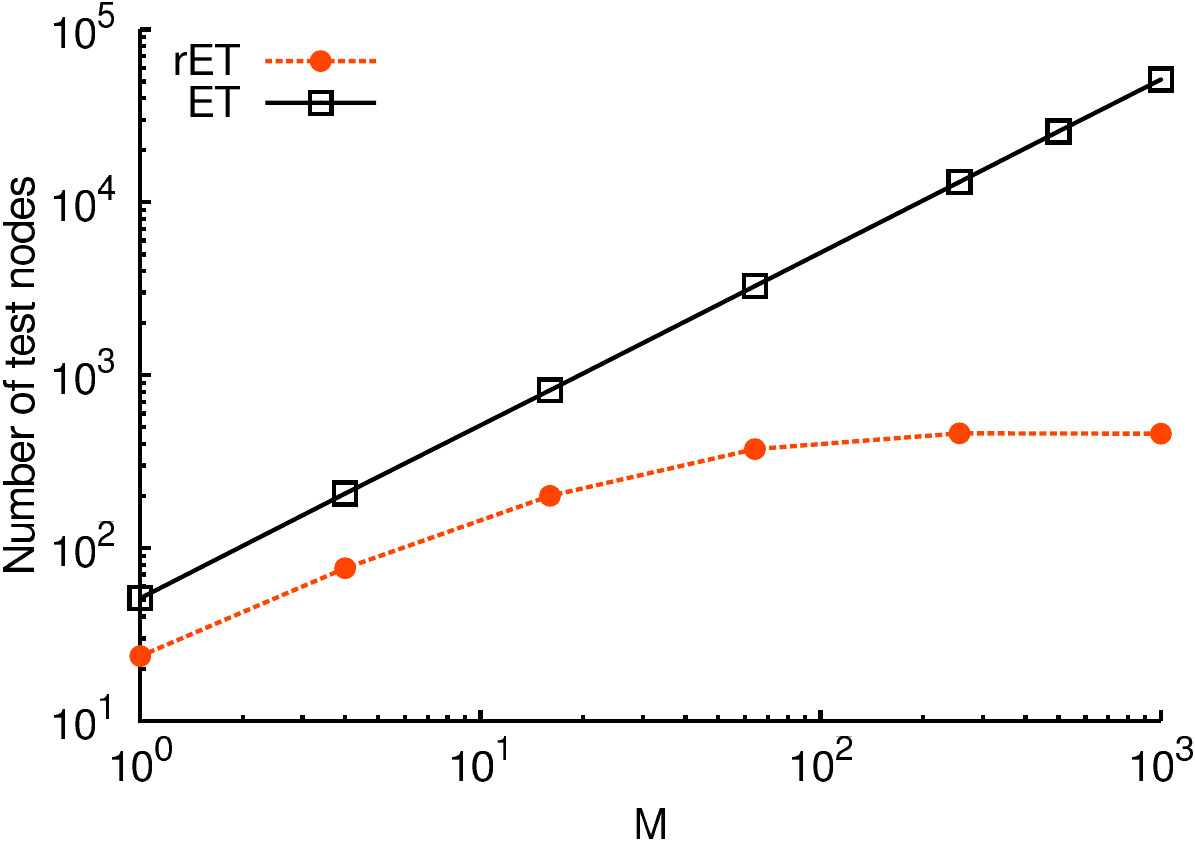}\label{fig:friedman_M_vs_complexity_nmin10_nfolds10}}
\caption{After variance reduction has stabilized ($M\simeq100$), further
increasing $M$ keeps enhancing the accuracy of the rET model without increasing
complexity (Friedman1, $n_{\min}=10$, $K=p=10$ and $t=t_{cv}^*$).}
\label{fig:M}
\end{figure}


%

\section{Conclusion}
\label{sec:ch07-conclusion}

Compression of randomized tree ensembles with $\ell_1$-norm regularization leads to
a drastic reduction of space complexity while preserving accuracy. The
complexity of the pruned model does not seem to be directly related to the
complexity of the original forest, \emph{i.e.} the number and complexity of each
randomized tree, as long as this forest has explored a large enough space of
variable interactions.

The strong compressibility of large randomized tree ensemble models suggests
that it could be possible to design novel algorithms based on tree-based
randomization which would scale in a better way to very high-dimensional input
spaces than the existing methods. To achieve this, one open question is how to
get  the compressed tree ensemble directly, \emph{i.e.} without generating a
huge randomized tree ensemble and then pruning it.

Tree-based  ensemble models may be interpreted as lifting  the original
input space towards a (randomly generated) high-dimensional discrete and sparse
representation, where each  induced feature corresponds to the indicator
function of a particular tree node, and takes the value $1$ for a given
observation if this observation reaches this node, and $0$ otherwise. The
dimension of this representation  is  on the order of $n  M(p)$, but the number
$s$ of non-zero components  for a given observation is only on the order  of
$M(p) \log n$. Compressed sensing theory~\cite{candes2008introduction} tells us
that high-dimensional sparsely representable observations may be compressed by
projecting them on a random subspace of dimension proportional to $s\log p$,
where  $p$ is the original dimension of the observations and $s \ll p$ is the
number of non-zero terms in their sparse representation basis. This suggests
that one could reduce the space complexity of tree-based method by applying
compressed sensing to their original input feature space if its dimension is
high, and/or to their induced feature space if $nM(p)$ is too large.


Since the publication of our work on this subject, several authors
have proposed similar ideas to post-prune a fully grown random forest
model: \citet{ren2015global} propose to iteratively remove or
re-weight the leaves of the random forest model, while
\citet{duroux2016impact} study the impact of pre-pruning on random
forest models.

\chapter{Exploiting input sparsity with decision tree}
\label{ch:tree-sparse}

\begin{remark}{Outline}
Many supervised learning tasks, such as text annotation, are characterized by
high dimensional and sparse input spaces where the input vectors of each sample
has only a few non zero values. We show how to exploit algorithmically the input
space sparsity within decision tree methods. It leads to significant speed up
both on synthetics and real datasets, while leading to exactly the same model.
We also reduce the required memory to grow such models by exploiting sparse
memory storage instead of dense memory storage for the input matrix.

\emph{This contribution is a joint work with Fares Hedayati and Panagiotis
Papadimitriou, working at \url{www.upwork.com}. The outcome of this research has
been proposed and merged in the
\emph{scikit-learn}~\cite{buitinck2013api,pedregosa2011scikit} open source
package.}
\end{remark}

Many machine learning tasks such as text annotation usually require training
over very big datasets with millions of web documents. Such tasks require
defining a mapping between the raw input space and the output space. For example
in text classification, a text document (raw input space) is usually mapped to a
vector whose dimensions correspond to all of the possible words in a dictionary
and the values of the vector elements are determined by the frequency of the
words in the document. Although such vectors have many dimensions, they are
often sparsely representable. For instance, the number of unique words
associated to a text document is actually small compared to the number of words
of a given language. We describe those samples with sparse input vectors as
having a few non zero values.

Exploiting the low density, i.e. the fraction of non zero elements, and the high
sparsity, i.e. the fraction of zero elements, is key to address such high dimensional
supervised learning tasks. Many models directly formulate their entire algorithm
to exploit the input sparsity. Linear models such as logistic regression or
support vector machine harness the sparsity by expressing most of their
operations as a set of linear algebra operations such as dot products who
directly exploit the sparsity to speed up computations.

Unfortunately, decision tree methods are not expressible only through linear
algebra operations. Decision tree methods are recursively partitioning the input
space by searching for the best possible splitting rules. As a consequence, most
machine learning packages either do not support sparse input vectors for
tree-based methods, only support stumps (decision tree with only one internal
node) or have a sub-optimal implementation through the simulation of a random
access memory as in the dense case. The only solution is often to densify the
input space which leads first to severe memory constraints and then to slow
training time.

In Section~\ref{sec:sparse-input-trees}, we present efficient algorithms to grow
vanilla decision trees, boosting and random forest methods on sparse input data.
In Section~\ref{sec:tree-sparse-prediction}, we describe how to adapt the
prediction algorithm of these models to sparse input data. In
Section~\ref{sec:sparse-experiments}, we show empirically the speed up obtained
by fitting decision trees with this input sparsity aware implementation.

\section{Tree growing}
\label{sec:sparse-input-trees}

During the decision tree fitting, the tree growing algorithm (see
Algorithm~\ref{algo:tree-grow}) interacts with the input space at two key points:
\begin{enumerate}

\item during the search of a splitting rule $s_t$ using a sample set
$\mathcal{L}_t$ at the expansion of node $t$ (see
line~\ref{alg-line:split-search} of Algorithm~\ref{algo:tree-grow});

\item during the data partitioning of the sample $\mathcal{L}_t$ into a left and
a right partition following the splitting rule $s_t$ at node $t$ (see
line~\ref{alg-line:sample-partitionning} of Algorithm~\ref{algo:tree-grow}).
\end{enumerate}

In this section, we show how to adapt decision tree at these three key points
to handle sparsely expressed data. While at the same time, we will show how to
harness the sparsity to speed up the original algorithm. We first explain how
node splitting is implemented in standard decision trees in
Section~\ref{subsec:standardsplitting} and then explain our efficient
implementation for sparse input data in
Section~\ref{subsec:sparse-splitting-rule-search}. In
Section~\ref{subsec:tree-data-partitionning}, we further describe how to
propagate samples with sparse input during the decision tree growing.

\subsection{Standard node splitting algorithm}
\label{subsec:standardsplitting}

During the decision tree growth, the crux of the tree growing algorithm in high
dimensional input space is the search of the best possible local splitting rule
$s_t$ (as described in Section~\ref{subsec:node-split}). Given a learning set
$\mathcal{L}_t$ reaching a node $t$, we search for the splitting rule $s_t$
among all the possible binary and axis-wise splitting rules
$\Omega(\mathcal{L}_t)$. We strive to maximize the impurity reduction $\Delta I$
obtained by dividing the sample set $\mathcal{L}_t$ into two partitions
$(\mathcal{L}_{t,r}, \mathcal{L}_{t,l})$. The splitting rule selection problem
(line \ref{alg-line:split-search} of the tree growing
Algorithm~\ref{algo:tree-grow}) is written as
\begin{equation}
s_t = \arg\max_{s \in \Omega(\mathcal{L}_t)}
\Delta I(\mathcal{L}_t, \mathcal{L}_{t,l}, \mathcal{L}_{t,r})
\label{eq:impurity }
\end{equation}
\noindent with
\begin{align}
\Omega(\mathcal{L}_t) = &
\Bigg\{s : s \in \bigcup_{j\in\{1,\ldots,p\}} Q(x_j), \nonumber \\
& \quad \mathcal{L}_{t,l} = \{(x,y)\in\mathcal{L}_{t}: s(x) = 1\}, \nonumber \\
& \quad \mathcal{L}_{t,r} = \{(x,y)\in\mathcal{L}_{t}: s(x) = 0\}, \nonumber \\
& \quad \mathcal{L}_{t,l}\not=\emptyset, \mathcal{L}_{t,r}\not=\emptyset \Bigg\}
\end{align}
\noindent where $Q(x_j)$ is the set of all splitting rules associated
to an input variable $x_j$.

Decision tree libraries carefully optimize this part of the algorithm to have
the proper computational complexity with a low constant. A careful design of the
algorithm for instance does not move around samples according to the partitions,
but instead move an identification number linked to each sample. The learning
set $\mathcal{L}$ is implemented as an array of row indices $L$ linked to the
rows of the input matrix $X$ and the output matrix $Y$. The sample set
$\mathcal{L}_t$ reaching a node $t$ is implemented as a slice of the array
$L[start_t:end_t[$ where the elements from the $start_t$ to the $`end_t-1`$
indices gives the indices of the samples reaching node $t$.

Let us take a small example with a set of $10$ training samples $L =
[0,\,1\,,\cdots, 9]$ illustrating the management of the array $L$. During the
tree growth (see Algorithm~\ref{algo:tree-grow}) when we partition the sample
set $\mathcal{L}_t=\{1,\ldots,10\}$ into two sample sets
$\mathcal{L}_{t,l}=\{9,\, 1, \,5, \,3\}$ and $\mathcal{L}_{t,r}=\{2, \,7, \,6,
\,4, \,8, \,0\}$. In practice, we modify the array $L$ such that from 0 to
$|\mathcal{L}_{1,l}|=4$ (resp. from $|\mathcal{L}_{1,l}|=4$ to
$|\mathcal{L}_1|=10$) are located the samples of the left child (resp. right
child). It leads to
\[
L = [\textcolor{orange}{9, \,1, \,5, \,3}, \,2, \,7, \,6, \,4,\, 8, \,0].
\]
\noindent We represent each sample set $\mathcal{L}_t$ as a slice $[start:end[$,
a chunk, of the array $L$. The sample set $\mathcal{L}_{1,l}$ is the slice
$[0:4[$ of $L$, while the sample set $\mathcal{L}_{1,r}$ is the slice $[4:10[$
of $L$. Now if the right node $\mathcal{L}_{1,r}$ is further split into a left
node with samples $\{6,0\}$ and a right node with samples $\{2, 7, 4, 8\}$ (in
orange), then $L[4:10[$ is further modified to reflect the split:
\[
L = [9, \,1,\, 5,\, 3,\, \textcolor{orange}{2,\, 7,\, 4,\, 8},\, 6,\,0].
\]

To further speed up the best splitting rule search with ordered variables, we
sort the possible thresholds associated to an ordered input variable
$x_j$ (programmatically $sort(X_j[L[start_t:end_t[[)$). By sorting the
possible thresholds sets, we can evaluate the impurity measure $I$ and the
impurity reduction $\Delta I$ in an online fashion. For instance, the Gini index
and entropy criteria can be computed by updating the class frequency in the left
and right split when moving from one splitting threshold to the next.

\subsection{Splitting rules search on sparse data}
\label{subsec:sparse-splitting-rule-search}

To handle sparse input data with decision trees, we need efficient procedures to
select the best splitting rules $s_t$ among all the possible splitting rules
knowing that we have a high proportion of zeros in the input matrix $X$. In this
section, we propose an efficient method to exploit the sparsity of the input
space with decision tree models. Our method takes advantage of the input
sparsity by avoiding sorting sample sets of a node along a feature unless there
are non zero elements at this feature. This approach speeds up training
substantially as extracting the possible threshold values and sorting them is a
costly but essential and ubiquitous component of tree-based models.

The splitting rule search algorithm for an ordered variable $x_j$ at a
node $t$ is divided in two parts (see Algorithm~\ref{algo:sparse-split}): (i) to
extract efficiently the non zero values associated to $x_j$ in the sample
partition $\mathcal{L}_t$ (line~\ref{alg-line:nnz-extraction}) and (ii) to
search separately among the splitting rules with the positive, negative or zero
threshold (line~\ref{alg-line:sparse-split-optim}).

\begin{algorithm}
\caption{Search for the best splitting rule $s_t^*$ given a sparse input
matrix $X$ and a set of samples $\mathcal{L}_t$}
\label{algo:sparse-split}
\begin{algorithmic}[1]
\Function{FindBestSparseSplit}{$X$, $\mathcal{L}_t$}
\For{$j=1$ to $p$}
    \parState{Extract strictly positive $X_{j,pos}$ and strictly negatives $X_{j,neg}$
              values from $X_j$ given $\mathcal{L}_t$.\label{alg-line:nnz-extraction}}
    \parState{Search for the best splitting rule of the form
              $s_j^*(x) = x_j \leq \tau$ with $\tau \in X_{j,pos}\cup \{0\}\cup X_{j,neg}$
              maximizing the impurity reduction
              $\Delta I$ over the sample set $\mathcal{L}_t$.
              \label{alg-line:sparse-split-optim}}
    \parState{Update $s^*$ if the splitting rule $s_j^*$ leads to higher impurity
              reduction.}
\EndFor
\State \Return $s^*$
\EndFunction
\end{algorithmic}
\end{algorithm}

To extract the non zero values of a sparse input matrix $X \in \mathbb{R}^{n
\times p}$ with sparsity $s$ in the context of the decision tree growth, we need
to perform efficiently two operations on matrices: (i) the column indexing for a
given input variable $j$ and (ii) the extraction of the non zero row values
associated to the set of samples $\mathcal{L}_t$ reaching the node $t$ in this
column $j$. The overall cost of extracting $|\mathcal{L}_t|$ samples at a column
$j$ from the input matrix $X$ should be proportional to the number of non zero
elements and not to $|\mathcal{L}_t|$.\footnote{We assume here that the matrix
$X$ is uniformly sparse.}

Among the different sparse matrix
representations~\cite{pissanetzky1984sparse,templates,hwu2009programming}, the sparse
csc matrix format is the most appropriate for tree growing as it allows
efficient column indexing, as required during node splitting. Let us show how to
perform an efficient extraction of the non zero values given the sample set
$\mathcal{L}_t$ using this matrix format. Note that using a compressed row
storage sparse format\footnote{The compressed row storage (csr) sparse array
format is made of three arrays $indptr$, $indices$ and $value$. The non zero
elements of the $i$-th row of the sparse csc matrix are stored from $indptr[i]$
to $indptr[i+1]$ in the $indices$ arrays, giving the column indices, and $value$
arrays, giving the stored values. It is the transposed version of the csc sparse
format.} would not be appropriate during the tree growth as we need to be able
to efficiently subsample input variables at each expansion of a new testing
node.

\begin{remark}{Compressed Sparse Column (CSC) matrix format}
The sparse csc matrix with $n_{nz}$ non zero elements is a data
structure composed of three arrays:
\begin{description}
\item[$indices \in \mathbb{Z}^{n_{nz}}$] containing the row indices of the non
zero elements.
\item[$data \in \mathbb{R}^{n_{nz}}$] containing the values of the non
zero elements.
\item[$indptr \in \mathbb{Z}^{p}$] containing the slice of the non zero
elements. For a column $j \in \{1,\ldots,p\}$, the row index and the values of
the non zero elements of columns $j$ are stored from $indptr[j]$ to
$indptr[j+1]$ in the $indices$ and $data$ arrays.
\end{description}

The non zero values associated to an input variable $j$ are located from
$indptr[j]$ to $indptr[j+1]-1$ in the $indices$ and the $data$ arrays (when
$indptr[j]=indptr[j+1]$, the column thus contains only zeros). Extracting them
requires to perform a set intersection between the sample set $\mathcal{L}_t$
reaching node $t$ and the non zero values $indices[indptr[j]:indptr[j+1][$ of
    the column $j$.

For instance, the following matrix $A$ has only 3 non zero elements, but we
would use an array of 20 elements:
\begin{equation}
A \in \mathbb{R}^{4 \times 5} = \begin{bmatrix}
a & 0 & 0 & b & 0 \\
0 & 0 & 0 & c & 0 \\
0 & 0 & 0 & 0 & 0 \\
0 & 0 & 0 & 0 & 0 \\
\end{bmatrix}.
\end{equation}
\noindent The csc representation of this matrix $A$ is given by
\begin{align*}
data &= \begin{bmatrix}a & b & c\end{bmatrix}, \\
indices &= \begin{bmatrix}0 & 0 & 1\end{bmatrix}, \\
inptr &= \begin{bmatrix}0 & 1 & 1 & 1 & 3 & 3\end{bmatrix}.
\end{align*}
\end{remark}

Let $n_{nz,j} = (indptr[j+1] -indptr[j])\,\forall j$ be the number of samples
with non zero values for input variable $j$ and let us assume that the
$indices$ of the input csc matrix array are sorted column-wise, i.e. for the
$j$-th row, the elements of $indices$ from $indptr[j]$ to $indptr[j+1]-1$ are
sorted. Standard intersection algorithms have the following time complexity:
\begin{enumerate}
\item in $O(|\mathcal{L}_t| \log{n_{nz,j}})$ by performing $|\mathcal{L}_t| $ binary
search on the sorted $n_{nz,j}$ nonzero elements; \label{en:bsearch}

\item in $O(|\mathcal{L}_t|\log{|\mathcal{L}_t|} + n_{nz,j} )$ by sorting the
sample set $\mathcal{L}_t$ and retrieving the intersection by iterating over
both arrays;

\item in $O(n_{nz,j})$ by maintaining a data structure such as a hash table of
$\mathcal{L}_t$ allowing to efficiently check if the elements of
$indices[indptr[j]$:$indptr[j+1][$ are contained in the sample partition
$\mathcal{L}_t$. \label{en:hash-aproach}
\end{enumerate}

The optimal intersection algorithm depends on the number of non zero elements
for input variable $j$ and the number of samples $|\mathcal{L}_t|$ reaching
node $t$. During the decision tree growth, we have two opposite situations:
either the size of the sample partition $|\mathcal{L}_t|$ is high with respect
to the number of non zero elements ($|\mathcal{L}_t| \approx O(n) \ggg
n_{nz,j}$) or, typically at the bottom of the tree, the partition size is small
with respect to the number of non zero elements ($n_{nz,j} \ggg
|\mathcal{L}_t|$). In the first case (i.e., at the top of the tree), the most
efficient approach is thus approach~(\ref{en:hash-aproach}), while in the
second case (i.e., at the bottom of the tree), approach~(\ref{en:bsearch})
should be faster. We first describe how to implement
approach~(\ref{en:hash-aproach}), then approach~(\ref{en:bsearch}), and
finally how to combine both approaches.

A straightforward implementation of approach~(\ref{en:hash-aproach}) is, at
each node, to allocate a hash table containing all training examples in that
node (in $O(|\mathcal{L}_t|)$) and then to compute the intersection by checking
if the non zero elements of the csc matrix belong to the hash table (in
$O(n_{nz,j})$). We can however avoid the overhead required for the allocation,
creation, and deallocation of the hash table by maintaining and exploiting a
mapping between the csc matrix and the sample set $\mathcal{L}_t$. Since the
array $L$ is constantly modified during the tree growth, we propose to use an array, denoted $mapping$, to keep track of the position of a sample $i$ in the array $L$ as illustrated in
Figure~\ref{fig:mapping}. During the tree growing, we keep the following
invariant:
\begin{equation}
mapping[L[\mathbf{i}]] = \mathbf{i}.
\end{equation}
\noindent In the above example, the array $L$ was $[0,\, 1,\ldots\,9]$ with
$mapping=[0,\, 1,\ldots\,9]$. After a few splits, the array $L$ has become
\[
L = [9, \,1, \,5,\, 3,\, 2,\, 7,\, 4,\, 8,\, 6,\, 0]
\]
\noindent and the associated $mapping$ array is
\[
mapping = [9,\, 1,\, 4,\, 3,\, 6,\, 2,\, 8,\, 5,\, 7,\, 0].
\]


\begin{figure}
\centering
\def\svgwidth{\textwidth}
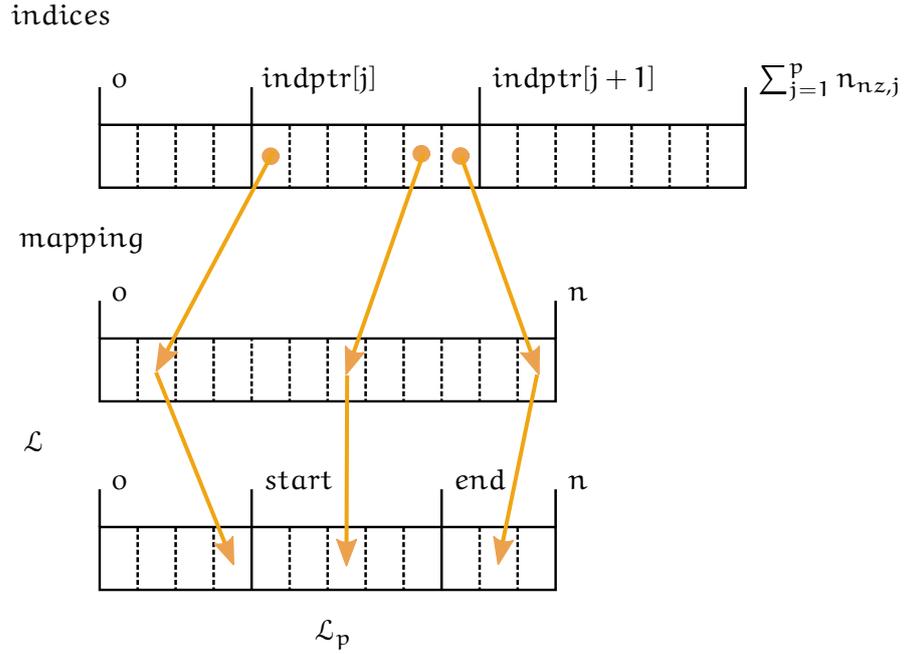
\caption{The array $mapping$ allows to efficiently compute the intersection
between the $indices$ array of the csc matrix and a sample set $\mathcal{L}_t$.}
\label{fig:mapping}
\end{figure}

Thanks to the $mapping$ array, we can now check in constant time $O(1)$ whether
a sample $i$ belongs to the sample set $\mathcal{L}_t$. Indeed, given that $\mathcal{L}_t$ is represented by a slice from an index $start$ to an
index $end - 1$ in $L$, sample $i$ belongs to $\mathcal{L}_t$ if the
following condition holds:
\begin{equation}
start \leq mapping[i] < end.
\end{equation}

To extract the non zero values of the csc matrix associated to the $j$-th input
variable in the sample set $\mathcal{L}_t$, we check the previous condition for
all samples from $indptr[j]$ to $indptr[j+1]-1$ in the $indices$ array. Thus, we
perform the intersection between $\mathcal{L}_t$ and the $n_{nz,j}$ non zero
values in $O(n_{nz,j})$. The whole method is described in
Algorithm~\ref{algo:extract-mapping}. Note that to swap samples in the array
$L$, we use a modified swap function (see Algorithm~\ref{algo:swap}) which
preserves the mapping invariant .

\begin{algorithm}
\caption{Return the $n_{neg}$ strictly negative and $n_{pos}$ positive values
$(X_{j, neg}, X_{j, pos})$ associated to the $j$-th variable from the sample set
$L[start:end[$ through a \emph{given $mapping$ satisfying $mapping[L[i]] = i$}. The
array $L$ is modified so that $L[start:start+n_{neg}[$ contains the samples with
negatives values, $L[start+n_{neg}:end - n_{pos}[$ contains the zero values and
$L[end - n_{pos}:end[$ the samples with positives values.}
\label{algo:extract-mapping}
\begin{algorithmic}[1]
\Function{extract\_nnz\_mapping}{$X$, $j$, $L$, $start$, $end$, $mapping$}
    \State $X_{j,pos} = []$
    \State $X_{j,neg} = []$
    \State $start_p$ = $end$
    \State $end_n$ = $start$
    \For{$k \in [X.indptr[j] \mbox{:} X.indptr[j+1][$}
        \State $index = indices[k]$
        \State $value = data[k]$
        \If{$start \leq mapping[index] < end$}
            \State $i$ = $mapping[index]$
            \If {$value > 0$}
                \State $X_{j,pos}$.\Call{append}{$value$}
                \State $start_p -= 1$
                \State \Call{Swap}{$\mathcal{L}$,  $i$, $start_p$, $mapping$}
            \Else
                \State $X_{j,neg}$.\Call{append}{$value$}
                \State \Call{Swap}{$\mathcal{L}$,  $i$, $end_n$, $mapping$}
                \State $end_n += 1$
            \EndIf
        \EndIf
    \EndFor
    \State \Return {$X_{j,pos}, X_{j,neg}, end_n - start, end - start_p$}
\EndFunction
\end{algorithmic}
\end{algorithm}

\begin{algorithm}
\caption{Swap two elements at positions $p_1$ and $p_2$ in the array $L$ in
place while maintaining the invariant of the $mapping$ array.}
\label{algo:swap}
\begin{algorithmic}[1]
\Function{Swap}{$L$, $p_1$, $p_2$, $mapping$}
\State $L[p_1]$, $L[p_2]$ = $L[p_2]$, $L[p_1]$
\State $mapping[L[p_1]] = p_1$
\State $mapping[L[p_2]] = p_2$
\EndFunction
\end{algorithmic}
\end{algorithm}

In practice, the number of non zero elements $n_{nz,j}$ of feature $j$ could be much
greater than the size of a sample set $\mathcal{L}_t$. This is likely to
happen near the leaf nodes. Whenever the tree is fully developed, there are only
a few samples reaching these nodes. The approach~(\ref{en:bsearch}) shown in
Algorithm~\ref{algo:extract-bsearch} exploits the relatively small size of the
sample set and performs repeated binary search on the $n_{nz,j}$ non zero
elements associated to the feature $j$.

\begin{algorithm}
\caption{Return the $n_{neg}$ strictly negative and $n_{pos}$ positive values
$(X_{j, neg}, X_{j, pos})$ associated to the $j$-th variable from the sample set
$L[start:end[$ through \emph{repeated binary search}. The array $L$ is modified so that
$L[start:start+n_{neg}[$ contains the samples with negatives values,
$L[start+n_{neg}:end - n_{pos}[$ contains the zero values and $L[end -
n_{pos}:end[$ the samples with positives values.}
\label{algo:extract-bsearch}
\begin{algorithmic}[1]
\Function{extract\_nnz\_bsearch}{$X$, $j$, $L$, $start$, $end$, $mapping$}
    \State $X_{j,pos} = []$
    \State $X_{j,neg} = []$
    \State $start_p$ = $end$
    \State $end_n$ = $start$
    \State $indices_j = X.indices[X.indptr[j]:X.indptr[j+1][$
    \State $data_j= X.data[X.indptr[j]:X.indptr[j+1][$
    \State $\mathcal{L}$ = \Call{sort}{$\mathcal{L}$, start, end}
    \For{$i \in [start:end[$}
        \parState{// Get the position of $L[i]$ in $indices_j$,
                  and -1 if it is not found:
                  $p$ = \Call{BinarySearch}{$L[i]$, $indices_j$}}
        \If {$p \neq -1$ }
            \If{$data_j[p]  > 0$}
                \State $start_p -= 1$
                \State $X_{j,pos}$.\Call{append}{$data_j[p]$}
                \State \Call{Swap}{$\mathcal{L}$,  $i$, $start_p$, $mapping$}
            \Else
                \State $X_{j,neg}$.\Call{append}{$data_j[p]$}
                \State \Call{Swap}{$\mathcal{L}$,  $i$, $end_n$, $mapping$}
                \State $end_n += 1$
            \EndIf
        \EndIf
    \EndFor
    \State \Return {$X_{j,pos}, X_{j,neg}, end_n - start, end - start_p$}
\EndFunction
\end{algorithmic}
\end{algorithm}

The optimal extraction of non zero values is a hybrid approach combining the
mapping-based algorithm (Algorithm~\ref{algo:extract-mapping}) and the binary
search algorithm (Algorithm~\ref{algo:extract-bsearch}). Empirical experiments
have shown that it is advantageous to use the mapping-based algorithm
whenever
\begin{equation}
\label{switch}
|\mathcal{L}_t| \times \log(n_{nz,j}) < 0.1 \times n_{nz,j}.
\end{equation}
\noindent and the binary search otherwise (see
Algorithm~\ref{algo:extract-nnz}). The formula is based on the computational
complexity of both algorithms. We have determined the constant of $0.1$
empirically.

\begin{algorithm}
\caption{Return the $n_{neg}$ strictly negative and $n_{pos}$ positive values
$(X_{j, neg}, X_{j, pos})$ associated to the $j$-th variable from the sample set
$L[start:end[$. The array $L$ is modified so that
$L[start:start+n_{neg}[$ contains the samples with negatives values,
$L[start+n_{neg}:end - n_{pos}[$ contains the zero values and $L[end -
n_{pos}:end[$ the samples with positives values.}
\label{algo:extract-nnz}
\begin{algorithmic}
  \Function{extract\_nnz}{$X$, $j$, $L$, $start$, $end$,  $mapping$}\\
    Let $n_{nz,j}$ be the number of non zero values in column $j$ of $X$.
    \If {$(end-start) \times \log(n_{nz,j}) < 0.1 \times n_{nz,j}$}
        \State \Return \Call{extract\_nnz\_mapping}{$X$, $j$, $L$, $start$, $end$, $mapping$}
    \Else
        \State \Return \Call{extract\_nnz\_bsearch}{$X$, $j$, $L$, $start$, $end$, $mapping$}
    \EndIf
\EndFunction
\end{algorithmic}
\end{algorithm}

Note that after extracting the non zero values, we need to sort the thresholds
of the splitting rules to search efficiently for the best one. Thanks to
Algorithm~\ref{algo:extract-nnz}, Algorithm~\ref{algo:extract-mapping} and
Algorithm~\ref{algo:extract-bsearch}, we have already made a three way partition
pivot as in the quicksort on the value 0. This speeds up the overall splitting
algorithm (the line~\ref{alg-line:sparse-split-optim} of
Algorithm~\ref{algo:sparse-split}). Instead of sorting the thresholds in the
sample set $\mathcal{L}_t$ in $O(|\mathcal{L}_t \log(\mathcal{L}_t)|)$, we can
perform the sort in $O(d\mathcal{L}_t \log(d\mathcal{L}_t))$ given an input
space density $d$.

As a further refinement, let us note that we can sometimes significantly speed
up the decision tree growth by avoiding to search for splitting rules on
constant input variables. To do so, we can cache the input variables that were
found constant during the expansion of the parents of the node of $t$. If an
input variable is found constant, caching this information avoids the overhead
of searching for a splitting rule when no valid one exists.

\subsection{Partitioning sparse data}
\label{subsec:tree-data-partitionning}

During the tree growth, we need to partition a sample set $\mathcal{L}_t = \{(x,
y) \in \mathcal{X} \times \mathcal{Y}\}_{i=1}^n)$ at a testing node $t$
according to a splitting rule $s_t(x)$. The splitting rule $s_t$ associated to
an ordered input variable is of the form $s_t(x) = 1(x_{F_t} \leq \tau_t)$,
where $\tau_t$ is a threshold constant on the $F_t$-th input variable.

During the tree growth, we have the constraint that the input data matrix is in
the csc sparse format. We can not convert the current sparse format to another
one as it would require to store both the new and old representations into
memory. An efficient way to split the sample set $\mathcal{L}_t$ into its left
$\mathcal{L}_{t,l}$ and right $\mathcal{L}_{t,r}$ subset is to use the
Algorithm~\ref{algo:extract-nnz}. It will extract the non zero values of a
given input variable, but also partition the array $L$ representing the sample
set $\mathcal{L}_t$ into three parts: (i) $L[start:start+n_{neg}[$ contains the
$n_{neg}$ samples with negatives values, (ii) $L[start+n_{neg}:end - n_{pos}[$
contains the elements with zero values and (iii) $L[end - n_{pos}:end[$ the
$n_{pos}$ samples with positives values. Once the non zero values have been
extracted, we have to partition the samples either with negative values
($L[start:start+n_{neg}[$) or with positive values ($L[end - n_{pos}:end[$)
according to the sign of the threshold $\tau_t$.

The complexity to partition once the data is
$O(\min\left(n_{nz,j},|\mathcal{L}_t| \times \log(n_{nz,j})\right))$ for a batch
of $\mathcal{L}_t$ samples instead of the usual $O(|\mathcal{L}_t|)$ with dense
input data.

\section{Tree prediction} \label{sec:tree-sparse-prediction}

The prediction of an unseen sample $x$ by a decision tree (see
Algorithm~\ref{algo:tree-predict}) is done by traversing the tree from the top
to the bottom. At each test node, a splitting rule of the form tests whether or
not the sample $x$ should go in the left or the right branch. The splitting rule
$s_t$ associated to an ordered input variable is of the form $s_t(x) = 1(x_{F_t}
\leq \tau_t)$, where $\tau_t$ is a threshold constant on the $F_t$-th input
variable.

We need to have an efficient row and column indexing of the input data matrix.
We discuss here two options: (i) using the dictionary of key (dok) sparse matrix
format and (ii) using a csr sparse matrix format\footnote{The compressed row
storage (csr) sparse array
format~\cite{pissanetzky1984sparse,templates,hwu2009programming} is made of three
arrays $indptr$, $indices$ and $value$. The non zero elements of the $i$-th row
of the sparse csc matrix are stored from $indptr[i]$ to $indptr[i+1]$ in the
$indices$ arrays, giving the column indices, and $value$ arrays, giving the
stored values. It is the transposed version of the csc sparse format.}.

The dictionary of key (dok) sparse matrix format store the non zero values in a
hash table whose keys are the pairs formed from the row and the column index. It
is straightforward to apply Algorithm~\ref{algo:tree-predict} with the dok
format. The computational complexity is thus unchanged.

To predict one or several unseen samples using a csr array, we need a procedure
to efficiently access to both the row and the column index without densifying the csr
matrix. We allocate two arrays $nz\_mask\in\mathcal{Z}^p$ with all elements
having the value ``$-1$'' and $nz\_value \in \mathbb{R}^p$ of size $p$. To
predict the $i$-th sample from a test set, we set in the array $nz\_mask$ the
value $i$ to the non zero values $indices[indptr[i]:indptr[i+1][$ associated to
this sample. The array $nz\_value$ is modified to contain the non zero values of
the $i$-th sample, i.e. $values[indptr[i]:indptr[i+1][$. During the tree
traversal (see Algorithm~\ref{algo:tree-predict}), we get the $j$-th input value
by first checking if it is zero with $nz\_mask[j]\not=i$, otherwise the value is
stored at $nz\_value[j]$. Assuming a proportion of zero elements $s$ for a batch
of $n$ test samples with $p$ input variables, the extra cost of using this
approach is $O(p + (1-s)np)$. Note that using the $nz\_mask$ and the $nz\_value$
arrays is more efficient than densifying each sample as it would add an extra
cost of $O(np)$.

The csr sparse format leads to a worse computational complexity than the dok
format, which has no extra computing cost. However, the csr approach was chosen
in the scikit-learn machine learning python library. The standard implementation
of the dok format in scipy is indeed currently implemented using the python dict
object. While with the csr format, we can work only with low level c arrays and
we can also easily release the global interpreter lock (GIL). Note that here the
csr format is also better suited than the csc format as the complexity is
independent of the number of samples to predict at once.

\section{Experiments} \label{sec:sparse-experiments}

In this section, we compare the training time and prediction time of decision
tree growing and prediction algorithms using either dense data representation
or sparse data representation. More specifically, we compare three input matrix
layouts using scikit-learn version 0.17: (i) the dense c array layout, a row
major order layout whose consecutive and contiguous elements are row values,
(ii) the dense fortran array layout, a column major order layout whose
consecutive and contiguous elements are column values, and (iii) the sparse
array layout, the csc sparse format during tree growing and the csr sparse
format during tree prediction (as proposed respectively in
Sections~\ref{subsec:sparse-splitting-rule-search} and
\ref{subsec:tree-data-partitionning}). The comparison will be made with stumps,
a decision tree with a single test node, and fully grown decision trees. These
results will be indicative of what could be gained in the context of boosting
methods using decision tree of low complexity such as stumps and random forest
methods using deep decision trees. Note that all splitting algorithms lead
exactly to the same decision tree structure and have the same generalization
performance.

We assess the effect of the input space density on synthetic datasets in
Section~\ref{sec:sparse-synthetic}. Then, we compare training times using each
of the three input matrix layouts on real datasets in
Section~\ref{sec:input-density-real}.

\subsection{Effect of the input space density on synthetic datasets}
\label{sec:sparse-synthetic}

As a synthetic experiment, we compare the decision tree growing algorithm and
tree prediction algorithm on synthetic regression tasks with $n=10^5$ samples and
$p=10^3$ features.  The input matrices are sparse random matrices whose non zero
elements are drawn uniformly in $\mathcal{N}(0; 1)$. The output vector is drawn
uniformly at random in $[0, 1]$. The input space density ranges from
$10^{-3}$ to $1$. Each point is an average over 20 experiments.

Figure~\ref{subfig:stump-fit} shows in logarithmic scale the computing times to
grow a single stump. We first note that column-based layouts (the fortran and
csc format) are more appropriate to grow decision tree on sparse data. While the
density is ranging from $10^{-3}$ to $10^{-1}$, the most expensive part of the
tree growing for a single stump is to retrieve the input values from a sample
set and to sort them. For a set of $n$ samples and $p$ features with a sparsity
$s$, the computational cost to grow a stump on dense data is $O(p n \log{n})$.
By contrast, growing a stump using a csc sparse input matrix has a computational
complexity of
\begin{align}
&O(p(1-s)n\log{((1-s)n)} + \min\{p n \log{(n(1-s)), p n(1-s)}\}) \nonumber \\
&= O(p(1-s)n\log{((1-s)n)})\}.
\label{eq:complexity-sparse-fit-stump}
\end{align}
\noindent The first term of Equation~\ref{eq:complexity-sparse-fit-stump}
corresponds to the sorting of non zero elements. The second term highlights the
contribution of the retrieval of the non zero values, which is always less
costly than the sorting operation. If the density is $1$ ($s=0$), both the dense
and sparse formats have the same complexity as shown in the right point of
Figure~\ref{subfig:stump-fit}. Overall with a sparse dataset, the csc format is
significantly faster as it leverages the sparsity. The bad performance of the
dense c layout compared to the fortran layout or csc layout can be explained by
the higher number of cache misses.

The time required to predict the training set using the fitted stump is shown in
Figure~\ref{subfig:stump-predict}. The only difference between the three input
matrix layouts (c, fortran and csr) is the access to the non zero elements. The
differences between the dense and the sparse input matrix format can be
explained by a better exploitation of the cache for the sparse format,
especially when the density is below $0.02$. When the density if over $0.02$,
the cost of copying the non zero values to the arrays $nz\_mask$ and
$nz\_values$ becomes dominant.

\begin{figure}
\centering
\subfloat[Learning time]{{\includegraphics[width=0.5\textwidth]{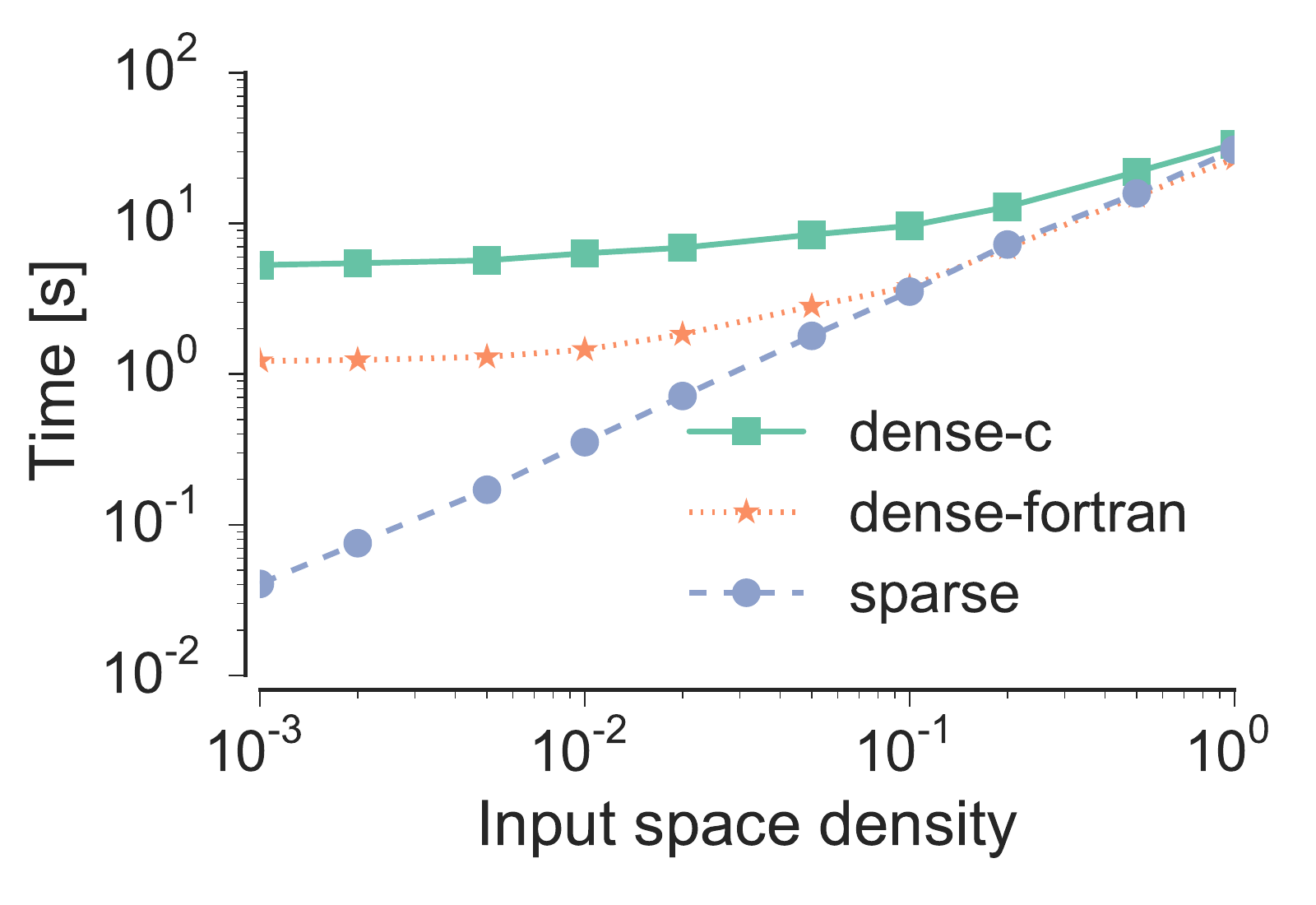}
\label{subfig:stump-fit}}}
\subfloat[Prediction time]{{\includegraphics[width=0.5\textwidth]{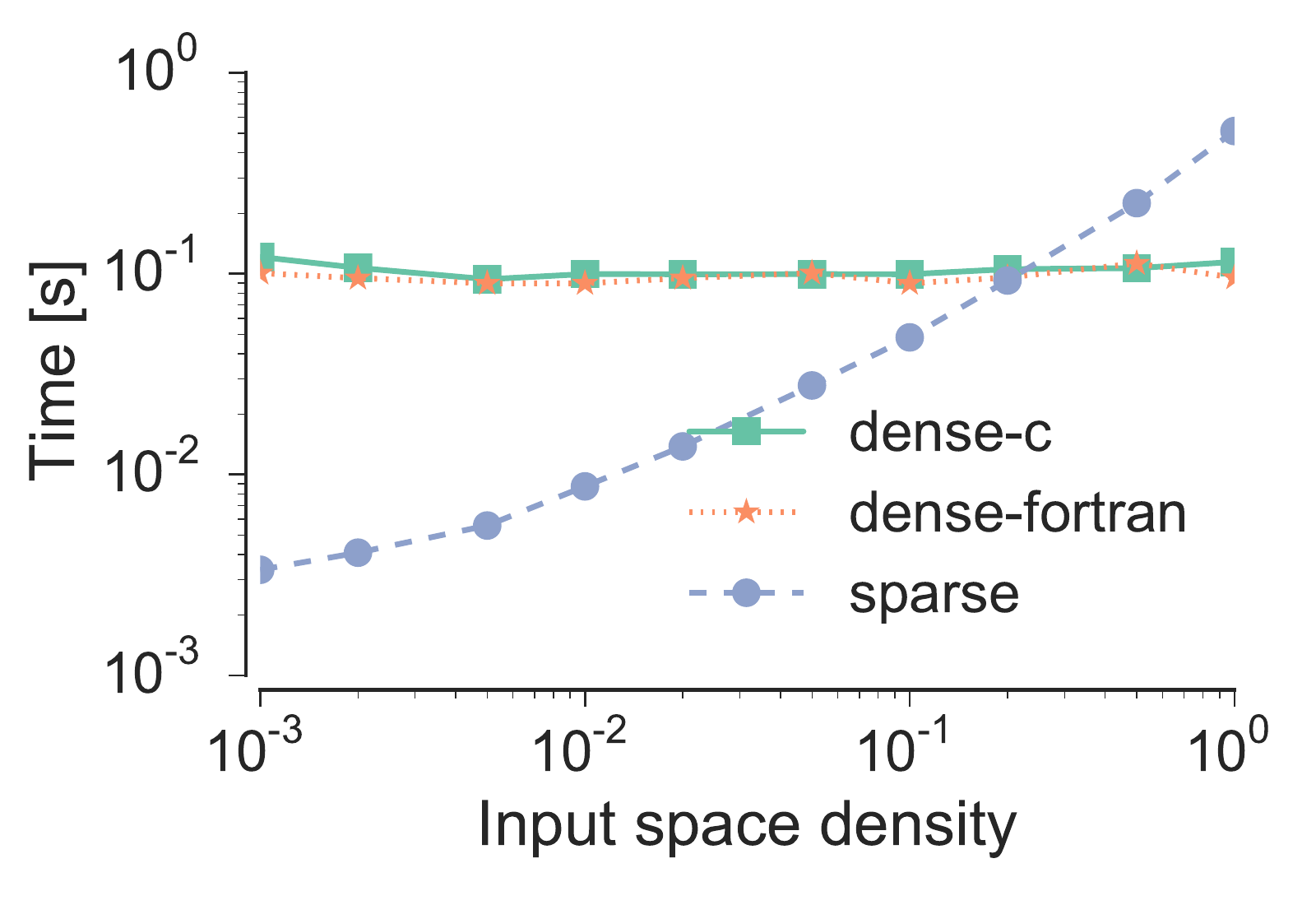}
\label{subfig:stump-predict}}}
\caption{Learning and prediction time of stumps as a function of the input space
density.}
\label{fig:sparse-input-stump}
\end{figure}

Figure~\ref{fig:cart-sparse-input} shows the time required to learn a fully
grown decision tree as a function of the dataset density. Note that the maximal
depth decreases as a function of the dataset density (see
Figure~\ref{subfig:dt-max-depth}) as the decision tree becomes more balanced. As
with the stump, the fortran layout is more appropriate than the c layout to grow
a fully grown decision tree on sparse data. The sparse csc algorithm is
significantly faster than the dense splitting algorithm if the density is
sufficiently low (here below $0.02$). The sparse splitting algorithm becomes
slower as the density increases (here beyond $0.02$) since the extraction of the
non zero values becomes more costly than finding the right split.

With fully developed decision trees, the prediction time (see
Figure~\ref{subfig:dt-predict}) is similar between its dense and sparse version.
The prediction time is lower when the input space density is high as the trees
are more balanced. We note that the prediction time is correlated with the
maximal depth of the tree.

Whenever the complexity of the decision tree lies between a stump and a fully
developed tree, the behavior moves from one extreme to the other.

\begin{figure}
\centering
\subfloat[Learning time]{{\includegraphics[width=0.7\textwidth]{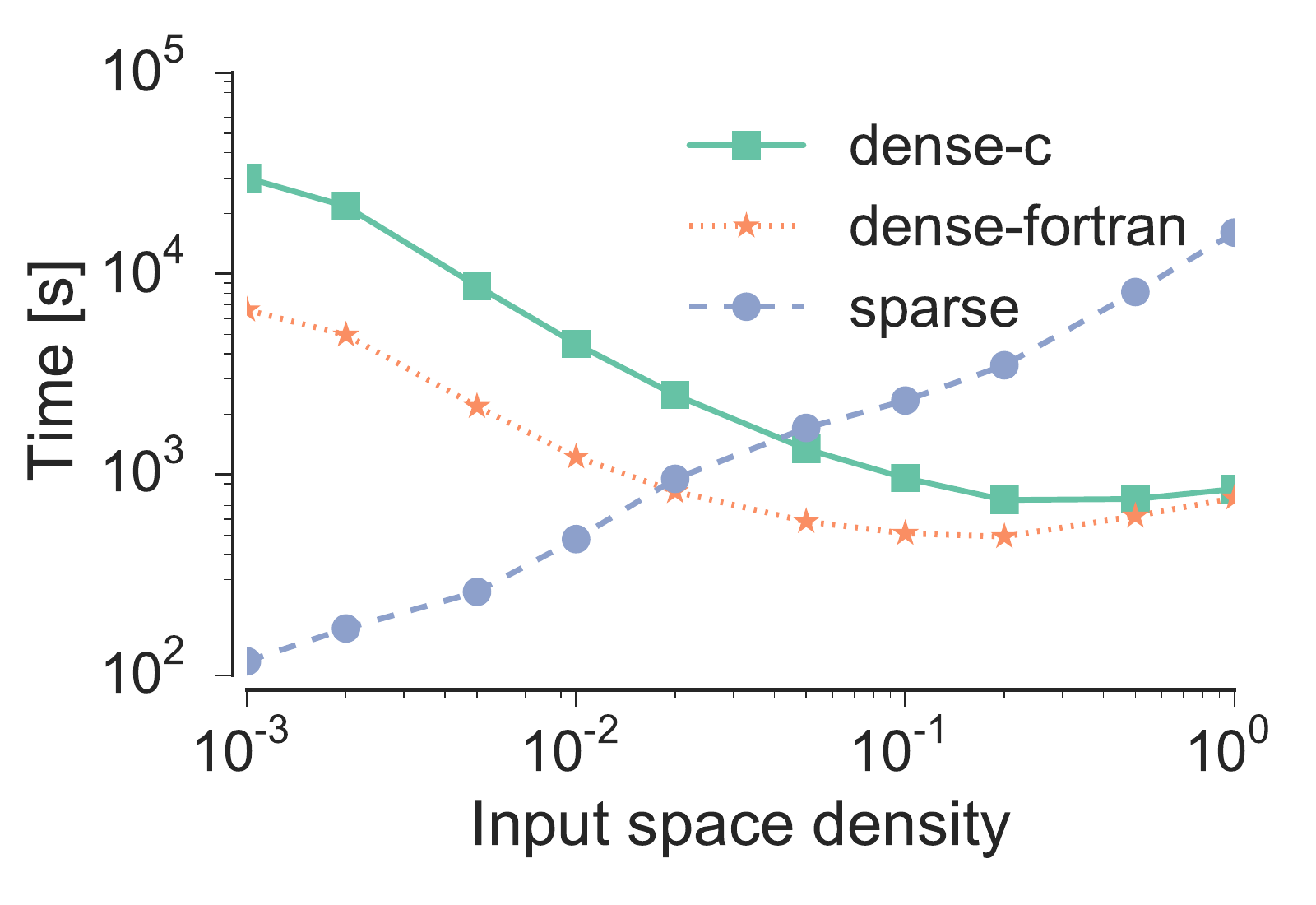}
\label{subfig:dt-fit}}} \\
\subfloat[Prediction time]{{\includegraphics[width=0.7\textwidth]{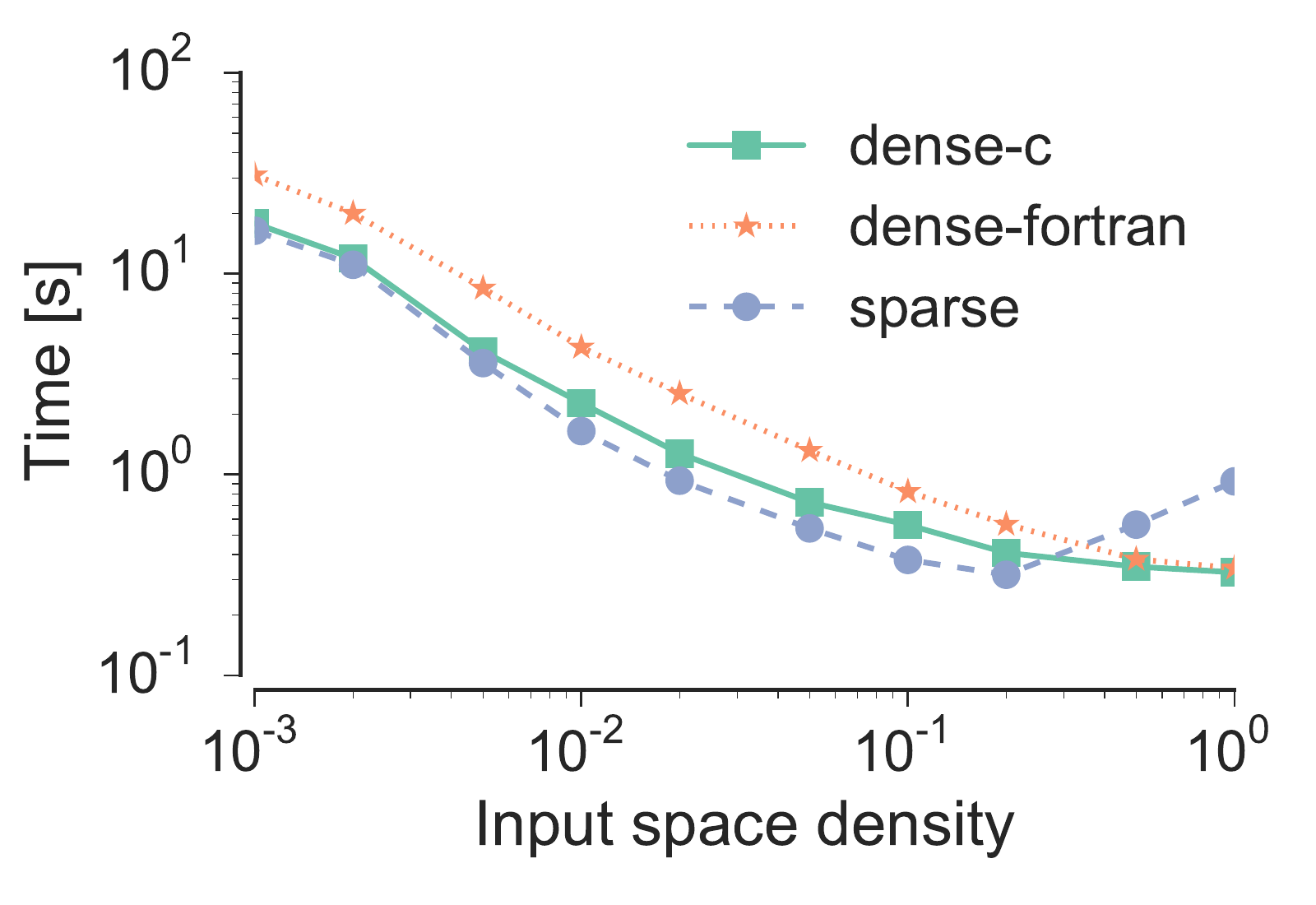}
\label{subfig:dt-predict}}} \\
\subfloat[Maximal depth]{{\includegraphics[width=0.7\textwidth]{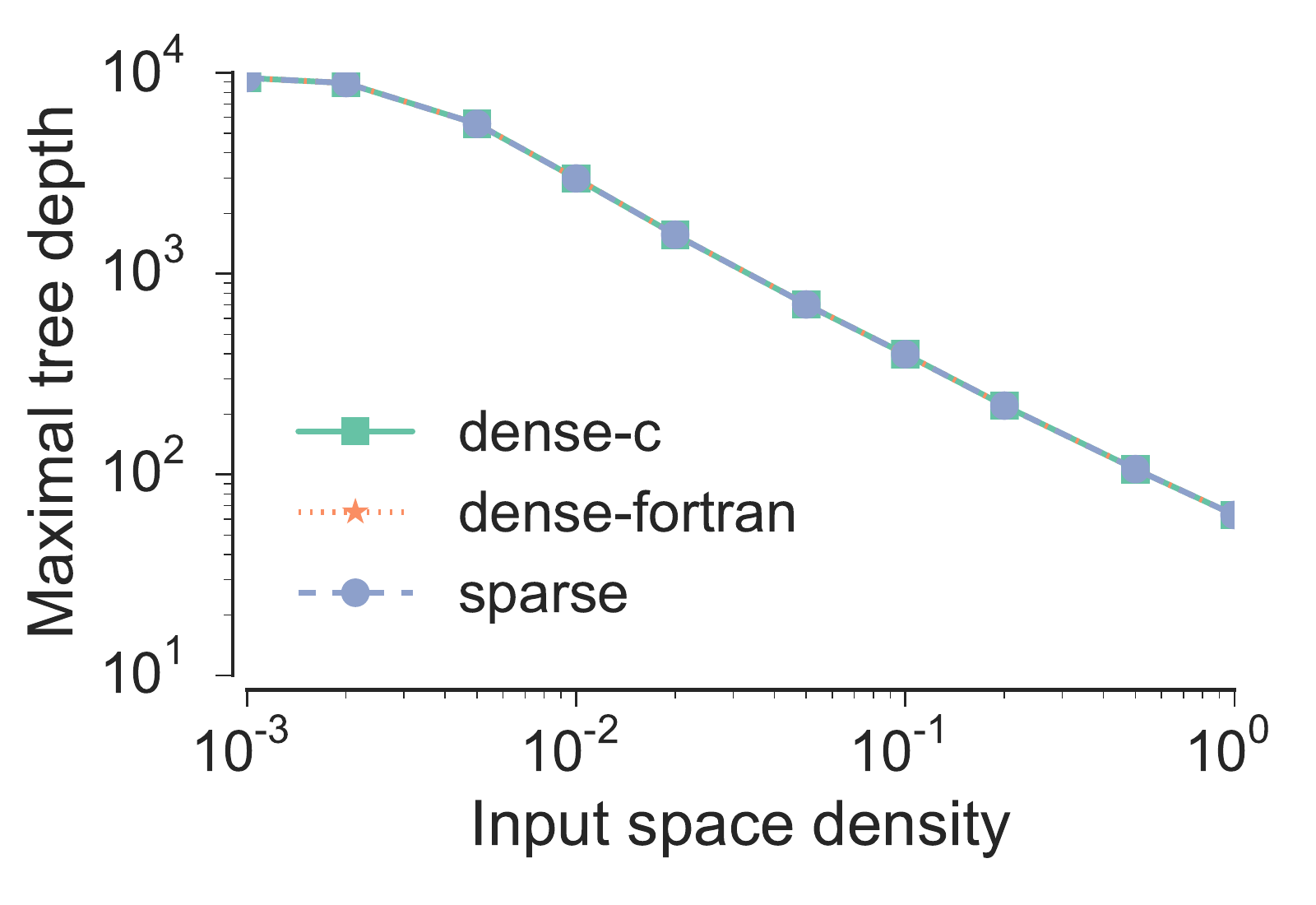}
\label{subfig:dt-max-depth}}}
\caption{Learning time, prediction time and maximal depth as a function of the input
space as a function of the input space density for fully grown decision trees.}
\label{fig:cart-sparse-input}
\end{figure}

\subsection{Effect of the input space density on real datasets}
\label{sec:input-density-real}

To further study the impact of the input space density, we have selected 9
datasets whose input space density ranges from $0.00034$ to $0.22$. These
datasets are presented in Table~\ref{table:dataset-properties} ordered by input
space density.

\begin{table}
\caption{Dataset properties ordered by input space density.}
\label{table:dataset-properties}
\centering
\setlength{\tabcolsep}{2pt}
\begin{footnotesize}
\begin{tabular}{lrrr}
\toprule
       Dataset &       n &        p &   Density \\
\midrule
 news20.binary~\cite{keerthi2005modified}                       &   19996 &  1355191 &  0.0003 \\
   20newsgroup~\cite{Lang95,rennie2001improving}                &   11314 &   130107 &  0.0012 \\
          rcv1~\cite{bekkerman2008data}                         &   23149 &    47236 &  0.0016 \\
  sector-scale~\cite{rennie2001improving,keerthi2008sequential} &    9619 &    55197 &  0.0029 \\
 farm-ads-vect~\cite{mesterharm2011active}         &    4143 &    54877 &  0.0036 \\
   E2006-train~\cite{kogan2009predicting}                       &   16087 &   150360 &  0.0083 \\
     mushrooms~\cite{Lichman:2013}                              &    8124 &      112 &  0.1875 \\
         mnist~\cite{lecun1998gradient}                         &   70000 &      784 &  0.1914 \\
       covtype~\cite{Lichman:2013}                              &  581012 &       54 &  0.2110 \\
\bottomrule
\end{tabular}
\end{footnotesize}
\end{table}

Table~\ref{table:real-dataset-Stump-chrono_fit} shows the time to train a single
stump. The fastest algorithm here is the tree growing algorithm with the input
sparse csc matrix. The speed up factor between the sparse and fortran memory
layout ranges from 1.1 to 23 times. Note that the fortran layout is always
faster than the c layout. The column major order layout is here better suited
for sparse dataset. The difference could be explained by fewer cache misses with
the fortran memory layout than with the c memory layout.

\begin{table}
\caption{The time (in second) required to train a stump using a sparse csc
layout is always faster than with the fortran or c dense memory layout on all
sparse selected datasets.}
\label{table:real-dataset-Stump-chrono_fit}
\centering
\begin{tabular}{lrrrr}
\toprule
       Dataset &       c & fortran &  sparse & fortran / sparse \\
\midrule
 news20.binary &     N/A &     N/A &   3.31 &              N/A \\
   20newsgroup &  213.18 &   15.77 &   0.79 &             20.1 \\
          rcv1 &  197.61 &   28.90 &  15.97 &              1.8 \\
  sector-scale &   58.56 &    7.10 &   1.08 &              6.6 \\
 farm-ads-vect &   26.03 &    3.22 &   0.14 &             23.0 \\
   E2006-train &  413.35 &   35.44 &   8.62 &              4.1 \\
     mushrooms &    0.03 &    0.02 &   0.02 &              1.2 \\
         mnist &    5.23 &    2.65 &   2.00 &              1.3 \\
       covtype &    3.89 &    2.57 &   2.35 &              1.1 \\
\bottomrule
\end{tabular}
\end{table}

Table~\ref{table:real-dataset-CART-chrono_fit} shows the time required to grow a
fully developed decision tree with c, fortran or sparse csc memory layout. The
dense fortran layout is here always faster than the dense c layout. The sparse
memory layout is faster by a factor between 1.3 and 7.5 than the fortran layout
when the input space density is below $0.3\%$.

\begin{table}
\caption{The time required to train a fully grown decision tree using a c, a fotran
or a sparse csc memory layout.}
\label{table:real-dataset-CART-chrono_fit}
\centering
\begin{tabular}{lrrrr}
\toprule
       Dataset &         c &  fortran &    sparse & fortran / sparse \\
\midrule
 news20.binary &      N/A &      N/A &   428.89 &              N/A \\
   20newsgroup &  4281.02 &   518.34 &    69.06 &              7.5 \\
          rcv1 &  1458.45 &   562.19 &   442.15 &              1.3 \\
  sector-scale &  5337.65 &   878.67 &   206.69 &              4.3 \\
 farm-ads-vect &   227.24 &    45.71 &     7.95 &              5.7 \\
   E2006-train &  2467.39 &   752.15 &  1083.62 &              0.7 \\
     mushrooms &     0.07 &     0.05 &     0.04 &              1.3 \\
         mnist &    80.09 &    56.11 &   178.59 &              0.3 \\
       covtype &    53.51 &    33.82 &   240.26 &              0.1 \\
\bottomrule
\end{tabular}
\end{table}

Note that we were unable to grow a decision tree with a dense memory layout on
the news20.binary dataset as it would require 108.4 Gigabyte to only store the
input matrix instead of 78 Megabytes.

\subsection{Algorithm comparison on 20 newsgroup}
\label{sec:sparse-20newsgroup}

Decision trees are rarely used in the context of sparse input datasets. One
reason is the lack of implementations exploiting sparsity during the decision
tree growth. With the previous experiments, we have shown that it increases
significantly the computing time, but also the amount of memory needed. With
the proposed tree growing and prediction algorithms, it is interesting to
compare the training time, prediction time and accuracy of some tree based
models, such as random forest or adaboost, with methods more commonly used in
the presence of sparse data.

We compare tree based methods to methods more commonly used on sparse datasets
on the 20 newsgroup dataset, which have $p=130107$ input variables, 11314 training
sample ands 7532 testing samples.

The compared algorithms are optimized on their respective hyper-parameters (see
Table~\ref{table:hyper-param-grid} for the details) using 5-fold cross
validation strategy on the training samples.

\begin{table}
\caption{Hyper-parameters grids.}
\label{table:hyper-param-grid}
\centering
\setlength{\tabcolsep}{2pt}
\begin{small}
\begin{tabular}{ll}
\toprule
\em{$k$-nearest neighbors} \\
\midrule
$k$, number of neighbors & $\{1,\ldots,10\}$ \\
\toprule
\em{Decision tree} \\
\midrule
$n_{\min}$, min. number of samples to split a node & $\{2, 5, 10, 15\}$ \\
\toprule
\em{Extra trees, random forest} \\
\midrule
$n_{\min}$, min. number of samples to split  a node & $\{2, 5, 10, 15\}$ \\
$k$, number of features drawn at each nodes &  $\{1, \log{p}, \sqrt{p}, 0.001p\}$\\
$M$, ensemble size & 100 or 1000 \\
\toprule
\em{Adaboost with decision trees as weak estimators} \\
\midrule
$n_{\min}$, min. number of samples to split  a node & $\{2, 5, 10, 15\}$ \\
$k$, number of features drawn at each nodes &  $\{1, \log{p}, \sqrt{p}, 0.001p\}$\\
$M$, ensemble size & 100 \\
$\mu$, learning rate &  $\{1, 0.1, 0.01\}$ \\
\toprule
\em{Bernoulli and multinomial naive Bayes} \\
\midrule
$\lambda$, additive smoothing parameter & $\{0, 0.001, 0.01, 0.1, 1\}$\\
\toprule
\em{Ridge classifier}\\
\midrule
$\lambda$, penalty parameter & $\{0.0001, 0.001, 0.01, 0.1, 1, 10\}$\\
\toprule
\em{Linear support vector machine classifier (SVC)}\\
\midrule
$\lambda$, penalty parameter & $\{10^i\}_{i=-5}^4$\\
\toprule
\em{Stochastic gradient descent (SGD)}\\
\midrule
Loss &  $\{\text{hinge}, \text{logistic}\}$ \\
Penalty constraint & $\{\ell_1, \ell_2, \text{elastic net}\}$\\
$\lambda$, penalty parameter & $ \{0.0001, 0.001, 0.01, 0.1, 1, 10\}$\\
Number of iterations & 100 \\
\bottomrule
\end{tabular}
\end{small}
\end{table}

The results obtained on the 20 newsgroup dataset are shown in
Table~\ref{table:20news-sparse} using scikit-learn version 0.17.1 and input
sparsity-aware implementations. The algorithm with the highest accuracy is the
linear estimator trained with ridge regression. It is closely followed by the
random forest ($m=1000$) model, the multinomial naives Bayes and extra-trees
($m=1000$). More generally, tree-based ensemble methods (random forest, extra
trees, and adaboost) show similar performance as linear methods (ridge, naive
bayes, linear SVC and SGD), with all methods from these two families reaching
at least $0.75$ of accuracy. On the other hand, the $k$-nearest neighbors and the
single decision tree perform very poorly (with an accuracy
below $0.6$).

We also note that increasing the number of trees from 100 to 1000 significantly
improves the performance of both random forests and extra trees. Their accuracy
increases respectively by $0.0573$ and $0.0409$ in absolute value. Building
tree ensembles also very significantly improves the accuracy with respect to
single trees (by at least $0.20$). This further suggests that the variance of
single trees is very high on this problem.

From a modeling perspective, growing decision tree ensemble on datasets with
sparse inputs is possible. From a training time perspective, the time needed to
grow and to optimize an ensemble of 100 trees is comparable to the time needed
to train linear models, e.g., with SGD or ridge regression. Note that
naive Bayes models are particularly fast to train compared to the other
estimators. However note that the comparison is dependent upon the chosen
hyper-parameters, the implementation and the grid size. From the point of view
of prediction time, tree ensemble methods are particularly slow compared to the
other estimators.

\begin{table}
\caption{Accuracy, training time (in second) and prediction time (in second) of
algorithms on 20 newsgroup sorted by accuracy score.}
\label{table:20news-sparse}
\centering
\setlength{\tabcolsep}{2pt}
\begin{tabular}{lrrr}
\toprule
                Estimator &  Training time &  Prediction time &  Accuracy \\
\midrule
    $K$-nearest neighbors &            295 &            22.35 &    0.457 \\
            Decision tree &            859 &             0.02 &    0.557 \\
                 Adaboost &          37867 &             8.84 &    0.752 \\
   Bernoulli naives Bayes &             25 &             0.52 &    0.765 \\
  Random forest ($m=100$) &          11766 &            12.71 &    0.778 \\
    Extra trees ($m=100$) &          22186 &            28.07 &    0.794 \\
                      SGD &          42776 &             0.25 &    0.814 \\
               Linear SVC &           2481 &             0.20 &    0.822 \\
   Extra trees ($m=1000$) &         213009 &           253.39 &    0.833 \\
 Multinomial naives Bayes &             14 &             0.11 &    0.833 \\
 Random forest ($m=1000$) &         114472 &           125.92 &    0.835 \\
                    Ridge &           5737 &             0.13 &    0.844 \\
\bottomrule
\end{tabular}
\end{table}

\section{Conclusion}

We propose an algorithm to grow decision tree models whenever the input space
is sparse. Our approach takes advantage of input sparsity to speed up the
search and selection of the best splitting rule during the tree growing. It
first speeds up the expansion of a tree node by extracting efficiently the non
zero threshold of the splitting rules used to partition data. Secondly, the
selection of best splitting rule is also faster as we avoid to sort data with
zero values. We reduce the memory needed as we do not need to densify the input
space. We also show how to predict samples with sparse inputs.


\chapter{Conclusions}
\label{ch:conclusions}


\section{Conclusions}

As we now gather or generate data at every moment, machine learning techniques
are emerging as ubiquitous tools in sciences, engineering, or society. Within
machine learning, this thesis focuses on supervised learning, which aims at
modelling input-output relationships only from observations of input-output
pairs, using tree-based ensemble methods, a popular method family exploiting
tree structured input-output models.  Modern applications of supervised
learning raise new computational, memory, and modeling challenges to existing
supervised learning algorithms. In this work, we identified and addressed the
following questions in the context of tree-based supervised learning methods:
(i) how to efficiently learn in high dimensional, and possibly sparse, output
spaces? (ii) how to reduce the memory requirement of tree-based models at
prediction time? (iii) how to efficiently learn in high dimensional and sparse
input spaces? We summarize below our main contributions and conclusions around
these three questions.

\paragraph{Learning in high dimensional and possibly sparse output spaces.}

Decision trees are grown by recursively partitioning the input space while
selecting at each node the split maximizing the reduction of some impurity
measure. Impurity measures have been extended to address with such models
multi-dimensional output spaces, so as to solve multi-label classification or
multi-target regression problems. However, when the output space is of high
dimension, the computation of the impurity becomes a computational bottleneck of
the tree growing algorithm. To speed up this algorithm, we propose to
approximate impurity computations during tree growing through random projections
of the output space. More precisely, before growing a tree, a few random
projections of the output space are computed and the tree is grown to fit these
projections instead of the original outputs. Tree leaves are then relabelled in
the original output space to provide proper predictions at test time. We show
theoretically that when the number of projections is large enough, impurity
scores, and thereby the learned tree structures and their predictions, are not
affected by this trick. We then exploit the randomization introduced by the
projections in the context of random forests, by building each tree of the
forest from a different randomly projected subspace. Through experiments on
several multi-label classification problems, we show that randomly projecting
the outputs can significantly reduce computing times at training without
affecting predictive performance. On some problems, the randomization induced by
the projections even allows to reach a better bias-variance tradeoff within
random forests, which leads to improved overall performance. In contrast with
existing works on random projections of the output, our proposed leaf
relabelling strategy also allows to avoid any decoding step and thus preserves
computational efficiency at prediction time with respect to standard unprojected
multi-output random forests.

Multi-output random forests build a single tree ensemble to predict all outputs
simultaneously. While often effective, this approach is justified only when the
individual outputs are strongly dependent (conditionally to the inputs). On the
other hand, building a separate ensemble for each output, as done in the binary
relevance / single target approach, is justified only when the outputs are
(conditionally) independent. In our second contribution, we build on gradient
boosting and random projections to propose a new approach that tries to bridge
the gap between these two extreme assumptions and can hopefully adapt
automatically to any intermediate output dependency structure. The idea of this
approach is to grow each tree of a gradient boosting ensemble to fit a random
projection of the original (residual) output space and then to weight this model
in the prediction of each output according to its ``correlation'' with this
output. Through extensive experiments on several artificial and real problems,
we show that the resulting method has a faster convergence than binary relevance
and that it can adapt better than both binary relevance and multi-output
gradient boosting to any output dependency structure. The resulting method is
also competitive with multi-output random forests. Although we only carried out
experiments with tree-based weak models, the resulting gradient boosting methods
are generic and can thus be used with any base regressor.

\paragraph{Reducing memory requirements of tree-based models at prediction time.}

One drawback of random forest methods is that they need to grow very large
ensembles of unpruned trees to achieve optimal performance. The resulting
models are thus potentially very complex, especially with large datasets, as
the complexity of unpruned trees typically depends linearly on the dataset
size. On the other hand, only very few nodes are required to make a prediction
for a given test example. Our investigation of the question of ensemble
compression started with the observation that the random forest model can be
viewed as  linear models in the node indicator space. Each of these binary
variables defining this space indicates whether or not a sample reaches a given
node of the forest. In the original linear representation of a forest in the
indicator space, non zero ``coefficients'' are given only to the leaf nodes. We
propose to post-prune the random forest model by selecting and re-weighting the
nodes of the linear model through the exploitation of sparse linear
estimators. More precisely, from the tree ensemble, we first extract node
indicator variables. Then, we project a sample set on this new representation
and select a subset of these variables through a Lasso model. The non zero
coefficients of the obtained Lasso model are later used to prune and to re-weight the
decision tree structure. The resulting post-pruning approach is shown
experimentally to reduce very significantly the size of random forests, while
preserving, and even sometimes improving, their predictive performance.

\paragraph{Learning in high dimensional and sparse input spaces.}

Some supervised learning tasks (e.g., text classification) need to deal with
high dimensional and sparse input spaces, where input variables have each only
a few non zero values. Dealing with input sparsity in the context of decision
tree ensembles is challenging computationally for two reasons: (i) it is more
difficult algorithmic-wise to exploit sparsity during the tree growth than for
example with linear models, leading to slow tree training algorithms requiring a
high amount of memory, (ii) the decision tree structures are very unbalanced,
which further affects computational complexity.  For these two reasons, linear
methods are often preferred to decision tree algorithms to learn with sparse
datasets. In our last contribution, we specifically developed an efficient
implementation of the tree growing algorithm to handle sparse input
variables. While previous implementations required to densify the input data
matrix, our implementation allows to directly fit decision trees on appropriate
sparse input matrices. It speeds up decision tree training on sparse input data
and saves memory by avoiding input data ``densification''.  We also show how to
predict unseen sparse input samples with a decision tree model. Note that in
this contribution we only focus on improving computing times without modifying
the original algorithm.

\section{Perspectives and future works}

We collect in this section some future research directions based on the presented
ideas of this thesis.

\subsection{Learning in compressed space through random projections}

\begin{itemize}

\item  The combination of random forest models with random projections adds two
new hyper-parameters to tree based methods: the choice and the size of the
random output projection subspace. It is not clear yet what would be good
default hyper-parameter choices. Extensive empirical studies and the
Johnson-Lindenstrauss lemma might help us to define good default values.
These two hyper-parameters also introduce randomization in the output space. It
would be interesting to further investigate how the input and output space
randomizations jointly modify the bias-variance tradeoff of the ensemble.

\item Single random projection of the output space with the gradient boosting
algorithm is a generic multi-output method usable with any single output
regressor. In this thesis, we specifically focused on tree-structured weak models.
We suggest  to investigate other kinds of weak models.

\item We have combined a dimensionality reduction method with an ensemble
method, random forest methods in Chapter~\ref{ch:rf-output-projections} and with
gradient boosting methods in Chapter~\ref{ch:gbrt-output-projection}, while
keeping the generation of the random projection matrix independent from the
supervised learning task. It would be interesting to investigate more
adaptive projection schemes. A simple
instance of this approach would be to draw a new random projection matrix
according to the residuals, e.g. by sub-sampling an output variable with a
probability proportional to the fraction of unexplained variance. An optimal
instance of this approach, but computationally more expensive, would compute a
projection maximizing the variance along each axis with the principal component
analysis algorithm.

\item Kernelizing the output of tree-based
methods~\cite{geurts2006kernelizing,geurts2007gradient} allows one to treat complex
outputs such as images, texts and graphs. It would be interesting to investigate
how to combine output kernel tree-based methods with random projection of the
output space to improve computational complexity and accuracy. This idea has
been studied~\cite{lopez2014randomized} in the context of the kernel principal
component algorithm and the kernel canonical correlation analysis algorithm.

\end{itemize}

\subsection{Growing and compressing decision trees}

\begin{itemize}

\item In the context of the $\ell_1$-based compression of random forests, we
first grow a whole forest, and then prune it. The post-pruning step is costly in
terms of computational time and memory as we start from a complex random forest
model. A first study in collaboration with Jean-Michel
Begon~\cite{begon2016joint} shows that we actually do not need to start from the whole
forest, and can grow a compressed random forest model greedily. Starting from a set of root
nodes, the idea is to sequentially develop the nodes that reduce the most the
error of the chosen loss function. The process continues until reaching a
complexity constraint, saving time and memory.

\item The space complexity of a decision tree model is linear in the number of (test
 and leaf) nodes, which is typically proportional to the training set size $n$. In the
context of $d$ outputs multi-output classification or regression tasks, or $d$
classes multi-class classification tasks, a leaf node is a constant model
stored as a vector of size $d$. The space complexity of a decision tree
model is thus $O(nd)$. With high dimensional output spaces or many classes, it thus may become
prohibitive to store decision tree or random forest models. We would like to
investigate two further approaches to compress such models: (i) by adapting the
(post)-pruning method developed in~\cite{joly2012l1} and in
Chapter~\ref{ch:rf-compression} to multi-output tasks and multi-class
classification tasks and (ii) by compressing exactly or approximately each
constant leaf model. Both approaches can be used together. For approach~(ii), an
exact solution would compute the constant leaf models on-the-fly at prediction
time by storing once the output matrix and the indices of the samples reaching
the leaf at learning time. With totally developed trees, it should not modify
the computational complexity of the prediction algorithm. If we agree to depart
from the vanilla decision tree model, it is also possible to approximate the
leaf model, for instance by keeping at the leaf nodes only the subset of the $k$
output-values reducing the most the error either at the leaf level as
in~\cite{prabhu2014fastxml} or at the tree level. Also, if the output space
is sparse, appropriate sparse data structures could help to further reduce
the memory footprint of the models.

\end{itemize}

\subsection{Learning in high dimensional and sparse input-output spaces}

\begin{itemize}

\item  In Chapter~\ref{ch:tree-sparse}, we have shown how to improve tree
growing efficiency in the case of  sparse high-dimensional input spaces.
However, the small fraction of non zero input values exploited for each split
typically leads to highly unbalanced tree structure. On such tasks, it would be
interesting to grow decision trees with multivariate splitting rules to make the
tree balanced. For instance in text-based supervised learning, the input text is
often converted to a vector through a bag-of-words, where each variable
indicates the presence of a single word. In this context, each test node
assesses the presence or the absence of a single word. The tree growing
algorithm lead to unbalanced trees as each training sample has only a small
fraction of all possible words. We propose to investigate node splitting methods
that would combine several sparse input variables into a dense one. In the text
example, we would generate new dense variables by collapsing several words
together. In a more general context, we could use random ``or'' or random
``addition'' functions of several sparse input variables.

Furthermore, while we have shown empirically that the implementation proposed in
Chapter ~\ref{ch:tree-sparse} indeed translates into a speed up and reduction of
memory consumption, it would be interesting to study formally its computational
complexity as a function of the input-space sparsity.

\item In the multi-label classification task, the output space is often very
large and \emph{sparse} (as in Chapter~\ref{ch:rf-output-projections} and
Chapter~\ref{ch:gbrt-output-projection}), having few ``non zero
values''\footnote{We assume that the majority class of each output is coded as
``0'' and the minority classes is coded with the value ``1''.}. It would be
interesting to exploit the output space sparsity to speed up the decision tree
algorithm. The algorithm interacts with the output space during the search of
the best split and the training of the leaf models. The search for the best split
for an ordered variable is done by first sorting the possible splitting rules
according to their threshold values, and then the best split is selected by
computing incrementally the reduction of impurity by moving samples from the
right partition to the left partition. The leaf models are constant estimators
obtained by aggregating output values. Both procedures require efficient
sample-wise indexing or row-wise indexing as provided by the compressed row
storage (csr) sparse matrix format. We propose to implement impurity functions
and leaf model training procedures to work with csr sparse matrices.

\end{itemize}


\appendix
\cleardoublepage
\part{Appendix}

\chapter{Description of the datasets}
\label{ch:datasets}

\section{Synthetic datasets}

\begin{itemize}
\item Friedman1~\cite{friedman1991multivariate} is a regression problem with
$p=10$ independent input variables of uniform distribution $\mathcal{U}(0,1)$.
We try to estimate the output $ y = 10\sin{(\pi{}\,x_1\,x_2)} + 20(x_3 -
\frac{1}{2})^2 + 10 x_4 + 6 x_5	+ \epsilon{} $, where $\epsilon$ is a Gaussian
noise $\mathcal{N}(0,1)$. There are $300$ learning samples and $2000$ testing
samples.

\item  Two-norm~\cite{breiman1996bias} is a binary classification problem with
$p=20$ normally distributed (and class-conditionally independent) input
variables: either from $\mathcal{N}(-a,1)$ if the class is $0$ or from
$\mathcal{N}(a,1)$ if the class is 1 (with $a = \frac{2}{\sqrt{20}}$). There are
$300$ learning and $2000$ testing samples.
\end{itemize}

\section{Regression datasset}

\begin{itemize}
\item SEFTi~\cite{SETFI} is a (simulated) regression problem which concerns the
tool level fault isolation in a semiconductor manufacturing.  One quarter of the
values are missing at random and were replaced by the median. There are $p=600$
input variables, $2000$ learning samples and $2000$ testing samples.
\end{itemize}

\section{Multi-label dataset}

Experiments are performed on several multi-label datasets:
the yeast~\cite{elisseeff2001kernel} and
the bird~\cite{briggs20139th}
datasets in the biology domain;
the corel5k~\cite{duygulu2002object} and
the scene~\cite{boutell2004learning}
datasets in the image domain;
the emotions~\cite{tsoumakas2008multi} and
the CAL500~\cite{turnbull2008semantic}
datasets in the music domain;
the bibtex~\cite{katakis2008multilabel},
the bookmarks~\cite{katakis2008multilabel},
the delicious~\cite{tsoumakas2008effective},
the enron~\cite{klimt2004enron},
the EUR-Lex (subject matters, directory codes and eurovoc descriptors)~\cite{mencia2010efficient}
the genbase~\cite{diplaris2005protein},
the medical\footnote{The medical dataset comes from the computational medicine
center's 2007 medical natural language processing challenge
 \url{http://computationalmedicine.org/challenge/previous}.},
the tmc2007~\cite{srivastava2005discovering}
datasets in the text domain
and the mediamill~\cite{snoek2006challenge}
dataset in the video domain.

Several hierarchical classification tasks
are also studied to increase the diversity in the number of label
and treated as multi-label classification task. Each node of the hierarchy
is treated as one label. Nodes of the hierarchy which never occured in the
training or testing set were removed.
The reuters~\cite{rousu2005learning},
WIPO~\cite{rousu2005learning} datasets are from the text domain. The
Diatoms~\cite{dimitrovski2012hierarchical} dataset is from the image domain.
SCOP-GO~\cite{clare2003machine}, Yeast-GO~\cite{barutcuoglu2006hierarchical}
and Expression-GO~\cite{vens2008decision} are from the biological domain.
Missing values in the Expression-GO dataset were inferred using
the median for continuous features and the most frequent value for categorical
features using the entire dataset.
The inference of a drug-protein interaction
network~\cite{yamanishi2011extracting} is also considered
either using the drugs to infer the interactions with the protein
(drug-interaction), either using the proteins to infer the interactions
with the drugs (protein-interaction).

Those datasets were selected to have a wide range of number of outputs $d$.
Their basic characteristics are summarized at
Table~\ref{tab:dataset_summary}. For more information on a particular dataset,
please see the relevant paper.

\begin{table}[htb]
\caption{Selected datasets have a number of labels $d$  ranging from 6 up to
         3993 in the biology, the text, the image, the video or the music domain.
         Each dataset has $n_{LS}$ training samples,  $n_{TS}$ testing
         samples and $p$ input features.}
\centering
\renewcommand{\tabcolsep}{1.4mm}
\begin{tabular}{@{} l rrrr  @{}}
\toprule
Datasets                      & $n_{LS}$ & $n_{TS}$ &      $p$ &        $d$\\
\midrule
emotions                      &      391 &      202 &       72 &         6 \\
scene                         &     1211 &     1196 &     2407 &         6 \\
yeast                         &     1500 &      917 &      103 &        14 \\
birds                         &       322 &      323 &    260 &    19 \\
tmc2007                       &    21519 &     7077 &    49060 &        22 \\
genbase                       &      463 &      199 &     1186 &        27 \\
reuters                       &     2500 &     5000 &    19769 &        34 \\
medical                       &      333 &      645 &     1449 &        45 \\
enron                         &     1123 &      579 &     1001 &        53 \\
mediamill                     &    30993 &    12914 &      120 &       101 \\
Yeast-GO                      &     2310 &     1155 &     5930 &       132 \\
bibtex                        &     4880 &     2515 &     1836 &       159 \\
CAL500                        &      376 &      126 &       68 &       174 \\
WIPO                          &     1352 &      358 &    74435 &       188 \\
EUR-Lex \footnotesize{(subject matters)} &    19348 &    10-cv &     5000 &       201 \\
bookmarks                     &    65892 &    21964 &     2150 &       208 \\
diatoms                       &     2065 &     1054 &      371 &      359 \\
corel5k                       &     4500 &      500 &      499 &       374 \\
EUR-Lex \footnotesize{(directory codes)} &    19348 &    10-cv &     5000 &       412 \\
SCOP-GO                       &     6507 &     3336 &     2003 &       465 \\
delicious                     &    12920 &     3185 &      500 &       983 \\
drug-interaction              &     1396 &      466 &      660 &      1554 \\
protein-interaction           &     1165 &      389 &      876 &      1862 \\
Expression-GO                 &     2485 &      551 &     1288 &      2717 \\
EUR-Lex \scriptsize{(eurovoc descriptors)} &    19348 &    10-cv &     5000 &      3993 \\
\bottomrule
\end{tabular}
\label{tab:dataset_summary}
\end{table}

\section{Multi-output regression datasets}

Multi-output regression is evaluated on several real world datasets:
the edm~\cite{karalivc1997first} dataset
in the industrial domain;
the water-quality~\cite{dvzeroski2000predicting}
dataset in the environmental domain;
the atp1d~\cite{spyromitros2012multi},
the atp7d~\cite{spyromitros2012multi},
the scm1d~\cite{spyromitros2012multi} and
the scm20d~\cite{spyromitros2012multi}
datasets in the price prediction domain;
the oes97~\cite{spyromitros2012multi} and
the oes10~\cite{spyromitros2012multi} datasets
in the human resource domain. The output of those datasets were normalized
to have zero mean and unit variance.

If the number of testing samples is unspecified, we use a $50\%$ of the
samples as training and validation set and $50\%$ of the samples as testing set.

\begin{table}[t]
\caption{Selected multi-output regression ranging from $d=2$ to $d=16$ outputs.}
\label{tab:dataset-multi-output regression}
\centering
\begin{tabular}{lrlrr}
\toprule
      Datasets &  $n_{LS}$ & $n_{TS}$ &  $p$ &  $d$ \\
\midrule
         atp1d &       337 &          &  411 &    6 \\
         atp7d &       296 &          &  411 &    6 \\
           edm &       154 &          &   16 &    2 \\
         oes10 &       403 &          &  298 &   16 \\
         oes97 &       334 &          &  263 &   16 \\
         scm1d &      8145 &     1658 &  280 &   16 \\
        scm20d &      7463 &     1503 &   61 &   16 \\
 water-quality &      1060 &          &   16 &   14 \\
\bottomrule
\end{tabular}
\end{table}


\manualmark{}
\markboth{\spacedlowsmallcaps{\bibname}}{\spacedlowsmallcaps{\bibname}} 
\refstepcounter{dummy}
\addtocontents{toc}{\protect\vspace{\beforebibskip}} 
\addcontentsline{toc}{chapter}{\tocEntry{\bibname}}
\bibliographystyle{abbrvnat}
\label{app:bibliography}
\bibliography{Bibliography}


\end{document}